\pdfoutput=1
\documentclass[msom]{informs4}

\OneAndAHalfSpacedXI

\usepackage{hyperref}
\usepackage{cleveref}
\crefname{equation}{Eq.}{Eqs.}
\Crefname{equation}{Eq.}{Eqs.}
\usepackage{wrapfig}
\usepackage{caption}
\usepackage{subcaption}
\usepackage{soul}
\usepackage{xurl}
\usepackage{enumitem}
\usepackage{tikz}

\usetikzlibrary{patterns}
\usepackage{pgfplots}
\pgfplotsset{compat = newest}
\usetikzlibrary{chains,shapes.multipart}
\usetikzlibrary{shapes,calc}
\usetikzlibrary{arrows,automata,arrows.meta,positioning}
\usepackage{adjustbox}
\usepackage{mathtools}
\usepackage{pagecolor}
\pagecolor{white}
\usepackage{multibib}

\usepackage{algorithm} 
\usepackage{algpseudocode} 

\usepackage{natbib}
 \bibpunct[, ]{(}{)}{,}{a}{}{,}%
 \def\bibfont{\small}%

\TheoremsNumberedThrough     %
\ECRepeatTheorems
\EquationsNumberedThrough    %

\newcommand{\truth}[0]{\theta}   %
\newcommand{\truMean}[0]{\mu_p}   %
\newcommand{\truSD}[0]{\sigma_p}%
\newcommand{\truDist}[0]{\pi_p}%
\newcommand{\query}[0]{s}%
\newcommand{\qError}[0]{\epsilon_s}%
\newcommand{\qErrorSD}[0]{\sigma_U}%
\newcommand{\priorAI}[0]{\pi_A} %
\newcommand{\aMean}[0]{\mu_A} %
\newcommand{\aSD}[0]{\sigma_A} %
\newcommand{\costQual}[0]{ \mathcal{E}}%
\newcommand{\costEffort}{\mathcal{I}} %
\newcommand{\costParam}{\lambda} %
\newcommand{\costHumanOnly}{\Gamma} %
\newcommand{\obj}{\mathcal{L}} %
\newcommand{\unique}[0]{d(\truth)} %
\newcommand{\EE}[2]{\mathbb{E}_{#1}\left[ #2\right]} %
\newcommand{\VV}[0]{\mathbb{V}}
\newcommand{\PP}[0]{\mathbb{P}}
\newcommand{\PL}[0]{\mathcal{P_L}}%

\newcommand{\truthAI}{\truth_A} %
\newcommand{\truthAIAlg}[1]{\truth_{A,#1}} %

\newcommand{\opteffort}[0]{\qErrorSD^\star(\truth)}
\newcommand{\opteffortAlg}[1]{\sigma_{U,#1}^{\star} (\truth)}
\newcommand{\optcosteffort}[0]{\costEffort^\star(\theta)}
\newcommand{\opterror}[0]{\costQual^\star(\theta)}
\newcommand{\optobj}[0]{\obj^\star(\theta)}

\newcommand{\MSEcritical}[0]{t} %
\newcommand{\thresholdA}[0]{\tau_a} %
\newcommand{\thresholdD}[0]{\tau_d} %
\newcommand{\outcome}[0]{\truth^\star}  %

\newcommand{\outcomeRetro}[0]{\tilde{\truth}}  %
\newcommand{\costQualRetro}[0]{\tilde{\costQual}}
\newcommand{\costManual}[0]{\tilde{\mathcal{M}}} %
\newcommand{\objRetro}{\tilde{\obj}} %
\newcommand{\qErrorSDRetro}[0]{\tilde{\sigma}_U}%

\newcommand{\1}[1]{\mathbf{1}\left\{#1\right\}}

\definecolor{revblue}{RGB}{0,0,255}%
\newcommand{\rev}[1]{{\color{black} #1}}
\definecolor{smorange}{RGB}{230,110,0}%
\newcommand{\smrev}[1]{\begingroup\color{black}#1\endgroup}%
\definecolor{jggreen}{RGB}{0,130,60}%
\newcommand{\jgsim}[1]{\begingroup\color{black}#1\endgroup}%
\newcommand{\jgmath}[1]{\begingroup\color{black}#1\endgroup}%
\newenvironment{jgmathblock}{\begingroup\color{black}}{\endgroup}%

\begin{document}

\RUNAUTHOR{Castro, Gao, and Martin}

\RUNTITLE{Human-AI Interactions and Societal Pitfalls}

\TITLE{Human-AI Interactions and Societal Pitfalls}

\ARTICLEAUTHORS{%
\AUTHOR{Francisco Castro, Jian Gao}
\AFF{Anderson School of Management, UC Los Angeles, Los Angeles, CA 90024, \EMAIL{francisco.castro@anderson.ucla.edu}, \EMAIL{jian.gao.phd@anderson.ucla.edu}} %
\AUTHOR{Sébastien Martin}
\AFF{Kellogg School of Management, Northwestern University, Evanston, IL 60208, \EMAIL{sebastien.martin@kellogg.northwestern.edu}}
} %

\ABSTRACT{%
When working with generative artificial intelligence (AI), users may see productivity gains, but the AI-generated content may not match their preferences exactly. 
To study this tension, we introduce a Bayesian framework in which heterogeneous users choose how much information to share with the AI, facing a trade-off between output fidelity and communication cost. 
We show that the interplay between these individual-level decisions and AI training may lead to societal challenges. Outputs may become more homogenized, especially when the AI is trained on AI-generated content, potentially triggering a homogenization death spiral, whose outcome is decided by the most unique users. And any AI bias may propagate to become societal bias. 
Solutions to the homogenization and bias issues include reducing human-AI
interaction frictions, retraining on data that reflects distinctive preferences, and retaining the most unique users' ability to work on their own. Together, these help preserve diverse and personalized outputs.
}


\maketitle

\section{Introduction}
\label{sec:intro}

Generative artificial intelligence (AI) systems, particularly large language models (LLMs), have improved at a rapid pace. For example, ChatGPT \rev{\smrev{reached}} 100 million users within two months of its 2022 launch \citep{Hu2023}. This progress is not limited to text generation, as demonstrated by other %
generative AI systems such as Midjourney %
(a text-to-image generative AI) and GitHub Copilot %
(an AI pair programmer that can autocomplete code).
\citet{eloundou2023} estimated that about 80\% of the U.S. workforce could be affected by the introduction of LLMs, and 19\% of the workers may have at least 50\% of their tasks impacted. 
\rev{A key factor in the rapid adoption of this technology is that AI can make users more productive by generating content in seconds.}
For example, \citet{noy2023} highlighted that ChatGPT can substantially improve productivity in writing tasks, and GitHub claimed that Copilot could increase developer productivity by up to 55\% \citep{kalliamvakou2023}.  

However, AI-assisted content is not always the same as content generated without AI. The boost in productivity may come at the expense of users' idiosyncrasies\rev{: AIs tend to generate outputs that do not fully match users'}
personal styles and tastes, %
preferences they would naturally express without AI. To let users express their preferences, many AI systems have ways to incorporate user feedback, often involving natural interactions (e.g., ChatGPT), and users can \rev{\smrev{also}} review and edit the AI-generated output themselves.
Users can therefore choose the extent to which they want to personalize the AI output.
While adding personalization can improve fidelity, it also requires extra time and \smrev{effort, potentially} reducing productivity.

Consider a simple example where we use ChatGPT to generate an abstract for this paper. \Cref{fig:chatgpt example} shows that ChatGPT's output with no information about our preferences was well-written and functional. However, it does not reflect our personal preferences when writing an abstract; the output is too long, it does not mention the type of model we use, and some sentences are overly verbose for our taste. 
To better match our style, we could provide more information by articulating a more detailed prompt (the second prompt in \Cref{fig:chatgpt example}). This also yields a functional outcome, perhaps closer to our writing style.
If we were to put more effort and time into this process (e.g., iterating on the prompt or adding manual edits), we could obtain a result even closer to what we would have done without AI.

\rev{In this work, we explore the trade-off between AI output fidelity, defined as how closely an AI's generated content matches users' personal style and preferences, and communication cost, defined as the amount of information users must provide (a measure of effort) to guide an AI's output to their liking.} %
We consider scenarios where users already know how to complete the task (e.g., researchers writing abstracts or software engineers writing code) and where the AI produces functionally correct content. Our focus is not on traditional quality metrics (such as grammar or absence of typos) but on the effort that users put into articulating their preferences (e.g., how much context, instructions, or examples users must communicate) so that the output closely matches their taste. \rev{In particular, we ask:

{\it How do differences between users' personal preferences and an AI's outputs affect their rational choice to engage with the AI or stop using it altogether?}
}

\begin{figure}[htbp]
    \begin{center}
            \scalebox{0.5}{\begin{tikzpicture}
% % Add axis labels
% \node[rotate=90, anchor=center, font=\bfseries\large] at (-6, 3) {prompt};
% \node[rotate=90, anchor=center, font=\bfseries\large] at (-6, -3) {Generated abstract};

% Draw left brace for prompts and text

\draw[decorate, decoration={brace, amplitude=8pt, mirror}, thick]
    (-6, 5) -- (-6, 1.00)
    node[midway, xshift=-12pt, rotate=90, font=\fontsize{12}{14}\selectfont\bfseries] {prompt};

\draw[decorate, decoration={brace, amplitude=8pt, mirror}, thick]
    (-6, 0.75) -- (-6, -8.25) node[midway, xshift=-12pt, rotate=90, font=\fontsize{12}{14}\selectfont\bfseries] {Generated abstract};

% Add prompt boxes with different fonts, sizes, and text wrapping
\node[font=\fontsize{12}{14}\selectfont\itshape, text width=11cm, align=center] (p1) at (0, 3) {Generate an abstract};
\node[font=\fontsize{12}{14}\selectfont\itshape, text width=11cm, align=center, right=0.5cm of p1] (p2) {Write a concise abstract for an Informs-related academic journal. Clearly explain the main findings in simple language, avoid “fancy” wording, and highlight the positive impact of AI. Be technical but not overly technical. Make sure to cover all key points from the paper (e.g., Bayesian model, homogenization, death spiral, and bias) and maintain an optimistic tone. Make sure not to go over 150 words.};

% Add generated text boxes with different styles and text wrapping
%ABSTRACT 1
\node[font=\fontsize{12}{12}\selectfont\rmfamily, text width=11cm, align=center] at (0,-3.75) {Generative AI tools such as large language models (LLMs) can significantly boost productivity, but they also introduce subtle societal risks. This paper models human-AI interaction as a communication problem in which users trade off effort (communication cost) for personalization (output fidelity). We show that users with more common preferences tend to accept AI’s default outputs with minimal interaction, while more unique users either engage in costly communication to improve fidelity or avoid AI altogether. This selective usage leads to homogenization—a systemic reduction in output diversity. We further show that as AI-generated content is used to train future AI systems, a "homogenization death spiral" can emerge, where diversity continues to diminish over time. Our model also evaluates the consequences of AI bias, showing that biases in the AI's prior can lead to unequal utility losses across users and potentially skew societal outcomes. However, improving the ease of communication with AI and allowing users to express their preferences more flexibly can mitigate both homogenization and bias. These findings highlight the importance of designing AI systems that support expressive human-AI collaboration to preserve diversity and equity in content generation};

%ABSTRACT 2
\node[font=\fontsize{12}{12}\selectfont\rmfamily, text width=11cm, align=center] at (11.5,-2.5) {
Generative AI tools boost productivity by helping users complete tasks quickly with minimal input. We develop a Bayesian model to study how users interact with AI, balancing the effort to express their preferences (communication cost) with how closely the AI output matches their intent (fidelity). Our model shows that users with common preferences rely more on default AI outputs, while those with unique preferences invest more effort or opt out entirely. This leads to homogenization—less diversity in AI-generated content. Over time, if AI trains on its own outputs, this can trigger a “death spiral” where diversity keeps shrinking. We also analyze AI bias and show that censoring extreme preferences harms users more than mild directional bias. Encouraging richer human-AI interaction—through better interfaces or customization—can limit bias and preserve diversity. Overall, our results highlight how thoughtful AI design can combine efficiency with inclusivity.
};

\end{tikzpicture}}
        \end{center}
    \caption{
    Two abstracts of this paper generated using ChatGPT GPT-4o. We provided a PDF version of this paper (without our current abstract) together with the two prompts (in two separate chats).
    }
    \label{fig:chatgpt example}
\end{figure}

We are also interested in the potential societal consequences of working with AI, particularly regarding content diversity and bias. First, content generated with AI assistance can become, on average, \rev{less original \citep{Chayka2025AI} \smrev{and, across users, more}} \emph{homogenized} toward an AI's default \smrev{outputs, which} are themselves partially influenced by the design and training choices of an AI company.
For example, using reinforcement learning from human feedback (RLHF) \citep{ziegler2020}, ChatGPT has been tuned to have a specific tone and language. If students use ChatGPT's help for their homework, \rev{the style of the work they submit may be influenced by and even default to} ChatGPT's \citep{kosmyna2025your}. 
Second, AIs are built by a few but used by many, and there is a risk that any AI bias could turn into a \emph{societal bias}. The AI training process may involve censoring (e.g., the choice of the dataset) and human input (e.g., RLHF), which could intentionally or unintentionally lead to bias. For example, some studies discuss ChatGPT's inclination towards left-leaning political stances \citep{hartmann2023, rozado23} and xAI's Grok was prompted to have specific political beliefs.\footnote{See \url{https://www.nytimes.com/2025/05/16/technology/xai-elon-musk-south-africa.html}.} \rev{\smrev{Because they balance output fidelity against communication cost, users} working with AI} could willingly produce less diverse content that is vulnerable to potential AI biases. \rev{We ask:

{\it How does AI-generated content affect the diversity of user outputs in the short and long term, and the diffusion of AI biases across user-generated content? What practical guidelines or mitigation strategies can users and practitioners adopt to reduce the homogenization and bias risks?}

}

\rev{To study these questions,} we propose a Bayesian model %
\rev{in which}, for a given task, rational users can share information with an AI with the intention of producing outputs that match their personal preferences \rev{(their type)}. The AI uses a distribution of population preferences and  a Bayesian update to generate an output with maximum expected fidelity given the information shared by users. Users choose the amount of information they share to minimize their expected utility loss, defined as the cost of communication plus the fidelity error from the AI's output\rev{\smrev{, or complete the task on their own at a fixed cost}}.

\subsection{\rev{Contributions}}

When solving a user's optimal decision, we find that her use of AI depends on how ``unique'' her preference is\rev{\smrev{, that is, how far it is from the population's mean preference (the default output)}}. Users with more common preferences simply accept the default output, avoiding any communication costs at the expense of a small fidelity mismatch. In contrast, users with more unique preferences share information with the AI to reduce fidelity error, albeit at a higher communication cost. For the most unique users, \rev{\smrev{getting a satisfactory output from the AI would cost more than doing the task themselves, so they work without AI}}. We highlight a non-obvious effect: the users who ``lose their voice'' the most with the introduction of AI\rev{\smrev{, in the sense that their output ends up furthest from their preference,}} are not the ones with the most common or the most unique preferences, but those in between. %

\rev{These individual rational choices can cause homogenization.}
We prove that\rev{\smrev{, in expectation, a user's output with AI is closer to the population's mean preference than her own preference is, and thus} less unique than what \smrev{she} would have done without AI. This \smrev{carries over to} the population level\smrev{: when everyone uses AI,}} \smrev{the} output distribution has a lower variance than the users' preference distribution. \rev{The implication is that working with AI can lead to less diverse outputs: %
the incentives in place, such as the need for working faster or being more productive, make it rational for users to accept functional but less \smrev{personalized} content.}

We also study the effects of AI bias, identifying who benefits or loses when
using an AI model that does not accurately reflect the population's preference
distribution. The censoring type of bias (e.g., \rev{\smrev{under-representing}} the more unique
preferences) \rev{lowers population utility, with the loss borne almost entirely
by users with uncommon preferences,} who rely on AI interactivity the most. This
may seem counterintuitive, as we might assume that the majority, i.e., users
with common preferences, would benefit from censorship. Yet, \rev{we find that
they barely do: these users accept the AI's default output with or without
censoring. In fact, an unbiased AI is not \smrev{always} the best choice: \smrev{when everyone uses the AI, training it} to slightly over-represent unique preferences \smrev{can improve population utility}.}
On the other hand, \rev{\smrev{a small directional bias}} (e.g., a slightly left- or right-leaning AI) \rev{\smrev{barely changes population utility, }}%
but any directional bias will influence users' output, even when they actively try to correct it.
This means that AI bias can propagate and become societal bias, implying that AI companies may have significant influence on the content we produce. On the positive side, users' interactions with an AI partially counter the effects of bias, \rev{and\smrev{, when everyone uses the AI,} they do so more strongly the lower the cost of communicating preferences to the AI.}

\rev{We then consider the long-term consequences of working with AI. In particular, we show that the homogenization}
phenomenon can be exacerbated when AI-generated content is used to train the next generation of AI. To capture this, we consider a multi-period version of our model in which the output distribution becomes the new AI prior. %
We show that users' rational decisions and the AI's training process can mutually reinforce each other, leading to a homogenization ``death spiral.'' As users interact with an AI trained on more homogenized content, it may become harder to steer the AI towards specific tastes. This additional effort may not be worth it for some users anymore, \rev{especially and increasingly for the most unique users}, exacerbating homogenization over time. \rev{\smrev{When everyone uses AI, whether the spiral ends in full collapse is decided by these most unique users: for a symmetric, atomless distribution of preferences, if preferences have bounded support, outputs collapse to a single point in finite time; with unbounded support and finite variance, they never fully collapse in finite time, though their variance may still vanish in the limit%
; and with sufficiently heavy tails, the variance of outputs stays infinite.} Practically speaking, our results \smrev{suggest that training data that cut off} the tails of the \smrev{preference distribution}, for example \smrev{by filtering out or} censoring certain views, \smrev{remove} precisely the preferences of the users (the most unique ones) \smrev{who} can sustain long-run diversity\smrev{, and may push} the loop toward collapse.
}

We show that the human side of the interaction can mitigate the risk of a homogenization spiral through three %
mechanisms: (i) human-AI interaction efficiency, (ii) human choice, and \rev{(iii) humans in the training loop.} %
(i) \rev{Solutions that} simplify human-AI interaction\rev{, such as \smrev{custom instructions \citep{openai2023-4}, user personalization}, or voice-based interactions \smrev{\citep{openai2023-5}},} serve as a counterforce to the death spiral and increase output diversity by encouraging users to share more informative signals for the same level of effort. \rev{\smrev{Our simulations suggest that such solutions are} even more effective when they ease human-AI interactions for the most unique users. \smrev{The pace of improvement also matters: when everyone uses the AI, sufficiently fast-declining interaction costs allow less homogenized content to reach the training data and can, in turn, restore diversity in full; while for preferences with bounded support, if the interaction cost falls too slowly outputs can be fully homogenized, even though interactions eventually become arbitrarily cheap.
}}
(ii) Human choice\rev{\smrev{, the} ability of users to opt out and work without \smrev{AI, can stop the spiral: we show that it keeps the variance of outputs above a positive floor at every round, and, in our simulations, it makes homogenization cyclical}.} %
\rev{\smrev{As the AI homogenizes, the} most unique users may \smrev{no longer} find the effort of working with AI worthwhile and} opt out of using AI entirely\rev{\smrev{. Their outputs bring diversity back into the next AI's training data,} which temporarily alleviates the \smrev{spiral until} it begins again.}
\rev{Retaining the ability to work without AI is thus \smrev{important} in maintaining diversity in human-AI generated content.}
\rev{(iii) Human data selection for training can shape the evolution of subsequent model generations. We show \smrev{numerically that} mixing curated human-generated data into each training generation can mitigate the homogenization spiral\smrev{, and that} prioritizing data that reflects the most distinctive preferences is especially effective at sustaining diversity. \smrev{Other ways of bringing human feedback into training, such as reinforcement learning from human feedback \citep{ziegler2020}, may play a similar role.}
}

Our results suggest that tasks that are either \rev{\smrev{costly}} to do without AI (e.g., image generation using Midjourney or Sora) or for which speed is particularly important (e.g., stylistic writing choices using Grammarly) are especially sensitive to the risks of homogenization and bias.

\section{Literature review}
\label{sec:related literautre}

\paragraph{Related studies on homogenization and bias.} 
Aligned with our theoretical findings, recent empirical studies indicate that generative AI may reduce the diversity of outputs. For instance, in the context of making short stories, \citet{doshi2024} provide experimental evidence showing that while working with generative AI can improve the creativity of written content, it can also substantially increase the similarity of stories compared to those written by humans only (see also \cite{wang2023human,anderson2024} and \cite{padmakumar2024}). %
\citet{shumailov2023} observe that the tails of the original content distribution disappear when AIs are successively trained from AI-generated content (they call it model collapse), while \citet{bommasani2022} demonstrate that algorithmic systems built on the same data or models tend to homogenize outcomes. In the context of recommendation systems, \citet{chaney18} use simulation to show that a feedback loop, where a recommendation system is trained on data from previous algorithmic recommendations, may homogenize user behavior.

\rev{\smrev{A growing literature studies how to mitigate model collapse as it pertains to how much real data and what data enters training. Iterative retraining remains stable when the initial model is accurate enough and a large enough share of each generation's training data is real \citep{bertrand2024}, collapse is avoided when synthetic data accumulate alongside the original real data rather than replacing them \citep{gerstgrasser2024}, and without enough fresh real data in each generation, quality or diversity progressively declines \citep{alemohammad2023}. Which data enter also matters: \citet{dohmatob2024} model generated data as cutting off the tail of the distribution and show that adding real tail data restores performance only when those data cover the part of the tail that was lost, while \citet{feng2025} show that filtering generated data with a sufficiently accurate, though not necessarily perfect, verifier can prevent the performance loss from training on generated data. Viewing training data as the product of rational human-AI interactions, our results on the human role in training further validate these ideas.   
}}

The issue of bias in generative AI has been studied from different perspectives. %
For example, \citet{rozado23} implemented 15 different political orientation tests on ChatGPT. The author found that ChatGPT's answers manifested a preference for left-leaning opinions in 14 of the 15 tests (see also \cite{hartmann2023}). 
\citet{bhat2023} discovered that people may incorporate AI suggestions into their writing, even when they disagree with the suggestions overall. Similarly, \citet{jakesch2023} showed that biased language models could influence the opinions expressed in people's writing and shift their viewpoints.

\rev{\smrev{Our study complements this body of work by modeling why these effects can arise: homogenization and the death spiral emerge, and AI bias spreads to users' outputs, from how users trade off fidelity against communication cost. Unlike model collapse, which is usually attributed to accumulated training errors, the spiral in our model arises even with a perfect AI, from users' rational choices. %
}} %

\paragraph{Related studies on human-AI interaction.} 
Our paper relates to recent modeling studies on human-AI interaction \smrev{\citep{agrawal2018}}. A stream of work in this space considers AIs as a support tool for decision-makers that helps improve different measures of performance, e.g., accuracy or value. 
 \citet{vericourt23} study a setting in which a human agent supervises an AI to make high-stakes decisions. They show that the agent may be subject to a verification bias and, in turn, \rev{\smrev{never learn}} whether the AI performs better than the agent\rev{\smrev{,}} because the agent can overrule the AI before observing the correctness of the AI's predictions. \citet{boyaci2023} consider a situation in which a human agent has to spend a cognitive cost collecting information in a decision process, whereas an AI can provide the agent with some additional information without cognitive cost. They show that the AI input can improve the overall accuracy of human decisions but may incur a higher propensity for certain types of errors. 
 \citet{bryce2023} consider the risks that algorithmic recommendations, when regarded as default actions, can have on the preferences and, ultimately, on the decisions of a decision maker, showing that withholding algorithmic recommendations can improve decision making.  
 We refer the reader to \citet{dai2023} for a related study \rev{\smrev{of a physician's decision whether to use AI}} and to \citet{bastani2022} for experimental evidence on how interpretable reinforcement learning algorithms can improve human performance in sequential decision-making.
 
Another stream of literature on human-AI interaction considers how human input can alter AI output. \citet{ibrahim21} studies strategies to elicit human judgment to improve algorithm-based predictions. They show that instead of a direct forecast from humans, eliciting the extent to which an algorithm's forecast should be adjusted leads to better forecasting accuracy. 
In an empirical study of human-algorithmic demand forecasting, \citet{balakrishnan2025human} examine a type of bias in which individuals average their own prediction (informed by private information) with the algorithm's. They find that this naïve weighting behavior is suboptimal and that feature transparency can help individuals to better adjust an algorithm's forecasts based on their private information.
\citet{chen2022} study the benefits of augmenting algorithmic decisions, such as pricing and forecasting, with human input in the form of guardrails. They conclude that with a large dataset, human augmentation offers no benefits; however, model \rev{\smrev{misspecification}} or data contamination can make human guardrails valuable. 

\rev{Our work combines \smrev{these two perspectives: users rely on an AI to} complete tasks at \smrev{a lower cost, and they choose how much costly information to convey in order to steer its output. Unlike in the prediction and decision settings above, there is no single correct output, since the information users provide is about their own preferences, which differ across users. Modeling this choice explicitly} (via a signal and a Bayesian update) \smrev{allows us} to analyze homogenization and bias of generated content as \smrev{by-products} of human-AI interactions.}

\paragraph{Related studies on generative AI.} With the increasing popularity of ChatGPT, there is growing interest across different fields in understanding its impact on people's lives, such as labor markets \citep{eloundou2023}, marketing \citep{brand2023}, healthcare \citep{sallam2023}, customer care \citep{yu2022emotion}, among others. 
Several \rev{\smrev{empirical studies measure}} the benefits of generative AI and \rev{\smrev{compare its performance with that of}} humans.
For instance, \citet{binz2023} tested GPT-3 with some experiments from the cognitive psychology literature. They find that GPT-3 can solve many of those tasks well and even sometimes outperform humans.
\citet{noy2023} show that ChatGPT can substantially improve productivity in mid-level professional writing tasks.
\rev{\smrev{Alongside these benefits}, our study \smrev{raises} a different concern: \smrev{the widespread use of generative AI}} may shape society's output in unintended ways, contributing to homogenization and bias.

\paragraph{Related studies on the modeling approach.} \rev{\smrev{Our model of} human-AI interaction \smrev{draws on three literatures. First, the}} user's decision is modeled \rev{\smrev{as}} an information design process \smrev{\citep{kamenica2011, saed2020, Vericourt2021}}\rev{\smrev{: the}} sender (i.e., the user) sends a signal to the receiver (i.e., the AI) to inform the receiver about a true state (i.e., the user's preference)\rev{\smrev{, and the}} utility of the sender is determined by the receiver's decision (i.e., the AI's output).
\rev{\smrev{Second, we use}} costly persuasion \citep{gentzkow2014} and the theory of rational inattention \citep{sims2003,matvejka2015rational} to model the user's communication cost when sending the signal.
In particular, we follow the standard way in the literature to model the cost of information as the expected reduction in entropy\rev{\smrev{, as in}} the cognitive cost defined \rev{\smrev{by}} \citet{boyaci2023}.
Note that we define the reduction in entropy relative to the population distribution of users' preferences rather than the AI's prior (see \Cref{sec:model}). As \citet{gentzkow2014} suggest, entropy reduction can be measured relative to any proper fixed reference belief. Using the population distribution as our reference belief highlights that communication cost depends on how difficult it is to distinguish one user's preference from others, independently of the AI's prior.
\rev{\smrev{Third}}, we model the AI's behavior as a Bayesian inference \citep{wei2021, xie2022}.
For instance, \citet{xie2022} \rev{\smrev{argue that}} the in-context learning of an LLM can be viewed as an implicit Bayesian inference. The prior of the LLM is formulated during training. Conditional on a prompt, the LLM characterizes a posterior distribution to generate an output.

\section{Model Setup}
\label{sec:model}

We use a Bayesian framework akin to \cite{kamenica2011} to represent the process of working with a generative AI to increase user productivity. There is a known task (e.g., writing an abstract), and different users have different preferences, corresponding to their idiosyncratic tastes for the task output.
While users could complete the task by themselves at a cost, they may choose to work with an AI. The AI ``knows'' the population's distribution of users' preferences (through its training) and can generate a functional output. %
Users can share information with the AI about their specific preferences for the task, which will, in turn, help the AI produce an output with varying degrees of \emph{fidelity}\rev{\smrev{, a} measure of how close the output is to what \smrev{the} user would have \smrev{produced on her own (her preference)}}.
However, sharing information requires effort, which entails a \emph{communication cost}.
When working with the AI, users must choose how much information they share (through prompting, rewriting, etc.) to balance the loss of fidelity and the cost of communication.

Formally, there is a continuum of users of type $\truth \in \Theta$, denoting a user's specific \textit{preference} about how to complete the task. We let $\Theta = \mathbb{R}$ and assume that the distribution of user preferences across the population is normal with mean $\truMean$ and standard deviation $\truSD$, with density $\truDist(\cdot)$.\footnote{A user's preferences should be represented by a high-dimensional space. However, restricting to one dimension makes the model more amenable to analysis while preserving its interpretability: we will view $\theta$ as a specific \emph{feature} of a user's preferences\rev{\smrev{, such as the length of an abstract or the political leaning of a news article (\Cref{eg: writer})}}.}
The AI has a prior belief  $\priorAI(\cdot)$ of the population distribution of preferences, which is normally distributed with mean $\aMean$ and standard deviation $\aSD$. To capture that the AI has been trained on a representative dataset, we assume that the AI's prior is exactly the population distribution, $\truDist(\cdot)\equiv\priorAI(\cdot)$ (this assumption is relaxed in \Cref{sec: Bias} to study the effects of a biased AI). 

A user  $\theta$ chooses a signal $\{\pi(\cdot|
\truth')\}_{\truth'\in\Theta}$\smrev{, a} mechanism to communicate her preference for how to complete the task. We assume that the signal follows a normal distribution centered on the true preference $\theta$ and with variance $\qErrorSD^2$, i.e., $\pi(\query|
\theta) = \frac{1}{\qErrorSD}\phi\left(\frac{\query-\truth}{\qErrorSD}\right)$ where $\phi(\cdot)$ is the density of a standard normal distribution. $\qErrorSD$ is a parameter that characterizes the user's choice of the signal and %
maps directly to the amount of information shared.
The AI then observes a signal realization $\query\in \Theta$ and the signal itself \rev{(it observes $\qErrorSD$ but not the preference $\truth$\smrev{, and uses $\qErrorSD$ only to weigh $\query$, not as evidence about $\truth$; see \Cref{sec:model-discussion}})}\label{resp:r3-signal1} and forms a posterior using Bayes' rule
\begin{equation}\label{eq:bayes-update}
\pi_A(\truth|\query) = \frac{\pi(s|
\theta)\pi_A(\theta)}{\int_\Theta \pi(s|
\theta)\pi_A(\theta)d\theta}.
\end{equation}
Once the AI forms a posterior, it aims to return an output with maximal {\it fidelity}, i.e., minimizing the expected discrepancy $(\hat \theta - \theta)^2$ (the ``fidelity error'') between the true preference $\theta$ and the AI output $\hat \theta$:

\begin{equation}
\label{eq: truthAI def}
  \truthAI(\query,\qErrorSD) \triangleq \argmin_{\hat{\truth}} \:\: \EE{\pi_A(\cdot|\query)}{(\hat{\truth}-\truth)^2} = \EE{\pi_A(\cdot|\query)}{\truth} = \dfrac{\aSD^2}{\aSD^2+\qErrorSD^2}\cdot \query + \dfrac{\qErrorSD^2}{\aSD^2+\qErrorSD^2}\cdot \aMean,
\end{equation}
that is,  $\truthAI(\query,\qErrorSD) $ is a weighted average between the signal realization and the prior mean.

Given the AI's output for different signals, a user with preference $\truth$ chooses a signal that minimizes her expected utility loss from interacting with the AI:
\begin{equation}  
    \label{eq: loss def}\tag{$\mathcal{P}_\truth$}
    \min_{\qErrorSD \geq0}  \left\{ \costQual(\truth, \qErrorSD) + \costParam \costEffort(\qErrorSD)\right\},
\end{equation}
where the first term, $\costQual(\truth, \qErrorSD)$, is the expected {\it fidelity error} and the second term, $\costParam\costEffort(\qErrorSD)$, is the expected {\it communication cost} (both will be formalized next). 
We term the parameter 
$\costParam>0$ the cost of human-AI interactions, or interaction cost for short; it can be interpreted as the minimum fidelity improvement that is worth a unit of effort. This cost is low when it is time-efficient and easy to communicate preferences to the AI (e.g., when using the advanced user personalization features of ChatGPT \citep{openai2023-5}) or when users care a lot about fidelity and are willing to invest time for it (e.g., an artist or writer).
We use  $\obj(\truth,\qErrorSD) = \costQual(\truth, \qErrorSD) + \costParam \costEffort(\qErrorSD)$ to denote the expected utility loss for a user $\truth$ choosing a signal characterized by $\qErrorSD$, and  $\opteffort$ to denote the optimal signal choice that solves \eqref{eq: loss def}. %

The expected fidelity error of user $\truth$ given a choice of signal parametrized by $\qErrorSD$ is then:
\begin{equation*}
    \costQual(\truth, \qErrorSD) \triangleq \EE{\pi(\cdot|\truth)}{(\truthAI(\query,\qErrorSD)-\truth)^2}.
\end{equation*}
To measure the expected communication cost of a user, we follow standard assumptions in the rational inattention \citep{sims2003,matvejka2015rational} and costly persuasion \citep{gentzkow2014} literature. In particular, we assume the expected communication cost of a signal to be proportional to the induced expected reduction in the uncertainty of the user's preference relative to the population distribution:
\[
 \costEffort(\qErrorSD) \triangleq \mathcal{H}(\truth)-\EE{\pi(\cdot|\truth)}{\mathcal{H}(\truth|\query)} =  \left[\ln(\truSD\sqrt{2 \pi e}) - \ln \left(\sqrt{\dfrac{\truSD^2\qErrorSD^2}{\truSD^2+\qErrorSD^2}}\sqrt{2 \pi e} \right)\right] = - \dfrac{1}{2} \ln \left(\dfrac{\qErrorSD^2}{\truSD^2+\qErrorSD^2} \right),
\]
where $\costEffort(\qErrorSD)$ is the mutual information, %
and $\mathcal{H}(\cdot)$ denotes the differential entropy.

Sharing the exact value of $\truth$ ($\qErrorSD=0$) requires an \emph{infinite} amount of information, $\costEffort (0)=+\infty$ (as an infinite amount of information is needed to define a real number with arbitrary precision). In this case, the AI outputs $\truth$ and $\costQual(\truth, 0) = 0$.
Conversely, providing an uninformative signal about $\truth$ ($\qErrorSD = +\infty$) requires no information, $\costEffort(+\infty)=0$. The AI outputs $\aMean$, and $\costQual(\truth, +\infty) = |\aMean-\theta|^2$. Since, in this case, the signal is not informative, we say that $\aMean$ is {\it the default output}.

In addition to completing the task with the AI, a user may decide to complete the task on her own, incurring no fidelity error. However, manual work takes time, which we model as a fixed utility cost $\costHumanOnly > 0$ that depends on the task but is the same for everyone.\footnote{\rev{In \Cref{sec:mitigation_hc}, we study the case when $\costHumanOnly$ is allowed to vary with user preferences.} } The user will choose the option that minimizes the expected utility loss: we define the optimal output $\outcome$ of a user $\theta$ and the corresponding expected utility loss $\optobj$ as
\begin{equation}
    \label{eq:decision w/ human only}\
    \outcome \triangleq \begin{cases} \truthAI(\query,\opteffort),\:\: s\sim \pi(\cdot|\truth) & \text{if }\obj(\truth,\opteffort) \le \costHumanOnly\\ 
        \truth & \text{otherwise}
           \end{cases}\quad \text{and}\quad \optobj \triangleq \min(\obj(\truth,\opteffort), \costHumanOnly).
\end{equation}
Therefore, $\outcome$ corresponds to an output that is either purely AI generated (if $\obj(\truth,\opteffort) \le \costHumanOnly$ with $\opteffort = +\infty$), purely human generated (if $\obj(\truth,\opteffort) > \costHumanOnly$),\footnote{\label{rev: r2-data}\rev{The output distribution of $\outcome$ includes outputs produced manually (which in practice may be unobserved). This is a conservative assumption: only partially \smrev{including human-generated} content \smrev{in the training data} would %
accelerate the homogenization \smrev{spiral %
(\Cref{sec:mitigation_hc})}.}} or the result of a human-AI interaction (otherwise).
In what follows, %
we will use  $\optcosteffort$ to denote $\costEffort(\opteffort)$ and $\opterror$ to denote $\costQual(\truth, \opteffort)$. \smrev{A user who works on her own produces her preference $\truth$, which is by definition what she would produce without the AI, and shares nothing with the AI, so for her we set $\opterror = \optcosteffort = 0$.}

\smrev{The next simple example helps to interpret our framework.}

\begin{example}[News Article]
\label{eg: writer}

A journalist wants to write an article about a piece of news and plans to use an AI (e.g., ChatGPT) to work faster. We represent the journalist's political orientation by $\truth$. For example, if $\truth>\truMean$, the journalist is more right-leaning than the average journalist.

Using the AI can speed up her writing process, but it may result in an article that does not precisely reflect her true orientation. The journalist has a process of interacting with the AI (combining custom prompts, back-and-forth with ChatGPT, manual edits, etc.) that she perfected to strike the right balance between the time it takes her (effort) and how the output fits her taste (fidelity). \rev{\smrev{This process conveys her orientation only approximately (the realization $\query$), and more precisely the more time she invests (a smaller $\qErrorSD$).}}

If the article is breaking news and the task is especially urgent (i.e., $\costParam$ is large%
), the journalist will invest less time in the human-AI interaction. The outcome is therefore more influenced by the AI's default choices. If the journalist has more time (i.e., $\costParam$ is low), she may carefully edit the article more and collaborate with the AI for longer to have an output that better fits her preferences. In either case, she may also anticipate that steering the AI toward her orientation will take so long that she prefers to write the article without its help, which also takes time, captured by the utility cost $\costHumanOnly$.

\end{example}

\subsection{Discussion of the Model}\label{sec:model-discussion}
\rev{\smrev{To summarize the timing: a user knows her preference $\truth$ and the AI's prior, and decides up front whether to work with the AI or on her own. If she works} with the AI, \smrev{she chooses the noise level $\qErrorSD$ of her signal; the AI observes $\qErrorSD$ and the realization $\query$, but not $\truth$, and returns $\truthAI(\query,\qErrorSD)$, which the user keeps. If she works on her own, she produces $\truth$ at cost $\costHumanOnly$. We now discuss these modeling choices.}}

Each value of $\truth$ represents a specific user preference, corresponding to a distinct and functional way of completing the task. As mentioned in \Cref{sec:intro}, we do not evaluate $\theta$ from a ``quality'' perspective. Rather, it represents an idiosyncratic user taste (e.g., political orientation or preferred vocabulary), and we care about how closely the output matches a user's actual preference. We focus on situations where output quality is not a concern, since both the AI and the user are capable of generating a suitable output on their own. This allows us to center our analysis on the impact of AI on the diversity and bias of human output, rather than on the separate question of whether AI does a ``better job'' at creating content than humans. In our model, \rev{\smrev{what distinguishes preferences is how frequent they are} in the population\smrev{: larger distances $|\truth-\truMean|$ to the mean (a user's} \emph{\smrev{uniqueness}}\smrev{)}} correspond to more extreme or rare preferences, and thus to outputs that are less likely to be produced by humans or generated by the AI. %

Our model of human-AI interaction is similar to the Sender–Receiver framework in the Bayesian Persuasion literature \citep{kamenica2011}. A key difference, however, is that in our setting, the human (Sender) chooses a signal after observing her type (the state of the world), whereas in the canonical Bayesian Persuasion model, Sender typically commits to a signal before the state is realized. 

In the usual Bayesian Persuasion setup, the signal represents the mechanism or experiment Sender uses to (partially) reveal the state of the world, and the realization is the outcome of that process. In our model, the signal results from the decision about how much effort (e.g., the amount of time spent in human-AI interaction, through longer prompts, back-and-forth conversations with the AI, or direct editing) a user invests in articulating her preference. This decision is captured in our framework by the choice of  $\qErrorSD^\star(\truth)$. The realization of the signal then corresponds to the actual meaning conveyed. This may deviate from the user's true preference due to noise, as a limited interaction does not allow the user to fully convey what she truly wants, but more informative signals (i.e., smaller values of $\opteffort$) will tend to have realizations closer to the true preference $\truth$.

We make \rev{four} comments \rev{\smrev{on how the AI and users behave}}.

\rev{First,\label{resp:r3-signal-2} the AI observes the signal and its realization, but never the preference $\truth$ itself, as in the persuasion literature where the receiver knows the experiment but not the state \citep{kamenica2011, gentzkow2014}. In our setting, how tightly a user has pinned down her preference could be legible in the prompt itself (its length and specificity, the number of
constraints and examples it carries, or the reference material attached to it), all of which the AI conditions on when generating its output. \smrev{The AI treats $\qErrorSD$ as a property of the message, which tells it how much to trust $\query$, and not as evidence about $\truth$.%
}}

Second, although users typically interact with an AI incrementally, our framework encapsulates these repeated interactions in one step by focusing on the total amount of information transmitted and the final AI output. \rev{\smrev{The signal thus represents the} entire interaction process (e.g., a full conversation with ChatGPT, manual edits, etc.)\smrev{, not a single prompt. It is a reduced form of that process: edits count only through the information they give the AI about the user's preference.}}

Third, an AI's output is typically stochastic in practice, and AI models can make mistakes. In our model, we instead assume that the AI is ``perfect,'' in the sense that it knows perfectly the distribution of user preferences and is able to select the best output (in expectation) given the user's signal. This is useful for simplicity, but also to show that homogenization \rev{\smrev{and the spread of any AI bias arise even with an AI that is otherwise perfect from} the strategic behavior of users%
\smrev{, and are not only due to AI imperfections, as is sometimes argued \citep{shumailov2023}}}\smrev{.} \rev{\smrev{Sampling outputs} from the AI's posterior $\pi_A(\cdot|s)$ \smrev{instead would} add output stochasticity \smrev{and} increase the fidelity error for a given amount of information\smrev{; we expect this would reduce} the homogenization issue \smrev{but} not eliminate it\smrev{.}}

Lastly, we assume that users know the AI prior and are able to anticipate the expected fidelity error and communication costs. Again, this is useful as we can then show that the negative effects we uncover \rev{\smrev{arise even when users understand the AI perfectly}}. This assumption is more realistic for experienced users. In addition, we assume that users commit ex-ante to using the AI or working manually.
Nevertheless, there could be situations where it would make sense to first try the AI and then revert to manual work if the realized outcome is suboptimal.
This variant is less tractable, so we study it numerically in \Cref{sec: robust test} of the Online Supplement, where we show the same homogenization spiral as in our main model.

\section{\rev{Human-AI Interactions: Homogenization and Bias}}
\label{sec: Homogenization}
Different users may interact with the AI differently, sharing varying amounts of information about their preferences or even choosing not to use the AI. 
\rev{In this section, we \smrev{characterize these individual choices and their consequences for the population. Because} the AI's output tends to revert \smrev{toward} its prior's mean\smrev{,} the population's outputs \smrev{become less diverse (\Cref{ssec: homogenization part}) and inherit any bias in the prior (\Cref{sec: Bias})}.}
\subsection{Homogenization}
\label{ssec: homogenization part}
\rev{
\smrev{Because communicating preferences to the AI takes} time and effort, a user may rationally settle for an AI output that does not fully match her preferences, trading fidelity for a lower communication cost. \smrev{When many users make this} trade-off\smrev{, their outputs} drift toward the AI's default output, \smrev{and} the content produced with the AI may end up less diverse than the \smrev{population's} preferences.

We first show how a user's optimal \smrev{use of} the AI depends on her uniqueness and \smrev{that her expected} output regresses toward the population's mean preference \smrev{(\Cref{ssec: individual conseq}). We then show that these choices} reduce the variance of the population's outputs \smrev{(\Cref{ssec: homogenization})}, a form of homogenization.
}
\subsubsection{Individual level: heterogeneous use \& regression to the mean}
\label{ssec: individual conseq}

In \Cref{prop: who benefits more}, we study the properties of the solution %
and show how a user's optimal choice depends on her \emph{uniqueness}, the distance of her preference to the population mean, $\unique \triangleq |\truth-\truMean|$. \Cref{sec: Characterization} of the Online Supplement characterizes $\opteffort$ in closed form; all proofs are also in the Online Supplement.

\begin{proposition}[Heterogeneous Use of AI]
    \label{prop: who benefits more}
    Under users' optimal signals, the following hold:
    \begin{enumerate}
        \item  More unique users have a higher utility loss: $\obj^\star(\truth)$ increases in $\unique$.\footnote{All references to ``increasing'' or ``decreasing'' functions  are meant in a weak sense (i.e., ``non-decreasing'').}
         \label{sprop:individual utility loss}
        \item More unique users \rev{share more information with the AI} (if they choose to use it): $\optcosteffort$ increases in $\unique$.\label{sprop:individual comm loss}
        \item Users work with the AI if they are below a uniqueness threshold $\thresholdA$: $\unique \le \thresholdA \Leftrightarrow \obj(\truth,\opteffort) \le \costHumanOnly$. \label{sprop:individual no AI}
        \item Users that work with the AI are characterized by another threshold 
         $\thresholdD\leq \thresholdA$ such that:
\begin{enumerate}
\item If $\unique\leq\thresholdD$, users choose an uninformative
signal ($\optcosteffort=0$) and their fidelity error increases with
their uniqueness.\label{sprop:individual default}
\item If $\thresholdD<\unique\leq\thresholdA$, users choose an
informative signal ($\optcosteffort>0$) and their fidelity error decreases with their uniqueness.\label{sprop:individual interact}
\end{enumerate}
    \end{enumerate}
\end{proposition}

\begin{figure}[htbp!]
\begin{center}
\begin{minipage}{1\textwidth}
    \scalebox{0.5}{\begin{tikzpicture}
    \begin{axis}[%
            xmin = -5, xmax = 5,
            ymin = -0.1, ymax = 1.5,
            xtick distance = 1,
            table/col sep=comma,
            width=0.8\textwidth,
            height=0.45\textwidth,
            title=test,
            xlabel style={align=center},
            xlabel={user preference $\truth$ \\ (a)},
            ylabel={},
            legend cell align = {left},
            legend style={at={(0.02,0.8)},anchor=west, fill opacity=0.8},
            title = {},
            font = \large
        ]

        % Coloring the areas based on the conditions
        \fill[magenta!10] (axis cs:-5,-2) rectangle (axis cs:-4.05855,2);
        \fill[yellow!10] (axis cs:-4.05855,-2) rectangle (axis cs:-0.86,2);
        \fill[red!0] (axis cs:-0.86,-2) rectangle (axis cs:0.86,2);
        \fill[yellow!10] (axis cs:0.86,-2) rectangle (axis cs:4.05855,2);
        \fill[magenta!10] (axis cs:4.05855,-2) rectangle (axis cs:5,2);

        \addplot [green, line width=0.5mm, opacity=1] table[x=theta,y=error]{PlotData/optimal_results.csv};
        \addplot [red, line width=0.5mm, opacity=0.9] table[x=theta,y=cost]{PlotData/optimal_results.csv};
        \addplot [blue, loosely dotted, line width=0.8mm, opacity=0.9] table[x=theta,y=value]{PlotData/optimal_results.csv};
        \addplot [green, line width=0.5mm, opacity=0.9] table[x=theta,y=errorGamma]{PlotData/optimal_results.csv};
        \addplot [green, line width=0.5mm, opacity=0.9] table[x=theta,y=errorGamma2]{PlotData/optimal_results.csv};

        \legend{
            fidelity error $\costQual^\star$,
            communication cost $\costParam \costEffort^\star$,
            utility loss $\obj^\star$,
        }
        \draw [black, densely dashed, line width=0.5mm, opacity=0.9] (0.86,-0.1) -- (0.86,1.5);
        \draw [black, densely dashed, line width=0.5mm, opacity=0.9] (-0.86,-0.1) -- (-0.86,1.5);
        \draw [black, dotted, line width=0.5mm, opacity=0.9] (4.05855,-0.1) -- (4.05855,1.5);
        \draw [black, dotted, line width=0.5mm, opacity=0.9] (-4.05855,-0.1) -- (-4.05855,1.5);
    \end{axis}

\end{tikzpicture}}
\end{minipage}%
\begin{minipage}{1\textwidth}
    \scalebox{0.5}{ \hspace{-13cm}\begin{tikzpicture}
    \begin{axis}[%
            xmin = -5, xmax = 5,
            ymin = -5, ymax = 5,
            xtick distance = 1,
            table/col sep=comma,
            width=0.8\textwidth,
            height=0.45\textwidth,
            title=test,
            xlabel style={align=center},
            xlabel={user preference $\truth$ \\ (b)},
            ylabel={},
            legend cell align = {left},
            legend style={at={(0.02,0.8)},anchor=west, fill opacity=0.8},
            title = {},
            font = \large
        ]

        % Coloring the areas based on the conditions
        \fill[magenta!10] (axis cs:-5,-5) rectangle (axis cs:-4.05855,5);
        \fill[yellow!10] (axis cs:-4.05855,-5) rectangle (axis cs:-0.86,5);
        \fill[red!0] (axis cs:-0.86,-5) rectangle (axis cs:0.86,5);
        \fill[yellow!10] (axis cs:0.86,-5) rectangle (axis cs:4.05855,5);
        \fill[magenta!10] (axis cs:4.05855,-5) rectangle (axis cs:5,5);

        \addplot [green, line width=0.5mm, opacity=0.9] table[x=theta,y=theta]{PlotData/optimal_resultsEoutput.csv};
        \addplot [red, line width=0.25mm, opacity=0.9] {0};
        \addplot [blue, line width=0.5mm, opacity=0.9] table[x=theta,y=E_output]{PlotData/optimal_resultsEoutput.csv};
        \addplot [blue, line width=0.5mm, opacity=0.9] table[x=theta,y=E_output1]{PlotData/optimal_resultsEoutput.csv};
        \addplot [blue, line width=0.5mm, opacity=0.9] table[x=theta,y=E_output2]{PlotData/optimal_resultsEoutput.csv};

        \legend{
            user preference $\truth$,
            AI's default output $\truMean$,
            expected output $\EE{\pi(\cdot|\truth)}{\outcome}$,
        }
        \draw [black, densely dashed, line width=0.5mm, opacity=0.9] (0.86,-5) -- (0.86,5);
        \draw [black, densely dashed, line width=0.5mm, opacity=0.9] (-0.86,-5) -- (-0.86,5);
        \draw [black, dotted, line width=0.5mm, opacity=0.9] (4.05855,-5) -- (4.05855,5);
        \draw [black, dotted, line width=0.5mm, opacity=0.9] (-4.05855,-5) -- (-4.05855,5);
    \end{axis}

\end{tikzpicture}}
\end{minipage}
\end{center}

    \caption{Users' optimal decisions and outputs as a function of their preference $\truth$. (a) Fidelity error, communication cost, and utility loss. (b) Expected output $\EE{\pi(\cdot|\truth)}{\outcome}$. Dashed and dotted lines mark $\unique = \thresholdD$ and $\unique = \thresholdA$; white, yellow, and red regions correspond to users who accept the default output, send information to the AI, and do not use the AI. We use $\truMean=0, \truSD=1, \costParam=1, \costHumanOnly = 1.4$.}

    \label{fig: optimal_results}
\end{figure}
\Cref{prop: who benefits more} establishes that users with more ``common'' preferences have a utility advantage (\Cref{sprop:individual utility loss}) and choose to provide a less informative signal to the AI (\Cref{sprop:individual comm loss}). The proposition sorts users into up to three groups. The most common users, with $\unique \le \thresholdD$ (\Cref{sprop:individual default}), accept the default output of the AI, $\aMean$, and have zero communication cost, but their fidelity error increases as they become more unique (white region in \Cref{fig: optimal_results}(a)). Users with \rev{\smrev{$\thresholdD < \unique \le \thresholdA$}} choose to %
\rev{send an informative signal to} the AI (\Cref{sprop:individual interact})\rev{\smrev{. Because the more unique among them share more information (\Cref{sprop:individual comm loss}),} their fidelity error \smrev{falls as their uniqueness grows, while their communication cost rises (yellow region} in \Cref{fig: optimal_results}(a)\smrev{).
For} the most unique users, $\unique > \thresholdA$ (\Cref{sprop:individual no AI}), \smrev{communication cost plus fidelity error would exceed the cost of working alone, so}} \smrev{the} no-AI option \rev{\smrev{is}} preferable (red region in \Cref{fig: optimal_results}(a)).

Many users have a positive fidelity error, so the AI's output does not perfectly match a user's preference. The next proposition shows that this mismatch occurs in a specific way: on average, a user's output $\outcome$ tends to revert toward the population's mean preference

\begin{proposition}[Regression Towards the Mean]
    \label{thm: output revert to mean}
    The expected chosen output $\EE{\pi(\cdot|\truth)}{\outcome}$ of any user $\truth$ is no farther from the population's mean than her preference is: $|\EE{\pi(\cdot|\truth)}{\outcome}- \truMean| \leq |\truth - \truMean|$. Moreover, the inequality is strict for almost all users who use the AI, $\unique < \thresholdA$ and $\truth \neq \truMean$.
\end{proposition}

We illustrate this result in \Cref{fig: optimal_results}(b). The most common users accept the AI's default output (\Cref{prop: who benefits more}), so their output is exactly the mean.
For more unique users, %
\rev{choosing to share more information with} the AI mitigates the regression towards the mean in the AI's output. However, \rev{\smrev{since conveying a preference exactly would require an infinite communication cost, the regression never fully vanishes}}. The mean reversion disappears only for those very unique users who choose to complete the task by themselves. \Cref{fig: optimal_results} also shows that users whose output changes the most with the AI are not the ones with the most common or the most unique preferences, but \rev{\smrev{those just inside the default region, with $\unique$ close to $\thresholdD$}} \smrev{(}\rev{peak of the \smrev{fidelity error,}} \smrev{the} green curve in \Cref{fig: optimal_results}(a)). For them, the default AI output is good enough to keep as is, even though it ignores their individual taste. As we show next, this regression towards the mean can translate into an issue at the population level. 

\subsubsection{Societal level: output homogenization}
\label{ssec: homogenization}
If users only chose to do the work by themselves, the distribution of their output would match the distribution of their preference, $\truth \sim \mathcal{N}(\truMean,\truSD^2)$. However, with the AI, the output $\outcome$ has a different distribution, since interacting with the AI tends to yield outputs closer to the mean $\truMean$ (cf. \Cref{thm: output revert to mean}). At the population level, this suggests \emph{homogenization}, where the output distribution has a lower variance than the population distribution of preferences. 

\begin{theorem}[Homogenization]
\label{thm: homogenization}
When everyone uses the AI ($\costHumanOnly = +\infty$), the variance of the population output is lower than the variance of the population preferences, $\VV(\outcome) < \VV(\truth)$, and strictly decreases in the cost of human-AI interactions $\costParam$. For any $\costHumanOnly > 0$, $\lim_{\costParam \rightarrow 0} \VV(\outcome) = \VV(\truth)$ and $\lim_{\costParam \rightarrow +\infty} \VV(\outcome) < \VV(\truth)$.
\end{theorem}
\begin{figure}[htbp!]
    \begin{center}
        \scalebox{0.5}{\begin{tikzpicture}
    \begin{axis}[%
            xmin = 0, xmax = 3,
            ymin = 0.2, ymax = 1,
            xtick distance = 0.5,
            table/col sep=comma,
            width=0.8\textwidth,
            height=0.45\textwidth,
            title=test,
            xlabel={cost of human-AI interactions $\costParam$},
            ylabel style={align=center},
            ylabel={diversity of outputs\\$\VV(\outcome)$},
            legend cell align = {left},
            legend pos = south west,
            legend style={fill opacity=0.8},
            title = {},
            font = \large
        ]
        \addplot [blue, thick] table[x=gamma,y=var_G1]{PlotData/variance_results.csv};
        \addplot [orange, thick] table[x=gamma,y=var_G1.5]{PlotData/variance_results.csv};
        \addplot [red, thick] table[x=gamma,y=var_G2]{PlotData/variance_results.csv};
        \addplot [purple, thick, dashed] table[x=gamma,y=var_G200]{PlotData/variance_results.csv};
        % \addplot +[dashed, mark=none] coordinates {(25.75, 0.8) (25.75, 1.0)};
        % \addplot[mark=*] coordinates {(25.75,0.8)} node[pin=150:{$\hat{N}_A$}]{} ;
        \legend{
        $\costHumanOnly = 1$,
        $\costHumanOnly = 1.5$,
        $\costHumanOnly = 2$,
        $\costHumanOnly = \infty$
        }
    \end{axis}
\end{tikzpicture}}
    \end{center}
    \caption{$\VV(\outcome)$ vs. $\costParam$; we set $\truMean=0$, $\truSD=1$.}
    \label{fig: var_bias_wrt_costParam}
\end{figure}
\Cref{thm: homogenization} formalizes the risk of homogenization and points to possible solutions. When everyone uses the AI, reducing the cost of human-AI interactions $\costParam$ encourages users to share more informative signals of their preferences, which limits homogenization. For $\costHumanOnly < +\infty$, 
some users choose to complete the task without AI when the cost of human-AI interactions is high, partially restoring the output's diversity. \Cref{fig: var_bias_wrt_costParam} shows that the variance of outputs first falls as $\costParam$ rises, and recovers once the most unique users stop using the AI when $\costParam$ is sufficiently high.
In the limit $\costParam \rightarrow +\infty$ with $\costHumanOnly < +\infty$, only two groups of users remain: those who complete the task themselves and those who accept the default AI output, and outputs stay homogenized (\Cref{thm: homogenization}). \Cref{thm: homogenization} and \Cref{fig: var_bias_wrt_costParam} suggest that enhancing the interactivity of AI tools (e.g., through better interfaces, multi-modal inputs, or real-time feedback mechanisms) to achieve a sufficiently low $\costParam$, \rev{alongside retaining the ability to work without AI (i.e., a finite $\costHumanOnly$)}, can raise fidelity and reduce homogenization.

\subsection{Bias}
\label{sec: Bias}

\Cref{ssec: homogenization part} shows that working with the AI changes many users' outputs ($\theta^\star \neq \theta$) and pulls them toward the AI's default. Yet generative AIs are not necessarily trained to reflect the population's preferences exactly. For example, an AI's training data may be censored to avoid illegal or dangerous behavior \citep{thompson2023}. The training of LLMs  uses Reinforcement Learning from Human Feedback \citep{ziegler2020}, in which a small group of humans ``teach'' the model what output is preferable. These training choices of a few can then influence the output of the entire population interacting with AI. 

We model this potential AI ``bias'' via an AI prior that does not exactly reflect the population's preference distribution (i.e., $\mu_A  \neq \truMean$ or $\sigma_A \neq \truSD$), leaving the true user preference distribution and the rest of the Bayesian inference unchanged. We refer to $\mu_A \neq \truMean$ as a \emph{directional bias} and to $\sigma_A < \truSD$ (a narrower prior that underrepresents uncommon preferences) as a \emph{censoring bias}.
In \Cref{eg: writer}, the AI may have a slight bias towards a political side (directional bias), or it may avoid extreme political views (censoring bias).

\subsubsection{Individual level: AI bias \& user utility }
\label{sec: a biased AI may benefit}
The next result shows which users a biased AI helps and which it hurts.

\begin{proposition}
\label{thm: imperf AI & error if pref known}
The uncapped optimal AI loss $\obj(\truth,\opteffort)$ is:
\begin{enumerate}
\item \label{sprop:imperf directional} strictly increasing in $|\aMean-\truth|$;
\item \label{sprop:imperf censoring} nonincreasing in $\aSD$ for
$\aSD\le|\aMean-\truth|$, and nondecreasing in $\aSD$ for
$\aSD\ge|\aMean-\truth|$.\footnote{The monotonicity in item 2 is
strict between two prior standard deviations on the same side of
$|\aMean-\truth|$ whenever the optimal signal at the one farther
from $|\aMean-\truth|$ is informative; see the proof.}
\end{enumerate}
The capped loss $\optobj$ satisfies the same properties, with
``strictly increasing'' in item 1 replaced by ``nondecreasing.''
\end{proposition}

 \Cref{sprop:imperf directional} states that the directional bias favors users the AI is biased toward and hurts the others.
 In \Cref{eg: writer}, if the AI leans slightly right, a left-leaning journalist must communicate more to obtain an article aligned with her views, while a right-leaning journalist obtains a close match with less effort. The ideal case for user $\theta$ is $\mu_A=\theta$, as the default AI output would correspond to a zero utility loss, $\optobj = 0$. 
 
 \Cref{sprop:imperf censoring} states that the censoring bias benefits users with common preferences. To see this, suppose $\aMean = \truMean$, and consider a user with ``common'' preferences less than a standard deviation away from the mean, i.e., $|\truMean-\theta|<\truSD$. \rev{\smrev{A slight censoring, with} $|\truMean-\theta| < \sigma_A < \truSD$\smrev{, never hurts her, and it strictly helps her if, without censoring, she would send an informative signal. A narrower prior makes the AI put less weight on the user's noisy signal: for a user close to} the mean, this \smrev{mainly removes noise from the output, whereas for a more unique user it mainly pulls the output away from her preference.}} \smrev{However,} \rev{\smrev{censoring}} hurts users with more unique preferences, who will incur higher communication costs to maintain a reasonable fidelity or will stop using the AI altogether. Both biases can increase some users' utility loss and decrease others'. 

The next result aggregates these effects over the population. It shows that directional and censoring biases have contrasting effects on the population's expected utility loss.
 \begin{proposition}
    \label{thm: optimize total loss}Let the expected population utility loss be $\PL(\aMean,\aSD)\triangleq \EE{\truDist(\cdot)}{\optobj}$, then
    \begin{enumerate}
        \item \label{sprop:pl-directional} $\dfrac{\partial \PL(\truMean,\truSD)}{\partial \aMean} = 0$ and $\PL(\aMean,\aSD)$ is minimized at $\mu_A = \truMean$. %
        \item \label{sprop:pl-censoring} $\dfrac{\partial \PL(\truMean,\truSD)}{\partial \aSD} < 0$ when $\costParam \ge 2\truSD^2$ and $\costHumanOnly = +\infty$.\footnote{\rev{This condition is sufficient but not necessary: in \Cref{fig: SigmaADiff}(a), every $\costParam$ plotted lies below $2\truSD^2$, yet the loss-minimizing prior variance exceeds $\truSD^2$ throughout.}}

    \end{enumerate}
\end{proposition}

The proposition first shows that, while any directional bias hurts the population utility, a small directional bias has only a second-order effect, since the derivative at $\aMean = \truMean$ is zero. Intuitively, if $\mu_A = \truMean + \varepsilon$ for $\varepsilon>0$ small, slightly less than half of the users (with $\truth > \truMean+\varepsilon/2$) benefit from the bias because the AI's default output moves closer to them, while the rest (with $\truth < \truMean+\varepsilon/2$) are hurt because it moves away. These two groups nearly balance each other, which limits the total loss of utility.

\begin{figure}[htbp!]
    \begin{center}
        \scalebox{0.5}{\begin{tikzpicture}
    \begin{axis}[%
            xmin = 0, xmax = 6,
            ymin = 0.14, ymax = 0.9,
            xtick distance = 1,
            table/col sep=comma,
            width=0.8\textwidth,
            height=0.45\textwidth,
            title=test,
            xlabel style={align=center},
            xlabel={AI's prior variance $\aSD^2$ \\ (a)},
            ylabel style={align=center},
            ylabel={population utility loss, $\PL$},
            legend cell align = {left},
            legend pos = north east,
            legend style={fill opacity=0.8},
            title = {},
            font = \large
        ]
        \draw [black, dashed, line width=0.5mm, opacity=0.9] (1,0) -- (1,1);
        
        % \addplot [blue, thick] table[x=sigmaA,y=value]{PlotData/ResultsSigmaA.csv};
        \addplot [blue, thick, smooth] table[x=sigmaA,y=gamma0.1]{PlotData/LossmuA_0.csv};
        \addplot [orange, thick, smooth] table[x=sigmaA,y=gamma0.3]{PlotData/LossmuA_0.csv};
        \addplot [red, thick, smooth] table[x=sigmaA,y=gamma0.5]{PlotData/LossmuA_0.csv};
        \addplot [green, thick, smooth] table[x=sigmaA,y=gamma1.0]{PlotData/LossmuA_0.csv};
        \addplot [purple, thick, smooth] table[x=sigmaA,y=gamma1.5]{PlotData/LossmuA_0.csv};
        
        \addplot[mark=*] coordinates {(1.3,0.153001175672755)} node[]{} ;
        \addplot[mark=*] coordinates {(1.5,0.291159002481224)} node[]{} ;
        \addplot[mark=*] coordinates {(1.7,0.374016705500883)} node[]{} ;
        \addplot[mark=*] coordinates {(1.9,0.498127394507586)} node[]{} ;
        \addplot[mark=*] coordinates {(2.1,0.572855029873921)} node[]{} ;
        \legend{
            $\costParam = 0.1$,
            $\costParam = 0.3$,
            $\costParam = 0.5$,
            $\costParam = 1.0$,
            $\costParam = 1.5$,
        }
    \end{axis}
\end{tikzpicture}}
	\scalebox{0.5}{\begin{tikzpicture}
    \begin{axis}[%
            xmin = -10, xmax = 10,
            ymin = -0.6, ymax = 0.25,
            xtick distance = 2.5,
            table/col sep=comma,
            width=0.8\textwidth,
            height=0.45\textwidth,
            title=test,
            xlabel style={align=center},
            xlabel={user preference $\truth$ \\ (b)},
            ylabel style={align=center},
            ylabel={difference in utility loss $\optobj$\\with increased $\aSD$},
            % legend cell align = {left},
            % legend pos = north east,
            % legend style={fill opacity=0.8},
            title = {},
            font = \large
        ]

        % Coloring the areas based on the conditions
        \fill[green!10] (axis cs:-10,-5) rectangle (axis cs:-1.3,5);
        \fill[red!20] (axis cs:-1.3,-5) rectangle (axis cs:-0.9,5);
        \fill[red!20] (axis cs:0.9,-5) rectangle (axis cs:1.3,5);
        \fill[green!10] (axis cs:1.3,-5) rectangle (axis cs:10,5);
        
        % \addplot [green, thick] table[x=theta,y=errorDiff]{PlotData/SigmaADiff.csv};
        % \addplot [red, thick] table[x=theta,y=costDiff]{PlotData/SigmaADiff.csv};
        \addplot [blue, thick, forget plot] table[x=theta,y=valueDiff]{PlotData/SigmaADiff.csv};
        \draw [black, dashed, line width=0.5mm, opacity=0.9] (-10,0) -- (10,0);

         % Adding legend entries for the colored areas
        \addlegendimage{area legend, fill=green!10}
        \addlegendentry{better off users}
        \addlegendimage{area legend, fill=red!20}
        \addlegendentry{worse off users}
        \addlegendimage{area legend, fill=white!90}
        \addlegendentry{indifferent users}
    \end{axis}
\end{tikzpicture}}
    \end{center}
    \caption{(a) $\PL$ vs. $\aSD^2$; the circles indicate the value of $\aSD^{2}$ that minimizes $\PL$. (b) The difference in utility loss $\optobj$ when $\aSD^2 = 4$ and $\aSD^2 = 1$, with $\costParam=1$. In both panels, we use $\aMean = \truMean=0, \truSD^2=1, \costHumanOnly=+\infty$.} 
    \label{fig: SigmaADiff}
\end{figure}

The case of censoring bias (\Cref{sprop:pl-censoring} of \Cref{thm: optimize total loss}) is more surprising. Unlike the effect of directional bias, setting $\sigma_A = \truSD$ (an unbiased prior) does not generally minimize the population utility loss $\PL(\aMean,\aSD)$. Both the proposition (for $\costParam \ge 2\truSD^2$) and \Cref{fig: SigmaADiff}(a) (for smaller $\costParam$) show that when everyone uses the AI ($\costHumanOnly = +\infty$), it is preferable to have $\sigma_A > \truSD$ (the opposite of censoring). 
Recall from \Cref{ssec: individual conseq} that when everyone uses the AI, users either accept the default AI output or, if they are more unique, %
\rev{share a more informative signal} with the AI. The choice of $\sigma_A$ affects only the utility of the latter group. With an AI that puts more weight on more unique preferences, $\aSD>\truSD$, the more unique %
users are better represented and thus can choose less informative signals (compared to $\aSD=\truSD$) to obtain a high fidelity output.
This is why a moderately wider prior improves the population utility. This effect is illustrated in \Cref{fig: SigmaADiff}(b): when increasing $\aSD$, common-preference users lose zero to a small amount of utility, but more unique users see a large improvement in utility loss.

\rev{For an AI company, this is a training-side lever. It is natural to
fit an AI model using all data from a population, but \smrev{in our model} that is not \smrev{always} the best choice: \smrev{when everyone uses the AI,} deliberately over-weighting distinctive content, so
that the trained prior is \smrev{moderately} wider than the population, \smrev{lowers the population's utility loss, with most of the gain going to the users with the most distinctive preferences} (\Cref{fig: SigmaADiff}(b)). How far to widen the prior depends on the interaction cost.
\Cref{fig: SigmaADiff}(a) shows the loss-minimizing prior variance rising with
$\costParam$\smrev{: in these examples,} the harder it is for users to communicate their preferences, the wider the prior distribution should be. %
\Cref{sec:hum-in-loop} gives this prescription a dynamic counterpart: what sustains diversity across model generations is, again, mass in the tails of the training distribution.
}

This result %
also warns against the potential negative effects of censoring bias. Decreasing $\aSD$ is particularly hurtful to the most unique users, \emph{who rely on human-AI interactions the most}. While censoring can be useful in preventing dangerous or illegal uses of AI (a benefit our utility measure does not capture), our results also highlight the importance of training AI on datasets that reflect a wide range of preferences.

\subsubsection{Societal level: AI bias becomes societal bias}
\label{sec: optimal AI model}

Another interpretation of \Cref{sprop:pl-directional} of \Cref{thm: optimize total loss} is that a small directional bias $|\aMean - \truMean|>0$ (referred to as \emph{AI bias} in this section) may be hard to detect through its effect on users' utility. %
However, it can still shift users' outputs $\outcome$. For example, users who accept the default output have $\outcome = \aMean$, directly inheriting the AI bias.
On the other hand, users may choose to share more information to correct this bias and maintain a high-fidelity output. To study which effect dominates, we analyze the consequences of the AI bias on the \emph{societal bias}, defined as the bias of the output distribution, $|\EE{\truDist(\cdot)}{\outcome} - \truMean|$. 
\begin{theorem}[Societal Bias Comparative Statics]
\label{thm: effect on distn outcomes} 
\rev{\smrev{Let $\mathcal{B}(\costParam,\costHumanOnly)$ denote the societal bias, %
}

\begin{enumerate}
\item it is lower than the AI bias: $\mathcal{B}(\costParam,\costHumanOnly)\le|\aMean-\truMean|$,

\item it is minimized when $\costParam \rightarrow 0$ \emph{or} $\costHumanOnly \rightarrow 0$: \smrev{%
$\lim_{\costParam\downarrow0}\mathcal{B}(\costParam,\costHumanOnly)=0$ and
$\lim_{\costHumanOnly\downarrow0}\mathcal{B}(\costParam,\costHumanOnly)=0$,}
\item it is maximized when $\costParam \rightarrow +\infty$ \emph{and} $\costHumanOnly \rightarrow +\infty$: \smrev{%
$\lim_{\costParam\to\infty,\costHumanOnly\to\infty}\mathcal{B}(\costParam,\costHumanOnly)=|\aMean-\truMean|$, and}
\item if everyone uses AI, it increases with the cost of human-AI interactions $\costParam$.%
\end{enumerate}}
\end{theorem}

\begin{figure}[htbp!]
    \begin{center}
        \scalebox{0.5}{\begin{tikzpicture}
    \begin{axis}[%
            xmin = 0, xmax = 3,
            ymin = 0.0, ymax = 0.7,
            xtick distance = 0.5,
            table/col sep=comma,
            width=0.8\textwidth,
            height=0.45\textwidth,
            title=test,
            xlabel={cost of human-AI interactions $\costParam$},
            ylabel style={align=center},
            ylabel={societal bias\\$|\EE{\truDist(\cdot)}{\outcome} - \truMean|$},
            legend cell align = {left},
            legend pos = north west,
            legend style={fill opacity=0.8},
            title = {},
            font = \large
        ]
        \addplot [blue, thick] table[x=gamma,y=Gamma1.0]{PlotData/GammamuA_1SigAs_1.csv};
        \addplot [orange, thick] table[x=gamma,y=Gamma1.5]{PlotData/GammamuA_1SigAs_1.csv};
        \addplot [red, thick] table[x=gamma,y=Gamma2.0]{PlotData/GammamuA_1SigAs_1.csv};
        \addplot [purple, thick, dashed] table[x=gamma,y=Gamma2000.0]{PlotData/GammamuA_1SigAs_1.csv};
        % \addplot +[dashed, mark=none] coordinates {(25.75, 0.8) (25.75, 1.0)};
        % \addplot[mark=*] coordinates {(25.75,0.8)} node[pin=150:{$\hat{N}_A$}]{} ;
        \legend{
        $\costHumanOnly = 1$,
        $\costHumanOnly = 1.5$,
        $\costHumanOnly = 2$,
        $\costHumanOnly = \infty$
        }
    \end{axis}
\end{tikzpicture}}
    \end{center}
    \caption{$|\EE{\truDist(\cdot)}{\outcome} - \truMean|$ vs. $\costParam$; we use $\truMean=0$, $\aMean=1$, $\truSD=\aSD=1$ (the AI bias is $|\aMean - \truMean|=1$).}
    \label{fig: bias_wrt_costParam}
\end{figure}

\Cref{thm: effect on distn outcomes} is illustrated in \Cref{fig: bias_wrt_costParam} and shows an encouraging result: human-AI interactions can partially prevent AI bias from becoming societal bias. 
In \Cref{eg: writer}, a left-leaning journalist may share a more informative signal about her preference to correct the output if the AI is biased to the right. 
This is particularly true when users can easily interact with the AI (a low $\costParam$) or simply stop using it (a low $\costHumanOnly$).
However, \Cref{thm: effect on distn outcomes} also states that when human-AI interactions are not efficient (high $\costParam$) and the task is large and laborious (high $\costHumanOnly$), rational users increasingly accept the AI bias, which in the limit is fully converted into a societal bias.\footnote{As in \Cref{fig: var_bias_wrt_costParam}, when $\costHumanOnly$ is finite, sufficiently high values of $\costParam$ lead the most unique users to prefer doing the work themselves, thereby reducing societal bias. }
For example, generative AI systems that favor speed over interactivity (e.g., the AI writing assistant Grammarly) or tackle complex tasks with limited interactivity (e.g., the image generator Midjourney) may fall into this category. Any bias they introduce may have a stronger influence on societal output than systems or settings with lower barriers for communicating preferences (e.g., ChatGPT). \rev{This also suggests a partial remedy. Even if a firm has not detected bias in the AI, the firm can still limit how much of that bias reaches the population's output by making human-AI interactions sufficiently cheap (a low $\costParam$).}

\section{Homogenization Spiral}
\label{sec:training-loop}
We now consider the potential long-term consequences of the homogenization phenomenon identified in \Cref{sec: Homogenization}.
\rev{We first introduce a self-training loop, in which AI-generated content trains the next generation of AI, and characterize how the AI's prior evolves across iterations (\Cref{sec:prior-dynamics}). The prior dynamics suggest homogenization over time, but whether diversity is eventually lost turns out to depend on the tails of the population's preference distribution, that is, on the most unique users (\Cref{sec:unique_key})}.

\rev{Throughout this section, we take $\costHumanOnly=+\infty$ and normalize
$\truMean=0$. Setting $\costHumanOnly=+\infty$ means every user works with the
AI, which isolates the self-training loop from the opt-out channel we study in
\Cref{sec: human stops death spiral}; $\truMean=0$ is without loss of generality. \Cref{sec:prior-dynamics} keeps the Gaussian population of \Cref{sec:model}. \Cref{sec:unique_key} allows a general symmetric population distribution $\truDist$ (still the AI's initial prior), because we study how its support and tails impact homogenization in the long-run. %
}

\subsection{The Self-Training Loop and Prior Dynamics}\label{sec:prior-dynamics}

As more and more content becomes AI-generated, it could become part of the training data for the next generation of AIs. Because of the homogenization issue (cf. \Cref{sec: Homogenization}), this would lead to an incorrect AI distribution of human preferences (the AI's prior). The next AI generation would be even more likely to return homogenized outputs, potentially resulting in a ``death spiral'' of homogenization, where the diversity of outputs diminishes over time.\footnote{
In our model, a death spiral occurs when the variance of outputs collapses to zero, i.e., $\limsup_{t\rightarrow \infty}  \VV(\outcome_t) = 0$.} 

We study this phenomenon within our model, considering a \emph{self-training loop} where the AI's prior distribution is periodically updated to be the output distribution (the distribution of $\outcome$).  \Cref{alg:self-training-loop} shows the procedure in detail.
At any period $t$,  given the AI prior $\pi_t$, users determine their optimal signal choice characterized by $\opteffortAlg{t}$. When a user works with the AI (in this section, every user does), the AI uses the signal realization $s\sim \pi(\cdot|\truth)$, with $\qErrorSD=\opteffortAlg{t}$, to generate an output, $\truthAI(\query,\opteffortAlg{t})$. This period's output $\outcome_t$ (cf. \Cref{eq:decision w/ human only}) determines the new prior distribution that the AI will use in the next period. 

\begin{algorithm}
\caption{Self-Training Loop}
\begin{algorithmic}[1]
\State \textbf{Input:} Prior $\pi_0 = \truDist$, and number of iterations $T$.
\For{$t = 0$ to $T-1$}
    \State User $\truth$ chooses $\opteffortAlg{t}$ (i.e., signal $\pi(\cdot|\truth)$ with $\qErrorSD=\opteffortAlg{t}$) solving Problem \eqref{eq: loss def} under prior $\pi_t$.
    \State  User $\truth$ decides whether to work with the AI or do the work herself.
    \State The output $\outcome_t$ is generated according to \Cref{eq:decision w/ human only}.
    \State The new prior $\pi_{t+1}$ is the distribution of $\outcome_t$.
\EndFor
\end{algorithmic}\label{alg:self-training-loop}
\end{algorithm}

Note that the prior affects the user's signal through Bayes' rule (see \Cref{eq:bayes-update}), while the AI's output depends on both the posterior and the user's signal (see \Cref{eq: truthAI def}). That is, there is an endogenous, nonlinear relationship between this period's prior and the next period's prior. While this complexity poses an analytical challenge, we are still able to simulate and provide a partial characterization of the evolution of the AI's prior which in turn allows us to understand how the self-training loop in Algorithm 1 shapes the priors and, ultimately, impacts the evolution of $\VV(\outcome_t)$.

\begin{figure}[t!]
    \begin{center}
    \begin{subfigure}{0.3\textwidth}
    \includegraphics[width=\linewidth]{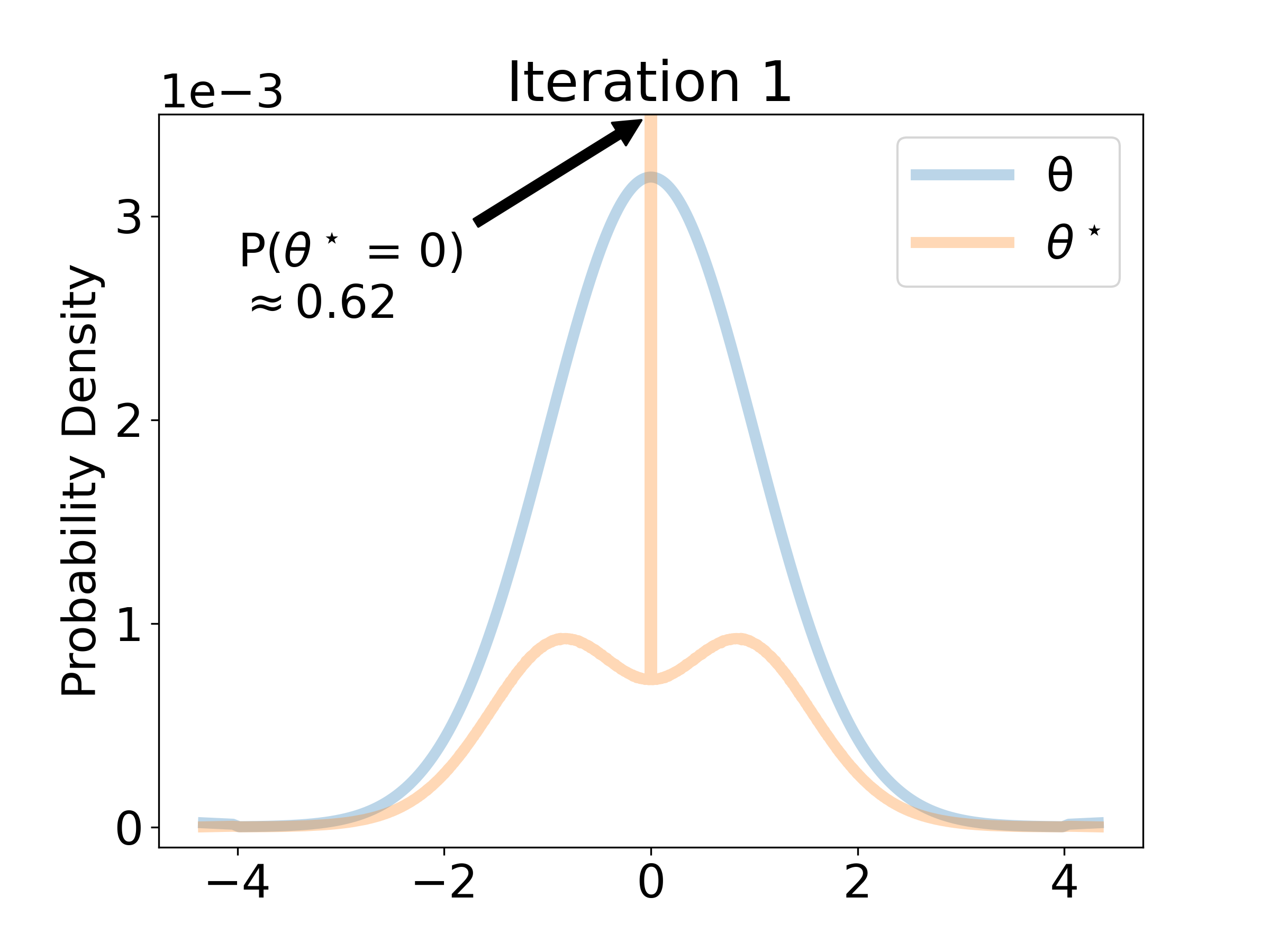}
        \caption{}
    \end{subfigure}
    \begin{subfigure}{0.3\textwidth}
        \includegraphics[width=\linewidth]{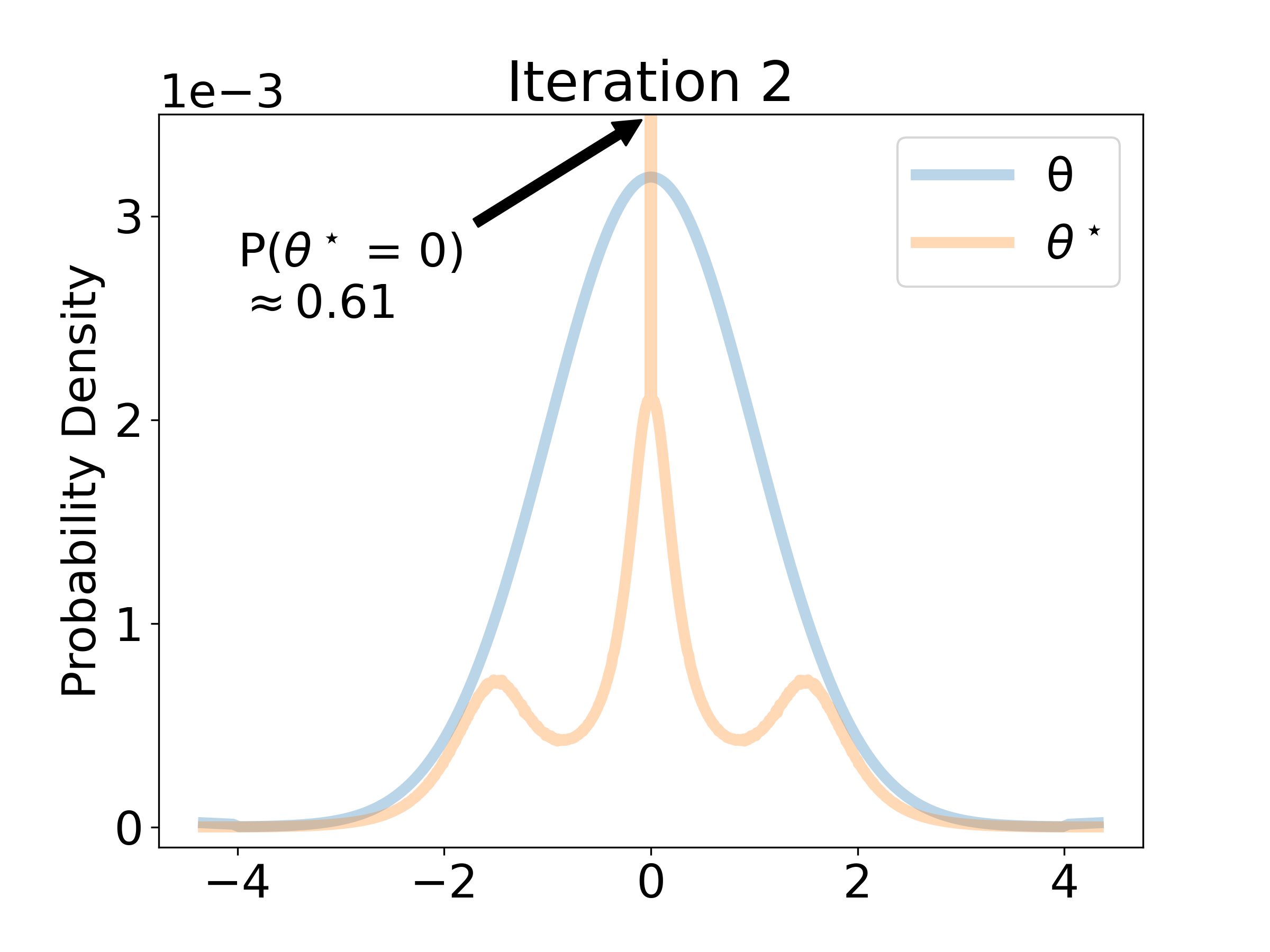}
        \caption{}
    \end{subfigure}
    \begin{subfigure}{0.3\textwidth}
        \includegraphics[width=\linewidth]{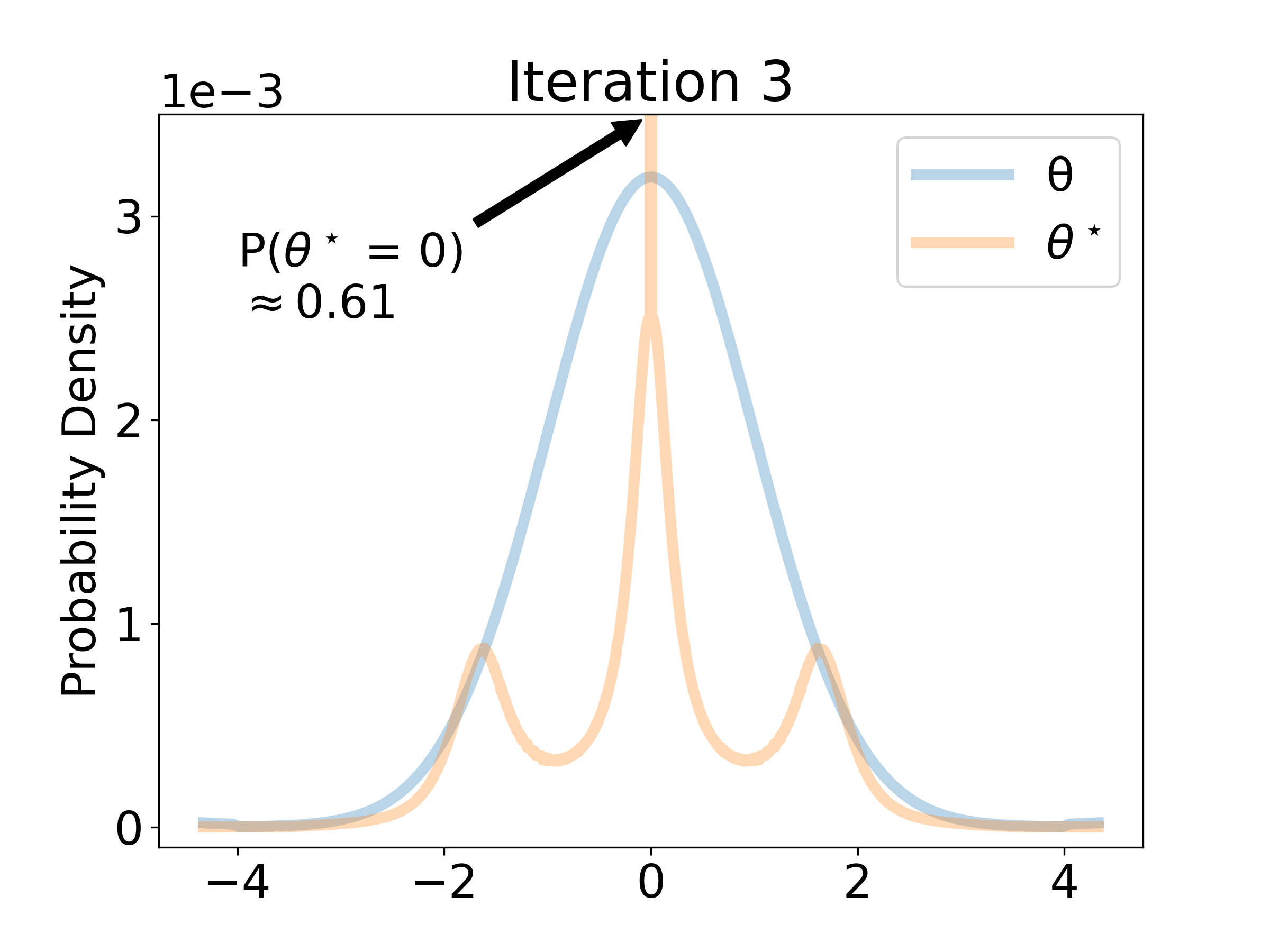}
        \caption{}
    \end{subfigure}
    \begin{subfigure}{0.3\textwidth}
    \includegraphics[width=\linewidth]{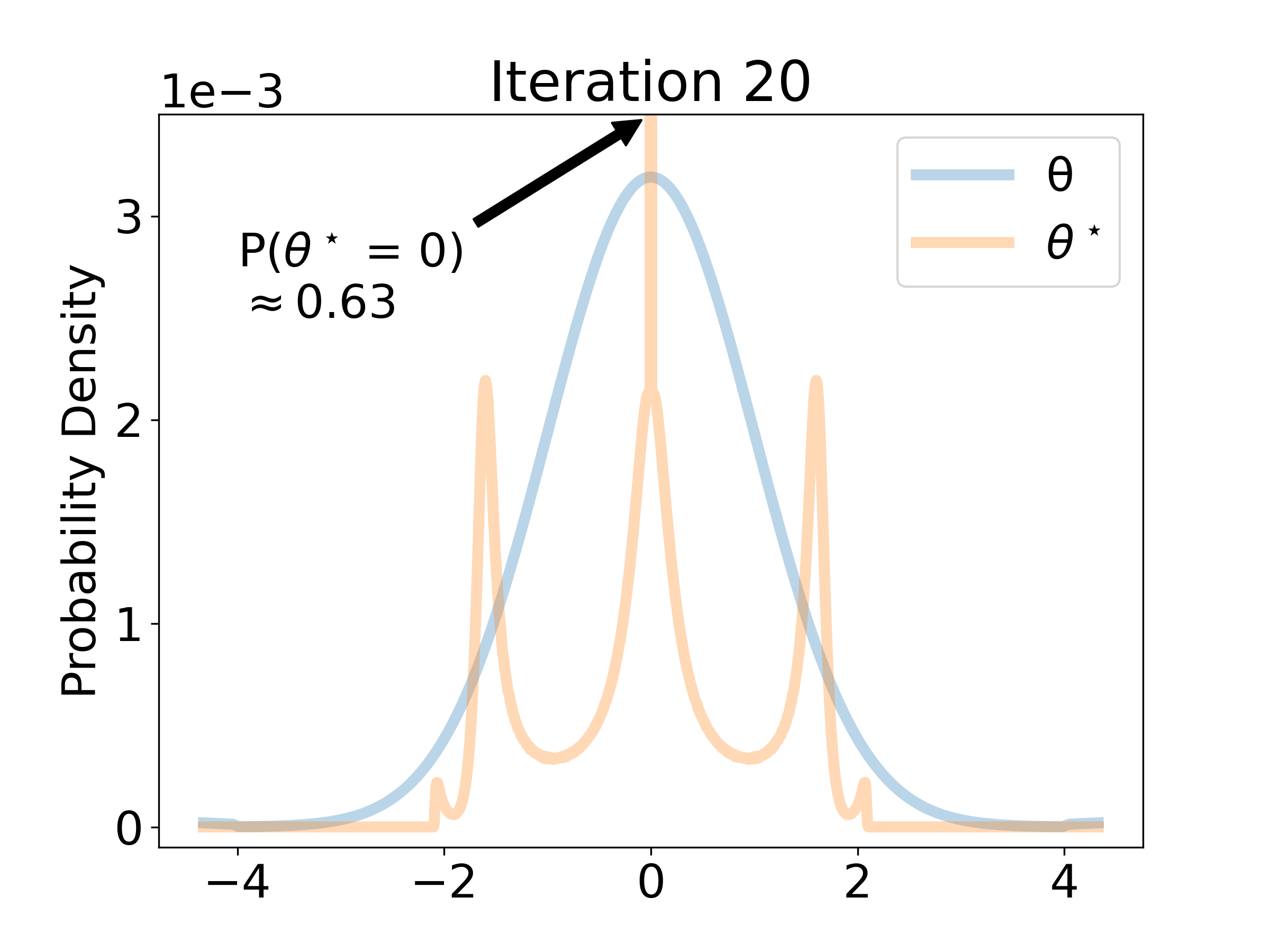}
        \caption{}
    \end{subfigure}
    \begin{subfigure}{0.3\textwidth}
        \includegraphics[width=\linewidth]{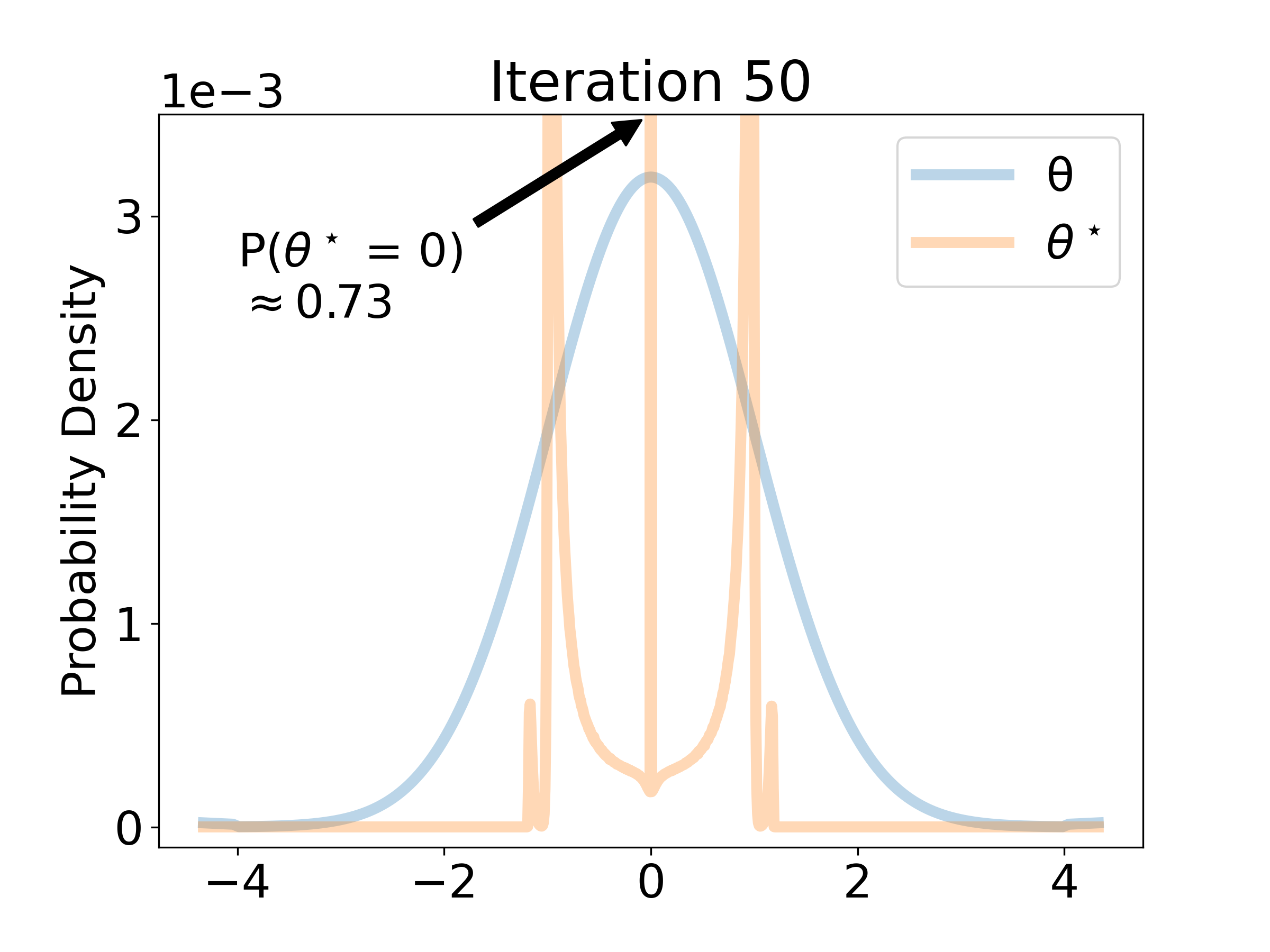}
        \caption{}
    \end{subfigure}
    \begin{subfigure}{0.3\textwidth}
        \includegraphics[width=\linewidth]{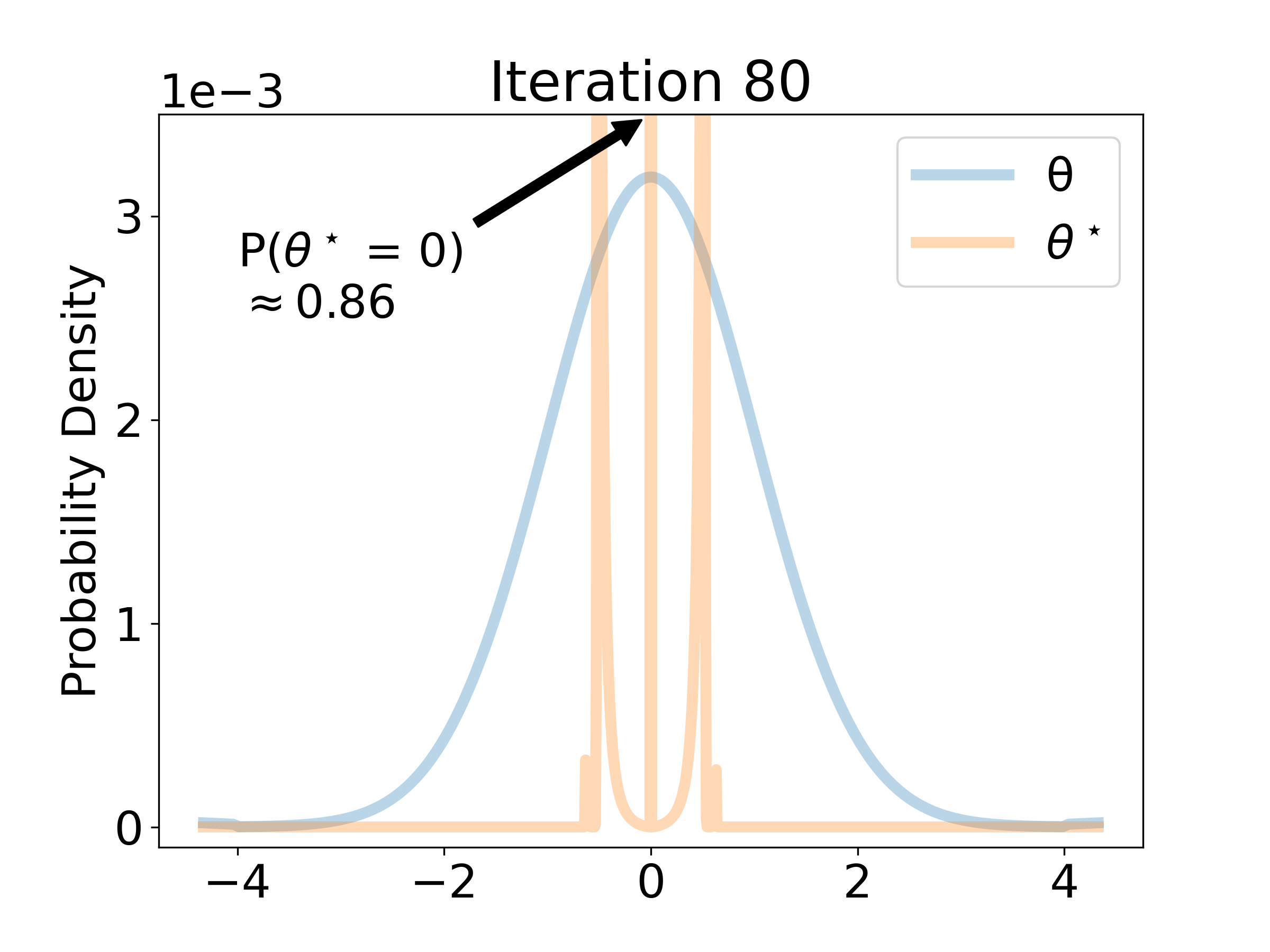}
        \caption{}
    \end{subfigure}
\end{center}
\caption{The evolution of the distribution of $\truth^\star$ across iterations. We use $\truMean=0$, $\truSD=1$, $\costParam=1$, $\Gamma = +\infty$. Notice that the point mass at $\truth^\star = 0$ is out of view; it exceeds 0.6 at every iteration.
}
\label{fig: iter converge detail}
\end{figure}

\begin{proposition}[Prior Dynamics]\label{prop:iter-den}
Under the assumptions in this section, the priors in \Cref{alg:self-training-loop} are symmetric and have finite second moments. The prior after the first iteration is
\begin{equation}\label{eq:den-1}
 \pi_1(z)
 =\left[1-2\bar\Phi(\thresholdD/\truSD)\right]\delta_0(z)
  +\EE{\truDist}{
    \frac{\phi\!\left((z-\truth\kappa(\truth))
                   /(\kappa(\truth)\opteffortAlg{0})\right)}
         {\kappa(\truth)\opteffortAlg{0}}
      \1{|\truth|>\thresholdD}},\footnote{$\delta_0(z)$ is the Dirac delta at zero.}
\end{equation}
where $\kappa(\truth)=\truSD^2/(\truSD^2+(\opteffortAlg{0})^{2})$. Moreover, for any $t$,
$\truthAIAlg{t}(\cdot,\opteffortAlg{t})$ is strictly increasing and

\begin{equation}\label{eq:den-t}
\begin{aligned}
 \pi_{t+1}(z)
 &=\PP_{\truDist}(\opteffortAlg{t}=\infty)\delta_0(z)+\EE{\truDist}{
  \frac{\phi\!\left((\vartheta_{A,t}(z,\truth)-\truth)
                         /\opteffortAlg{t}\right)}
       {\opteffortAlg{t}\,
        \truthAIAlg{t}'(\vartheta_{A,t}(z,\truth),\opteffortAlg{t})}
   \1{\opteffortAlg{t}<\infty}},
\end{aligned}
\end{equation}
where $\vartheta_{A,t}(z,\truth)$ is the inverse of $\truthAIAlg{t}(\cdot,\opteffortAlg{t})$ on its range, and $\truthAIAlg{t}'$ is its derivative.

\end{proposition}

To understand the proposition, we simulate \Cref{alg:self-training-loop}.\footnote{We use the Lloyd-Max algorithm %
to discretize all distributions and obtain a discretized version of \Cref{alg:self-training-loop}\rev{, on a 1001-point grid spanning $[-4.33,4.33]$}. \rev{As any discretization, this grid truncates the support, so the simulated population is bounded (see \Cref{sec:unique_key}).}} As shown in \Cref{fig: iter converge detail}, the prior becomes increasingly concentrated and departs substantially from its initial distribution. In the first iteration, users with $|\theta|\leq\thresholdD$ provide uninformative signals and select the default output, generating a point mass at zero (the initial prior mean), whereas users with $|\theta|>\thresholdD$ exert effort to improve fidelity and each contribute a normal density with mode at $\theta \kappa(\theta)$ with $\kappa(\theta)<1$. The continuous component of the new prior is therefore a mixture of these individual, partially homogenized densities. As illustrated in \Cref{fig: iter converge detail}(a), aggregating this mixture produces two new modes away from zero: because more unique users disclose more information, the prior regresses to the mean more rapidly near the center than in the tails, preserving diversity through these intermediate peaks. \Cref{prop:iter-den} formalizes this initial update in \Cref{eq:den-1} and subsequent updates in \Cref{eq:den-t}. %
In later iterations, the prior preserves the same broad structure: a point mass at zero, from users who select the default output, alongside a continuous mixture of densities from users who exert effort, %
which in turn helps preserve heterogeneity in the AI's prior.

The prior is thus shaped by two countervailing forces: the mass at zero and the new modes. In our simulations, once a mode exists in the prior, it attracts the AI's output, much as the prior mean does in \Cref{sec: Homogenization}. A user who puts more effort into increasing fidelity sends a more precise signal, and the AI's posterior pulls her output toward the mode closest to her signal. This strengthens the mode in the next iteration and can even increase output diversity after the first iteration. Two remarks are in order: (i) AI homogenization may not only slowly have everyone create the same content (i.e., concentrate around zero), but also tend to polarize outputs around a few other options; and (ii) the modes can act as a counterbalance to homogenization, helping maintain output heterogeneity in the long run.

\rev{As suggested by \Cref{fig: iter converge detail} and \Cref{fig: iter converge base-a}, \smrev{the} net effect of the self-training loop is that the prior increasingly concentrates and, thus, the variance of outputs decreases over time. While suggestive, these \smrev{simulations} do not prove the existence of a death spiral.
}

\begin{figure}[htbp!]
        \begin{center}
        \begin{subfigure}{0.32\textwidth}
            \includegraphics[width=\linewidth]{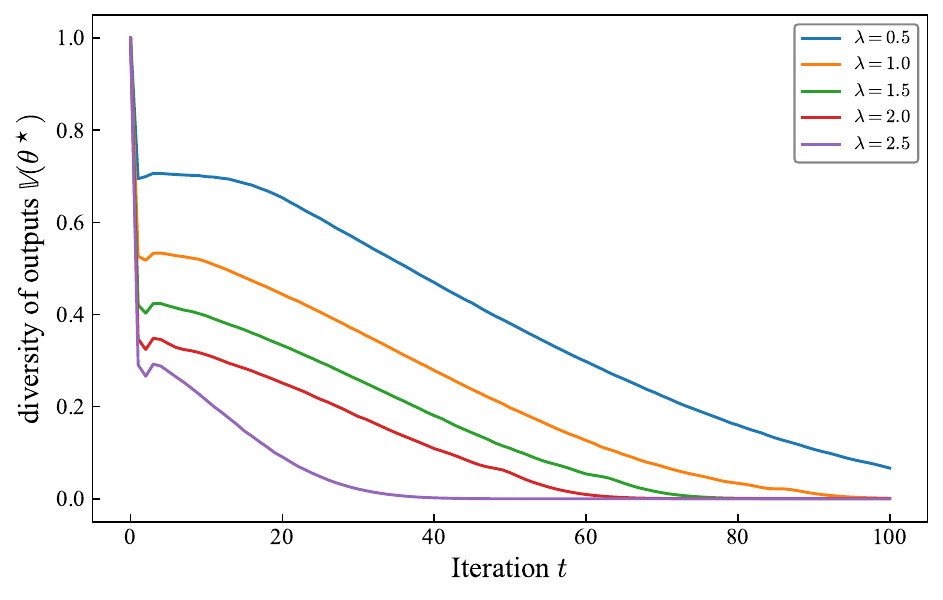}
            \caption{}
            \label{fig: iter converge base-a}
        \end{subfigure}
        \begin{subfigure}{0.32\textwidth}
            \includegraphics[width=\linewidth]{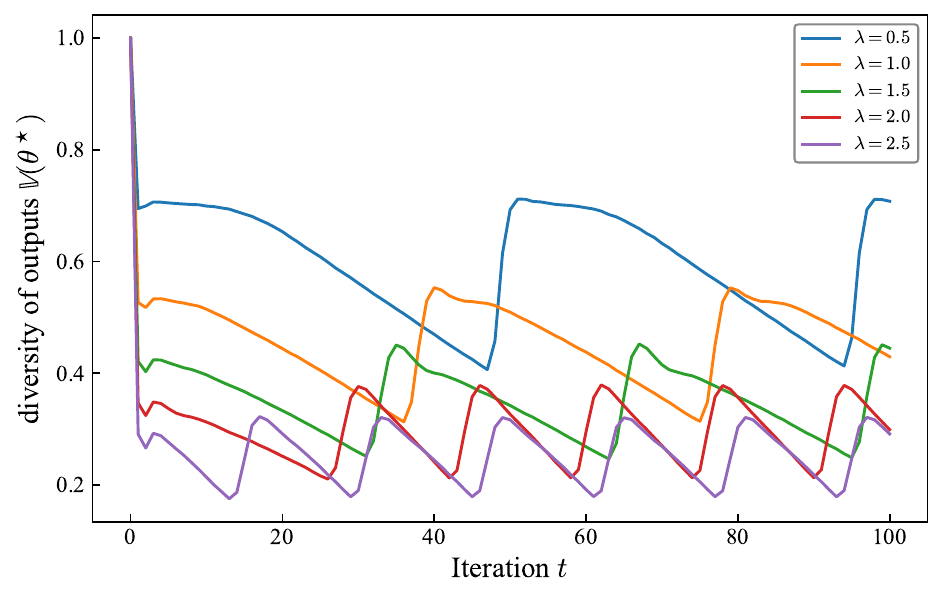}
            \caption{}
            \label{fig: iter converge base-b}
        \end{subfigure}
        \begin{subfigure}{0.32\textwidth}
            \includegraphics[width=\linewidth]{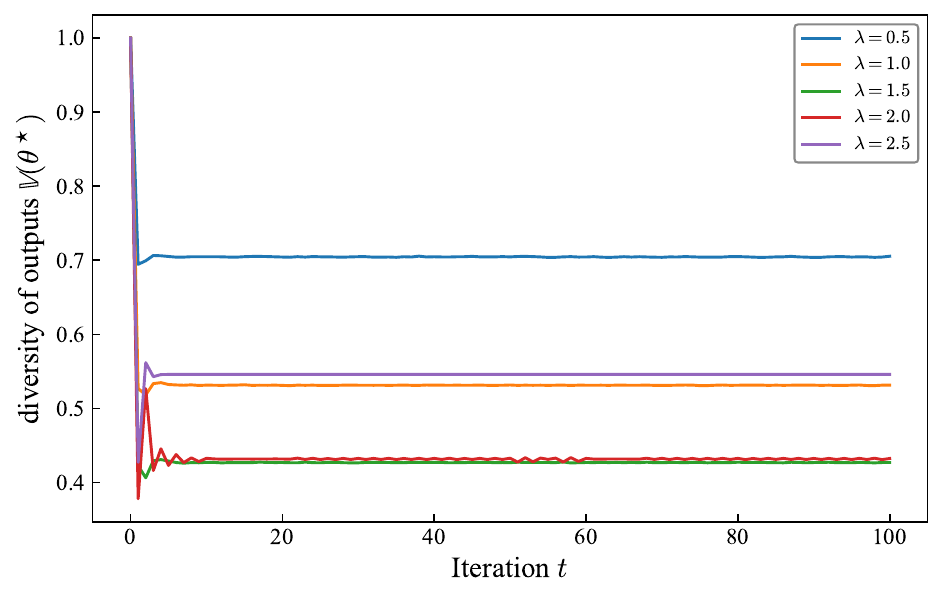}
            \caption{}
            \label{fig: iter converge base-c}
        \end{subfigure}
    \end{center}
    \caption[Iterative change of the variance of $\truthAI^\star$]{Variance of outputs, $\VV(\outcome_t)$, over iterations. We use $\truMean=0$, $\truSD=1$. (a) $\costHumanOnly=\infty$; (b) $\costHumanOnly=10$; (c) $\costHumanOnly=2$. %
    }
    
\end{figure}
\rev{
\subsection{Homogenization Driver: The Most Unique Users}
\label{sec:unique_key}

In this section, we show that whether the self-training loop ends in collapse depends on the support and \smrev{tails of the population distribution $\truDist$. With} bounded support, the \smrev{output collapses to the default} in finite time. \smrev{With unbounded support and finite variance}, finite-time collapse is impossible, though \smrev{the variance may still vanish} in the limit. \smrev{With} sufficiently heavy tails, \smrev{the variance of outputs} remains infinite for all time. \smrev{In each case, what decides the outcome is whether} the most unique users \smrev{keep communicating with the AI}.

We \smrev{start with} a bounded support\smrev{, so that no preference lies beyond some distance from the mean}.

\begin{theorem}[Bounded support: finite-time collapse] \label{thm:bounded-atomless-collapse} Suppose the population distribution $\truDist$ is symmetric, atomless, and has bounded support. Then, for any $\costParam>0$ %
the population output collapses in finite time: there exists a finite $T$ such that $\pi_T=\delta_0$.
\end{theorem}

To understand this result, let us consider the most unique users, those at the edge of the support, e.g., at a distance $B$ from the mean. As the AI's prior concentrates, the same signal buys a user of this type a worse output, because a more homogenized prior pulls the AI's output further toward the center. Thus, for a given communication cost, her fidelity deteriorates. To restore it, she must share a more informative signal, but the communication cost rises sharply for more informative signals.\footnote{Because preferences are continuous, reducing the fidelity error requires an increasingly precise signal, whose communication cost, the mutual information between her preference and her signal, grows without bound as the signal's precision improves.} Facing this trade-off, she \smrev{only partly compensates and accepts a worse output}, which in turn leads to an even more homogenized prior in the next round. As the effect compounds, a given level of fidelity requires increasingly \smrev{costly} signals, while the most she could ever save by communicating is the fidelity error she would incur under an uninformative signal, at most $B^2$. Eventually\smrev{, a cheap signal barely moves the AI's output, and any signal precise enough to move it costs more than} this bounded benefit, \smrev{so} even the most unique users settle for the default. As this happens, the AI's new prior increasingly loses mass at the tails, and the output collapses to a point mass at the mean after finitely many iterations\smrev{. Our simulations, whose grid truncates the support, are consistent with this result:} $\VV(\outcome_t)$ \smrev{collapses (}\Cref{fig: iter converge base-a}).

The picture changes once the population has unbounded support, where very unique users always exist.

\begin{proposition}[Unbounded finite-variance support prevents finite-time collapse]
\label{thm:unbounded-finite-variance-no-collapse}
Suppose the population distribution $\truDist$ is symmetric, atomless, and has unbounded support with finite variance. %
Then, for any $\costParam>0$, %
the population output \smrev{has not collapsed at any finite time}: $\pi_t\ne\delta_0$ for every finite $t\ge0$.
\end{proposition}

In \Cref{thm:unbounded-finite-variance-no-collapse} there are no ``most extreme'' types. However concentrated the prior becomes, some users' types are unique enough that using an uninformative signal would incur a high fidelity error (on the order of $\truth^2$), whereas a
fixed, moderately informative signal costs only a bounded amount \smrev{and, for users far enough from the mean, recovers more fidelity than it costs}. These users therefore always prefer to keep communicating, introducing mass away from the default at every iteration, so the prior can never equal $\delta_0$ in finite time. This does not preclude a gradual loss of \smrev{diversity, as the} variance may still tend to zero in the \smrev{limit, but} it rules out the abrupt, finite-time collapse of the bounded case.\footnote{For the Gaussian population, we believe that the death spiral does occur, but proving it remains open. Simulations cannot settle the question as any numerical method must truncate or approximate the tails, the case in which \Cref{thm:bounded-atomless-collapse} predicts collapse.}

\smrev{Finally, when the tails} of the population distribution are \smrev{heavy enough, so that very unique users are relatively common,} the variance of outputs does not spiral to zero.

\begin{proposition}[Two-sided polynomial tails]
\label{cor:two-sided-polynomial-tails}
Fix $1<\zeta<2$.
Suppose the population distribution $\truDist$ is symmetric and admits a density $p$ such that, for some constants $A,c_-,c_+>0$,
\[
    c_- |\theta|^{-\zeta-1}
    \le p(\theta)
    \le c_+ |\theta|^{-\zeta-1}
    \qquad\text{for all } |\theta|\ge A .
\]
Then, for any $\costParam>0$, %
$\VV(\outcome_{t}) = \infty$ for all $t\geq 0$.
\end{proposition}

\smrev{The condition $1<\zeta<2$ describes} symmetric Pareto-type \smrev{tails, for which} the mean is finite but the variance is infinite\smrev{. The} mass of very unique users is \smrev{then} large enough that their persistent communication keeps the output variance infinite at every iteration.

Taken together, these results \smrev{show} that {\it it is the most unique users who can prevent full homogenization}. For these users, the output the AI would generate with little communication is too far from their preference to be useful, so they must keep sharing information to obtain an output that works for them. This sustained communication, along with enough mass of such users, is what can \smrev{keep the variance of outputs from vanishing}. The flip side is a warning\smrev{. Our results compare populations with different tails, with} the AI's \smrev{initial prior equal to the population, so they do not directly cover an AI trained on data that misrepresent the population. They suggest, however, that training data that cut off the tails of the preference distribution, for example through} aggressive filtering or \smrev{the censoring of certain views, would remove} precisely the preferences of the users who sustain long-run diversity and \smrev{could push} the loop toward collapse.
}

\section{Mitigation Strategies: Humans Can Stop the Homogenization Spiral}
\label{sec: human stops death spiral}
\rev{The spiral, when it occurs, is not a property of the AI alone but rather a by-product of the interaction between humans and the AI. The AI homogenizes outputs, users tend to limit their communication rather than exert enough effort to counter homogenization, and the result becomes the next generation's training data. 
\smrev{Human} choices, such as \smrev{how} AI-interactive systems \smrev{are designed, what data the AI is trained on, and} how AI is used, \smrev{are} part of this cycle\smrev{, so} they offer an opportunity to intervene and slow homogenization.
In this section, we discuss three mitigation strategies: (i) human-AI interaction efficiency, (ii) human choice, and (iii) humans in the training loop. We \smrev{discuss how each one counters the spiral and show} that some can halt or reverse it. \smrev{Unless stated otherwise, we return to the normal population of \Cref{sec:model}, consider an unbiased AI, and now also allow $\costHumanOnly$ to be finite.}}

\subsection{Human-AI interaction efficiency}\label{sec:mit-efficiency}

\rev{The first mechanism concerns the efficiency of the human-AI interaction itself, as governed by the interaction cost $\costParam$. When interacting with the AI is cheap, users can afford to send more informative signals and obtain outputs that are closer to their true preferences. Formally, \smrev{in the first round, the noise} of the signal a user optimally sends shrinks as the interaction cost falls: $\opteffortAlg{0}$ increases in $\costParam$. Because more informative signals feed more of the population's underlying diversity back into the AI's prior, a lower $\costParam$ \smrev{should slow} the homogenization spiral\smrev{, and \Cref{fig: iter converge base-a} illustrates this}: when everyone uses the AI, a lower $\costParam$ sustains a higher variance of outputs over time.
}

\rev{
The efficiency of human-AI interactions, however, need not remain fixed. As AI systems mature, communicating with them tends to become easier over \smrev{time, through} richer personalization, multimodal input, and more intuitive ways to express preferences. We capture this by letting the interaction cost vary across iterations, writing $\{\costParam_t\}_{t\geq 0}$ for the resulting sequence of per-round interaction costs. The next result shows that\smrev{, along a suitable sequence of falling costs,} diversity does not merely decline more slowly\smrev{: the variance of outputs returns all the way to the population variance. It also shows that the pace of this improvement is decisive: it is possible to construct a sequence of decreasing costs along which diversity is lost entirely, even though interaction costs eventually become arbitrarily cheap.}
}

\begin{proposition}[The Speed of Cheaper Interactions Decides Diversity]
\label{prop: recover diversity}
Suppose $\costHumanOnly = +\infty$ \smrev{and the population distribution $\truDist$ is symmetric and atomless with finite variance, and fix an initial interaction cost $\costParam_0\in(0,2\truSD^2)$.}
\begin{enumerate}
\item \emph{Fast improvement recovers diversity.} There exists a sequence of interaction costs that is strictly decreasing towards zero such that %
$
\lim_{t\rightarrow\infty} \VV(\outcome_t) = \truSD^2.
$
\item \emph{\smrev{Slow improvement does not recover diversity.}} \smrev{If $\truDist$ also has bounded support, %
there exists a sequence of interaction costs that is strictly decreasing towards zero %
and a finite $T$ such that %
the output collapses by round $T$: $\pi_T=\delta_0$, so that $\VV(\outcome_t)=0$ for every $t\geq T-1$.}
\end{enumerate}
\end{proposition}

\rev{
\smrev{Item 1 shows that diversity can be recovered in full even though the interaction cost stays strictly positive throughout: interactions never become frictionless}. Every round still \smrev{pulls outputs toward homogenization}, yet a sufficiently fast-decreasing $\costParam_t$ enables users to share more informative signals, and this counterforce can outpace the loop. %
Along the second sequence, the cost stays above a fixed positive level long enough for the prior to collapse; %
 once the prior is $\delta_0$, the AI ignores every signal, so later improvements come too late. 
Whether easier human-AI interactions defeat the death spiral therefore depends not only on whether interactions eventually become cheap, but on how fast they do.%
}

\label{resp:lambda-theta}
\rev{
We note that even if the cost of human-AI interactions remains constant over time, this cost may differ
across users, reducing the effect of homogenization. Indeed, some users may be more adept at communicating with the AI, or they may have access to better tools that facilitate interaction, i.e., $\costParam$ may depend on a user's type $\truth$. In particular, if the most unique users have a low interaction cost, they will find it easier to express their preferences than the most common users, who may be more susceptible to homogenization. This can help preserve diversity for longer: if unique users have sufficiently small interaction cost, in line with \Cref{prop: recover diversity}, diversity cannot be fully lost, since at least these users will always be able to communicate their preferences to the AI. However, if interaction costs for all users are uniformly bounded away from zero, as in \Cref{thm:bounded-atomless-collapse}, diversity may still be lost, albeit at a slower pace. Our simulations in \Cref{sec: lambda theta} of the Online Supplement confirm this discussion: when the interaction cost falls with uniqueness, the variance of outputs declines more slowly than under a constant cost.
}

\rev{
This points to a concrete design lever\smrev{:} improving how users convey their preferences to the AI, \smrev{early and continually,} especially for the most unique users, can on its own \smrev{counter} the loss of diversity that the spiral would otherwise produce.
}

\subsection{Human choice}\label{sec:mitigation_hc}
A second mechanism is the possibility that humans may choose to not use AI. For tasks that can also be done by hand (i.e., $\costHumanOnly$ is finite), \Cref{fig: iter converge base-b} \rev{\smrev{shows that the variance of outputs first declines (as in the spiral), then rebounds} within a couple of iterations, and then \smrev{declines again}}, resulting in a cyclical behavior.
As \rev{\smrev{homogenization}} progresses, it becomes more and more costly for users to get the AI to return an output that is close to their true preferences. At some point, it becomes preferable for the most unique users to stop using the AI and do the work themselves. As soon as this happens, the next AI prior \rev{\smrev{recovers the true preferences of}} the most unique users. In turn, this added mass at the extremes of the AI prior acts as an attractor, which enables other users to lower their fidelity error. In just a few iterations, the AI output's diversity recovers substantially. However, at this point, the AI becomes good enough so that most users choose to rely on it again, restarting the cycle. 
The following result formalizes this intuition and shows that, when $\costHumanOnly$ is finite, the variance of the AI's outputs either does not vanish or it cycles.

\begin{theorem}[Human Choice can Stop the Spiral]
\label{theorem: lower-bounded variance}
When $\costHumanOnly$ is finite, the variance of the AI's outputs is bounded away from zero, i.e., $\inf_{t\geq0}\VV(\outcome_t)>0$. It may oscillate, but remains above the positive bound. %

\end{theorem}
\rev{
\smrev{\Cref{theorem: lower-bounded variance} assumes that all users face the same finite} cost of working without the AI\smrev{, but} this cost could depend on a user's type.
For example, performing the task by hand may require an increasing level of effort as preferences become more unique, leading to an increasing cost of performing the task manually ($\costHumanOnly(\truth)$ increasing in $|\truth|$). Alternatively, the most unique users may be exactly the ones most capable of working alone, leading to a decreasing cost of performing the task manually ($\costHumanOnly(\truth)$ decreasing in $|\truth|$). The following corollary extends \Cref{theorem: lower-bounded variance} to \smrev{a type-dependent and symmetric manual cost $\costHumanOnly(\truth)$. %
It again guarantees a positive floor on the variance of outputs at every round, above which the variance may oscillate}.
}

\rev{
\begin{corollary}[Human Choice with Type-Dependent Manual Cost]
\label{corollary: lower-bounded variance}
Let the cost of working without the AI $\costHumanOnly(\cdot)$ be continuous\smrev{,}
\smrev{with
$\costHumanOnly(\truMean+x)=\costHumanOnly(\truMean-x)$ for every $x$.
Suppose} a positive mass of users would rather complete the task themselves than accept a fully homogenized \smrev{output:}
\[
 \smrev{\PP_{\truDist}\bigl((\truth-\truMean)^2>\costHumanOnly(\truth)\bigr)>0.}
\]
Then the conclusion of \Cref{theorem: lower-bounded variance} holds.
\end{corollary}
}
\rev{
The logic of this result follows closely that of \Cref{theorem: lower-bounded variance}. As the AI homogenizes, obtaining a high-fidelity output would require increasingly costly informative signals, so each user's loss from using the AI rises towards the cost of full homogenization, $(\truth-\mu_p)^2$. Users whose loss crosses the cost of performing the task manually, $\costHumanOnly(\truth)$, stop using the AI, and their output helps restore diversity to the AI's next prior. In \Cref{theorem: lower-bounded variance}, these users coincide with the most unique, but \Cref{corollary: lower-bounded variance} opens the possibility for any type of user to opt out when working with the AI becomes too costly.
}

\rev{
\Cref{corollary: lower-bounded variance} \smrev{thus changes who sustains diversity, and our simulations (\Cref{fig:human-choice-optout}) bring} to light two drivers behind the diversity recovery due to human choice: a floor set by users with low $\costHumanOnly$, and the spiral's long-run dynamics driven by the most unique users.
Consider the situation in which $\costHumanOnly$ increases in $|\truth|$. In this case,
a middle band of moderately unique users may opt out if the loss of working with the AI dominates the cost of working by themselves. \Cref{fig:human-choice-optout-a} \smrev{shades} a range of types (\smrev{$0.64\le|\truth|\le 1.27$}) whose loss \smrev{stays} above their manual cost \smrev{at all three iterations. In our simulations}, this band of users persistently opts out, guaranteeing a minimum level of diversity, and the long-run dynamics then hinge on the most unique users, who keep interacting with the AI.
}

\rev{
\smrev{\Cref{fig:human-choice-optout-b} compares three shapes of $\costHumanOnly(\truth)$.} When $\costHumanOnly(\truth)$ increases moderately, we recover the cyclical behavior of \Cref{fig: iter converge base-b}: as the prior concentrates, high-fidelity outputs become ever more expensive, until the most unique users are briefly pushed out; their output then replenishes the tails of the prior, diversity rebounds, they return to work with the AI, and the spiral restarts.
When instead $\costHumanOnly(\truth)$ increases \smrev{steeply, faster than $\truth^2$ (the most these users can lose by accepting the default output), they} always work with the AI, so \smrev{their own work never replenishes the tails}, and the variance converges monotonically to the level set by the middle band of users\smrev{: the lowest} level the cycle \smrev{visits}.
When $\costHumanOnly(\truth)$ instead decreases in $|\truth|$ fast enough, the most unique users always work without the AI, \smrev{and} the variance settles at a stable, relatively high level.
}

\rev{
In many cases, diversity can survive, but who preserves it, and at what level,
depends on how the cost of working manually varies across users, and in
particular on how it impacts the most unique ones. Human choice can thus act as a
counterforce to the death spiral. \smrev{The} mechanism can also run in
reverse: if the most unique users lose some of their skills, so that working
manually becomes more costly for them, the mechanism that replenishes the tails
weakens and homogenization worsens.
This \smrev{suggests that AI companies that care about the diversity of their future models, which train on what these users produce, have reason to support users in} retaining, and developing,
distinctive ways of working.
}

\begin{figure}[t!]
\centering
\begin{subfigure}[b]{0.45\textwidth}
    \centering
    \begin{tikzpicture}
    \begin{axis}[%
            xmin = -5, xmax = 5,
            ymin = 0, ymax = 15,
            table/col sep=comma,
            width=7cm,
            height=4.9cm,
            xlabel={user type $\truth$},
            ylabel={cost},
            legend cell align = {left},
            legend pos = north east,
            legend style={fill opacity=0.8, font=\scriptsize},
            title = {},
            font = \small
        ]
        % types whose loss exceeds Gamma(theta) at all three iterations (PlotData/humanChoice_Lt.csv)
        \addplot [draw=none, fill=gray!20, forget plot] coordinates {(-1.27,0) (-0.64,0) (-0.64,15) (-1.27,15)} -- cycle;
        \addplot [draw=none, fill=gray!20, forget plot] coordinates {(0.64,0) (1.27,0) (1.27,15) (0.64,15)} -- cycle;
        \addplot [black, very thick] table[x=theta,y=Gamma]{PlotData/humanChoice_Lt.csv};
        \addplot [blue!25, thick] table[x=theta,y=L_early]{PlotData/humanChoice_Lt.csv};
        \addplot [blue!55, thick] table[x=theta,y=L_mid]{PlotData/humanChoice_Lt.csv};
        \addplot [blue!90!black, thick] table[x=theta,y=L_low]{PlotData/humanChoice_Lt.csv};
        \legend{
        {$\costHumanOnly(\truth)$},
        {$L_t$ early},
        {$L_t$ mid-descent},
        {$L_t$ low point}
        }
    \end{axis}
\end{tikzpicture}
    \caption{}
    \label{fig:human-choice-optout-a}
\end{subfigure}
\hfill
\begin{subfigure}[b]{0.45\textwidth}
    \centering
    \begin{tikzpicture}
    \begin{axis}[%
            xmin = 0, xmax = 51,
            ymin = 0.4, ymax = 0.9,
            table/col sep=comma,
            unbounded coords=discard,
            width=7cm,
            height=4.9cm,
            xlabel={round $t$},
            ylabel={diversity of outputs $\VV(\outcome_t)$},
            ylabel style={align=center},
            legend cell align = {left},
            legend pos = south west,
            legend style={fill opacity=0.8, font=\scriptsize},
            title = {},
            font = \small
        ]
        \addplot [orange, very thick] table[x=round,y=var_mod]{PlotData/humanChoice_var.csv};
        \addplot [blue, very thick, dashed] table[x=round,y=var_steep]{PlotData/humanChoice_var.csv};
        \addplot [teal, very thick, dashdotted] table[x=round,y=var_dec]{PlotData/humanChoice_var.csv};
        % Var(theta)=1 benchmark is off-scale with ymax=0.9; restore if ymax >= 1:
        % \addplot [gray, dotted, thick, forget plot] coordinates {(0,1) (51,1)};
        % \node[gray, font=\normalsize, anchor=north east] at (axis cs:50,0.995) {$\VV(\truth)$};
        \addplot [gray!60, dotted, forget plot] coordinates {(0,0.4922) (51,0.4922)};
        \legend{
        moderate increase (cycles),
        steep increase (no recovery),
        decreasing (stable floor)
        }
    \end{axis}
\end{tikzpicture}
    \caption{}
    \label{fig:human-choice-optout-b}
\end{subfigure}
\caption{Human choice with a type-dependent manual cost $\costHumanOnly(\truth)$. (a) A user's optimal loss from working with
the AI, $L_t(\truth)$, at three iterations
($\costHumanOnly(\truth)=0.2+0.5\truth^2$)\smrev{; shaded: types whose loss exceeds $\costHumanOnly(\truth)$ at all three iterations}.
(b) Variance of outputs when $\costHumanOnly(\truth)$ increases
moderately ($0.2+0.5\truth^2$), increases steeply (cubic beyond
$|\truth|=3$), or decreases ($0.3+2.7e^{-\truth^2/2}$). We use $\truMean=0$,
$\truSD=\aSD=1$, $\costParam=1$.}
\label{fig:human-choice-optout}
\end{figure}
\rev{
\subsection{Humans in the training loop}\label{sec:hum-in-loop}

So far, we have focused on how individual users respond to a homogenizing AI. A complementary mechanism operates on the role that humans can play in selecting the data for training, and how that can influence new model generations. AI companies can, and in practice do, incorporate human feedback into how they progressively train their models in a variety of ways. For example, they can mix curated human-generated data into each training generation \smrev{and, depending on whose feedback is used, rely on} reinforcement learning from human feedback \citep{ziegler2020}. These can help mitigate the homogenization spiral by ensuring that the AI's prior does not become too concentrated, and that the most unique users are represented in the training data.
In this \smrev{subsection}, we illustrate this idea \smrev{with} a simple mechanism: injecting fresh human-generated data into the training loop. We \smrev{find, numerically, that which users the injected data comes from is decisive: while data from common users can help, even a little data from} the most unique users has a larger impact on sustaining diversity.

To capture this within our framework, we modify the self-training update in \Cref{alg:self-training-loop} to blend the AI's own output distribution with fresh human data. We use the population's preference distribution $\truDist$ as a proxy for human-generated data and the self-training update $\pi_{t+1}$ as a proxy for AI-generated data \smrev{(in this subsection, $\costHumanOnly=+\infty$, so all outputs are AI-generated)}. Two parameters govern the injection: the injection size $\alpha\in[0,1]$ sets \emph{how much} human data enters the new prior, and the offset $\beta\in[0,1]$ sets \emph{which} users it comes from. Specifically, we skip the most typical $\beta$ fraction of the population and draw the $\alpha$ fraction of human data from the two \smrev{bands} just beyond it \smrev{(with $\alpha+\beta\leq 1$, so the bands fit within the population);} a larger $\beta$ \smrev{thus} injects data from increasingly unique users. Formally, let $q_x$ denote the $x$-quantile of $\truDist$. The injected data occupies the two symmetric tail bands
$$
S_{\alpha,\beta}=\big[q_{(1-\beta-\alpha)/2},q_{(1-\beta)/2}\big]\:\cup\:\big[q_{(1+\beta)/2},q_{(1+\beta+\alpha)/2}\big],
$$
each carrying mass $\alpha/2$, so $S_{\alpha,\beta}$ holds an $\alpha$ fraction of the population. The prior carried into the next iteration is then 
\begin{equation}
\label{eq:posterior-update-injection}
\hat{\pi}_{t+1}(\theta)=(1-\alpha)\pi_{t+1}(\theta)+\truDist(\theta)\1{\theta\in S_{\alpha,\beta}},
\end{equation}
which keeps a $(1-\alpha)$ share of the AI's self-generated update and supplies the rest as human data drawn from $S_{\alpha,\beta}$. Because $\truDist$ places a mass of exactly $\alpha$ on $S_{\alpha,\beta}$, the updated $\hat{\pi}_{t+1}$ is a valid distribution.

\begin{figure}[htbp!]
    \begin{center}
    \begin{subfigure}{0.3\textwidth}
        \includegraphics[width=\linewidth]{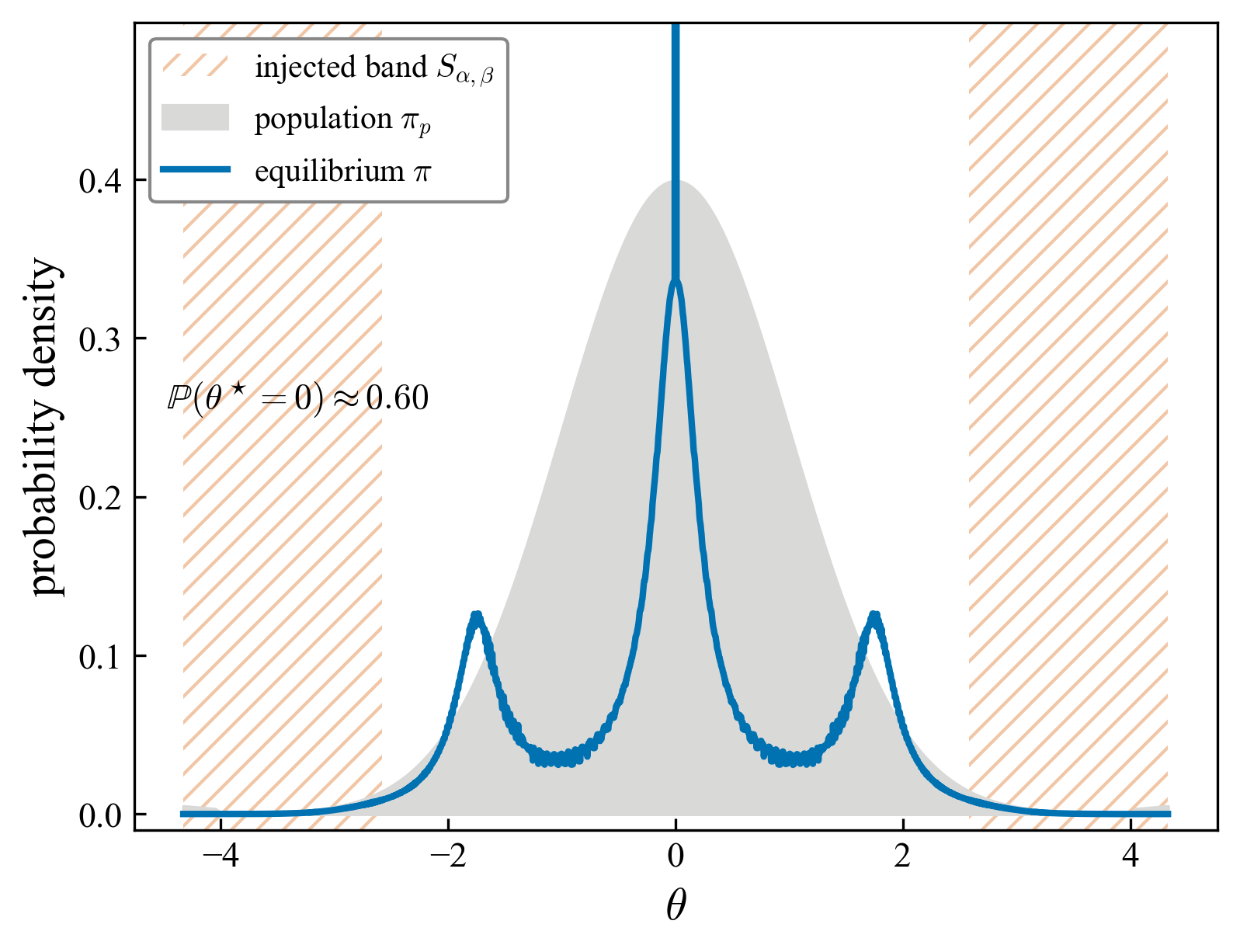}
        \caption{}
        \label{fig:injection-tail}
    \end{subfigure}
    \begin{subfigure}{0.3\textwidth}
        \includegraphics[width=\linewidth]{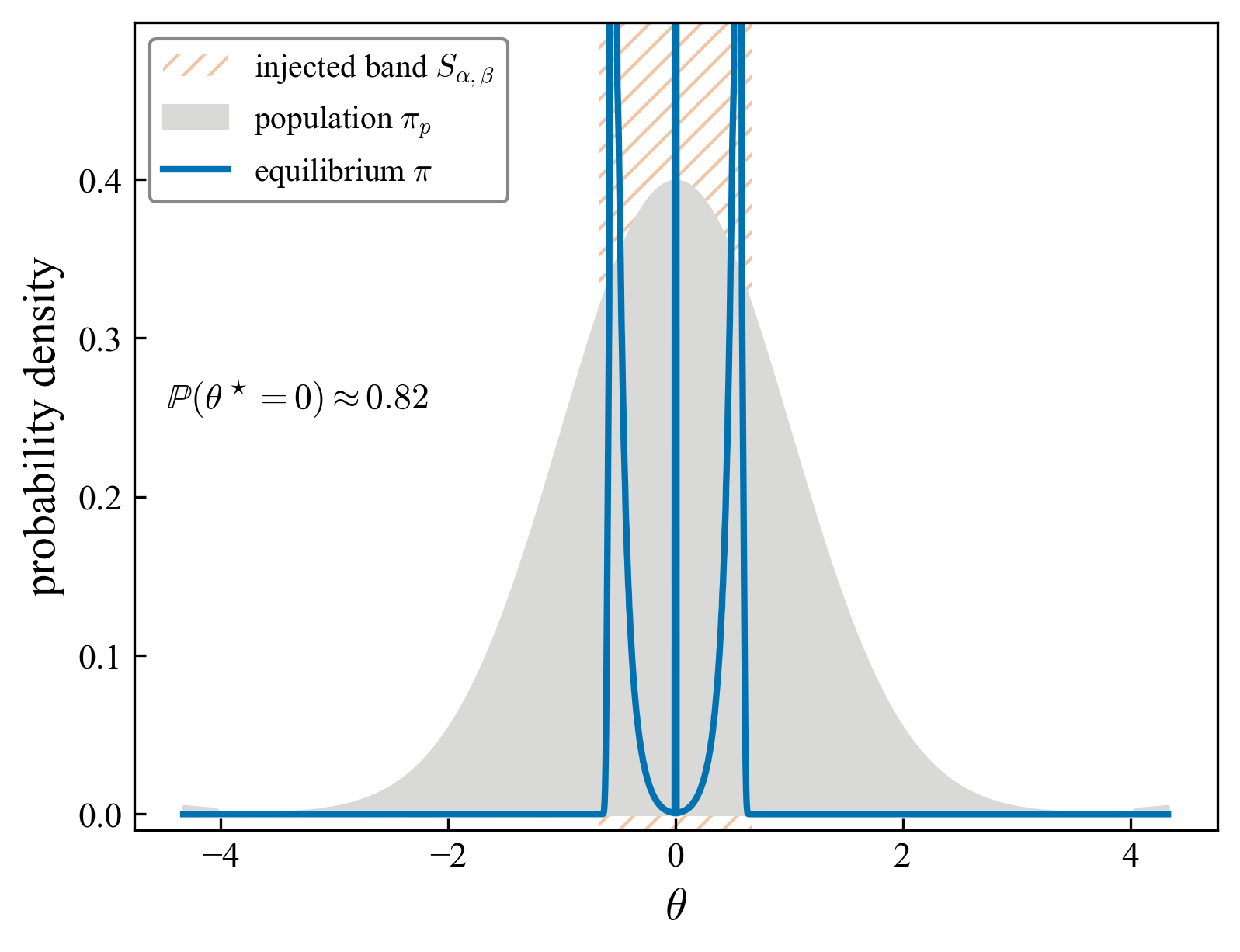}
        \caption{}
        \label{fig:injection-middle}
    \end{subfigure}
    \begin{subfigure}{0.3\textwidth}
        \includegraphics[width=\linewidth]{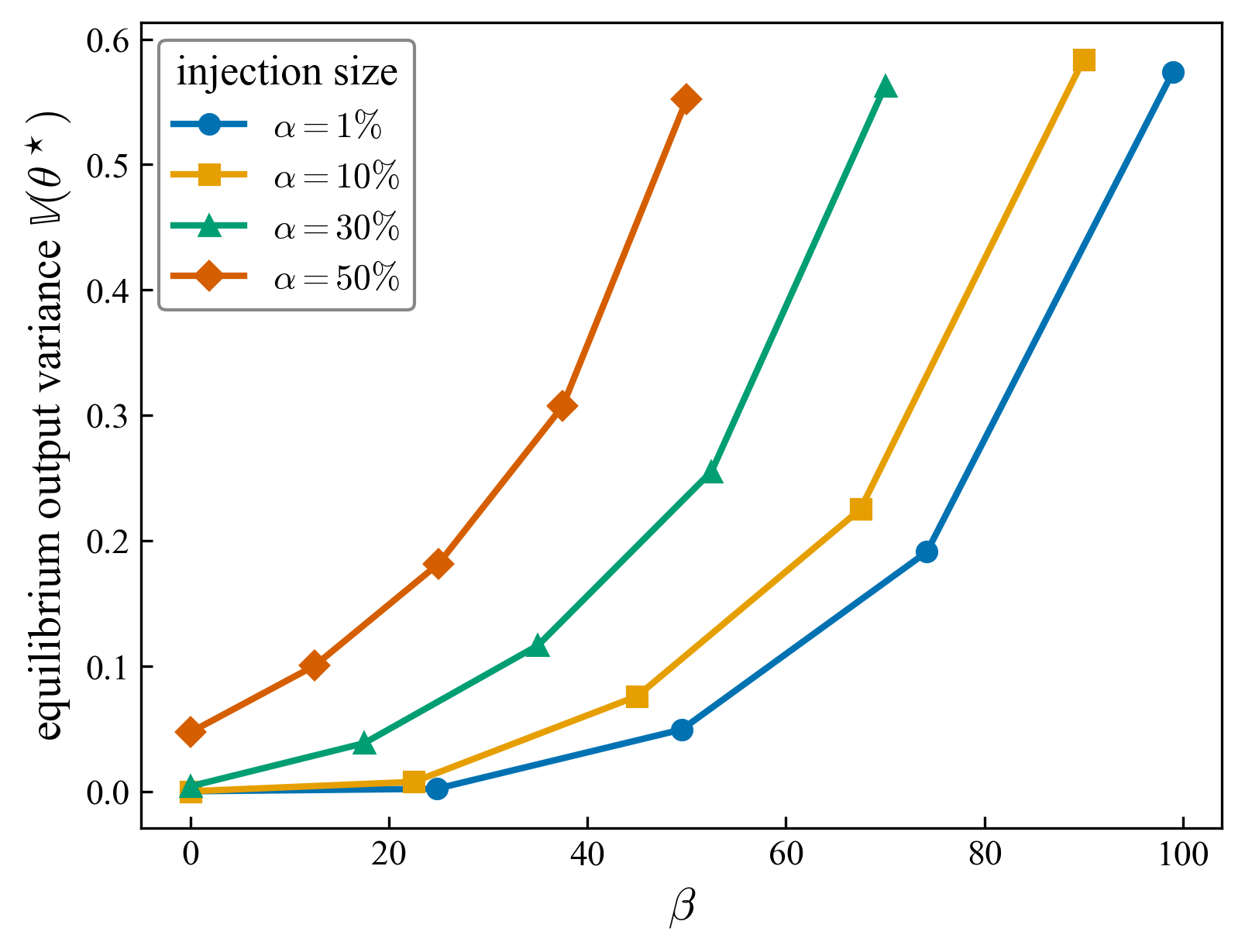}
        \caption{}
        \label{fig:injection-sweep}
    \end{subfigure}
    \end{center}
    \caption{Injecting human-generated data into the training loop. (a) An injection of $\alpha=1\%$ of the population mass drawn from the extreme tail (offset $\beta=99\%$). (b) An injection of $\alpha=50\%$ drawn from around the mean ($\beta=0\%$). (c) Variance of outputs once the loop has converged %
    as a function of the offset $\beta$. We use $\truMean=0$, $\truSD=\aSD=1$, $\costParam=1$, and $\Gamma=\infty$. %
    }
    \label{fig:injection}
\end{figure}

\Cref{fig:injection} shows that injecting human-generated data into the prior mitigates \smrev{homogenization, and that} \emph{which} human data is injected \smrev{decides whether} \emph{how much} \smrev{of it matters}. Content generated by just $1\%$ of the population drawn from the extreme tail sustains an output variance of \smrev{$0.57$} (\Cref{fig:injection-tail}), whereas content generated by fifty times as much of the population but drawn around the mean sustains only \smrev{$0.05$} (\Cref{fig:injection-middle}). \smrev{The tail injection also keeps the outputs closer to} the population's \smrev{distribution: about \jgsim{$60\%$} of them sit at the default output, against about $80\%$ under the central injection, and the rest spread much more widely}. \Cref{fig:injection-sweep} shows \smrev{the pattern} across injection sizes\smrev{: adding more data drawn around the mean ($\beta=0$)
does not help as much as adding a small amount drawn from the tails ($\beta$ close to $1$).
}

The reason is that the self-training loop progressively \smrev{erodes the mass in the} tails, while \smrev{common preferences are close to what the AI already generates by default. Human-generated} content from the most common preferences therefore mostly duplicates what the AI already tends to generate and does little to counter homogenization, whereas tail preferences restore exactly the mass the loop tends to \smrev{remove}. This is the same counterforce at work throughout this section, now acting through the AI's training rather than through users' behavior: whether diversity is replenished endogenously, by unique users who opt out of using the AI, or exogenously, by a trainer who carefully selects human generated content for training, what sustains the prior's mass at the tails is the preferences of the most unique users.

Injecting human-generated data into the training loop can also be interpreted as regularizing the self-training update $\pi_{t+1}$. Indeed, we could solve for the next prior $\hat{\pi}_{t+1}$ by minimizing 
$$
\min_{\pi\geq 0, \int \pi(\theta) d\theta=1}\;\,D_{KL}(\pi_{t+1}\,\|\,\pi) + \eta\,D_{KL}(h\,\|\,\pi), %
$$
where $D_{KL}$ is the Kullback–Leibler divergence, $\eta$ is a regularization parameter, and $h$ is an appropriately chosen reference distribution. This problem has a closed-form solution given by $\hat{\pi}_{t+1} = \frac{\pi_{t+1}+\eta h}{1+\eta}$. By setting $\eta= \alpha/(1-\alpha)$ and \smrev{$h(\theta)= \pi_p(\theta)\1{\theta\in S_{\alpha,\beta}}/\alpha$}, we recover \Cref{eq:posterior-update-injection} and the resulting insights. This formulation, however, is more general, as it allows for any reference distribution $h$ to be used, not just the population distribution $\truDist$. In practice, this means that a trainer could inject human-generated content from any source, not necessarily the population itself, and \smrev{we expect} similar effects on diversity \smrev{as long as that source carries enough mass in the tails}.

\smrev{In short, to limit} homogenization, training data should deliberately \smrev{include} users with unique perspectives: \smrev{in our simulations, a small sample drawn from the tails preserves the diversity} of AI outputs more effectively than a far larger \smrev{sample} of more common content.%

}

\section{Conclusions}
\label{sec:conclusion}

The widespread introduction of generative AI enables significant productivity gains. However, we show that the power of these tools may lead users to accept homogenized or biased outputs \rev{\smrev{rather than express}} their particular preferences, even when \rev{\smrev{they could}}. At the societal level, this can lead to homogenization (reinforced by training loop effects) and the potential influence of AI training choices on the societal output. These risks \rev{\smrev{are likely to be}} particularly strong for labor-intensive tasks (e.g., image/sound generation) or with AI tools that favor speed over preference-sharing (e.g., grammar assistants). \rev{Our results suggest that \smrev{content produced} by users with the most unique preferences plays a key role in \smrev{limiting homogenization}. Accordingly,}
we also show that enabling easier human-AI communication, training AI models on data \rev{that carry distinctive preferences}, \rev{and sustaining users' ability to work on their own} can limit these negative effects\rev{\smrev{; in particular, %
lowering interaction costs can reduce societal bias and improve diversity}}\smrev{.} %

\rev{\smrev{Our Bayesian framework is a stylized representation of how users interact with an AI}} (cf. \Cref{sec:model-discussion}). %
\rev{For example, \smrev{our baseline assumes} that a one-dimensional normal distribution can represent the vast space of human preferences, and that human-AI communication reduces to a normal signal and Bayesian inference. We expect our qualitative insights to extend to higher dimensions, since the same forces operate along each attribute, but a richer multidimensional model is a promising avenue for future work, potentially revealing which dimensions homogenize faster and how correlation across them shapes the overall loss of diversity. %
}

\rev{We also take the cost of working without the AI, $\costHumanOnly$, and the cost of human-AI interaction, $\costParam$, as primitives. \Cref{sec: human stops death spiral} allows $\costHumanOnly$ to vary across users and $\costParam$ to fall over time, but a full treatment of user-specific $\costParam$
(i.e., heterogeneous skill at prompting) remains open, as does modeling where these costs come from (e.g., how an AI provider chooses the user interface, or how users learn to lower their interaction cost with experience).}
\rev{Our analysis further targets settings where the user already knows how to complete the task and the AI's output is functionally correct, so that the object of interest is \emph{fidelity} to the user's preference rather than absolute \emph{quality}. Jointly modeling quality, and whether users can assess it, is a natural next step.} 
\rev{Finally, the population's preferences are fixed: the AI's prior changes over time, but users' tastes do not. Allowing exposure to AI output to reshape preferences themselves \smrev{could} compound the mechanism we identify.}

In sum, AIs offer a new medium for production and creation, but they may also filter and even replace the expression of our preferences, styles, and tastes. Improving human-AI interactions and encouraging users to voice their unique views can help avoid these societal pitfalls.

\bibliographystyle{informs2014}
\bibliography{references} 

\begin{thebibliography}{47}
\providecommand{\natexlab}[1]{#1}
\providecommand{\url}[1]{\texttt{#1}}
\providecommand{\urlprefix}{URL }

\bibitem[{Agrawal et~al.(2018)Agrawal, Gans, \protect\BIBand{}
  Goldfarb}]{agrawal2018}
Agrawal A, Gans J, Goldfarb A (2018) \emph{Prediction, Judgment, and
  Complexity: A Theory of Decision-Making and Artificial Intelligence}, 89--110
  (University of Chicago Press).

\bibitem[{Alemohammad et~al.(2024)Alemohammad, Casco-Rodriguez, Luzi, Humayun,
  Babaei, LeJeune, Siahkoohi, \protect\BIBand{} Baraniuk}]{alemohammad2023}
Alemohammad S, Casco-Rodriguez J, Luzi L, Humayun AI, Babaei H, LeJeune D,
  Siahkoohi A, Baraniuk RG (2024) Self-consuming generative models go {MAD}.
  \emph{International Conference on Learning Representations}.

\bibitem[{Alizamir et~al.(2020)Alizamir, de~V\'{e}ricourt, \protect\BIBand{}
  Wang}]{saed2020}
Alizamir S, de~V\'{e}ricourt F, Wang S (2020) Warning against recurring risks:
  An information design approach. \emph{Management Science} 66(10):4612--4629.

\bibitem[{Anderson et~al.(2024)Anderson, Shah, \protect\BIBand{}
  Kreminski}]{anderson2024}
Anderson BR, Shah JH, Kreminski M (2024) Homogenization effects of large
  language models on human creative ideation. \emph{Proceedings of the 16th
  Conference on Creativity \& Cognition}, 413--425 (ACM).

\bibitem[{Balakrishnan et~al.(2025)Balakrishnan, Ferreira, \protect\BIBand{}
  Tong}]{balakrishnan2025human}
Balakrishnan M, Ferreira KJ, Tong J (2025) Human-algorithm collaboration with
  private information: Na{\"\i}ve advice-weighting behavior and mitigation.
  \emph{Management Science} .

\bibitem[{Bastani et~al.(2022)Bastani, Bastani, \protect\BIBand{}
  Sinchaisri}]{bastani2022}
Bastani H, Bastani O, Sinchaisri WP (2022) Improving human decision-making with
  machine learning. \emph{Available at arXiv:2108.08454} .

\bibitem[{Bertrand et~al.(2024)Bertrand, Bose, Duplessis, Jiralerspong,
  \protect\BIBand{} Gidel}]{bertrand2024}
Bertrand Q, Bose J, Duplessis A, Jiralerspong M, Gidel G (2024) On the
  stability of iterative retraining of generative models on their own data.
  \emph{International Conference on Learning Representations}.

\bibitem[{Bhat et~al.(2023)Bhat, Agashe, Oberoi, Mohile, Jangir,
  \protect\BIBand{} Joshi}]{bhat2023}
Bhat A, Agashe S, Oberoi P, Mohile N, Jangir R, Joshi A (2023) Interacting with
  next-phrase suggestions: How suggestion systems aid and influence the
  cognitive processes of writing. \emph{Proceedings of the 28th International
  Conference on Intelligent User Interfaces}, 436--452.

\bibitem[{Binz \protect\BIBand{} Schulz(2023)}]{binz2023}
Binz M, Schulz E (2023) Using cognitive psychology to understand \text{GPT-3}.
  \emph{Proceedings of the National Academy of Sciences} 120(6).

\bibitem[{Bommasani et~al.(2022)Bommasani, Creel, Kumar, Jurafsky,
  \protect\BIBand{} Liang}]{bommasani2022}
Bommasani R, Creel K, Kumar A, Jurafsky D, Liang P (2022) Picking on the same
  person: Does algorithmic monoculture lead to outcome homogenization? Oh AH,
  Agarwal A, Belgrave D, Cho K, eds., \emph{Advances in Neural Information
  Processing Systems}.

\bibitem[{Boyac\i{} et~al.(2023)Boyac\i{}, Canyakmaz, \protect\BIBand{}
  de~V\'{e}ricourt}]{boyaci2023}
Boyac\i{} T, Canyakmaz C, de~V\'{e}ricourt F (2023) Human and machine: The
  impact of machine input on decision making under cognitive limitations.
  \emph{Management Science} .

\bibitem[{Brand et~al.(2023)Brand, Israeli, \protect\BIBand{} Ngwe}]{brand2023}
Brand J, Israeli A, Ngwe D (2023) Using \text{GPT} for market research.
  \emph{Available at SSRN 4395751} .

\bibitem[{Chaney et~al.(2018)Chaney, Stewart, \protect\BIBand{}
  Engelhardt}]{chaney18}
Chaney AJB, Stewart BM, Engelhardt BE (2018) How algorithmic confounding in
  recommendation systems increases homogeneity and decreases utility.
  \emph{Proceedings of the 12th ACM Conference on Recommender Systems},
  224–232, RecSys '18.

\bibitem[{Chayka(2025)}]{Chayka2025AI}
Chayka K (2025) {A.I.} is homogenizing our thoughts. \emph{The New Yorker}
  \urlprefix\url{https://www.newyorker.com/culture/infinite-scroll/ai-is-homogenizing-our-thoughts}.

\bibitem[{Chen et~al.(2022)Chen, Hu, \protect\BIBand{} Li}]{chen2022}
Chen N, Hu M, Li W (2022) Algorithmic decision-making safeguarded by human
  knowledge. \emph{Available at arXiv:2211.11028} .

\bibitem[{Dai \protect\BIBand{} Singh(2023)}]{dai2023}
Dai T, Singh S (2023) Artificial intelligence on call: The physician’s
  decision of whether to use \text{AI} in clinical practice. \emph{Available at
  SSRN 3987454} .

\bibitem[{de~V\'{e}ricourt \protect\BIBand{} Gurkan(2023)}]{vericourt23}
de~V\'{e}ricourt F, Gurkan H (2023) Is your machine better than you? {You} may
  never know. \emph{Management Science} .

\bibitem[{de~V\'{e}ricourt et~al.(2021)de~V\'{e}ricourt, Gurkan,
  \protect\BIBand{} Wang}]{Vericourt2021}
de~V\'{e}ricourt F, Gurkan H, Wang S (2021) Informing the public about a
  pandemic. \emph{Management Science} 67(10):6350--6357.

\bibitem[{Dohmatob et~al.(2024)Dohmatob, Feng, Yang, Charton, \protect\BIBand{}
  Kempe}]{dohmatob2024}
Dohmatob E, Feng Y, Yang P, Charton F, Kempe J (2024) A tale of tails: Model
  collapse as a change of scaling laws. \emph{Proceedings of the 41st
  International Conference on Machine Learning}, volume 235 of
  \emph{Proceedings of Machine Learning Research}, 11165--11197 (PMLR).

\bibitem[{Doshi \protect\BIBand{} Hauser(2024)}]{doshi2024}
Doshi AR, Hauser OP (2024) Generative {AI} enhances individual creativity but
  reduces the collective diversity of novel content. \emph{Science Advances}
  10(28):eadn5290.

\bibitem[{Eloundou et~al.(2023)Eloundou, Manning, Mishkin, \protect\BIBand{}
  Rock}]{eloundou2023}
Eloundou T, Manning S, Mishkin P, Rock D (2023) \text{GPTs} are \text{GPTs}: An
  early look at the labor market impact potential of large language models.
  \emph{Available at arXiv:2303.10130} .

\bibitem[{Feng et~al.(2025)Feng, Dohmatob, Yang, Charton, \protect\BIBand{}
  Kempe}]{feng2025}
Feng Y, Dohmatob E, Yang P, Charton F, Kempe J (2025) Beyond model collapse:
  Scaling up with synthesized data requires verification. \emph{International
  Conference on Learning Representations}.

\bibitem[{Gentzkow \protect\BIBand{} Kamenica(2014)}]{gentzkow2014}
Gentzkow M, Kamenica E (2014) Costly persuasion. \emph{American Economic
  Review} 104(5):457--62.

\bibitem[{Gerstgrasser et~al.(2024)Gerstgrasser, Schaeffer, Dey, Rafailov,
  Korbak, Sleight, Agrawal, Hughes, Pai, Gromov, Roberts, Yang, Donoho,
  \protect\BIBand{} Koyejo}]{gerstgrasser2024}
Gerstgrasser M, Schaeffer R, Dey A, Rafailov R, Korbak T, Sleight H, Agrawal R,
  Hughes J, Pai DB, Gromov A, Roberts D, Yang D, Donoho DL, Koyejo S (2024) Is
  model collapse inevitable? {B}reaking the curse of recursion by accumulating
  real and synthetic data. \emph{Conference on Language Modeling}.

\bibitem[{Hartmann et~al.(2023)Hartmann, Schwenzow, \protect\BIBand{}
  Witte}]{hartmann2023}
Hartmann J, Schwenzow J, Witte M (2023) The political ideology of
  conversational \text{AI}: Converging evidence on \text{ChatGPT's}
  pro-environmental, left-libertarian orientation. \emph{Available at
  arXiv:2301.01768} .

\bibitem[{Hu(2023)}]{Hu2023}
Hu K (2023) \text{ChatGPT} sets record for fastest-growing user base.
  \urlprefix\url{https://www.reuters.com/technology/chatgpt-sets-record-fastest-growing-user-base-analyst-note-2023-02-01/},
  \text{Last} accessed: 2024-02-05.

\bibitem[{Ibrahim et~al.(2021)Ibrahim, Kim, \protect\BIBand{} Tong}]{ibrahim21}
Ibrahim R, Kim SH, Tong J (2021) Eliciting human judgment for prediction
  algorithms. \emph{Management Science} 67(4):2314--2325.

\bibitem[{Jakesch et~al.(2023)Jakesch, Bhat, Buschek, Zalmanson,
  \protect\BIBand{} Naaman}]{jakesch2023}
Jakesch M, Bhat A, Buschek D, Zalmanson L, Naaman M (2023) Co-writing with
  opinionated language models affects users’ views. \emph{Proceedings of the
  2023 CHI conference on human factors in computing systems}, 1--15.

\bibitem[{Kalliamvakou(2023)}]{kalliamvakou2023}
Kalliamvakou E (2023) Research: quantifying {GitHub Copilot}’s impact on
  developer productivity and happiness.
  \urlprefix\url{https://github.blog/2022-09-07-research-quantifying-github-copilots-impact-on-developer-productivity-and-happiness/},
  \text{Last} accessed: 2024-02-05.

\bibitem[{Kamenica \protect\BIBand{} Gentzkow(2011)}]{kamenica2011}
Kamenica E, Gentzkow M (2011) {Bayesian} persuasion. \emph{American Economic
  Review} 101(6):2590--2615.

\bibitem[{Kosmyna et~al.(2025)Kosmyna, Hauptmann, Yuan, Situ, Liao,
  Beresnitzky, Braunstein, \protect\BIBand{} Maes}]{kosmyna2025your}
Kosmyna N, Hauptmann E, Yuan YT, Situ J, Liao XH, Beresnitzky AV, Braunstein I,
  Maes P (2025) Your brain on {ChatGPT}: Accumulation of cognitive debt when
  using an {AI} assistant for essay writing task. \emph{arXiv preprint
  arXiv:2506.08872} .

\bibitem[{Mat{\v{e}}jka \protect\BIBand{} McKay(2015)}]{matvejka2015rational}
Mat{\v{e}}jka F, McKay A (2015) Rational inattention to discrete choices: A new
  foundation for the multinomial logit model. \emph{American Economic Review}
  105(1):272--298.

\bibitem[{McLaughlin \protect\BIBand{} Spiess(2023)}]{bryce2023}
McLaughlin B, Spiess J (2023) Algorithmic assistance with
  recommendation-dependent preferences. \emph{Proceedings of the 24th ACM
  Conference on Economics and Computation}, 991.

\bibitem[{Noy \protect\BIBand{} Zhang(2023)}]{noy2023}
Noy S, Zhang W (2023) Experimental evidence on the productivity effects of
  generative artificial intelligence. \emph{Science} 381(6654):187--192.

\bibitem[{OpenAI(2023{\natexlab{a}})}]{openai2023-4}
OpenAI (2023{\natexlab{a}}) Custom instructions for \text{ChatGPT}.
  \urlprefix\url{https://openai.com/blog/custom-instructions-for-chatgpt},
  \text{Last} accessed: 2024-02-05.

\bibitem[{OpenAI(2023{\natexlab{b}})}]{openai2023-5}
OpenAI (2023{\natexlab{b}}) \text{ChatGPT} can now see, hear, and speak.
  \urlprefix\url{https://openai.com/blog/chatgpt-can-now-see-hear-and-speak},
  \text{Last} accessed: 2024-02-05.

\bibitem[{Padmakumar \protect\BIBand{} He(2024)}]{padmakumar2024}
Padmakumar V, He H (2024) Does writing with language models reduce content
  diversity? \emph{The Twelfth International Conference on Learning
  Representations}.

\bibitem[{Rozado(2023)}]{rozado23}
Rozado D (2023) The political biases of \text{ChatGPT}. \emph{Social Sciences}
  12(3).

\bibitem[{Sallam(2023)}]{sallam2023}
Sallam M (2023) {ChatGPT} utility in healthcare education, research, and
  practice: Systematic review on the promising perspectives and valid concerns.
  \emph{Healthcare} 11(6).

\bibitem[{Shumailov et~al.(2024)Shumailov, Shumaylov, Zhao, Papernot, Anderson,
  \protect\BIBand{} Gal}]{shumailov2023}
Shumailov I, Shumaylov Z, Zhao Y, Papernot N, Anderson R, Gal Y (2024) {AI}
  models collapse when trained on recursively generated data. \emph{Nature}
  631(8022):755--759.

\bibitem[{Sims(2003)}]{sims2003}
Sims CA (2003) Implications of rational inattention. \emph{Journal of Monetary
  Economics} 50(3):665--690.

\bibitem[{Thompson(2023)}]{thompson2023}
Thompson SA (2023) Uncensored chatbots provoke a fracas over free speech.
  \urlprefix\url{https://www.nytimes.com/2023/07/02/technology/ai-chatbots-misinformation-free-speech.html},
  \text{Last} accessed: 2024-02-05.

\bibitem[{Wang et~al.(2023)Wang, Yang, \protect\BIBand{} Sun}]{wang2023human}
Wang W, Yang M, Sun T (2023) Human-{AI} co-creation in product ideation: The
  dual view of quality and diversity. \emph{Available at SSRN 4668241} .

\bibitem[{Wei et~al.(2021)Wei, Xie, \protect\BIBand{} Ma}]{wei2021}
Wei C, Xie SM, Ma T (2021) Why do pretrained language models help in downstream
  tasks? {An} analysis of head and prompt tuning. Beygelzimer A, Dauphin Y,
  Liang P, Vaughan JW, eds., \emph{Advances in Neural Information Processing
  Systems}.

\bibitem[{Xie et~al.(2022)Xie, Raghunathan, Liang, \protect\BIBand{}
  Ma}]{xie2022}
Xie SM, Raghunathan A, Liang P, Ma T (2022) An explanation of in-context
  learning as implicit {Bayesian} inference. \emph{International Conference on
  Learning Representations}.

\bibitem[{Yu et~al.(2022)Yu, Xue, Jia, \protect\BIBand{} Tan}]{yu2022emotion}
Yu Y, Xue W, Jia L, Tan Y (2022) When emotion {AI} meets strategic users.
  \emph{Available at SSRN 4218083} .

\bibitem[{Ziegler et~al.(2020)Ziegler, Stiennon, Wu, Brown, Radford, Amodei,
  Christiano, \protect\BIBand{} Irving}]{ziegler2020}
Ziegler DM, Stiennon N, Wu J, Brown TB, Radford A, Amodei D, Christiano P,
  Irving G (2020) Fine-tuning language models from human preferences.
  \emph{Available at arXiv:1909.08593} .

\end{thebibliography}

\newpage

\begin{APPENDICES}
\fontsize{10}{11}\selectfont

\newpage
\pagenumbering{arabic}
\renewcommand{\thepage}{e-\arabic{page}}
\setcounter{page}{1}
\section*{Online Supplement}
\section{Characterization of Users' Optimal Decision}
\label{sec: Characterization}
\rev{\smrev{This section characterizes the} user's optimal decision\smrev{, which the proofs of the main results use.}} We first find the closed form of the expected fidelity error $\costQual(\truth, \qErrorSD)$ %
\rev{\smrev{and then derive}} the optimal solution to Problem \eqref{eq: loss def}.
As in \Cref{ssec: homogenization part}, we assume $\aMean=\truMean$ and $\aSD=\truSD$.
\rev{\smrev{The proofs of the main results are in: \Cref{sec:proof of homo,ec-sec: proofs training loop,ec-sec: proofs mitigation} follow Sections~\ref{sec: Homogenization} to~\ref{sec: human stops death spiral} of the paper, and \Cref{sec: robust test} reports robustness tests.}}

\begin{proposition}
    \label{thm: AI & error if pref known}
   For any $\truth, \qErrorSD$, the fidelity error is 
    \begin{equation}
       \label{eq: MSE}
       \costQual(\truth, \qErrorSD) =\dfrac{\qErrorSD^2(\truSD^4+\qErrorSD^2(\truMean-\truth)^2)}{(\truSD^2+\qErrorSD^2)^2}
   \end{equation}
    Furthermore, 
    \begin{itemize}
        \item $\costQual(\truth, \qErrorSD)$ \smrev{increases} in $\unique$.
        \item $\lim_{\qErrorSD^2 \to 0}\costQual(\truth, \qErrorSD) = 0$,\ $\lim_{\qErrorSD^2 \to \infty }\costQual(\truth, \qErrorSD) = \unique^2 $
        \item If $\unique \ge \truSD/\sqrt{2}$, $\costQual(\truth, \qErrorSD)$ is monotonically increasing in $\qErrorSD$; if $\unique < \truSD/\sqrt{2}$, there exists a threshold $\MSEcritical > 0$ such that $\costQual(\truth, \qErrorSD)$ increases in $1/\qErrorSD \in (0, \MSEcritical) $ and decreases in $1/\qErrorSD \in (\MSEcritical, \infty)$.
    \end{itemize}
\end{proposition}

\begin{proof}{Proof of \Cref{thm: AI & error if pref known}.}

By the definition of $\costQual(\truth, \qErrorSD)$ and \Cref{eq: truthAI def}, let $\qError \triangleq \query - \truth \sim N(0,\qErrorSD^2)$
\begin{equation*}
    \begin{aligned}
        \costQual(\truth, \qErrorSD) & = \EE{\pi(\cdot|\truth)}{(\truthAI(\query,\qErrorSD)-\truth)^2}
         = \EE{\pi(\cdot|\truth)}{\left(\dfrac{\truSD^2}{\truSD^2+\qErrorSD^2} \qError+\dfrac{\qErrorSD^2}{\truSD^2+\qErrorSD^2} (\truMean-\truth) \right)^2 }
         = \dfrac{\qErrorSD^2(\truSD^4+\qErrorSD^2(\truMean-\truth)^2)}{(\truSD^2+\qErrorSD^2)^2}.
    \end{aligned}
\end{equation*}
It is clear that $\costQual(\truth, \qErrorSD)$ increases in $(\truMean-\truth)^2$.

\begin{itemize}
    \item

$\lim_{\qErrorSD^2 \to 0}\costQual(\truth, \qErrorSD) = \lim_{\qErrorSD^2 \to 0}\dfrac{\qErrorSD^2(\truSD^4+\qErrorSD^2(\truMean-\truth)^2)}{(\truSD^2+\qErrorSD^2)^2} = 0$, and
$
    \lim_{\qErrorSD^2 \to \infty}\costQual(\truth, \qErrorSD) %
     =(\truMean-\truth)^2.
     $
\item
Take the derivative of $\costQual(\truth, \qErrorSD)$ with respect to $\qErrorSD^2$:
$$
\begin{aligned}
\dfrac{\partial \costQual(\truth, \qErrorSD)}{\partial \qErrorSD^2}  & = \dfrac{\partial \dfrac{\qErrorSD^2(\truSD^4+\qErrorSD^2(\truMean-\truth)^2)}{(\truSD^2+\qErrorSD^2)^2}}{\partial \qErrorSD^2}\\
& = \dfrac{(\truSD^4+2\qErrorSD^2(\truMean-\truth)^2)(\truSD^2+\qErrorSD^2)^2-2(\truSD^2+\qErrorSD^2)\qErrorSD^2(\truSD^4+\qErrorSD^2(\truMean-\truth)^2)}{(\truSD^2+\qErrorSD^2)^4}\\
& = \dfrac{\truSD^2(\truSD^4+\qErrorSD^2(2(\truMean-\truth)^2-\truSD^2))}{(\truSD^2+\qErrorSD^2)^3}
\end{aligned}
$$
    which is non-negative for all $\qErrorSD \ge 0$ if and only if \smrev{$|\truMean-\truth| \ge \truSD/\sqrt{2}$. When $|\truMean-\truth| < \truSD/\sqrt{2}$}, $\frac{\partial \costQual(\truth, \qErrorSD)}{\partial \qErrorSD^2} $ is positive for $\qErrorSD \in \left(0, \sqrt{\frac{\truSD^4}{\truSD^2-2(\truMean-\truth)^2}} \right)$, and is negative for $\qErrorSD \in \left( \sqrt{\frac{\truSD^4}{\truSD^2-2(\truMean-\truth)^2}}, \infty \right)$, so $\MSEcritical = \sqrt{\frac{\truSD^2-2(\truMean-\truth)^2}{\truSD^4}}$.
\end{itemize}
\Halmos
\end{proof}

Given \Cref{thm: AI & error if pref known}, we can solve Problem \eqref{eq: loss def} and derive the following \Cref{lemma: characterization of sol}.

\begin{lemma}
\label{lemma: characterization of sol}
\smrev{Let $S\triangleq\truSD^2$, $D\triangleq(\truth-\truMean)^2$, $\Delta\triangleq S^2+4\costParam(D-S)$, and
$w_1\triangleq\costParam/(S+\sqrt{\Delta})$.} \jgmath{There is a threshold $\thresholdD(\costParam)>0$ such that:}
\begin{enumerate}[label=(\roman*), ref=(\roman*), leftmargin=*]
\item\label{lem1:sol} the optimal solution to Problem~\eqref{eq: loss def} is
\begin{equation}
\label{eq:opt sol}
    \qErrorSD^{\star} =
    \begin{cases}
    \sqrt{\dfrac{\smrev{w_1} S}{1-\smrev{w_1}}} & \unique \smrev{>} \thresholdD(\costParam),\\[1ex]
    \infty & \text{otherwise;}
    \end{cases}
\end{equation}
\jgmath{the optimal solution is unique when $\unique\neq\thresholdD(\costParam)$, and} \smrev{at
$\unique=\thresholdD(\costParam)$ a tie between the default ($\qErrorSD=\infty$) and an informative signal is
resolved in favor of the default;}
\item\label{lem1:small} \smrev{when $\costParam<S$, $\thresholdD(\costParam)^2$ is the unique root of
$D=\obj(w_1;D,\costParam)$ in the interval $\big(\max\{0,\,S-S^2/(4\costParam)\},\ S/2+\costParam/4\big)$, where
$\obj(w;D,\costParam)=Sw+(D-S)w^2-\frac{\costParam}{2}\ln w$ is the objective of Problem~\eqref{eq: loss def}
written in terms of $w=\qErrorSD^2/(\qErrorSD^2+S)$;}
\item\label{lem1:large} when $\costParam\smrev{\ge} S$, $\thresholdD(\costParam)=\sqrt{S/2+\costParam/4}$;
\item\label{lem1:bound} $\thresholdD(\costParam)^2\ge\max\{0,\,S-S^2/(4\costParam)\}$ for all $\costParam$\smrev{,
with strict inequality when $\costParam<S$};
\item\label{lem1:mono} $\thresholdD(\costParam)$ is strictly increasing in $\costParam$;
\item\label{lem1:smooth} \smrev{$\thresholdD(\costParam)$ is continuously differentiable in $\costParam$;}
\item\label{lem1:tie} \smrev{at $\unique=\thresholdD(\costParam)$, both the default and the informative signal
are optimal when $\costParam<S$, and the default is the unique optimum when $\costParam\ge S$.}
\end{enumerate}
\end{lemma}

\begin{proof}{Proof of \Cref{lemma: characterization of sol}.}

\begin{jgmathblock}
The proof has three parts. Part~1 rewrites the problem in one variable and collects elementary facts
about the first-order condition. Part~2 finds the optimal signal for every $D$, separately for
$0<\costParam<S$ (Case~1) and $\costParam\ge S$ (Case~2); this identifies the threshold
$\thresholdD(\costParam)$ and proves Equation~\eqref{eq:opt sol}. Part~3 shows that
$\thresholdD(\costParam)$ is strictly increasing and continuously differentiable in $\costParam$.

\textbf{Part 1: reformulation and the first-order condition.}
\end{jgmathblock}
\smrev{Write $S\jgmath{\triangleq}\truSD^2>0$, $D\jgmath{\triangleq}(\truth-\truMean)^2$, and
$w\jgmath{\triangleq}\qErrorSD^2/(\qErrorSD^2+S)$\jgmath{, so that $w\in(0,1]$ and $w=1$ corresponds to an uninformative signal, $\qErrorSD=\infty$}. For $\costParam>0$ the objective is}
\begin{equation}\smrev{\label{eq: main problem w}
\obj(w;D,\costParam)\jgmath{\triangleq}S w+(D-S)w^2-\frac{\costParam}{2}\log w,
\qquad 0<w\leq1\jgmath{,}}
\end{equation}
\smrev{\jgmath{where, by Equation~\eqref{eq: MSE}, $Sw+(D-S)w^2=w(1-w)S+w^2D$ is the fidelity error and
$-\frac12\log w$ is the communication cost in Problem~\eqref{eq: loss def}.} It tends to infinity as
$w\downarrow0$, so a minimum exists and $w=0$ is never admissible at finite loss. An uninformative signal, $w=1$, has loss $D$.}

\smrev{Multiplying the first derivative of $\obj$, $\obj_w$, by $2w>0$ gives}
\begin{equation}\smrev{\label{eq: Q def}
\jgmath{Q(w)\triangleq}\;2w\obj_w=4(D-S)w^2+2Sw-\costParam.}
\end{equation}
\begin{jgmathblock}
Since $2w>0$, $Q$ has the sign of $\obj_w$ on $(0,1]$: $\obj$ decreases where $Q<0$ and increases where
$Q>0$. Two values of $Q$ will be used repeatedly, $Q(0)=-\costParam<0$ and
\begin{equation}\label{eq: Q at 1}
Q(1)=4(D-S)+2S-\costParam=4\Big(D-\frac S2-\frac{\costParam}{4}\Big),
\end{equation}
so $Q(1)$ is positive, zero or negative depending on whether $D$ is larger than, equal to or smaller than
$S/2+\costParam/4$. Let
\begin{equation}\label{eq: Delta def}
\Delta\triangleq S^2+4\costParam(D-S),
\end{equation}
so that the discriminant of $Q$ is $4\Delta$. When $\Delta\ge0$ and $D\ne S$, the roots of $Q$ are
\begin{align*}
w_+ &\triangleq \frac{-S+\sqrt\Delta}{4(D-S)}
= \frac{(\sqrt\Delta-S)(\sqrt\Delta+S)}{4(D-S)(\sqrt\Delta+S)}
= \frac{\Delta-S^2}{4(D-S)(S+\sqrt\Delta)}
= \frac{\costParam}{S+\sqrt\Delta},\\
w_- &\triangleq \frac{-S-\sqrt\Delta}{4(D-S)},
\end{align*}
where the third equality uses Equation~\eqref{eq: Delta def}. The expression $\costParam/(S+\sqrt\Delta)$
is positive for every $D$ and, when $D=S$ (where $Q$ is linear), is its unique root $\costParam/(2S)$.
\end{jgmathblock}
\smrev{The only possible interior local minimum is}
\begin{equation}\smrev{\label{eq: w1 def}
w_1\jgmath{\triangleq}\frac{\costParam}{S+\sqrt\Delta}\jgmath{\;(=w_+).}}
\end{equation}

\smrev{\jgmath{Indeed, consider the sign of $D-S$ and of $\Delta$. If $D\ge S$, then $Q$ is convex (linear if
$D=S$) with $Q(0)<0$, so $w_1$ is its only positive root, and $\obj$ decreases on $(0,w_1]$ and
increases on $[w_1,\infty)$.} If $D<S$ and \jgmath{$\Delta>0$, so that the
concave $Q$ has two distinct roots $0<w_1<w_2$}, the other root \jgmath{$w_2\triangleq w_-$} is a local maximum\jgmath{:
$Q<0$ on $(0,w_1)$, $Q>0$ on $(w_1,w_2)$ and $Q<0$ on $(w_2,\infty)$}. If $\Delta=0$ \jgmath{(which
requires $D<S$)}, then there is a unique root and we can write \jgmath{
$Q(w)=-4(S-D)(w-w_1)^2$, which vanishes only at the unique root $w_1=\costParam/S$; hence $\obj_w\le0$ on
$(0,1)$ with equality only at $w_1$, and} $\obj$ is strictly decreasing on
\jgmath{$(0,1]$}, and \jgmath{$w=1$ is the unique optimum}.
 If $\Delta<0$, the derivative is negative throughout $(0,1)$ and again \jgmath{$w=1$}
is the unique optimum.}

\begin{jgmathblock}
\textbf{Part 2: the optimal signal.} Part~1 has settled the following for every pair $(D,\costParam)$.
If $\Delta<0$ or $\Delta=0$ (both require $D<S$), then $w=1$ is the unique optimum. If $D\ge S$, then
$\obj$ decreases up to $w_1$ and increases afterwards, so the optimum is $w_1$ when $w_1<1$ and $w=1$ when
$w_1\ge1$. If $D<S$ and $\Delta>0$, then $w_1$ is a local minimum and $w_2$ a local maximum, so the optimum
is either $w_1$ or $w=1$, and which of the two wins is still open. In what follows, we want to show
that these facts together imply a threshold rule: there is a threshold
$\eta(\costParam)>0$ such that $w_1$ is the unique optimum when $D>\eta(\costParam)$ and $w=1$ is the unique optimum when $D<\eta(\costParam)$; setting $\thresholdD(\costParam)\triangleq\sqrt{\eta(\costParam)}$
then gives Equation~\eqref{eq:opt sol} once a tie is resolved in favor of $w=1$, as in the statement; this is item~\ref{lem1:sol}, and along the way we also obtain
items~\ref{lem1:small}--\ref{lem1:bound} and~\ref{lem1:tie}. We consider two cases according to the value of $\costParam$.

\textbf{Case 1:} $0<\costParam<S$.
\end{jgmathblock}
\smrev{Put}
\begin{equation}\smrev{\label{eq: ell r}
\ell\jgmath{\triangleq}\max\{0,S-S^2/(4\costParam)\},\qquad
r\jgmath{\triangleq}S/2+\costParam/4\jgmath{,}}
\end{equation}
\smrev{\jgmath{so that $\ell<r<S$: $r<S$ because $\costParam<2S$, and $\ell<r$ because 
$r-(S-S^2/(4\costParam))=(\costParam-S)^2/(4\costParam)>0$. Note also that, by
Equation~\eqref{eq: w1 def}, $w_1\le\costParam/S<1$ whenever $w_1$ is real. We want to show that
$\eta(\costParam)\in(\ell,r)$ where $\eta(\costParam)$ is the unique root of
$D=\obj(w_1;D,\costParam)$, and we distinguish three ranges of $D$.}}

\smrev{\jgmath{\textit{Case 1(a): $D\ge r$.}} $w_1$ is feasible and \jgmath{is the unique optimum}:
at $D=r$ the second stationary point is $1$ \jgmath{(because $Q(1)=0$ by Equation~\eqref{eq: Q at 1})}, and
for $D>r$ the derivative at $1$ is positive \jgmath{($Q(1)>0$)} and the first stationary point is the unique
minimum in $(0,1)$\jgmath{; in both cases $\obj$ increases on $(w_1,1)$, so $\obj(w_1)<\obj(1)=D$. (For
$D\ge S$ this is the convex case of Part~1, with $w_1<1$ because $Q(1)>0$.) In turn, there is no solution to $D=\obj(w_1;D,\costParam)$ in $[r,\infty)$, so $\eta(\costParam)<r$.}}

\smrev{\jgmath{\textit{Case 1(b): $D\leq \ell$.}} \jgmath{This range is nonempty only if $\ell>0$, i.e., $\costParam>S/4$;
then $D\le S-S^2/(4\costParam)$, which by Equation~\eqref{eq: Delta def} means $\Delta\le0$. By Part~1,
$Q\le0$ on $(0,1]$, $\obj$ is strictly decreasing, and $w=1$ is the unique optimum. Again, there is no solution to $D=\obj(w_1;D,\costParam)$ in $(0,\ell]$, so $\eta(\costParam)>\ell$.}}

\smrev{\jgmath{\textit{Case 1(c): $\ell< D<r$.}} \jgmath{The sign pattern of
$Q$} shows that for $D\in(\ell,r)$ the only candidates for a global minimum are $w_1\in(0,1)$ and $1$.
 \rev{Indeed, $r<S$, so \jgmath{$Q$} is concave in $w$, its value at $w=1$
is $4(D-r)<0$ \jgmath{by Equation~\eqref{eq: Q at 1}}, and $w_1<1$ because $\costParam<S$; hence the other
root $w_2$ also satisfies $w_2<1$.} \jgmath{Thus $\obj$ decreases on $(0,w_1)$, increases on $(w_1,w_2)$
and decreases again on $(w_2,1]$: both $w_1$ and $w=1$ are local minima, and whichever has the
smaller loss is the global optimum. To compare them,} define}
\begin{equation}\smrev{\label{eq: g def}
g(D,\costParam)\jgmath{\triangleq}D-\obj(w_1;D,\costParam)\qquad\text{on }[\ell,r].}
\end{equation}
In what follows, we fix $\costParam\in(0,S)$ and show in three steps that the function
$D\mapsto g(D,\costParam)$ has exactly one zero in the interval $(\ell,r)$ of values of $D$.

\smrev{\jgmath{\textit{Step (i): the sign of $g$ at the endpoints $D=\ell$ and $D=r$.}} At $D=\ell=0$, $g<0$, because $w_1<1$
has strictly positive loss \jgmath{($\obj(w_1;0,\costParam)=Sw_1(1-w_1)-\frac{\costParam}{2}\log w_1>0$)}
while the default has loss zero. If $\ell>0$, then $\Delta=0$ and $w_1=\costParam/S<1$; strict
decrease of $\obj$ between $w_1$ and $1$ likewise gives $g(\ell,\costParam)<0$. At $D=r$,
\jgmath{$\Delta=(S-\costParam)^2$ by Equation~\eqref{eq: Delta def}, so $\sqrt\Delta=S-\costParam$ and} the
roots are $w_1=\costParam/(2S-\costParam)<1$ and $1$, and $\obj_w>0$ between them. Hence
$g(r,\costParam)>0$.}

\smrev{\jgmath{\textit{Step (ii): $g$ is strictly increasing in $D$ and strictly decreasing in $\costParam$.}}
 On the interior \jgmath{$(\ell,r)$}, \rev{\jgmath{we have} $\Delta>0$
\jgmath{(because $D>\ell\ge S-S^2/(4\costParam)$)}, so that \jgmath{$w_1$ and hence} $g$ \jgmath{are}
continuously differentiable jointly in $(D,\costParam)$\jgmath{;
thus}} stationarity and the chain rule give}
\begin{equation}\smrev{\label{eq: g derivs}
g_D=1-w_1^2>0,\qquad
g_{\costParam}=\frac12\log w_1<0.}
\end{equation}
\begin{jgmathblock}
Indeed, $w_1$ depends on $D$ and on $\costParam$, but $\obj_w(w_1)=0$ by Equation~\eqref{eq: Q def}, so:
\[
\frac{d\,\obj(w_1;D,\costParam)}{dD}
=\obj_w(w_1)\frac{\partial w_1}{\partial D}+\frac{\partial\obj}{\partial D}\Big|_{w=w_1}=w_1^2,
\]
\[
\frac{d\,\obj(w_1;D,\costParam)}{d\costParam}
=\obj_w(w_1)\frac{\partial w_1}{\partial\costParam}+\frac{\partial\obj}{\partial\costParam}\Big|_{w=w_1}
=-\frac12\log w_1,
\]
and $0<w_1<1$ gives the strict signs in Equation~\eqref{eq: g derivs}.
\end{jgmathblock}

\smrev{\jgmath{\textit{Step (iii): existence and uniqueness of the threshold.} By Steps (i) and (ii),} the
intermediate value theorem and strict increase in $D$ yield a unique $\eta(\costParam)\in(\ell,r)$ with
$g=0$. Below it \jgmath{$w=1$ is the unique optimum}, above it \jgmath{$w_1$ is the unique optimum}, and at it exactly these two actions minimize the
objective. \jgmath{Together with Cases~1(a) and~1(b), this proves
Equation~\eqref{eq:opt sol} for $0<\costParam<S$ with $\thresholdD(\costParam)=\sqrt{\eta(\costParam)}$;
the optimum is unique for $D\ne\eta(\costParam)$, and $\eta(\costParam)>\ell\ge S-S^2/(4\costParam)$.}}

\begin{jgmathblock}
Thus, for $0<\costParam<S$, item~\ref{lem1:sol} holds. Moreover, since $g(D,\costParam)=0$ if and only if
$D=\obj(w_1;D,\costParam)$ by Equation~\eqref{eq: g def}, $\thresholdD(\costParam)^2=\eta(\costParam)$ is the
unique root of $D=\obj(w_1;D,\costParam)$ in $(\ell,r)$, so item~\ref{lem1:small} holds;
$\eta(\costParam)>\ell$ is the strict inequality of item~\ref{lem1:bound} for $\costParam<S$; and, since both
$w_1$ and $w=1$ are optimal at $D=\eta(\costParam)$, the first half of item~\ref{lem1:tie} holds.
\end{jgmathblock}

\smrev{\jgmath{\textbf{Case 2:} $\costParam\ge S$.}
\jgmath{Let $r\triangleq S/2+\costParam/4$ as in Equation~\eqref{eq: ell r}. We want to show that $w=1$ is the unique optimum for $D\le r$ and $w_1$ is the unique optimum for $D>r$, i.e.,
$\eta(\costParam)=r$. We consider three ranges of $D$.}}

\smrev{\jgmath{\textit{Case 2(a): $D=r$.}} At $D=r$, \jgmath{substituting $r-S=\costParam/4-S/2$ into
$\obj(w;r,\costParam)-r=Sw+(r-S)w^2-\frac{\costParam}{2}\log w-r$ and collecting the $S$-terms and the
$\costParam$-terms gives}}
\begin{equation}\smrev{\label{eq: L at r}
\obj(w;r,\costParam)-r
=-\frac S2(1-w)^2
+\frac{\costParam}{4}(w^2-1-2\log w)
\quad(0<w<1).}
\end{equation}
\smrev{\jgmath{We now show that the right-hand side of Equation~\eqref{eq: L at r} is positive, so that
$\obj(w;r,\costParam)>r=\obj(1;r,\costParam)$ for every $w\in(0,1)$, i.e., $w=1$ is the unique optimum at
$D=r$.}
 Indeed $w^2-1-2\log w>2(1-w)^2$ on $(0,1)$: the difference has derivative $-2(w-1)^2/w<0$ and
vanishes at $1$\jgmath{, so it is positive for $w<1$}. \rev{Hence\jgmath{, by Equation~\eqref{eq: L at r},}
$\obj(w;r,\costParam)-r>\frac{\costParam-S}{2}(1-w)^2\geq0$ \jgmath{for $0<w<1$}.}}

\smrev{\jgmath{\textit{Case 2(b): $D<r$.}} For $D<r$, the loss difference
increases further by $(r-D)(1-w^2)>0$\jgmath{, because
$\obj(w;D,\costParam)-D=[\obj(w;r,\costParam)-r]+(r-D)(1-w^2)$. Hence, by Case~2(a), $w=1$ is the unique optimum for $D<r$ as well}.}

\smrev{\jgmath{\textit{Case 2(c): $D>r$.}} For $D>r$, $\obj_w(1)=2D-S-\costParam/2>0$
\jgmath{(that is, $Q(1)>0$ by Equation~\eqref{eq: Q at 1})}, and \jgmath{$Q$} has
exactly one root in $(0,1)$ that is a local minimum \jgmath{(a quadratic with $Q(0)<0<Q(1)$ changes sign
exactly once on $(0,1)$, from negative to positive)}; that root \jgmath{is $w_1$, and it} is the unique \jgmath{optimum, since
$\obj$ decreases before it and increases after it up to $w=1$}. \jgmath{Combining Cases~2(a)--(c):} \jgmath{thus} $\eta=r$ for all
$\costParam\geq S$, \jgmath{in particular} $\eta(S)=3S/4$\jgmath{; the optimum is unique for every $D$, and
Equation~\eqref{eq:opt sol} holds with $\thresholdD(\costParam)=\sqrt{S/2+\costParam/4}$}.
\rev{Moreover $r-(S-S^2/(4\costParam))=(\costParam-S)^2/(4\costParam)\geq0$, with equality only at
$\costParam=S$, so the stated bound $\eta\geq\max\{0,S-S^2/(4\costParam)\}$ holds for all $\costParam$.}}

\begin{jgmathblock}
Thus, for $\costParam\ge S$, item~\ref{lem1:sol} holds with $\thresholdD(\costParam)=\sqrt{S/2+\costParam/4}$,
which is item~\ref{lem1:large}; the bound of item~\ref{lem1:bound} holds for all $\costParam$; and, since
$w=1$ is the unique optimum at $D=r=\eta(\costParam)$ (Case~2(a)), the second half of item~\ref{lem1:tie} holds.
\end{jgmathblock}

\begin{jgmathblock}
\textbf{Part 3: monotonicity and differentiability of $\thresholdD(\costParam)$.} Part~2 has
established Equation~\eqref{eq:opt sol}: the threshold is $\eta(\costParam)=S/2+\costParam/4$ for
$\costParam\ge S$ (Case~2) and the unique zero of $g$ in $(\ell,r)$ for $0<\costParam<S$ (Case~1); the
optimum is unique whenever $D\ne\eta(\costParam)$; and $\eta(\costParam)\ge\max\{0,S-S^2/(4\costParam)\}$,
strictly for $\costParam<S$. It remains to show that $\thresholdD(\costParam)=\sqrt{\eta(\costParam)}$ is strictly increasing and continuously
differentiable on $(0,\infty)$, i.e., items~\ref{lem1:mono} and~\ref{lem1:smooth}%
For $\costParam\ge S$, $\eta(\costParam)=S/2+\costParam/4$ by Case~2, so $\eta'(\costParam)=1/4>0$. For $0<\costParam<S$, $\eta(\costParam)$ is the zero of $g$ found in Step~(iii), and
Equation~\eqref{eq: g derivs} applies at $(\eta(\costParam),\costParam)$ because $\Delta>0$ there
(Case~1(c)).
\end{jgmathblock}
\smrev{The implicit function theorem gives}
\begin{equation}\smrev{\label{eq: eta prime}
\eta'(\costParam)=\jgmath{-\frac{g_{\costParam}}{g_D}=}
\frac{-\log m}{2(1-m^2)}>0,
\qquad m\jgmath{\triangleq}w_1(\eta(\costParam),\costParam)\in(0,1).}
\end{equation}

\smrev{\jgmath{Hence $\eta$ is continuously differentiable and strictly increasing on $(0,S)$ and on $[S,\infty)$
separately, and it remains to check $\costParam=S$.} As $\costParam\uparrow S$, both $\ell$ and
$r$ tend to $3S/4$, so $\eta$ is continuous there \jgmath{($\eta(\costParam)$ lies between $\ell$ and $r$ by
Case~1, and $\eta(S)=3S/4$ by Case~2)}. The formula for $w_1$ then gives $m\to1$\jgmath{: writing $\Delta(D,\costParam)$ for the discriminant
of Equation~\eqref{eq: Delta def}, which is increasing in $D$, and using $\ell<\eta(\costParam)<r$, we have}
\[
\smrev{0<\Delta(\eta(\costParam),\costParam)<\Delta(r,\costParam)=S^2+4\costParam(r-S)=(S-\costParam)^2
\longrightarrow0\quad\text{as }\costParam\uparrow S,}
\]
\smrev{where the left inequality holds because $\Delta(D,\costParam)>0$ exactly when $D>S-S^2/(4\costParam)$,
and $\eta(\costParam)>\ell\ge S-S^2/(4\costParam)$. Hence, by Equations~\eqref{eq: w1 def}
and~\eqref{eq: eta prime},
$m=w_1(\eta(\costParam),\costParam)=\costParam/\big(S+\sqrt{\Delta(\eta(\costParam),\costParam)}\big)\to S/S=1$
as $\costParam\uparrow S$}\jgmath{. Moreover, by Equation~\eqref{eq: eta prime} and
L'H\^opital's rule,}}
\[
\smrev{\lim_{\costParam\uparrow S}\eta'(\costParam)
=\lim_{m\uparrow1}\frac{-\log m}{2(1-m^2)}\jgmath{=\lim_{m\uparrow1}\frac{-1/m}{-4m}}=\frac14.}
\]
\smrev{This equals the right derivative\jgmath{, $\eta'(S^+)=1/4$}. The mean value theorem on each side
therefore shows that $\eta$ is continuously differentiable at $S$. Consequently
$\thresholdD=\sqrt\eta$ is strictly increasing and continuously differentiable on $(0,\infty)$\jgmath{.
By the chain rule, $\thresholdD'(\costParam)=\eta'(\costParam)/(2\thresholdD(\costParam))$; inserting
Equation~\eqref{eq: eta prime} for $0<\costParam<S$ and $\eta'(\costParam)=1/4$ for $\costParam\ge S$ gives}}
\[
\smrev{\thresholdD'(\costParam)=
\frac{-\log m}{4\thresholdD(\costParam)(1-m^2)}
\quad(0<\costParam<S),\qquad
\thresholdD'(\costParam)=\frac1{8\thresholdD(\costParam)}
\quad(\costParam\geq S),}
\]
\smrev{and\jgmath{, since $\thresholdD(S)=\sqrt{3S/4}$,} $\thresholdD'(S)=1/(4\sqrt{3S})$. \jgmath{Thus, items~\ref{lem1:mono} and~\ref{lem1:smooth} hold.} Choosing the default in the two-action tie gives \eqref{eq:opt sol}. \jgmath{This completes the proof.}}

\Halmos
\end{proof}

\section{Results in \Cref{sec: Homogenization}.\protect\footnote{In all the proofs, we use $\qError \triangleq \query - \truth \sim N(0, \qErrorSD^2)$ to denote the noise of a signal.}}
\label{sec:proof of homo}

\subsection{Proofs for \Cref{ssec: homogenization part}}

\begin{proof}{Proof of \Cref{prop: who benefits more}.}
\smrev{Put $S=\truSD^2$, $D=(\truth-\truMean)^2$, and let
$U(D)=\min_{0<w\leq1}\obj(w;D,\costParam)$ denote the uncapped (not accounting for $\costHumanOnly$) AI loss,
where
$\obj(w;D,\costParam)=Sw+(D-S)w^2-\frac{\costParam}{2}\ln w$ is the objective of Problem~\eqref{eq: loss def}
written in terms of $w=\qErrorSD^2/(\qErrorSD^2+S)$. We prove item (1) (and at the same time the continuity of $U(D)$), then (3), (2) and finalize with (4).

For item (1), note that for any $D$, a positive minimizer exists by the coercivity of $\obj(w;D,\costParam)$ at $w=0$. For $D_2>D_1$,
evaluation at an optimizer $w_2>0$ for $D_2$ gives}
\[
\smrev{U(D_2)\geq U(D_1)+(D_2-D_1)w_2^2>U(D_1).}
\]
Evaluation at an optimizer for $D_1$ also gives $U(D_2)-U(D_1)\leq D_2-D_1$, so $U$ is continuous. It satisfies $U(0)=0$ and $U(D)\to\infty$ as $D\to\infty$: otherwise, $U(D_n)\leq M$ along some $D_n\to\infty$ would imply $w_n\geq e^{-2M/\costParam}$ from the communication cost, whereas the fidelity loss is at least $D_n w_n^2$, a contradiction.
Thus the capped loss $\min\{U(D),\costHumanOnly\}$ is nondecreasing, proving item 1. 

For item (3), note that for finite $\costHumanOnly>0$,
there is a unique $\thresholdA>0$ satisfying
$U(\thresholdA^2)=\costHumanOnly$ (because $U$ is continuous and increasing ranging from 0 to $\infty$); working with the AI, including its tie,
is equivalent to $\unique\leq\thresholdA$. If
$\costHumanOnly=\infty$, set $\thresholdA=\infty$. This proves
item 3.

\smrev{For item (2): In the informative region of \Cref{lemma: characterization of sol},
the first-order condition and implicit differentiation give}
\[
\smrev{\frac{d w_1}{dD}=-\frac{2w_1^2}{\sqrt\Delta}<0,
\qquad
\frac{d}{dD}\left(-\frac12\log w_1\right)
=\frac{w_1}{\sqrt\Delta}>0.}
\]
\smrev{Below the default threshold, $\thresholdD(\costParam)$, the communication cost is zero. At a
two-action tie we select the default, after which} \rev{\smrev{(when
$\costParam<S$)}} \smrev{communication
cost goes up. Hence item 2 holds, including at the threshold}.

\smrev{For item 4 define explicitly
$\thresholdD=\min\{\thresholdD(\costParam),\thresholdA\}$.
An adopter with $\unique\leq\thresholdD$ selects the default and
has fidelity loss $D$. An adopter with
$\thresholdD<\unique\leq\thresholdA$ uses the informative optimum.
For the latter, we have}
\[
\smrev{\costQual=S w_1+(D-S)w_1^2
=\frac S2w_1+\frac{\costParam}{4},
\qquad
\frac{d\costQual}{dD}=-\frac{S w_1^2}{\sqrt\Delta}<0.}
\]
\smrev{%
This proves item (4).%
}
\Halmos
\end{proof}

\begin{proof}{Proof of \Cref{thm: output revert to mean}.}
As shown in item 3 of \Cref{prop: who benefits more}, if \smrev{$\unique > \thresholdA$}, users will work on their own and $\truth^\star = \truth$, so $|\mathbb{E}[\truth^\star|\truth] - \truMean| = |\truth - \truMean|$. 

If \smrev{$\unique \le \thresholdA$}, $\truth^\star = \truthAI^\star$\rev{\smrev{, where $\truthAI^\star \triangleq \truthAI(\query,\opteffort)$ denotes the AI's output under the optimal signal}}. By \Cref{eq: truthAI def}, we know $\mathbb{E}[\truthAI|\truth] = \dfrac{\truSD^2}{\truSD^2+\qErrorSD^2}\cdot \truth + \dfrac{\qErrorSD^2}{\truSD^2+\qErrorSD^2}\cdot \truMean$, so  $|\mathbb{E}[\truthAI^\star|\truth] - \truMean| = \dfrac{\truSD^2}{\truSD^2+\qErrorSD^{\star2}(\truth)}|\truth - \truMean|$
which equals $0$ if $\truth = \truMean$.

Additionally, since $\obj(\truth, \qErrorSD) \to \infty$ as $\qErrorSD \to 0$ and $\qErrorSD = \infty$ is feasible, we must have $\qErrorSD^{\star}(\truth) > 0$. 
Thus, $|\mathbb{E}[\truthAI^\star|\truth] - \truMean| < |\truth - \truMean|$ whenever $\truth \ne \truMean$.
\Halmos
\end{proof}

In what follows, we prove a more detailed version of \Cref{thm: homogenization}\smrev{.}
\begin{repeattheorem}[\Cref{thm: homogenization}](Full version)
When everyone uses AI (\smrev{$\costHumanOnly=+\infty$),
$\VV(\outcome)<\VV(\truth)$ and $\VV(\outcome)$} strictly decreases
in \smrev{$\costParam>0$. For every fixed
$\costHumanOnly\in(0,\infty]$,}
\[
\smrev{\lim_{\costParam\downarrow0}\VV(\outcome)=\truSD^2,
\qquad
\lim_{\costParam\to\infty}\VV(\outcome)
=\EE{\truDist}{(\truth-\truMean)^2
\1{(\truth-\truMean)^2>\costHumanOnly}}<\truSD^2,}
\]
\smrev{where the expectation on the right is zero when
$\costHumanOnly=\infty$}. In addition,
\smrev{$\VV(\outcome)<\VV(\truth)$ if
$\costParam\geq\truSD^2/2$, or if
$\costHumanOnly\leq\hat{\costHumanOnly}$, or if
$\costHumanOnly\geq\tilde{\costHumanOnly}$, for some positive
$\hat{\costHumanOnly},\tilde{\costHumanOnly}$ depending on
$\costParam$ and $\truSD$}.
\end{repeattheorem}

\begin{proof}{Proof of \Cref{thm: homogenization}.}
We start with definitions and some preliminary analysis that will be used later in the proof. 
\smrev{Throughout this proof $S=\truSD^2$, $D=(\truth-\truMean)^2$, and
$U_\costParam(D)=\min_{0<w\leq1}\obj(w;D,\costParam)$ is the
uncapped AI loss. Write $w_\costParam(D)$ for the selected optimizer,
including $w_\costParam(D)=1$ in the default region. Let
$\phi(x)=(2\pi)^{-1/2}e^{-x^2/2}$ be the standard normal density} and
\smrev{$q(u)=\phi(u/\truSD)/\truSD$. Let}
\[
\smrev{H_\costParam(D)=S w_\costParam(D)(1-w_\costParam(D))
+(1-w_\costParam(D))^2D,}
\]
\smrev{be the second moment of $(\truthAI(\query,\opteffort)-\truMean)$ conditional on $\truth$. Note that $0\leq H_\costParam(D)\leq S/4+D$.
\Cref{lemma: characterization of sol} gives uniqueness of the
optimizer, $w_\costParam(D)$, except possibly at the two types with
$D=\thresholdD(\costParam)^2$. These types have population measure
zero. Thus the optimizer and the decision to work with AI are even about
$\truMean$ almost everywhere. The conditional mean of the output
minus $\truMean$ is consequently odd about $\truMean$. Integrability
follows from the preceding second-moment bound on $H_\costParam(D)$, and hence
$\EE{}{\outcome}=\truMean$.}

\smrev{For later reference, if $\thresholdA\leq\thresholdD(\costParam)$},
\begin{equation}\smrev{\label{eq: var general - small}
\VV(\outcome)=2\int_{\thresholdA}^{\infty}u^2q(u)\,du.}
\end{equation}
\smrev{If $\thresholdA>\thresholdD(\costParam)$, then}
\begin{equation}\smrev{\label{eq: var general}
\VV(\outcome)=
2\int_{\thresholdD(\costParam)}^{\thresholdA}
H_\costParam(u^2)q(u)\,du
+2\int_{\thresholdA}^{\infty}u^2q(u)\,du.}
\end{equation}

\textbf{\smrev{Strict decrease when $\costHumanOnly=\infty$.}}
\smrev{Put $\alpha=\costParam/S$, $x=(\truth-\truMean)/\truSD$,
$a(\alpha)=\thresholdD(S\alpha)/\truSD$, and, for $x>a(\alpha)$},
\[
\smrev{r(x,\alpha)=\sqrt{1+4\alpha(x^2-1)},\qquad
w(x,\alpha)=\frac{\alpha}{1+r(x,\alpha)},\qquad
\mathcal H(x,\alpha)=w(1-w)+(1-w)^2x^2.}
\]
\smrev{Then $\VV(\outcome)=2S\int_{a(\alpha)}^\infty
\mathcal H(x,\alpha)\phi(x)\,dx$. On the informative branch,}
\[
\smrev{w_\alpha=\frac1{2r},\qquad
2\mathcal H_\alpha
=\frac{1-2x^2+2w(x^2-1)}{r}
=\frac12\left(1+\frac{1-4x^2}{r}\right).}
\]

Now, we want to use the above to compute the derivative of
$\VV(\outcome)$ with respect to $\costParam$. Both the integrand
and the lower limit $a(\alpha)$ depend on $\alpha$, so we argue
that we can differentiate under the integral and apply the
Leibniz rule. The only obstacle is the factor $1/r$ in
$\mathcal H_\alpha$, where $r^2=1+4\alpha(x^2-1)$ is the
(normalized) discriminant $\Delta/S^2$. Since $r^2$ is increasing
in $x^2$, it suffices to show that it is bounded away from zero
at the threshold.

For every $\alpha_0\ne1$, $r(a(\alpha_0),\alpha_0)>0$: for
$\alpha_0<1$ this follows from the strict inequality in
\Cref{lemma: characterization of sol} part (iv), which places
$a(\alpha_0)^2$ strictly above $1-1/(4\alpha_0)$, the point where
$r^2$ vanishes; for $\alpha_0>1$, part (iii) gives
$r(a(\alpha_0),\alpha_0)=\alpha_0-1$. (At $\alpha_0=1$ we have
$r(a(1),1)=0$, which is why this value is excluded for now.)
By continuity of $a(\cdot)$ and of $r^2$, there exist a
neighborhood $N$ of $\alpha_0$, a point $b<a(\alpha)$ for all
$\alpha\in N$, and $\delta>0$ such that $r^2\geq\delta$ on
$[b,\infty)\times N$. On this fixed set, $w=\alpha/(1+r)\leq\alpha$ and
$1/r\leq1/\sqrt\delta$ give
$|\mathcal H_\alpha|+|\mathcal H|\leq C(1+x^2)$ for some $C$
depending on $N$; on the actual informative domain we also have
$0\leq\mathcal H\leq1/4+x^2$, since $w(1-w)\leq1/4$ and
$(1-w)^2\leq1$. Both bounds are integrable against $\phi$. The
second shows that $\VV(\outcome)$ is finite; the first, by
dominated convergence, justifies differentiation under the
integral on $[b,\infty)$, and together with continuity of
$\mathcal H$ at $(a(\alpha_0),\alpha_0)$ it justifies the
Leibniz boundary term. Consequently, for $\alpha\ne1$,
\[
\frac{d\VV(\outcome)}{d\costParam}
=-2\mathcal H(a(\alpha)+,\alpha)\phi(a(\alpha))a'(\alpha)
+\frac1{2\sqrt{2\pi}}G(\alpha),
\]
where $\mathcal H(a(\alpha)+,\alpha)$ denotes the limit along the
informative branch (at the threshold itself the default is
selected and $\mathcal H=0$), and
\[
G(\alpha)=\int_{a(\alpha)}^\infty
\left(1+\frac{1-4x^2}{\sqrt{1+4\alpha(x^2-1)}}\right)e^{-x^2/2}\,dx.
\]
Note that the first term in the derivative of $\VV(\outcome)$ is nonpositive, since $\mathcal H\geq0$ and
$a'\geq0$ by \Cref{lemma: characterization of sol} part (v).

In turn, to show strict decrease, it is enough to show
that $G(\alpha)<0$. The integrand of $G$ has the
sign of the numerator $1-2x^2+2w(x^2-1)$ in $2\mathcal H_\alpha$.
We treat two ranges. When $\alpha>1/2$, the integration domain
lies entirely in $x^2>1/2$ and the numerator is negative pointwise,
so $G<0$ immediately. When $\alpha\leq1/2$, the integrand can be
positive near the threshold, so we bound the positive part on
$[a(\alpha),1]$ and show it is dominated by the negative tail on
$[1,\infty)$. The value $\alpha=1$ is excluded for now because the
derivative formula was not established there; we return to it
at the end.

\emph{Case $\alpha>1/2$, $\alpha\neq1$.}
By \Cref{lemma: characterization of sol} part (iv),
$a(\alpha)^2\geq1-1/(4\alpha)>1/2$, so every $x$ in the
integration domain satisfies $x^2>1/2$. For such $x$ the numerator
$1-2x^2+2w(x^2-1)$ is negative: if $x^2\geq1$, then $w\leq1$
gives $1-2x^2+2w(x^2-1)\leq-1$; if $1/2<x^2<1$, then $w>0$ gives
$1-2x^2+2w(x^2-1)<1-2x^2<0$. Hence the integrand of $G$ is
strictly negative on $[a(\alpha),\infty)$, so $G(\alpha)<0$ and
$d\VV(\outcome)/d\costParam<0$.

\emph{Case $\alpha\leq 1/2$.} \smrev{Here the integrand of $G$,
$g(x)\triangleq\big(1+\tfrac{1-4x^2}{r}\big)e^{-x^2/2}$, can be
positive near the threshold, so a pointwise argument does not
suffice. Instead we split the domain at $x=1$. We first show that
the tail contributes a fixed negative amount,
$\int_1^\infty g<-K$ with $K=(2\sqrt2-1)e^{-1/2}$. We then bound
the head $\int_{a(\alpha)}^1 g$ from above by a quantity smaller
than $K$, treating separately the sub-cases $a(\alpha)\geq1/2$
and $a(\alpha)<1/2$; the latter forces $\alpha<1/6$, which gives
the needed control on $r$. Summing the two pieces yields
$G(\alpha)<0$.}

\emph{The tail $[1,\infty)$.} For $x\geq1$,
\[
\smrev{r(x,\alpha)\leq\sqrt{2x^2-1},\qquad
\frac{4x^2-1}{\sqrt{2x^2-1}}>2\sqrt2\,x,}
\]
\smrev{where the latter inequality can be derived by squaring both sides of it. Hence the tail satisfies}
\[
\smrev{\int_1^\infty\left(1+\frac{1-4x^2}{r}\right)e^{-x^2/2}\,dx
<-(2\sqrt2-1)e^{-1/2}=-K.}
\]
\smrev{Here we used
$\int_1^\infty e^{-x^2/2}\,dx\leq
\int_1^\infty xe^{-x^2/2}\,dx=e^{-1/2}$.
The following strict bounds (which we will use later) use only rational inequalities:}
\[
\smrev{K>\frac{27}{25}>\frac12+\frac1{\sqrt3}.}
\]
\smrev{Indeed $e<11/4<25/9$, $\sqrt2>7/5$, and
$1/\sqrt3<29/50$. %
}

\emph{The head $[a(\alpha),1]$, sub-case $a(\alpha)\geq1/2$.} %
If $a(\alpha)\geq1/2$, then, since also $a(\alpha)<1$ (as
$a(\alpha)^2<1/2+\alpha/4\leq5/8$ by the upper endpoint of the
interval in \Cref{lemma: characterization of sol} part (ii)),
on $[a(\alpha),1]$ we have $r\leq1$ and $1-4x^2\leq0$. 
Therefore
\[
\smrev{\int_{a(\alpha)}^1\left(1+\frac{1-4x^2}{r}\right)e^{-x^2/2}\,dx
\leq\int_{1/2}^{1/\sqrt2}(2-4x^2)\,dx<\frac14<K,}
\]
where $\int_{1/2}^{1/\sqrt2}(2-4x^2)\,dx=\frac{2\sqrt2}{3}-\frac56<\frac14$; %
\smrev{which proves $G<0$ in this case.}

\emph{Sub-case $a(\alpha)<1/2$.} \smrev{We first show that this
forces $\alpha<1/6$; by strict monotonicity of the threshold, it
suffices to show that at $\alpha=1/6$ the type $x=1/2$ defaults.
For this type the default loss is $D/S=1/4$, and the informative
candidate of \Cref{lemma: characterization of sol} is
$m=w_1=(2-\sqrt2)/6\in(1/12,1/10)$. Using the first-order condition satisfied by $m$ to simplify the expression of the  normalized loss  at $m$ gives}
\[
\smrev{\frac m2+\frac1{24}-\frac1{12}\log m
>\frac1{12}+\frac{\log10}{12}>\frac14,}
\]
\smrev{using $m>1/12$, $m<1/10$, and $\log10>2$. So the default is
strictly better, $a(1/6)>1/2$, and $a(\alpha)<1/2$ implies
$\alpha<1/6$. For $\alpha<1/6$ and $0\leq x\leq1/2$ we have
$r^2\geq1-4\alpha>1/3$, and hence}

\[
\smrev{\int_{a(\alpha)}^{1/2}
\left(1+\frac{1-4x^2}{r}\right)e^{-x^2/2}\,dx
\leq\int_0^{1/2}\left(1+\sqrt3(1-4x^2)\right)\,dx
=\frac12+\frac1{\sqrt3}.}
\]
\smrev{On $[1/2,1]$ the integrand is bounded above by
$(2-4x^2)f(x)$ with $f(x)=e^{-x^2/2}$, whose integral is negative:
integrating by parts and using concavity of $f$ on $[1/2,1]$,}
\[
\begin{split}
\smrev{\int_{1/2}^1(2-4x^2)f(x)\,dx}
&\smrev{=4f(1)-2f(1/2)-2\int_{1/2}^1f(x)\,dx}\\
&\smrev{\leq\tfrac72e^{-1/2}-\tfrac52e^{-1/8}<0.}
\end{split}
\]
\smrev{The last inequality follows from
$e^{3/8}>1+3/8+(3/8)^2/2>7/5$.
Combining the three intervals gives
$G(\alpha)<1/2+1/\sqrt3-K<0$.}

\smrev{This proves strict decrease on each of $(0,S)$ and $(S,\infty)$.
At $\costParam=S$,} \rev{\smrev{the threshold $\thresholdD$ is continuously
differentiable, in particular continuous, by
\Cref{lemma: characterization of sol}, so}} \smrev{the optimizer converges pointwise for every
$D\ne3S/4$ as the cost approaches $S$ from either side.
The bound $H_\costParam(D)\leq S/4+D$ and dominated convergence
therefore give continuity of the population variance at $S$.
Strict decrease on all of $(0,\infty)$ follows. %
}

\textbf{\smrev{The limit as $\costParam\downarrow0$.}}
\smrev{Fix a type with squared distance $D$ and let
$0<\alpha=\costParam/S<1$. If the type works with AI, its fidelity
error $Sw_\costParam(D)+(D-S)w_\costParam(D)^2$ is at most the
optimal loss $U_\costParam(D)$ (the communication cost is
nonnegative), which is at most the loss at the feasible weight
$w=\sqrt\alpha$:}
\[
\smrev{\EE{}{(\outcome-\truth)^2\mid\truth}
\leq\obj(\sqrt\alpha;D,\costParam)
=S\sqrt\alpha+(D-S)\alpha-\frac{\costParam}{4}\log\alpha.}
\]
\smrev{If the type works alone, the left side is zero. Taking
expectations over $\truth$ and using $\EE{}{D}=S$ gives}
\[
\smrev{\EE{}{(\outcome-\truth)^2}
\leq S\sqrt\alpha-\frac{\costParam}{4}\log\alpha\longrightarrow0
\quad\text{as }\costParam\downarrow0.}
\]
\smrev{Thus $\outcome\to\truth$ in $L^2$, so means and variances
converge and $\VV(\outcome)\to S$. The argument does not depend
on who works with AI, so it holds for every fixed
$\costHumanOnly\in(0,\infty]$, and it requires no decision at the
excluded value $\costParam=0$. Together with strict decrease,
it gives $\VV(\outcome)<S$ for all $\costParam>0$ when
$\costHumanOnly=\infty$.}

\textbf{\smrev{The limit as $\costParam\to\infty$.}}
\smrev{The default $w=1$ has loss $D$, so $U_\costParam(D)\leq D$.
Since the fidelity term is nonnegative, the cost alone satisfies
$-\frac{\costParam}{2}\log w_\costParam(D)\leq D$, i.e.}
\[
\smrev{w_\costParam(D)\geq e^{-2D/\costParam}\longrightarrow1
\quad\text{as }\costParam\to\infty.}
\]
\smrev{Hence the fidelity error $Sw(1-w)+Dw^2$ at $w=w_\costParam(D)$
tends to $D$, and $H_\costParam(D)=Sw(1-w)+(1-w)^2D$ tends to $0$.
As fidelity error $\leq U_\costParam(D)\leq D$, also
$U_\costParam(D)\to D$.}

\smrev{Now fix $\costHumanOnly<\infty$. Types with
$D<\costHumanOnly$ work with the AI for every $\costParam$, since
$U_\costParam(D)\leq D<\costHumanOnly$. Types with
$D>\costHumanOnly$ work alone for all $\costParam$ large enough
(depending on $D$), since $U_\costParam(D)\to D>\costHumanOnly$.
Types with $D=\costHumanOnly$ have zero population mass. Since
$\EE{}{\outcome}=\truMean$, $\VV(\outcome)$ is the expectation
over $\truth$ of the conditional second moment of
$\outcome-\truMean$, which equals $H_\costParam(D)$. This integrand is bounded by $S/4+D$, which
is integrable, and converges pointwise for almost every type to
$D\1{D>\costHumanOnly}$. Dominated convergence gives}
\[
\smrev{\lim_{\costParam\to\infty}\VV(\outcome)
=\EE{\truDist}{D\1{D>\costHumanOnly}}.}
\]
\smrev{When $\costHumanOnly=\infty$ everyone works with the AI and the same
argument gives limit zero. For finite $\costHumanOnly>0$, the
limit equals $S-\EE{\truDist}{D\1{D\leq\costHumanOnly}}<S$,
since the subtracted term is strictly positive. %
}

\textbf{\smrev{The additional finite-manual-cost conclusions.}}
\smrev{Since $\EE{}{\outcome}=\truMean$ and $S=\EE{\truDist}{D}$,}
\[
\smrev{\Delta_V(\costHumanOnly)\triangleq S-\VV(\outcome)
=\EE{\truDist}{\big(D-H_\costParam(D)\big)\1{\unique\leq\thresholdA}}.}
\]
\smrev{For a user who takes the default $H_\costParam(D)=0$, so the integrand is
$D>0$. For a user who uses an informative signal, with $w=w_\costParam(D)\in(0,1)$,}
\[
\smrev{D-H_\costParam(D)=w\{(2-w)D-S(1-w)\}\geq\frac{Sw^2}{2}>0
\qquad\text{whenever } D\geq S/2,}
\]
\smrev{since the bracket is increasing in $D$ and equals $Sw/2$ at
$D=S/2$. We show $\Delta_V(\costHumanOnly)>0$ in each of the three
cases of the theorem.}

\smrev{\emph{Case $\costParam\geq S/2$.} By
\Cref{lemma: characterization of sol} parts (iii) and (iv),
$\thresholdD(\costParam)^2\geq S/2$, so every user who uses an informative signal
has $D>S/2$ and the integrand is positive for each of them. The users who take the default  $\{\unique\leq\min\{\thresholdD(\costParam),\thresholdA\}\}$
form a set of positive measure because
$\thresholdD(\costParam)>0$ and $\thresholdA>0$. Hence
$\Delta_V(\costHumanOnly)>0$ for every $\costHumanOnly>0$.}

\smrev{\emph{Case $\costHumanOnly\leq\hat{\costHumanOnly}\triangleq\thresholdD(\costParam)^2$.}
On the default region $U_\costParam(D)=D$, so
$U_\costParam(\costHumanOnly)=\costHumanOnly$ and therefore
$\thresholdA^2=\costHumanOnly$. Every user that works with the AI then has
$D\leq\costHumanOnly\leq\thresholdD(\costParam)^2$ and selects the
default, giving
$\Delta_V(\costHumanOnly)=\EE{\truDist}{D\1{D\leq\costHumanOnly}}>0$.}

\smrev{\emph{Case $\costHumanOnly$ large.} Fix $\costParam>0$. A
type works with the AI once $\costHumanOnly\geq U_\costParam(D)$, which is
finite, so $\1{\unique\leq\thresholdA}\to1$ pointwise as
$\costHumanOnly\to\infty$. Since $|D-H_\costParam(D)|\leq S/4+2D$ is
integrable, dominated convergence gives
$\Delta_V(\costHumanOnly)\to\EE{\truDist}{D-H_\costParam(D)}=\Delta_V(\infty)$,
which equals $S-\VV(\outcome)$ when $\costHumanOnly=\infty$ and is
therefore positive by the first part of the proof.
Hence there is a finite $\tilde{\costHumanOnly}>0$ such that
$\Delta_V(\costHumanOnly)>0$ for all
$\costHumanOnly\geq\tilde{\costHumanOnly}$.}
\Halmos
\end{proof}

\subsection{Results in \Cref{sec: Bias}.}
\label{sec:proof of Bias}

\begin{proof}{Proof of \Cref{thm: imperf AI & error if pref known}.}
Put $D=(\aMean-\truth)^2$, $a=\aSD^2>0$, and $q=\qErrorSD^2\in(0,\infty]$.
Under the AI prior $(\aMean,\aSD)$, the loss $\obj(\truth,\qErrorSD)$
of user $\truth$ with signal variance $q$ depends on $(\truth,\aMean,\aSD)$
only through $(D,a)$. In turn, we write the loss as
\begin{equation}
 \label{eq: loss with bias prior}
 \Psi(D,a,q)\triangleq\obj(\truth,\sqrt q)
 =\underbrace{\frac{q^2D+a^2q}{(a+q)^2}}_{\text{fidelity error}}
 +\underbrace{\frac{\costParam}{2}\log\Big(1+\frac{\truSD^2}q\Big)}_{\text{communication cost}},
 \qquad \Psi(D,a,\infty)=D.
\end{equation}
Since $\Psi$ is continuous in $q\in (0,\infty]$ and diverges as $q\downarrow0$,
an optimal $q$ exists, and the uncapped optimal loss is
$\obj(\truth,\opteffort)=V_a(D)\triangleq\min_{q\in(0,\infty]}\Psi(D,a,q)$.

We make two observations about $\Psi$ that we will use later in the proof. First, it is affine in $D$ with
slope $(q/(a+q))^2>0$. Second, $a$ enters only through the fidelity
error, and for finite $q$
\[
 \frac{\partial\Psi(D,a,q)}{\partial a}=\frac{2q^2(a-D)}{(a+q)^3},
\]
so at fixed finite $q$, $\Psi$ is strictly decreasing in $a$ on $(0,D]$
and strictly increasing on $[D,\infty)$, while $\Psi(D,a,\infty)=D$ does
not depend on $a$.

\emph{Item 1.} Let $D_2>D_1$ and let $q_2$ be optimal at $D_2$. By
affinity in $D$,
\[
 V_a(D_2)=\Psi(D_2,a,q_2)
 =\Psi(D_1,a,q_2)+\Big(\frac{q_2}{a+q_2}\Big)^2(D_2-D_1)
 >\Psi(D_1,a,q_2)\ge V_a(D_1).
\]

\emph{Item 2.} Let $0<a_1<a_2\le D$ and let $q_1$ be optimal at $a_1$.
Using the $a_1$-optimal signal at $a_2$,
\[
 V_{a_1}(D)=\Psi(D,a_1,q_1)\ge\Psi(D,a_2,q_1)\ge V_{a_2}(D),
\]
where the first inequality uses the monotonicity of $\Psi$ in $a$ on
$(0,D]$ and is strict if $q_1<\infty$. Symmetrically, let
$D\le a_1<a_2$ and let $q_2$ be optimal at $a_2$; then
\[
 V_{a_2}(D)=\Psi(D,a_2,q_2)\ge\Psi(D,a_1,q_2)\ge V_{a_1}(D),
\]
strictly if $q_2<\infty$. This proves item 2. In both chains the
comparison is strict whenever the optimal signal at the prior variance
farther from $D$ is informative, which is the footnote.

\emph{Capped loss.} Since $\optobj=\min\{V_a(D),\costHumanOnly\}$ and
$x\mapsto\min\{x,\costHumanOnly\}$ is nondecreasing, both items hold
weakly for $\optobj$.
\Halmos
\end{proof}

\begin{proof}{Proof of \Cref{thm: optimize total loss}.}
We begin by establishing some facts that we will later use to prove Items 1 and 2 in the statement. Let \smrev{$S=\truSD^2>0$}, \smrev{$a=\aSD^2>0$, $m=\aMean$, and let}
\[
 \smrev{f_p(x)=\frac1{\truSD}\phi\!\left(\frac{x-\truMean}{\truSD}\right)}
\]
\smrev{be the population density, where $\phi$ is the standard normal density.
For $D\ge0$ and signal variance $q\in(0,\infty]$, write}
\begin{equation}
 \smrev{
 \Psi(D,a,q)
 =\frac{q^2D+a^2q}{(a+q)^2}
       +\frac{\costParam}{2}\log\!\left(1+\frac Sq\right),
 \qquad \Psi(D,a,\infty)=D,
 \qquad V_a(D)=\min_{q\in(0,\infty]}\Psi(D,a,q).}
\end{equation}
\smrev{The cost tends to infinity at $q=0$, and the objective is continuous on
$(0,\infty]$, so the minimum exists. In particular, $0\le V_a(D)\le D$ (see also the proof of \Cref{thm: imperf AI & error if pref known}).
At fixed $a$, changing variables to $w=q/(a+q)$ shows that $V_a$ is the
infimum of affine functions of $D$ with slopes $w^2\in(0,1]$.
For $D_2>D_1\ge0$, comparison at optimal weights gives}
\[
 \smrev{0<w_2^2(D_2-D_1)
 \le V_a(D_2)-V_a(D_1)\le D_2-D_1.}
\]
\smrev{Thus $V_a$ is concave, continuous, strictly increasing,} and \smrev{$V_a(0)=0$.
Define the even continuous function}
\[
 \smrev{J_a(z)=\min\{V_a(z^2),\costHumanOnly\},}
\]
\smrev{with the cap omitted when $\costHumanOnly=\infty$.
It is nondecreasing in $|z|$, satisfies $0\le J_a(z)\le z^2$, and is
strictly positive for $z\ne0$}.

\begin{enumerate}
\item \smrev{First, for the minimization assertion, %
suppose
$m>\truMean$, put $\bar\mu=(m+\truMean)/2$, and pair
$\truth_1=\bar\mu+x$ with $\truth_2=\bar\mu-x$, $x>0$.}
Let
\[
 \smrev{\Delta_i=J_a(m-\truth_i)-J_a(\truMean-\truth_i),\qquad
 f_i=f_p(\truth_i).}
\]
\smrev{The first type is closer to $m$, and symmetry gives
$\Delta_1\le0$, $\Delta_2=-\Delta_1\ge0$, while $f_1<f_2$.
Consequently the comparison is}
\[
 \smrev{0\le-\Delta_1f_1\le\Delta_2f_2,
 \qquad
 \Delta_1f_1+\Delta_2f_2
 =\Delta_2(f_2-f_1)\ge0.}
\]
\smrev{The last inequality is strict on an interval of positive length:
at $\truth_1=m$ one has
$\Delta_1=-J_a(m-\truMean)<0$, and continuity preserves strictness nearby.
All integrands are absolutely integrable since $J_a(z)\le z^2$.
Integrating the above inequality over $x>0$, which covers each type exactly once, gives
$\PL(m,\aSD)-\PL(\truMean,\aSD)=\int_0^\infty(\Delta_1f_1+\Delta_2f_2)\,dx>0$.
Reflection gives the same conclusion when $m<\truMean$.}

\smrev{For differentiability, change variables $z=m-\truth$ to obtain}
\[
 \smrev{\PL(m,\aSD)=\int_{\mathbb R}J_a(z)f_p(m-z)\,dz.}
\]
\smrev{Fix $R>0$. For $|m-\truMean|\le R$, the elementary inequality
$(z-(m-\truMean))^2\ge z^2/2-R^2$ gives}
\[
 \smrev{|\partial_m f_p(m-z)|
 \le C_R(1+|z|)e^{-z^2/(4S)}}
\]
\smrev{for a finite constant $C_R$ depending only on $R,S$.
Multiplication by $J_a(z)\le z^2$ gives an integrable majorant.
The mean value theorem and dominated convergence therefore justify}
\[
 \smrev{\frac{\partial\PL(m,\aSD)}{\partial m}
 =\int_{\mathbb R}J_a(z)\partial_m f_p(m-z)\,dz.}
\]
\smrev{At $m=\truMean$ the integrand equals
$J_a(z)\,z f_p(\truMean-z)/S$, which is odd and absolutely integrable so its integral is zero. This proves the stated derivative, and in fact it
holds for every $a>0$, not just $a=S$}.

\item \smrev{Now let $\costHumanOnly=\infty$ and hold $m=\truMean$ fixed.
We first justify the derivative in $a$ at $a=S$ for arbitrary
$\costParam>0$. 
We first justify the derivative in $a$ at $a=S$ for arbitrary
$\costParam>0$. The argument has three steps. First, for each type $D$ we
differentiate $V_a(D)$ at $a=S$ by an envelope argument, which requires
the optimizer at $a=S$ to be unique and to vary continuously with $a$.
Second, we pass the derivative through the population expectation by
dominated convergence. Third, we sign the resulting integrand.

We begin with uniqueness. A finite concave function of $D$ is nondifferentiable at
at most countably many points: its left and right derivatives are monotone,
and the nonempty intervals between them are disjoint. At a differentiability
point of $V_S$, every optimal weight satisfies $w^2=V_S'(D)$, by comparison
with that optimizer at $D+h$ for positive and negative $h$.
Thus the optimizer is unique except at countably (zero measure) many $D$.%
}

We now show that a unique optimizer at $a=S$ varies continuously with
$a$. For fixed $D>0$, every optimizer, for every $a>0$, satisfies
\[
 \smrev{\frac{\costParam}{2}\log\!\left(1+\frac Sq\right)\le D.}
\]

\smrev{In terms of $r=q/(S+q)$, which maps $(0,\infty]$ onto $(0,1]$, this
bound reads $r\ge e^{-2D/\costParam}$, so all optimizers belong to the compact
interval $[e^{-2D/\costParam},1]$, independently of $a$. For $D=0$, the only
optimizer is $q=\infty$. On this interval the objective is jointly continuous
in $(a,r)$. If the optimizer $q_0$ at $a=S$ is unique, then optimizers $q_h$
at $a=S+h$ converge to $q_0$ (in the variable $r$) as $h\to0$: by compactness,
it suffices to consider a convergent subsequence, and passing to the limit in
$\Psi(D,S+h,q_h)\le\Psi(D,S+h,q_0)$ shows that its limit is optimal at $a=S$,
hence equals $q_0$.}

\smrev{The fixed-signal partial derivative is}
\[
 \smrev{\partial_a\Psi(D,a,q)=\frac{2q^2(a-D)}{(a+q)^3},
 \qquad \partial_a\Psi(D,a,\infty)=0.}
\]
\smrev{It is jointly continuous on the preceding compact set. For $h>0$,
optimality gives the two bounds}
\[
 \smrev{\frac{\Psi(D,S+h,q_h)-\Psi(D,S,q_h)}h
 \le\frac{V_{S+h}(D)-V_S(D)}h
 \le\frac{\Psi(D,S+h,q_0)-\Psi(D,S,q_0)}h.}
\]
\smrev{For $h<0$ the bounds reverse. The mean value theorem, the convergence
$q_h\to q_0$, and continuity of the fixed-signal derivative show that}
\[
 \smrev{\left.\partial_aV_a(D)\right|_{a=S}
 =\frac{2q_0^2(S-D)}{(S+q_0)^3}}
\]
\smrev{at every type with a unique optimizer, including defaults via the value
zero at $q_0=\infty$}.

\smrev{For the second step,we need a bound on $|V_{S+h}(D)-V_S(D)|/|h|$ that is integrable
over the population and uniform in $h$ near $0$.
For $a\in[S/2,3S/2]$ and all $q\in(0,\infty]$,}

\[
 \smrev{|\partial_a\Psi(D,a,q)|
 \le\frac{2(a+D)}a\le2+\frac{4D}{S}.}
\]
\smrev{Comparing optimal values at two values of $a$ yields the same Lipschitz
bound for their difference quotient. Since
$\EE{\truDist}{D}=S$, this is an integrable majorant,} so \smrev{dominated
convergence justifies passing the derivative through the population
expectation. Writing $w^\star=q_0/(S+q_0)$ gives}
\[
 \smrev{\left.\frac{\partial\PL(\truMean,\sqrt a)}{\partial a}\right|_{a=S}
 =\EE{\truDist}{\frac2S w^{\star2}(1-w^\star)(S-D)
                      \1{|\truth-\truMean|>\thresholdD}}.}
\]
\smrev{The indicator follows from \Cref{lemma: characterization of sol}; defaults
have $w^\star=1$ and contribute zero, and the threshold types have zero
population probability. No differentiability of that threshold is needed.}

\smrev{Finally, we sign the integrand $\tfrac2S w^{\star2}(1-w^\star)(S-D)$
of the preceding display when $\costParam\ge2S$. It vanishes when $w^\star=1$
and otherwise has the sign of $S-D$. At $a=S$, in terms of the weight
$w=q/(S+q)$, the objective is}

\[
 \smrev{\ell_\lambda(D,w)=Sw(1-w)+Dw^2-\frac{\costParam}{2}\log w,}
\]
\smrev{and taking the default corresponds to $w=1$, where $\ell_\lambda(D,1)=D$.
For $0<w<1$, the strict inequality $-\log w>1-w$ and $\costParam\ge2S$
give, for $D\le S$,}
\[
 \smrev{\ell_\lambda(D,w)-D
 >(1-w)\{Sw-D(1+w)+S\}
 =(1-w)(1+w)(S-D)\ge0.}
\]
\smrev{Thus all types with $D\le S$ take the default, and their integrand is
zero. For $D>S$ the integrand is nonpositive; moreover,}
\[
 \smrev{\partial_w\ell_\lambda(D,1)=2D-S-\costParam/2>0
 \quad\hbox{when}\quad D>S/2+\costParam/4,}
\]
\smrev{so every optimum on this tail is interior and the integrand is
strictly negative there. The Gaussian population assigns positive
probability to the tail. Hence}
\[
 \smrev{\left.\frac{\partial\PL}{\partial(\aSD^2)}
       \right|_{(\aMean,\aSD)=(\truMean,\truSD)}<0,
 \qquad
 \left.\frac{\partial\PL}{\partial\aSD}
       \right|_{(\truMean,\truSD)}
 =2\truSD\left.\frac{\partial\PL}{\partial(\aSD^2)}
       \right|_{(\truMean,\truSD)}<0.}
\]

\end{enumerate}
\Halmos
\end{proof}

\begin{proof}{Proof of \Cref{thm: effect on distn outcomes}.}
\smrev{All four items concern the societal bias, so the proof first writes it as an integral over users,
\Cref{eq: societal bias}, and then proves items~1--4 in order from that integral. To obtain the integral, we
write a user's loss in terms of the weight that the AI output puts on the AI's prior mean. The loss that matters is
the uncapped optimal AI loss $V_{\costParam}$, the loss of working with the AI before it is compared with
$\costHumanOnly$. We recall from the proof of \Cref{thm: optimize total loss} that it is continuous and strictly
increasing in the user's squared distance to the AI's mean, and show, as in the proof of
\Cref{prop: who benefits more}, that it is unbounded. These three properties give a threshold distance
$\thresholdA$, finite whenever $\costHumanOnly$ is, such that a user works with the AI exactly when she is within
$\thresholdA$ of the AI's mean (\Cref{eq: bias works with AI}); this threshold is the domain of the integral.
Items~1 and~4 pair users on either side of $\aMean$. Items~2 and~3 take limits of the integral as the costs vary:
its integrand is dominated by $|\aMean-\truth|$ times the population density
(\Cref{eq: bias dominating}), so dominated convergence lets us pass the limit inside.}

\smrev{\emph{Loss as a function of the weight.} As in the proof of \Cref{thm: optimize total loss}, let
$S\triangleq\truSD^2$, $a\triangleq\aSD^2$, $m\triangleq\aMean$ and $D\triangleq(m-\truth)^2$, and let
$f_p(x)\triangleq\truSD^{-1}\phi((x-\truMean)/\truSD)$ be the population density. The quantities $S$, $a$ and
$m$ stay fixed throughout, while items~2--4 vary the interaction cost $\costParam$ (and items~2--3 also vary
$\costHumanOnly$). We therefore index by $\costParam$ the loss, its minimum and the optimal weight defined
below. The optimal signal $\opteffort$ also depends on $\costParam$, but we keep the main text's notation for
it.}

\smrev{By \Cref{eq: truthAI def}, a user who sends a signal with finite noise level $\qErrorSD>0$ receives the
AI output $\truthAI(\query,\qErrorSD)=(1-w)\query+wm$, where}
\[
 \smrev{w\triangleq\frac{\qErrorSD^2}{a+\qErrorSD^2}\in(0,1),
 \qquad\hbox{equivalently}\qquad
 \qErrorSD^2=\frac{aw}{1-w},}
\]
\smrev{is the weight on the AI's prior mean. The uninformative signal ($\qErrorSD=\infty$) gives the default
output $m$. We represent it by $w=1$, the weight at which the AI output equals $m$ whatever the signal. A
noiseless signal ($\qErrorSD=0$, so $w=0$) has infinite communication cost and is never optimal. Hence the
relevant weights are $w\in(0,1]$, and in terms of $w$ the utility loss $\obj(\truth,\qErrorSD)$ is}
\begin{equation}
 \smrev{\label{eq: bias loss in w}
 \ell_{\costParam}(D,w)\triangleq
 \underbrace{aw(1-w)+Dw^2}_{\text{fidelity error}}
 +\underbrace{\costParam c_a(w)}_{\text{communication cost}},
 \qquad
 c_a(w)\triangleq\costEffort(\qErrorSD)=\frac12\ln\!\left(\frac{S+(a-S)w}{aw}\right).}
\end{equation}
\smrev{We write $V_{\costParam}(D)\triangleq\min_{0<w\le1}\ell_{\costParam}(D,w)$ for the uncapped optimal AI
loss.}

\smrev{In the notation of the earlier proofs, $\ell_{\costParam}(D,w)$ is the
loss $\Psi(D,a,q)$ of \Cref{eq: loss with bias prior} with $q=\qErrorSD^2$ and $w=q/(a+q)$, so
$V_{\costParam}(D)$ is the optimal loss $V_a(D)$ of the proofs of
\Cref{thm: imperf AI & error if pref known,thm: optimize total loss}, with the index changed from $a$ to
$\costParam$.}

\smrev{The communication cost $c_a$ is continuous on $(0,1]$, satisfies $c_a(1)=0$ and $c_a(w)\to\infty$ as
$w\downarrow0$, and is strictly decreasing because}
\begin{equation}
 \smrev{\label{eq: bias ca derivative}
 c_a'(w)=-\frac{S}{2w\{S+(a-S)w\}}<0,}
\end{equation}
\smrev{where $S+(a-S)w=S(1-w)+aw>0$.}

\smrev{\emph{Properties of the optimal loss.} As shown at the start of the proof of
\Cref{thm: optimize total loss}, a minimizer exists for every $D\ge0$, and $V_{\costParam}$ is concave,
continuous and strictly increasing, with $V_{\costParam}(0)=0$ and $V_{\costParam}(D)\le D$.}

\smrev{The threshold below also needs $V_{\costParam}(D)\to\infty$ as $D\to\infty$. Since $V_{\costParam}$ is
increasing, it suffices to show that it is unbounded. The proof of \cref{sprop:individual utility loss} of
\Cref{prop: who benefits more} shows this for $a=S$; we reproduce its argument for general $a$. Suppose that
$V_{\costParam}(D_n)\le M$ along a sequence $D_n\to\infty$. Then the minimizers $w_n$ satisfy
$\costParam c_a(w_n)\le M$, which keeps them bounded away from $0$ because $c_a(w)\to\infty$ as
$w\downarrow0$. But they also satisfy $D_nw_n^2\le M$, which forces $w_n\to0$, a contradiction.}

\smrev{\emph{Who works with the AI.} By \Cref{eq:decision w/ human only}, user $\truth$ works with the AI,
including at a tie, if and only if $V_{\costParam}(D)\le\costHumanOnly$. If $\costHumanOnly<\infty$, then, as in
the proof of \cref{sprop:individual no AI} of \Cref{prop: who benefits more}, there is a unique
$\thresholdA=\thresholdA(\costParam,\costHumanOnly)\in(0,\infty)$ with
$V_{\costParam}(\thresholdA^2)=\costHumanOnly$, because $V_{\costParam}$ is
continuous and strictly increasing from $0$ to $\infty$. The subscript of $\thresholdA$ is a label, not the
prior variance $a$. If $\costHumanOnly=\infty$, set $\thresholdA\triangleq\infty$. In both cases,}
\begin{equation}
 \smrev{\label{eq: bias works with AI}
 \hbox{user $\truth$ works with the AI}
 \iff V_{\costParam}(D)\le\costHumanOnly
 \iff |m-\truth|\le\thresholdA.}
\end{equation}
\smrev{The threshold is measured from the AI's mean $m$. It is the threshold of \Cref{prop: who benefits more}
when the AI prior is unbiased ($m=\truMean$ and $a=S$).}

\smrev{\emph{The optimal weight.} Let $\opteffort$ be an optimal signal of user $\truth$, that is, a solution of
\eqref{eq: loss def}. It is defined for every user, including those who end up working alone. If several
signals are optimal, $\opteffort$ is any one of them; we show below that this choice matters only for a set of
users of population probability zero. The optimal weight is
$w^\star_{\costParam}(\truth)\triangleq\opteffort^2/(a+\opteffort^2)$, with $w^\star_{\costParam}(\truth)=1$ when
$\opteffort=\infty$. By definition,}
\[
 \smrev{\ell_{\costParam}\bigl(D,w^\star_{\costParam}(\truth)\bigr)=V_{\costParam}(D)
 \qquad\hbox{and}\qquad 0<w^\star_{\costParam}(\truth)\le1.}
\]

\smrev{We now show that the choice among optimal signals matters only on a null set. Fix $\costParam$. The argument of
the proof of \cref{sprop:pl-censoring} of \Cref{thm: optimize total loss}, applied to $V_{\costParam}$ (our prior
variance $a$) instead of $V_S$ (the case $a=S$), shows that the optimal weight is unique except at countably
many values of $D$. That argument uses two facts: a concave function is differentiable at all but countably many
points of $(0,\infty)$, and at each differentiability point every optimal weight satisfies
$w^2=V_{\costParam}'(D)$. The exceptional values of $D$ are these countably many points together with $D=0$.
They correspond to countably many types $\truth=m\pm\sqrt D$, a set of population probability zero. Outside this set,
$w^\star_{\costParam}(\truth)=\sqrt{V_{\costParam}'(D)}$ with $V_{\costParam}'$ monotone. Hence, whichever
choice is made, $w^\star_{\costParam}$ is measurable and is determined up to a null set, so the population
integrals below do not depend on the choice.}

\smrev{It also follows that two users at the same distance $x$ from the AI's mean have the same optimal weight,
for almost every $x>0$ with respect to Lebesgue measure. Indeed, $m-x$ and $m+x$ have the same $D=x^2$, so}
\begin{equation}
 \smrev{\label{eq: bias symmetric weight}
 w^\star_{\costParam}(m-x)=w^\star_{\costParam}(m+x)\qquad\hbox{for almost every }x>0.}
\end{equation}

\smrev{\emph{The societal bias as an integral.} We now express the signed bias
$b_{\costParam,\costHumanOnly}\triangleq\EE{\truDist}{\outcome}-\truMean$ as an integral over users. Since the
output in \Cref{eq:decision w/ human only} depends on the signal, this expectation is over both the user's
preference $\truth\sim\truDist$ and her signal noise. The
societal bias of the statement is then $\mathcal B(\costParam,\costHumanOnly)=|b_{\costParam,\costHumanOnly}|$.
We need two facts about the output of each user, conditional on her preference $\truth$: its conditional
mean, which is what the bias averages, and its conditional variance, which we use only to show that
$\outcome$ is integrable, so that the bias is well defined.}

\smrev{First, consider a user $\truth$ who works with the AI and write $w^\star=w^\star_{\costParam}(\truth)$. If
$\opteffort<\infty$, she sends $\query=\truth+\qError$ with $\qError\sim\mathcal N(0,\opteffort^2)$ and receives
$(1-w^\star)\query+w^\star m=\truth+w^\star(m-\truth)+(1-w^\star)\qError$. Since
$\opteffort^2=aw^\star/(1-w^\star)$,}
\[
 \smrev{\EE{}{\outcome\mid\truth}=\truth+w^\star(m-\truth),
 \qquad
 \VV(\outcome\mid\truth)=(1-w^\star)^2\opteffort^2=aw^\star(1-w^\star)\le\frac a4.}
\]
\smrev{If $\opteffort=\infty$, then $w^\star=1$ and she receives $m$, so the same mean formula holds and the
variance is zero. A user who works alone outputs $\truth$.}

\smrev{Next, $\outcome$ is integrable. For a user who works with the AI, the Cauchy--Schwarz inequality bounds
the conditional mean absolute deviation by the square root of the conditional variance:}
\[
 \smrev{\EE{}{\bigl|\outcome-\truth-w^\star(m-\truth)\bigr|\Bigm|\truth}\le\frac{\sqrt a}2,
 \qquad\hbox{so}\qquad
 \EE{}{|\outcome|\mid\truth}\le|\truth|+|m-\truth|+\frac{\sqrt a}2.}
\]
\smrev{A user who works alone has $\EE{}{|\outcome|\mid\truth}=|\truth|$. Both bounds are integrable under the
Gaussian population, so $\outcome$ is integrable.}

\smrev{Finally, we average the conditional means over the population. Since $\EE{\truDist}{\truth}=\truMean$,
the signed bias is the population average of $\EE{}{\outcome\mid\truth}-\truth$, which equals
$w^\star_{\costParam}(\truth)(m-\truth)$ for users who work with the AI and zero for the others. By
\Cref{eq: bias works with AI},}
\begin{equation}
 \smrev{\label{eq: societal bias}
 \begin{aligned}
 b_{\costParam,\costHumanOnly}
 &=\int_{\mathbb R}\1{V_{\costParam}(D)\le\costHumanOnly}\,w^\star_{\costParam}(\truth)(m-\truth)
   f_p(\truth)\,d\truth\\
 &=\int_{|m-\truth|\le\thresholdA}w^\star_{\costParam}(\truth)(m-\truth)f_p(\truth)\,d\truth.
 \end{aligned}}
\end{equation}

\smrev{We also record a bound on the integrand of \Cref{eq: societal bias}. Since the indicator
and $w^\star_{\costParam}$ lie in $[0,1]$, for every $\costParam$ and $\costHumanOnly$,}
\begin{equation}
 \smrev{\label{eq: bias dominating}
 \bigl|\1{V_{\costParam}(D)\le\costHumanOnly}\,w^\star_{\costParam}(\truth)(m-\truth)f_p(\truth)\bigr|
 \le|m-\truth|f_p(\truth).}
\end{equation}
\smrev{The right side does not depend on $\costParam$ or $\costHumanOnly$, and it is integrable, because
$\int_{\mathbb R}|m-\truth|f_p(\truth)\,d\truth\le|m|+\EE{\truDist}{|\truth|}<\infty$. Items~2 and~3 use it as
the dominating function.}

\smrev{\emph{Item 1.} Item~1 states that the societal bias is at most the AI bias,
$\mathcal B(\costParam,\costHumanOnly)\le|\aMean-\truMean|$. We also show that the signed bias is nonzero, with
the sign of $\aMean-\truMean$, when $\aMean\ne\truMean$; item~4 uses this sign. We pair each user $m-x$ below
the AI's mean with the user $m+x$ above it. For almost every such pair, the two users have the same optimal
weight, by \Cref{eq: bias symmetric weight}, while, unless $m=\truMean$, the population density is higher at the
one closer to $\truMean$.}

\smrev{Substitute $\truth=m-x$ on $\{\truth<m\}$ and $\truth=m+x$ on $\{\truth>m\}$ in
\Cref{eq: societal bias}. Both pieces are finite because the integrand is absolutely integrable. By
\Cref{eq: bias symmetric weight}, the two pieces carry the same weight $w^\star_{\costParam}(m+x)$ for almost
every $x$, so they combine into one integral with this common weight inside:}
\begin{equation}
 \smrev{\label{eq: bias paired}
 b_{\costParam,\costHumanOnly}
 =\int_0^{\thresholdA}x\,w^\star_{\costParam}(m+x)\{f_p(m-x)-f_p(m+x)\}\,dx.}
\end{equation}
\smrev{The sign of the factor $f_p(m-x)-f_p(m+x)$ depends on whether $m$ is above, below or equal to
$\truMean$, so we treat these three cases.}

\smrev{\emph{Case $m>\truMean$.} We show that $0<b_{\costParam,\costHumanOnly}\le m-\truMean$. For $x>0$,
$(m-\truMean-x)^2<(m-\truMean+x)^2$, so the user $m-x$ is closer to $\truMean$ than the user $m+x$, and}
\begin{equation}
 \smrev{\label{eq: bias density gap}
 f_p(m-x)-f_p(m+x)>0\qquad\hbox{for every }x>0.}
\end{equation}

\smrev{By \Cref{eq: bias density gap} and $w^\star_{\costParam}>0$, the integrand of
\Cref{eq: bias paired} is strictly positive on $(0,\thresholdA)$. This interval has positive length because
$\thresholdA>0$, so $b_{\costParam,\costHumanOnly}>0$.}

\smrev{For the upper bound, we replace the weight by its maximum value $1$ and extend the integral to all
$x>0$. Both steps can only increase the integral, because $w^\star_{\costParam}\le1$ and, by
\Cref{eq: bias density gap}, $x\{f_p(m-x)-f_p(m+x)\}>0$ for every $x>0$. Hence}
\[
 \smrev{b_{\costParam,\costHumanOnly}
 \le\int_0^\infty x\{f_p(m-x)-f_p(m+x)\}\,dx
 =\int_{\mathbb R}(m-\truth)f_p(\truth)\,d\truth=m-\truMean.}
\]
\smrev{The middle equality undoes the pairing substitution. Hence
$0<\mathcal B(\costParam,\costHumanOnly)\le m-\truMean$.}

\smrev{\emph{Case $m<\truMean$.} We show that $m-\truMean\le b_{\costParam,\costHumanOnly}<0$. This case is
symmetric: now $f_p(m-x)-f_p(m+x)<0$ for every $x>0$. So the integrand of \Cref{eq: bias paired} is strictly
negative on $(0,\thresholdA)$, which gives $b_{\costParam,\costHumanOnly}<0$. Replacing the weight by $1$ and
extending the integral to all $x>0$ can only decrease it, which gives $b_{\costParam,\costHumanOnly}\ge
m-\truMean$. Hence $0<\mathcal B(\costParam,\costHumanOnly)\le\truMean-m$.}

\smrev{\emph{Case $m=\truMean$.} Now $f_p(m-x)=f_p(m+x)$ for every $x$, so \Cref{eq: bias paired} gives
$b_{\costParam,\costHumanOnly}=0$.}

\smrev{This proves item~1.}

\smrev{\emph{Item 2.} Item~2 states that the societal bias vanishes in two limits:
$\lim_{\costParam\downarrow0}\mathcal B(\costParam,\costHumanOnly)=0$ for each fixed
$\costHumanOnly\in(0,\infty]$, and $\lim_{\costHumanOnly\downarrow0}\mathcal B(\costParam,\costHumanOnly)=0$ for
each fixed $\costParam>0$. In both limits we show that the integrand of \Cref{eq: societal bias} tends to zero for
every user $\truth$. Dominated convergence, with the bound \Cref{eq: bias dominating}, then gives
$b_{\costParam,\costHumanOnly}\to0$.}

\smrev{\emph{Limit $\costParam\downarrow0$.} Fix $\costHumanOnly\in(0,\infty]$ and a user $\truth$, and let
$\costParam\downarrow0$. If she works with the AI, her expected output is pulled toward the AI's mean $m$ by
$w^\star_{\costParam}(\truth)(m-\truth)$, and this pull is her integrand in \Cref{eq: societal bias} up to the
indicator and $f_p(\truth)$. We show that $w^\star_{\costParam}(\truth)(m-\truth)\to0$: when communication is
cheap, the expected output of a user who works with the AI stays close to her own preference, so the AI's mean,
and with it the AI bias $m-\truMean$, has little effect on it.}

\smrev{First, the optimal loss tends to zero. To see this, bound it by the loss of the signal with variance
$\qErrorSD^2=\costParam$, whose weight is $w=\costParam/(a+\costParam)$. With $q=\costParam$ in
\Cref{eq: loss with bias prior}, its fidelity error and communication cost are}
\[
 \smrev{\frac{\costParam^2D+a^2\costParam}{(a+\costParam)^2}\le\frac{\costParam^2D}{a^2}+\costParam
 \qquad\hbox{and}\qquad
 \frac{\costParam}{2}\ln\!\left(1+\frac S{\costParam}\right).}
\]
\smrev{Both tend to zero as $\costParam\downarrow0$, the second because $\costParam\ln(1/\costParam)\to0$. Hence}
\[
 \smrev{0\le V_{\costParam}(D)
 \le\ell_{\costParam}\Big(D,\frac{\costParam}{a+\costParam}\Big)
 \le\costParam+\frac{\costParam^2D}{a^2}+\frac{\costParam}{2}\ln\!\left(1+\frac S{\costParam}\right)
 \longrightarrow0.}
\]

\smrev{Second, the optimal loss bounds the displacement. In \Cref{eq: bias loss in w}, the fidelity error is at
least $Dw^2$ and the communication cost is nonnegative, so}
\[
 \smrev{D\,w^\star_{\costParam}(\truth)^2\le\ell_{\costParam}\bigl(D,w^\star_{\costParam}(\truth)\bigr)
 =V_{\costParam}(D),
 \qquad\hbox{that is,}\qquad
 |m-\truth|\,w^\star_{\costParam}(\truth)\le V_{\costParam}(D)^{1/2}\longrightarrow0.}
\]
\smrev{The indicator in \Cref{eq: societal bias} is at most $1$, so the integrand tends to zero for every
$\truth$. Dominated convergence gives $b_{\costParam,\costHumanOnly}\to0$ as $\costParam\downarrow0$.}

\smrev{\emph{Limit $\costHumanOnly\downarrow0$.} Fix $\costParam>0$ and let $\costHumanOnly\downarrow0$. Every
user other than $\truth=m$ eventually works alone. A user with $D>0$ has $V_{\costParam}(D)>0$,
because $V_{\costParam}$ is strictly increasing with $V_{\costParam}(0)=0$. Once
$\costHumanOnly<V_{\costParam}(D)$, she works alone by \Cref{eq: bias works with AI}, so the indicator in
\Cref{eq: societal bias} is zero. The user with $D=0$ has a zero integrand because $m-\truth=0$. Dominated
convergence gives $b_{\costParam,\costHumanOnly}\to0$ as $\costHumanOnly\downarrow0$. This proves item~2.}

\smrev{\emph{Item 3.} Item~3 states that the societal bias tends to the AI bias when both costs grow:
$\mathcal B(\costParam,\costHumanOnly)\to|\aMean-\truMean|$ as $\costParam\to\infty$ and
$\costHumanOnly\to\infty$ jointly. Take any sequence $(\costParam_n,\costHumanOnly_n)\to(\infty,\infty)$ and fix
a user $\truth$. We show that she eventually works with the AI and that her optimal weight tends to $1$, so
that her integrand in \Cref{eq: societal bias} tends to $(m-\truth)f_p(\truth)$, from which dominated
convergence gives the limit.}

\smrev{First, the optimal weight tends to $1$. The optimal loss is at most the loss of the uninformative
signal, $\ell_{\costParam}(D,1)=D$, and it is at least its communication cost, since the fidelity error is
nonnegative. Hence, for every $\costParam>0$,}
\[
 \smrev{\costParam\,c_a\bigl(w^\star_{\costParam}(\truth)\bigr)\le V_{\costParam}(D)\le D,
 \qquad\hbox{so}\qquad
 0\le c_a\bigl(w^\star_{\costParam_n}(\truth)\bigr)\le\frac D{\costParam_n}\longrightarrow0.}
\]
\smrev{Because $c_a$ is continuous and strictly decreasing with $c_a(1)=0$, this forces
$w^\star_{\costParam_n}(\truth)\to1$.}

\smrev{Second, the user eventually works with the AI. Indeed, $V_{\costParam_n}(D)\le D<\costHumanOnly_n$ for
all large $n$, so the indicator in \Cref{eq: societal bias} equals $1$ by \Cref{eq: bias works with AI}.}

\smrev{Hence the integrand of \Cref{eq: societal bias} tends to $(m-\truth)f_p(\truth)$ for every $\truth$, and
dominated convergence with the bound \Cref{eq: bias dominating} gives}
\[
 \smrev{\lim_{n\to\infty}b_{\costParam_n,\costHumanOnly_n}
 =\int_{\mathbb R}(m-\truth)f_p(\truth)\,d\truth=m-\truMean.}
\]
\smrev{Taking absolute values gives $\mathcal B(\costParam_n,\costHumanOnly_n)\to|\aMean-\truMean|$. This proves
item~3.}

\smrev{\emph{Item 4.} Item~4 states that, if everyone uses the AI ($\costHumanOnly=\infty$), the societal bias
increases with the interaction cost $\costParam$. We show that $\mathcal B(\costParam,\infty)$ is nondecreasing
in $\costParam$, and strictly increasing when $\aMean\ne\truMean$. Since $\costHumanOnly=\infty$, we have
$\thresholdA=\infty$ and every user works with the AI. We compare two costs $\costParam_1>\costParam_2>0$ in
three steps: at the higher cost, every user's optimal weight is weakly larger (Step~1) and, far enough from $m$,
strictly larger (Step~2); inserting this into the paired formula, \Cref{eq: bias paired}, compares the biases
(Step~3). The increase is strict only when $\aMean\ne\truMean$, because if $\aMean=\truMean$ the factor
$f_p(m-x)-f_p(m+x)$ in \Cref{eq: bias paired} vanishes and the bias is zero for every $\costParam$.}

\smrev{\emph{Step 1: a higher cost gives a weakly larger weight.} Fix a user $\truth$, and write
$w_1\triangleq w^\star_{\costParam_1}(\truth)$ and $w_2\triangleq w^\star_{\costParam_2}(\truth)$ for her optimal
weights at the two costs. Each weight is at least as good as the other one at its own cost:}
\[
 \smrev{\ell_{\costParam_1}(D,w_1)\le\ell_{\costParam_1}(D,w_2)
 \qquad\hbox{and}\qquad
 \ell_{\costParam_2}(D,w_2)\le\ell_{\costParam_2}(D,w_1).}
\]
\smrev{When we add these two inequalities, the fidelity errors cancel, which leaves}
\[
 \smrev{(\costParam_1-\costParam_2)\bigl\{c_a(w_1)-c_a(w_2)\bigr\}\le0.}
\]
\smrev{Since $\costParam_1>\costParam_2$ and $c_a$ is strictly decreasing, it follows that $w_1\ge w_2$,
whichever optimal weights are chosen. Thus}
\begin{equation}
 \smrev{\label{eq: bias weak weights}
 w^\star_{\costParam_1}(\truth)\ge w^\star_{\costParam_2}(\truth)\qquad\hbox{for every }\truth.}
\end{equation}

\smrev{\emph{Step 2: far from $m$, the weight is strictly larger.} Let
$\bar D\triangleq a/2+\costParam_1S/(4a)$. This is a cutoff on the squared distance $D=(m-\truth)^2$, chosen so
that the derivative of the loss at $w=1$, computed next, is positive when $D>\bar D$; beyond it, the
uninformative signal ($w=1$) is therefore not optimal at either cost. Consider a user with
$D>\bar D$, with $w_1$ and $w_2$ as in Step~1. By \Cref{eq: bias loss in w,eq: bias ca derivative}, the
derivative of the loss at $w=1$ is}
\[
 \smrev{\partial_w\ell_{\costParam_1}(D,1)
 =2D-a-\frac{\costParam_1S}{2a}>0,}
\]
\smrev{and the same holds with $\costParam_2<\costParam_1$ in place of $\costParam_1$. So a weight slightly below
$1$ has a lower loss than $w=1$ at either cost, and both optimal weights lie in $(0,1)$. There, they satisfy the
first-order conditions}
\[
 \smrev{\partial_w\ell_{\costParam_1}(D,w_1)=0
 \qquad\hbox{and}\qquad
 \partial_w\ell_{\costParam_2}(D,w_2)=0.}
\]
\smrev{By \Cref{eq: bias weak weights}, $w_1\ge w_2$, so it remains to exclude $w_1=w_2$. Suppose that
$w_1=w_2$. The fidelity error does not depend on $\costParam$, so subtracting the two
conditions would give $(\costParam_1-\costParam_2)c_a'(w_1)=0$, which contradicts
\Cref{eq: bias ca derivative}. Hence}
\begin{equation}
 \smrev{\label{eq: bias strict weights}
 w^\star_{\costParam_1}(\truth)>w^\star_{\costParam_2}(\truth)\qquad\hbox{whenever }(m-\truth)^2>\bar D.}
\end{equation}

\smrev{\emph{Step 3: comparing the biases.} By \Cref{eq: bias paired} with $\thresholdA=\infty$, the difference
of the two signed biases is}
\[
 \smrev{b_{\costParam_1,\infty}-b_{\costParam_2,\infty}
 =\int_0^\infty x\,\bigl\{w^\star_{\costParam_1}(m+x)-w^\star_{\costParam_2}(m+x)\bigr\}
 \{f_p(m-x)-f_p(m+x)\}\,dx.}
\]
\smrev{By \Cref{eq: bias weak weights}, the first brace is nonnegative for every $x>0$. By
\Cref{eq: bias strict weights}, it is strictly positive for $x>\sqrt{\bar D}$, a set of positive length. The sign
of the second brace depends on the case.}

\smrev{If $m>\truMean$, the second brace is strictly positive for every $x>0$ by \Cref{eq: bias density gap}. The
integrand is then nonnegative, and strictly positive for $x>\sqrt{\bar D}$, so
$b_{\costParam_1,\infty}>b_{\costParam_2,\infty}$. Since $b_{\costParam,\infty}>0$ by item~1,
$\mathcal B(\costParam,\infty)=b_{\costParam,\infty}$ is strictly increasing in $\costParam$.}

\smrev{If $m<\truMean$, the second brace is strictly negative for every $x>0$, and the same argument gives
$b_{\costParam_1,\infty}<b_{\costParam_2,\infty}$. Since $b_{\costParam,\infty}<0$ by item~1,
$\mathcal B(\costParam,\infty)=-b_{\costParam,\infty}$ is strictly increasing in $\costParam$.}

\smrev{If $m=\truMean$, then $\mathcal B(\costParam,\infty)=0$ for every $\costParam$ by item~1, which is
nondecreasing in $\costParam$.}

\smrev{This proves item~4.}
\Halmos
\end{proof}

\section{Results in \Cref{sec:training-loop}.}
\label{ec-sec: proofs training loop}

\subsection{Results in \Cref{sec:prior-dynamics}.}

\begin{proof}{Proof of \Cref{prop:iter-den}.}
\smrev{We split the proposition into three claims, which we state now. Claim~1: every prior $\pi_t$ of
\Cref{alg:self-training-loop} is symmetric and has a finite second moment. Claim~2: at every round $t$, the AI output
$\truthAIAlg{t}(\cdot,\opteffortAlg{t})$ of each user who sends an informative signal is strictly increasing in the
signal realization $\query$, and the next prior $\pi_{t+1}$ is given by \Cref{eq:den-t}. Claim~2 restricts the
monotonicity to informative signals because, for the uninformative signal, the AI returns the default output whatever
the signal realization; this is how we read the monotonicity in the statement. Claim~3: the prior after the first
iteration, $\pi_1$, is given by \Cref{eq:den-1}.}

\smrev{Part~1 proves Claim~1 by induction on $t$. Along the way, it shows that, at every round, optimal signals exist,
and so does a tie-breaking rule, that is, a measurable choice among them that treats users $\truth$ and $-\truth$
alike (paragraph \emph{Ties} below). Part~2 proves Claim~2 at a fixed round $t$. It uses Steps~1.1 and~1.2 of Part~1
at round $t$: optimal signals and a tie-breaking rule exist, they are not noiseless, and the default output, the mean
of the symmetric prior $\pi_t$, is $0$. Part~2 first shows that, for a nondegenerate $\pi_t$ and every informative signal, the AI output
is strictly increasing in the signal realization, with a strictly positive derivative (the case $\pi_t=\delta_0$ is
handled separately). So it can be inverted, and a
change of variables from a user's signal realization to her output gives the density of her output. Part~2 then
averages these densities over users. Part~3 proves Claim~3 by applying the formula of Part~2 at round~$0$. There the
prior is normal, so the AI output is linear in the signal realization, and the inverse and the derivative in
\Cref{eq:den-t} are explicit.}

\smrev{Before Part~1, six short paragraphs fix the setting and the notation (\emph{Setting} and \emph{Signals}), write
the AI output for a general prior (\emph{The AI output}), recall the tie-breaking convention (\emph{Ties}), and say how
we read the Dirac delta and the fractions in \Cref{eq:den-1,eq:den-t} (the two \emph{Reading} paragraphs).}

\smrev{\emph{Setting.} As in \Cref{sec:training-loop}, $\costHumanOnly=\infty$, so every user works with the AI
(\Cref{eq:decision w/ human only}), and $\truMean=0$. As in \Cref{sec:prior-dynamics}, the population is Gaussian and
is the AI's initial prior, $\pi_0=\truDist=\mathcal N(0,\truSD^2)$, and the interaction cost $\costParam>0$ is fixed.
At round $t$, user $\truth$ solves Problem~\eqref{eq: loss def} with the AI's prior $\pi_t$ in place of $\priorAI$
(\Cref{alg:self-training-loop}). Her output $\outcome_t$ at round $t$ is the AI output of
\Cref{eq:decision w/ human only}, and $\pi_{t+1}$ is the distribution of $\outcome_t$ over the population and the
signal noise (\Cref{alg:self-training-loop}). We write her utility loss as
$\obj_{\pi_t}(\truth,\qErrorSD)\triangleq\costQual_{\pi_t}(\truth,\qErrorSD)+\costParam\costEffort(\qErrorSD)$. Only
the fidelity error depends on the AI's prior: $\costQual_{\pi_t}$ is computed with the AI output $\truthAIAlg{t}$
under the prior $\pi_t$, given below. The communication cost $\costEffort(\qErrorSD)=\frac12\ln(1+\truSD^2/
\qErrorSD^2)$ of \Cref{sec:model} is computed under the population $\truDist$, which does not change, so it depends
only on $\qErrorSD$, not on the user or on $t$.}

\smrev{\emph{Signals.} For a finite noise level $\qErrorSD\in(0,\infty)$, we write $v\triangleq\qErrorSD^2$ for the
noise variance. The signal realization of user $\truth$ is then $\query=\truth+\sqrt vZ$, where the standard normal
noise $Z\sim\mathcal N(0,1)$ is independent of $\truth$; only the scale $\sqrt v$ depends on the user's choice.
Conditional on $\truth$, the signal realization has density $\phi_v(\query-\truth)$, where
$\phi_v(y)\triangleq v^{-1/2}\phi(y/\sqrt v)$ is the $\mathcal N(0,v)$ density. The uninformative signal is
$\qErrorSD=\infty$.}

\smrev{\emph{The AI output.} For a noise variance $v\in(0,\infty)$, the AI output is the mean of the AI's posterior
given
the signal realization (the first two expressions in \Cref{eq: truthAI def}). With the prior $\pi_t$ in place of
$\priorAI$, Bayes' rule (\Cref{eq:bayes-update}) gives}
\begin{equation}\smrev{\label{eq: iter-den posterior mean}
 \truthAIAlg{t}(\query,\sqrt v)=\frac{\int x\,\phi_v(\query-x)\,\pi_t(dx)}{\int\phi_v(\query-x)\,\pi_t(dx)}.}
\end{equation}
\smrev{This is \Cref{ec-mi-eq: posterior mean} of \Cref{ec-sec: proofs unique key}, with $\pi=\pi_t$, so the results
of that section that we cite in Part~2 are about this same formula. The linear
formula on the right of \Cref{eq: truthAI def} holds only for a normal prior, so we use it only at round~$0$ (Part~3).
For the uninformative signal, the AI returns the mean of $\pi_t$, its default output.}

\smrev{\emph{Ties.} When several signals are optimal for a user, we use a tie-breaking rule in the sense of the
Setting of \Cref{ec-sec: proofs unique key}: a measurable choice $\opteffortAlg{t}$ among the optimal signals that
treats users $\truth$ and $-\truth$ alike. Step~1.1 shows that such a rule exists at every round. We fix one
tie-breaking rule at each round; the results hold whichever are fixed.}

\smrev{\emph{Reading \Cref{eq:den-1,eq:den-t} as measures.} These equations write each updated prior, $\pi_1$ and
$\pi_{t+1}$, with the Dirac delta $\delta_0(z)$ next to a density in $z$. We read them, and the decompositions of
priors below, as identities of measures: \Cref{eq:den-t} states that
$\pi_{t+1}(dz)=\PP_{\truDist}(\opteffortAlg{t}=\infty)\,\delta_0(dz)+f_{t+1}(z)\,dz$. Here $f_{t+1}$ is the density
part, contributed by the users who send an informative signal: it is the expectation in \Cref{eq:den-t}. Step~2.5
defines $f_{t+1}$ as an integral over these users, \Cref{eq: iter-den mixture}, and shows that this integral is the
expectation. Every formula for a density holds for Lebesgue-almost every output value $z$.}

\smrev{\emph{Reading the fractions in \Cref{eq:den-1,eq:den-t}.} Inside each expectation, the product of the fraction
and the indicator is read as $0$ when the indicator is $0$, without evaluating the fraction, which is not defined for
users who send the uninformative signal. The fraction in \Cref{eq:den-t} is also read as $0$ when $z$ lies outside the
range of $\truthAIAlg{t}(\cdot,\opteffortAlg{t})$, where the inverse $\vartheta_{A,t}(z,\truth)$ is not defined.}

\smrev{\emph{Part 1: proof of Claim~1.} Claim~1 states that every $\pi_t$ is symmetric around $0$ and has a finite
second moment. The prior $\pi_0=\mathcal N(0,\truSD^2)$ is symmetric, with second moment $\truSD^2$. Suppose that, for
some $t\ge0$, the prior $\pi_t$ is symmetric and has a finite second moment. We show the same for $\pi_{t+1}$ in four
steps: optimal signals exist (Step~1.1); they are not noiseless, and the uninformative signal has loss $\truth^2$
(Step~1.2); optimality then bounds the second moment of the output (Step~1.3); and users $\truth$ and $-\truth$ have
mirror-image outputs (Step~1.4).}

\smrev{\emph{Step 1.1: optimal signals exist.} An optimal signal exists for every user by
\Cref{ec-mm-lemma: wellposed}(i), whose three assumptions hold here. First, the Gaussian population $\truDist$ is
atomless with a finite first absolute moment $\int|x|\,\truDist(dx)$. Second, $\pi_t$ is symmetric by the induction
hypothesis, and it has a finite first absolute moment $\int|x|\,\pi_t(dx)$ because, by the induction hypothesis, its
second moment is finite. Third, $0<\costParam<\infty$. The same lemma shows that a measurable choice among the optimal
signals that treats $\truth$ and $-\truth$ alike exists, that is, a tie-breaking rule as in the paragraph \emph{Ties}
above.}

\smrev{\emph{Step 1.2: the uninformative and the noiseless signals.} Since $\pi_t$ is symmetric with a finite first
absolute moment, its mean is $0$. The uninformative signal therefore yields the default output $0$, with fidelity
error $\truth^2$ and no communication cost, so its utility loss is $\obj_{\pi_t}(\truth,\infty)=\truth^2$. A noiseless
signal ($\qErrorSD=0$) is never optimal: its communication cost is $\costEffort(0)=\infty$ (\Cref{sec:model}), and
$\costParam>0$, so its utility loss is infinite. Hence $\opteffortAlg{t}\in(0,\infty]$. We write
$v^\star_t(\truth)\triangleq(\opteffortAlg{t})^2\in(0,\infty]$ for the optimal noise variance.}

\smrev{\emph{Step 1.3: the second moment of $\pi_{t+1}$.} Let $\EE{}{\cdot\mid\truth}$ denote the expectation over the
signal noise of user $\truth$. Her optimal signal does at least as well as the uninformative one, and the
communication cost is nonnegative. Hence the fidelity error of her optimal signal is at most her optimal utility loss,
which is at most $\truth^2$:}
\[
 \smrev{\EE{}{(\outcome_t-\truth)^2\mid\truth}=\costQual_{\pi_t}(\truth,\opteffortAlg{t})
 \le\obj_{\pi_t}(\truth,\opteffortAlg{t})\le\obj_{\pi_t}(\truth,\infty)=\truth^2.}
\]
\smrev{We now average over $\truth\sim\truDist$. Using $(x+y)^2\le2x^2+2y^2$ with $x=\outcome_t-\truth$ and
$y=\truth$,}
\[
 \smrev{\EE{}{(\outcome_t)^2}\le2\,\EE{}{(\outcome_t-\truth)^2}+2\,\EE{\truDist}{\truth^2}
 \le4\,\EE{\truDist}{\truth^2}=4\truSD^2<\infty.}
\]
\smrev{Since $\pi_{t+1}$ is the distribution of $\outcome_t$ (\Cref{alg:self-training-loop}), its second moment is
finite.}

\smrev{\emph{Step 1.4: symmetry of $\pi_{t+1}$.} We first derive a reflection identity between the AI outputs of
users $\truth$ and $-\truth$, which the mirror-image argument below uses. We also deduce from it that the two users
have the same utility loss for every noise level, hence the same optimal signals, so choosing the same signal for
both, as the tie-breaking rule does, is optimal for each. Fix a noise variance $v\in(0,\infty)$. Because $\pi_t$ is
symmetric, replacing $x$ by $-x$ in
both integrals of \Cref{eq: iter-den posterior mean} shows that the AI output is odd in the signal realization:}
\[
 \smrev{\truthAIAlg{t}(-\query,\sqrt v)=-\truthAIAlg{t}(\query,\sqrt v)\qquad\hbox{for every }\query\in\mathbb R.}
\]
\smrev{The signal realization $-\truth+\sqrt vZ$ of user $-\truth$ has the same distribution as $-(\truth+\sqrt vZ)$,
because $Z$ and $-Z$ have the same distribution. By oddness, the AI outputs of the two users therefore satisfy the
reflection identity}
\begin{equation}\smrev{\label{eq: iter-den reflection}
 \truthAIAlg{t}(-\truth+\sqrt vZ,\sqrt v)\ \hbox{has the same distribution as}\ -\truthAIAlg{t}(\truth+\sqrt
 vZ,\sqrt v).}
\end{equation}
\smrev{Hence the two users have the same fidelity error:}
\[
 \smrev{\costQual_{\pi_t}(-\truth,\sqrt v)
 =\EE{}{\bigl(-\truthAIAlg{t}(\truth+\sqrt vZ,\sqrt v)+\truth\bigr)^2}
 =\costQual_{\pi_t}(\truth,\sqrt v).}
\]
\smrev{They also have the same communication cost $\costEffort(\sqrt v)$, which depends only on the noise level. So
their utility losses coincide for every noise level in $(0,\infty)$, and also for the uninformative signal, whose loss
is $\truth^2$ for both. The noiseless signal is optimal for neither (Step~1.2). Users $\truth$ and $-\truth$ therefore
have the same set of optimal signals.}

\smrev{Next, the two users' outputs are mirror images in distribution. Since the tie-breaking rule treats $\truth$ and
$-\truth$ alike, $\opteffortAlg{t}$ takes the same value at $\truth$ and at $-\truth$. Let $\mu_\truth$ denote the
distribution of the output $\outcome_t$ of user $\truth$, and for a Borel set $E$ let $-E\triangleq\{-z:z\in E\}$. If
the selected signal is finite, the reflection identity \Cref{eq: iter-den reflection} gives the mirror-image identity}
\begin{equation}\smrev{\label{eq: iter-den mirror}
 \mu_{-\truth}(E)=\mu_\truth(-E)\qquad\hbox{for every Borel set }E.}
\end{equation}
\smrev{If the selected signal is uninformative, both outputs are $0$, and \Cref{eq: iter-den mirror} holds as well.}

\smrev{Finally, the population $\truDist$ is symmetric. For every Borel set $E$,}
\[
 \smrev{\PP(\outcome_t\in-E)=\int\mu_\truth(-E)\,\truDist(d\truth)
 =\int\mu_{-\truth}(E)\,\truDist(d\truth)
 =\int\mu_\truth(E)\,\truDist(d\truth)=\PP(\outcome_t\in E).}
\]
\smrev{The second equality is the mirror-image identity \Cref{eq: iter-den mirror}, and the third replaces $\truth$
by $-\truth$, which leaves $\truDist$ unchanged. Hence $\pi_{t+1}$ is symmetric. This completes the induction and
proves Claim~1.}

\smrev{\emph{Part 2: proof of Claim~2.} Fix a round $t\ge0$. Claim~2 states that, for every user with
$\opteffortAlg{t}<\infty$, the AI output $\truthAIAlg{t}(\cdot,\opteffortAlg{t})$ is strictly increasing in the
signal realization, and that $\pi_{t+1}$ satisfies \Cref{eq:den-t}. By Part~1, $\pi_t$ is symmetric. A symmetric
distribution concentrated on a single point is $\delta_0$, the point mass at $0$. So either $\pi_t=\delta_0$, and no
user communicates, as we show next, or $\pi_t$ is nondegenerate, which Step~2.1 needs for strict monotonicity. The
next paragraph treats $\pi_t=\delta_0$, and the rest of Part~2 treats a nondegenerate $\pi_t$.}

\smrev{Suppose first that $\pi_t=\delta_0$. The noiseless signal is never optimal (Step~1.2), so we consider noise
levels $\qErrorSD\in(0,\infty]$. Under this prior, the AI's posterior is $\delta_0$ for every such signal, so every
such signal yields the output $0$ and the fidelity error $\truth^2$. A finite signal has a strictly positive
communication cost, $\costEffort(\qErrorSD)=\frac12\ln(1+\truSD^2/\qErrorSD^2)>0$, while the uninformative signal has
none. Hence the uninformative signal is the unique optimum of every user. No user sends an informative signal, so the
monotonicity part of Claim~2 holds trivially. Every output is $0$, so $\pi_{t+1}=\delta_0$. This is \Cref{eq:den-t}
with $\PP_{\truDist}(\opteffortAlg{t}=\infty)=1$ and an expectation term equal to $0$. So Claim~2 holds when
$\pi_t=\delta_0$.}

\smrev{From now on, $\pi_t$ is nondegenerate. We prove Claim~2 in six steps. They follow the route of the opening
roadmap; here we state what each step proves and which later step uses it.}
\begin{itemize}
\item \smrev{Step~2.1: the AI output has a strictly positive, continuous derivative in the signal realization. This is
the monotonicity half of Claim~2, and Steps~2.2 and~2.3 use it to invert the AI output and to change variables.}
\item \smrev{Step~2.2: the inverse of the AI output, which appears in \Cref{eq:den-t}.}
\item \smrev{Step~2.3: the density of one user's output is the fraction in \Cref{eq:den-t}.}
\item \smrev{Step~2.4: this density is jointly measurable in the user and the output, which Step~2.5 needs to exchange
two integrals.}
\item \smrev{Step~2.5: averaging the densities over users gives \Cref{eq:den-t}.}
\item \smrev{Step~2.6: the mass of $\pi_{t+1}$ at the default output is exactly the coefficient of $\delta_0$ in
\Cref{eq:den-t}, because the density part has no atoms.}
\end{itemize}

\smrev{\emph{Step 2.1: the AI output is strictly increasing.} This step shows that, for every noise variance
$v\in(0,\infty)$, the AI output is continuously differentiable in the signal realization, with a strictly positive
derivative, and that it is a bijection from $\mathbb R$ onto an open interval. Fix $v\in(0,\infty)$, and write
$g_v(\query)\triangleq\truthAIAlg{t}(\query,\sqrt v)$.}

\smrev{The result that gives the derivative is \Cref{ec-mm-lemma: posterior}(i), applied with $\pi=\pi_t$. It yields
three facts. First, $g_v$ is continuously differentiable. Second, the maps $(\query,v)\mapsto g_v(\query)$ and
$(\query,v)\mapsto g_v'(\query)$ are continuous on $\mathbb R\times(0,\infty)$; Step~2.4 uses this. Third, the
derivative of the AI output with respect to the signal realization, which is the derivative $\truthAIAlg{t}'$ of the
statement, is a posterior variance:}
\begin{equation}\smrev{\label{eq: iter-den derivative}
 g_v'(\query)=\frac1v\,\VV_{\pi_t}\bigl(X\mid X+\sqrt vZ=\query\bigr),}
\end{equation}
\smrev{where $X\sim\pi_t$ is independent of $Z$. This is the identity \Cref{ec-mi-eq: posterior derivative}.
\Cref{ec-mi-lemma: derivative} states it for priors with bounded support, and \Cref{ec-mm-lemma: posterior}(i)
extends it to every prior, with no moment or support condition. We need the extension because $\pi_t$ may have
unbounded support (for instance, $\pi_0$ is Gaussian).}

\smrev{We now show that this derivative is strictly positive. By Bayes' rule (\Cref{eq:bayes-update}), the posterior
given the signal realization $\query$ has a density with respect to $\pi_t$ that is proportional to
$x\mapsto\phi_v(\query-x)$, which is strictly positive. So the posterior gives positive mass to exactly the sets to
which $\pi_t$ does. Since $\pi_t$ is nondegenerate, so is the posterior, and its variance is strictly positive. By
\Cref{eq: iter-den derivative}, the derivative $g_v'(\query)$ is therefore strictly positive.}

\smrev{Hence $g_v$ is strictly increasing. Its range $g_v(\mathbb R)$ is an interval, as the image of $\mathbb R$
under a continuous function, and it is open, because a strictly increasing function on $\mathbb R$ attains no largest
or smallest value. This gives the bijection announced at the start of the step, from $\mathbb R$ onto the open
interval $g_v(\mathbb R)$; Step~2.2 uses it to define the inverse. Taking $v=v^\star_t(\truth)$ for a user with
$\opteffortAlg{t}<\infty$ shows that
$\truthAIAlg{t}(\cdot,\opteffortAlg{t})$ is strictly increasing in the signal realization. This proves the first half
of Claim~2.}

\smrev{\emph{Step 2.2: the inverse of the AI output.} For $v\in(0,\infty)$, let $\vartheta_v$ be the inverse of $g_v$
on its range $g_v(\mathbb R)$, which exists by Step~2.1. For a user with $\opteffortAlg{t}<\infty$, the inverse
$\vartheta_{A,t}(z,\truth)$ of the statement is $\vartheta_{v^\star_t(\truth)}(z)$ for $z$ in the range.}

\smrev{\emph{Step 2.3: the output density of one user.} Fix a user $\truth$ with $\opteffortAlg{t}<\infty$, and write
$v=v^\star_t(\truth)$. Her output is $\outcome_t=g_v(\query)$, where the signal realization $\query=\truth+\sqrt vZ$
has density $\phi_v(\query-\truth)$ given $\truth$. By Step~2.1, $g_v$ is strictly increasing and continuously
differentiable, and its derivative \Cref{eq: iter-den derivative} never vanishes. The change of variables
$z=g_v(\query)$ therefore shows that, conditional on $\truth$, the output $\outcome_t$ has the density}
\begin{equation}\smrev{\label{eq: iter-den conditional density}
 k_{\truth,v}(z)\triangleq
 \begin{cases}
 \dfrac{\phi_v\bigl(\vartheta_v(z)-\truth\bigr)}{g_v'\bigl(\vartheta_v(z)\bigr)}, & z\in g_v(\mathbb R),\\[10pt]
 0, & z\notin g_v(\mathbb R).
 \end{cases}}
\end{equation}
\smrev{Being a conditional density, it integrates to one.}

\smrev{We now check that $k_{\truth,v}$ is the fraction in \Cref{eq:den-t}. Since $v=(\opteffortAlg{t})^2$, three
substitutions translate the notation of \Cref{eq: iter-den conditional density} into that of the statement:}
\[
 \smrev{\phi_v(y)=\frac{1}{\opteffortAlg{t}}\,\phi\Bigl(\frac{y}{\opteffortAlg{t}}\Bigr),\qquad
 g_v'=\truthAIAlg{t}'(\cdot,\opteffortAlg{t}),\qquad
 \vartheta_v(z)=\vartheta_{A,t}(z,\truth).}
\]
\smrev{Hence, for $z$ in the range $g_{v^\star_t(\truth)}(\mathbb R)$,}
\[
 \smrev{k_{\truth,v^\star_t(\truth)}(z)
 =\frac{\phi\bigl((\vartheta_{A,t}(z,\truth)-\truth)/\opteffortAlg{t}\bigr)}
 {\opteffortAlg{t}\,\truthAIAlg{t}'\bigl(\vartheta_{A,t}(z,\truth),\opteffortAlg{t}\bigr)},}
\]
\smrev{which is the fraction in \Cref{eq:den-t}. Outside the range, both $k_{\truth,v^\star_t(\truth)}(z)$ and, by the
paragraph \emph{Reading the fractions} before Part~1, the fraction in \Cref{eq:den-t} are $0$. With both cases
checked, the density of the output of every user who sends an informative signal is the fraction in \Cref{eq:den-t}
at every $z$.}

\smrev{\emph{Step 2.4: joint measurability of the densities.} Step~2.5 averages the densities of Step~2.3 over users
and exchanges two integrals by Tonelli's theorem, which requires the integrand to be jointly measurable. Let
$\mathcal C_t\triangleq\{\truth:\opteffortAlg{t}<\infty\}$ be the set of users who send an informative signal; the
first fact below shows that it is measurable. This step shows that}
\begin{equation}\smrev{\label{eq: iter-den measurable}
 (\truth,z)\mapsto k_{\truth,v^\star_t(\truth)}(z)\quad\hbox{is jointly measurable on }\mathcal C_t\times\mathbb R.}
\end{equation}

\smrev{First, the optimal noise variance is measurable in the user. The tie-breaking rule makes
$\truth\mapsto v^\star_t(\truth)$ measurable, so $\mathcal C_t$ is a measurable set.}

\smrev{Second, the AI output and its derivative are jointly continuous. By \Cref{ec-mm-lemma: posterior}(i), as
recalled in Step~2.1, the maps $(\query,v)\mapsto g_v(\query)$ and $(\query,v)\mapsto g_v'(\query)$ are
continuous on $\mathbb R\times(0,\infty)$.}

\smrev{Third, the inverse is jointly measurable. For this argument only, we extend $\vartheta_v$ to all of $\mathbb
R$: we set $\vartheta_v(z)\triangleq-\infty$ for $z\le\inf g_v(\mathbb R)$ and $\vartheta_v(z)\triangleq+\infty$ for
$z\ge\sup g_v(\mathbb R)$. Since $g_v$ is continuous and strictly increasing, the equivalence}
\[
 \smrev{\vartheta_v(z)>r\iff z>g_v(r)}
\]
\smrev{holds for every $z\in\mathbb R$ and $r\in\mathbb R$, including, with this extension, for $z$ outside the range.
We apply it to the map $(\truth,z)\mapsto\vartheta_{v^\star_t(\truth)}(z)$ on $\mathcal C_t\times\mathbb R$, with
values in $[-\infty,\infty]$. Such a map is measurable as soon as, for each rational $r$, the set where it exceeds $r$
is measurable. By the equivalence, this set is}
\[
 \smrev{\bigl\{(\truth,z)\in\mathcal C_t\times\mathbb R:z>g_{v^\star_t(\truth)}(r)\bigr\},}
\]
\smrev{This set is measurable because $\truth\mapsto v^\star_t(\truth)$ is measurable (first fact) and
$(r,v)\mapsto g_v(r)$ is continuous (second fact), so $(\truth,z)\mapsto\vartheta_{v^\star_t(\truth)}(z)$ is
measurable on $\mathcal C_t\times\mathbb R$.}

\smrev{Next, the set where the output lies in the range is measurable. The set $R_t\triangleq\{(\truth,z)\in\mathcal
C_t\times\mathbb R:z\in g_{v^\star_t(\truth)}(\mathbb R)\}$ of pairs in which $z$ lies in the range of
$g_{v^\star_t(\truth)}$ is where the inverse is finite, so it is measurable.}

\smrev{Finally, we assemble the density. On $R_t$, it is the quotient
$\phi_v(\vartheta_v(z)-\truth)/g_v'(\vartheta_v(z))$ with $v=v^\star_t(\truth)$. Its numerator and strictly positive
denominator are measurable, as compositions of measurable maps with the continuous maps $(y,v)\mapsto\phi_v(y)$ and
$(\query,v)\mapsto g_v'(\query)$. On the rest of $\mathcal C_t\times\mathbb R$, the density is $0$. This proves
\Cref{eq: iter-den measurable}.}

\smrev{\emph{Step 2.5: averaging over users.} We split the population into the users in $\mathcal C_t$ and the others,
who output $0$ (Step~1.2). For every Borel set $E$, Step~2.3 gives}
\[
 \smrev{\PP(\outcome_t\in E)=\PP_{\truDist}(\opteffortAlg{t}=\infty)\,\delta_0(E)
 +\int_{\mathcal C_t}\int_Ek_{\truth,v^\star_t(\truth)}(z)\,dz\,\truDist(d\truth).}
\]
\smrev{By Tonelli's theorem for nonnegative functions, which applies by the joint measurability
\Cref{eq: iter-den measurable} of Step~2.4, the two integrals can be exchanged. Hence}
\begin{equation}\smrev{\label{eq: iter-den mixture}
 \pi_{t+1}(dz)=\PP_{\truDist}(\opteffortAlg{t}=\infty)\,\delta_0(dz)+f_{t+1}(z)\,dz,
 \qquad
 f_{t+1}(z)\triangleq\int_{\mathcal C_t}k_{\truth,v^\star_t(\truth)}(z)\,\truDist(d\truth).}
\end{equation}
\smrev{By Step~2.3 and the paragraph \emph{Reading the fractions} before Part~1, $f_{t+1}(z)$ is the expectation in
\Cref{eq:den-t}, so \Cref{eq: iter-den mixture} is \Cref{eq:den-t}, written as an identity of measures.}

\smrev{\emph{Step 2.6: the mass at the default output.} \Cref{eq: iter-den mixture} decomposes $\pi_{t+1}$; this
step reads off from it the exact mass of the single point $0$. The function $f_{t+1}$ integrates to
$\PP_{\truDist}(\opteffortAlg{t}<\infty)\le1$, the population share of users who send an informative signal. In
particular, it is finite for Lebesgue-almost every $z$. A measure with a density has no atoms, so the mass of
$\pi_{t+1}$ at the default output is exactly the first coefficient in \Cref{eq: iter-den mixture}, the population
share of users who send the uninformative signal at round $t$:}
\[
 \smrev{\pi_{t+1}(\{0\})=\PP_{\truDist}(\opteffortAlg{t}=\infty).}
\]
\smrev{Together with Step~2.1, this proves Claim~2.}

\smrev{\emph{Part 3: proof of Claim~3.} Claim~3 states that $\pi_1$ is given by \Cref{eq:den-1}. We apply
\Cref{eq: iter-den mixture} at $t=0$. Its two terms are the coefficient of $\delta_0$ and the density $f_1$, which is
the integral over the users who send an informative signal of their individual densities
$k_{\truth,v^\star_0(\truth)}$. So we compute the coefficient, then each individual density, and then substitute both.}

\smrev{\emph{The coefficient of $\delta_0$.} At round $0$, the AI's prior is $\pi_0=\truDist$, so user $\truth$ solves
Problem~\eqref{eq: loss def} of \Cref{sec:model} with $\truMean=0$, and her uniqueness is $\unique=|\truth|$. To find
which users send the uninformative signal, we use \Cref{lemma: characterization of sol}, whose threshold
$\thresholdD(\costParam)$ ignores the option to work alone. We first check that the threshold $\thresholdD$ of the
statement is this threshold:}
\[
 \smrev{\thresholdD=\thresholdD(\costParam).}
\]
\smrev{The threshold $\thresholdD$ of the statement comes from item~4 of \Cref{prop: who benefits more}, which accounts
for the option to work alone: the proof of that item sets $\thresholdD=\min\{\thresholdD(\costParam),\thresholdA\}$,
where $\thresholdA$ is the threshold of item~3 beyond which users work alone. Here $\costHumanOnly=\infty$, so no
user works alone, and the proof of item~3 sets $\thresholdA=\infty$. Hence the minimum is $\thresholdD(\costParam)$.}

\smrev{By \Cref{lemma: characterization of sol}\ref{lem1:sol}, the uninformative signal is the unique optimum when
$|\truth|<\thresholdD$, and an informative signal is the unique optimum when $|\truth|>\thresholdD$.}

\smrev{The two remaining types $\truth=\pm\thresholdD$ have population probability $0$, so their choice does not affect
$\pi_1$. We still describe them, because the tie convention of \Cref{lemma: characterization of sol} may differ there
from the tie-breaking rule fixed at round~$0$. At these two types, \Cref{lemma: characterization of sol}\ref{lem1:tie}
distinguishes two cases. If $\costParam<\truSD^2$, both the uninformative and an informative signal are optimal there;
\Cref{lemma: characterization of sol}\ref{lem1:sol} resolves this tie in favor of the uninformative signal, but the
tie-breaking rule may choose the informative one. If $\costParam\ge\truSD^2$, the uninformative signal is the unique
optimum. Either way, up to a $\truDist$-null set of users,}
\begin{equation}\smrev{\label{eq: iter-den default set}
 \{\truth:\opteffortAlg{0}=\infty\}=\{\truth:|\truth|\le\thresholdD\}.}
\end{equation}
\smrev{Since $\truth/\truSD\sim\mathcal N(0,1)$, and $\bar\Phi=1-\Phi$ is the standard normal survival function,}
\[
 \smrev{\PP_{\truDist}(\opteffortAlg{0}=\infty)=\PP_{\truDist}(|\truth|\le\thresholdD)=1-2\bar\Phi(\thresholdD/\truSD).}
\]
\smrev{This is the coefficient of $\delta_0$ in \Cref{eq:den-1}.}

\smrev{\emph{The density.} Consider a user with $|\truth|>\thresholdD$, who sends an informative signal, and write
$v=v^\star_0(\truth)=(\opteffortAlg{0})^2$. The prior $\pi_0=\mathcal N(0,\truSD^2)$ is the normal prior of
\Cref{sec:model} with mean $0$ and standard deviation $\truSD$, so the linear formula in \Cref{eq: truthAI def}
applies:}
\[
 \smrev{g_v(\query)=\frac{\truSD^2}{\truSD^2+v}\,\query=\kappa(\truth)\,\query,
 \qquad
 \kappa(\truth)=\frac{\truSD^2}{\truSD^2+(\opteffortAlg{0})^2},}
\]
\smrev{with $\kappa(\truth)$ as in the statement. Hence $g_v(\mathbb R)=\mathbb R$, $\vartheta_v(z)=z/\kappa(\truth)$
and $g_v'=\kappa(\truth)$. Substituting these into \Cref{eq: iter-den conditional density}, and using
$z/\kappa(\truth)-\truth=(z-\truth\kappa(\truth))/\kappa(\truth)$ and
$\phi_v(y)=\phi(y/\opteffortAlg{0})/\opteffortAlg{0}$, gives}
\[
 \smrev{k_{\truth,v}(z)=\frac{\phi_v\bigl(z/\kappa(\truth)-\truth\bigr)}{\kappa(\truth)}
 =\frac{1}{\kappa(\truth)\opteffortAlg{0}}\,
 \phi\!\left(\frac{z-\truth\kappa(\truth)}{\kappa(\truth)\opteffortAlg{0}}\right).}
\]

\smrev{\emph{Conclusion.} By \Cref{eq: iter-den default set}, the set $\mathcal C_0$ over which $f_1$ in
\Cref{eq: iter-den mixture} integrates is $\{|\truth|>\thresholdD\}$, up to a $\truDist$-null set. Substituting the
coefficient of $\delta_0$ and the individual densities into \Cref{eq: iter-den mixture} at $t=0$, and reading the
expectation in \Cref{eq:den-1} by the
paragraph \emph{Reading the fractions} before Part~1, gives \Cref{eq:den-1}. This proves Claim~3.}
\Halmos
\end{proof}

\subsection{Results in \Cref{sec:unique_key}.}\label{ec-sec: proofs unique key}
\rev{
\smrev{This subsection has two parts. The first proves \Cref{thm:bounded-atomless-collapse}, for populations with
bounded support. The second proves \Cref{thm:unbounded-finite-variance-no-collapse,cor:two-sided-polynomial-tails},
for populations with unbounded support. The setting and notation below are shared by the two parts. They are also
used by \Cref{ec-sec: proofs mit-efficiency}, which proves \Cref{prop: recover diversity}, the efficiency result of
\Cref{sec:mit-efficiency}; there, the interaction cost may change across rounds, and the start of
\Cref{ec-sec: proofs mit-efficiency} states this change. The plan of the proof of
\Cref{thm:bounded-atomless-collapse} comes after the setting.}

\paragraph{\smrev{Setting.}} \smrev{Throughout this subsection, as in \Cref{sec:training-loop}, we set
$\costHumanOnly=+\infty$ and $\truMean=0$, fix $\costParam>0$, and consider \Cref{alg:self-training-loop} with
$\pi_0=\truDist$. The population distribution $\truDist$ does not change over time; only the AI's prior $\pi_t$ does.
A user of type $\truth$ chooses the noise level $\qErrorSD\in[0,+\infty]$ of her signal, and we write
$v\triangleq\qErrorSD^2$ for the variance of the signal noise. A finite choice $v>0$ sends the signal
$\query=\truth+\sqrt{v}Z$, with
$Z\sim\mathcal{N}(0,1)$ independent of $\truth$, so the signal noise is $\qError\triangleq \query-\truth=\sqrt{v}Z$.
Smaller $v$ means more precise communication, and $v=\infty$ is the uninformative signal. The noiseless signal $v=0$
is excluded below, after the communication cost is defined.}

\smrev{We describe a user's choice by the variance $v$. As in the main text, however, the AI's output, the
communication cost and the loss take the standard deviation $\sqrt v$ as their argument. Some proofs, such as those
of \Cref{ec-mi-lemma: cost}(v) and \Cref{ec-mi-lemma: selector}, also use the precision $u\triangleq1/v$, with $u=0$ for
the uninformative signal.}

\emph{\smrev{The AI's output.}} \smrev{Given a prior $\pi$, a finite $v>0$ and a signal realization $\query$, the AI
outputs the posterior mean, as in \Cref{eq: truthAI def}, but under the prior $\pi$ instead of a normal prior:}
\begin{equation}\smrev{\label{ec-mi-eq: posterior mean}
\truthAIAlg{\pi}(\query,\sqrt{v})\triangleq\EE{\pi}{X\,|\,X+\sqrt{v}Z=\query}
=\frac{\int x\,e^{-(\query-x)^2/(2v)}\,\pi(dx)}{\int e^{-(\query-x)^2/(2v)}\,\pi(dx)},}
\end{equation}
\smrev{where $X\sim\pi$ is independent of $Z$. Both integrals in \Cref{ec-mi-eq: posterior mean} are finite for
every prior $\pi$, because $e^{-(\query-x)^2/(2v)}$ and $|x|e^{-(\query-x)^2/(2v)}$ are bounded in $x$. For the
uninformative signal, the AI outputs its default output. Every prior that we consider is symmetric around $0$, and
we set $\truthAIAlg{\pi}(\query,\infty)\triangleq0$, the center of symmetry; it is the mean of the prior whenever the
prior has a finite mean. For the priors of the loop, symmetry is proved in the proof of
\Cref{ec-mi-cor: support preserved} in the bounded case and in the unbounded part below.}

\emph{\smrev{The communication cost.}} \smrev{The communication cost is the mutual information of \Cref{sec:model},
computed under the fixed population distribution:}
\begin{equation}\smrev{\label{ec-mi-eq: population mi}
\costEffort(\sqrt{v})\triangleq I(\truth;\truth+\sqrt{v}Z),\qquad \truth\sim\truDist,\quad Z\sim\mathcal{N}(0,1)
\text{ independent},}
\end{equation}
\smrev{with $\costEffort(\infty)=0$. The cost of a signal with variance $v$ is a function of $v$ alone: it is the
information that this signal would carry about a type drawn from the population, whatever variances other users
choose. Since it is computed under $\truDist$, not under the AI's current prior $\pi_t$, it is the same at every
round. \Cref{sec:model} writes
the mutual information as a difference of differential entropies of $\truth$, which requires $\truth$ to have a
density with finite differential entropy. \Cref{ec-mi-eq: population mi} is the mutual information itself, which is
defined for every population,
with values in $[0,\infty]$; \Cref{ec-mi-lemma: cost}(i) below shows that it is finite for every $v>0$ when
$\truSD^2$ is finite.}

\smrev{If $\truDist=\mathcal{N}(0,\truSD^2)$, then $\costEffort(\sqrt{v})=\frac12\ln(1+\truSD^2/v)$, the closed form of
\Cref{sec:model}. For other populations with a finite second moment $\truSD^2$ (see the notation below), this closed
form is in general not the cost; \Cref{ec-mi-lemma: cost}(i) below shows that it is an upper bound.}

\smrev{The noiseless signal $v=0$ is never optimal in the settings where we use these definitions: in each of them,
it is worse than the uninformative signal, whose loss $\truth^2$ is finite. When $\truDist$ is
atomless, as in all results of \Cref{sec:unique_key}, it has infinite cost, $I(\truth;\truth)=\infty$: by the
data-processing inequality,
$I(\truth;\truth)\geq I(\truth;\truth+\sqrt{v}Z)$ for every $v>0$, and the right side tends to $\infty$ as
$v\downarrow0$ (\Cref{ec-mi-lemma: cost}(iii) in the bounded case and \Cref{ec-mm-lemma: cost}(ii) in the
unbounded case). So its loss is infinite. In
\Cref{ec-mi-prop: collapse lower-bounded costs}, where the interaction costs are bounded below by some $\ell>0$, it is
dominated
by the uninformative signal under its condition \eqref{ec-mm-eq: coercive}, with
$\costParam\geq\ell$. That proposition, in \Cref{ec-sec: proofs mit-efficiency}, is a bounded-support collapse result
that allows atoms. In both settings, atomless populations and that proposition, we restrict users to
$v\in(0,\infty]$.}

\emph{\smrev{The user's loss.}} \smrev{Under the prior $\pi$, a type-$\truth$ user who sends a signal with finite
variance $v$ has expected utility loss}
\begin{equation}\smrev{\label{ec-mi-eq: user loss}
\obj_\pi(\truth,\sqrt{v})\triangleq\costQual_\pi(\truth,\sqrt{v})+\costParam\costEffort(\sqrt{v}),
\qquad
\costQual_\pi(\truth,\sqrt{v})\triangleq\EE{}{\left(\truthAIAlg{\pi}(\truth+\sqrt{v}Z,\sqrt{v})-\truth\right)^2},}
\end{equation}
\smrev{where the expectation is over the signal noise. The first term is the fidelity error and the second the
communication cost. The uninformative signal has loss $\obj_\pi(\truth,\infty)=\truth^2$.}

\emph{\smrev{Optimal signals.}} \smrev{Given $\pi$ and $\truth$, the user chooses a variance that minimizes
\Cref{ec-mi-eq: user loss} over $v\in(0,\infty]$. We call such a variance optimal for $\truth$ under $\pi$. That an
optimal variance exists is not assumed, and neither is a measurable choice among optimal variances (see the next
paragraph). For the populations of \Cref{thm:bounded-atomless-collapse}, both are proved in
\Cref{ec-mi-lemma: selector}; for the unbounded part, in \Cref{ec-mm-lemma: wellposed}.}

\emph{\smrev{Tie-breaking.}} \smrev{At a given round, several signal variances may be optimal for the same user, so the
loop needs a rule that picks one. A} \emph{\smrev{tie-breaking rule}} \smrev{for $(\pi,\costParam)$ is any choice of an
optimal variance $v^\star(\truth)\in(0,\infty]$ for every type $\truth$ in the support of $\truDist$ that is
measurable in $\truth$ (throughout, measurable means Borel measurable) and treats $\truth$ and $-\truth$ alike, that is,
$v^\star(-\truth)=v^\star(\truth)$. Unless
stated otherwise, the results hold for every tie-breaking rule, which may change from round to round.}

\smrev{Outside the support of $\truDist$, a $\truDist$-null set that does not affect the next prior, we extend a rule by
$v^\star(\truth)\triangleq\infty$, so that it is defined, measurable and even on all of $\mathbb{R}$. Along the
loop, $v^\star_t$ denotes the rule used at round $t$, so that $v^\star_t(\truth)=(\opteffortAlg{t})^2$. When we
discuss a single round with prior $\pi$, we write $v^\star_\pi$ for the rule used in that round; it also depends on
$\costParam$ and on which rule is chosen, which the notation leaves implicit.}

\emph{\smrev{The self-training update.}} \smrev{Given the prior $\pi$, the output of type $\truth$ is}
\begin{equation}\smrev{\label{ec-mi-eq: type output}
Y_\pi^{\truth}\triangleq
\begin{cases}
\truthAIAlg{\pi}\left(\truth+\sqrt{v_\pi^\star(\truth)}Z,\sqrt{v_\pi^\star(\truth)}\right),
& v_\pi^\star(\truth)<\infty,\\[1mm]
0, & v_\pi^\star(\truth)=\infty,
\end{cases}}
\end{equation}
\smrev{The output depends on the rule $v^\star_\pi$, which the notation leaves implicit. The next prior is the
distribution of $Y_\pi^{\truth}$ when $\truth\sim\truDist$. In
\Cref{alg:self-training-loop}, $\outcome_t=Y_{\pi_t}^{\truth}$ with the rule $v^\star_t$: starting from
$\pi_0=\truDist$, $\pi_{t+1}$ is the distribution of $\outcome_t$, $t=0,1,2,\ldots$.}

\smrev{The loop collapses in finite time if $\pi_T=\delta_0$ for some finite $T$. The point mass $\delta_0$ at the
default output is absorbing: under the prior $\delta_0$, every finite signal has posterior mean $0$, and the
uninformative signal also produces $0$, so the next prior is $\delta_0$ again.}

\emph{\smrev{Notation for distributions.}} \smrev{For a distribution $\pi$ on $\mathbb{R}$, we write $H_\pi$ for the
closed convex hull of its support and $J_\pi$ for the interior of $H_\pi$, and}
\[
\smrev{m_2(\pi)\triangleq\int x^2\,\pi(dx)}
\]
\smrev{for its second moment. We use $H_\pi$ and $J_\pi$ to describe the range of the posterior mean
(\Cref{ec-mi-lemma: range}) and the support of the outputs (\Cref{ec-mi-lemma: structure} and
\Cref{ec-mi-cor: support preserved}). For the population, we write $\truSD^2\triangleq\EE{}{\truth^2}$ and
$m_1\triangleq\EE{}{|\truth|}$, $\truth\sim\truDist$, whenever they are finite. Since $\truMean=0$, $\truSD^2$ is the
population variance, as in \Cref{sec:model}.}

\emph{\smrev{The assumptions of \Cref{thm:bounded-atomless-collapse}.}} \smrev{Under these assumptions,
$\truDist$ is symmetric and $H_{\truDist}=[-B,B]$ for some $B\in(0,\infty)$. Atomless means that
$\truDist(\{x\})=0$ for every $x\in\mathbb{R}$: no single preference value has positive mass. This is the sense in
which preferences are continuous. It does not require a density, so \Cref{thm:bounded-atomless-collapse} applies to
every symmetric, atomless distribution supported on a bounded interval, including singular continuous ones.}

\smrev{Atomlessness is used only through the communication cost, mainly to make it diverge as the noise vanishes
(\Cref{ec-mi-lemma: cost}(iii)); it also enters item (iv) of that lemma. In a population with finitely many
types, exact revelation has a finite cost, and the argument below, which relies on this divergence, does not apply.}

\emph{\smrev{Existence of tie-breaking rules.}} \smrev{For the populations of \Cref{thm:bounded-atomless-collapse}, a
tie-breaking rule exists at every round. We show this by induction over rounds. Two results are combined, and
neither needs a rule at the round being constructed. \Cref{ec-mi-lemma: selector} below needs only the current prior: it constructs a
rule on $[-B,B]$, extended by $\infty$ outside as above, for every symmetric prior supported on $[-B,B]$. The proof
of \Cref{ec-mi-cor: support preserved} needs only the rules of earlier rounds: its induction step at round $t$ shows
that $\pi_t$ is again symmetric and supported on $[-B,B]$, including the case $\pi_t=\delta_0$, using only the
rules of rounds before $t$.}

\smrev{The prior $\pi_0=\truDist$ is symmetric and supported on $[-B,B]$, so \Cref{ec-mi-lemma: selector} gives a rule
at round $0$. Suppose that rules have been fixed up to round $t-1$. Then $\pi_t$ is symmetric and supported on
$[-B,B]$, by \Cref{ec-mi-cor: support preserved} and its proof. So \Cref{ec-mi-lemma: selector} gives a rule at
round $t$.}

\paragraph{\smrev{Plan of the proof.}} \smrev{The proof of \Cref{thm:bounded-atomless-collapse} rests on five facts,
(a)--(e), proved in this order in the lemmas and the corollary below. A sixth step, (f), is the proof of the theorem
itself: it combines (a)--(e) in an argument by contradiction. Each constant is introduced with what it depends on.
The constants of (b)--(e) do not depend on the prior, the type, the round or the tie-breaking rule; this uniformity
is what lets (f) apply (d) and (e) at every round.}
\begin{enumerate}[label=\textup{(\alph*)}]
\item \smrev{Until the loop collapses, every prior is symmetric, and the convex hull of its support is the interval
$[-B,B]$ of the population. From round $1$ on, it also has a positive density on $(-B,B)$, plus possibly a mass at
$0$. This is \Cref{ec-mi-cor: support preserved} and its proof. So the supports of the priors do not shrink, and the
proof must measure concentration in another way: through the second moment in (c), and through limit points in (f).}

\smrev{The proof of \Cref{ec-mi-cor: support preserved} uses three facts about posterior means under Gaussian
noise. The derivative of the posterior mean in the signal is a posterior variance divided by $v$
(\Cref{ec-mi-lemma: derivative}). Under a symmetric, nondegenerate prior with bounded support, the posterior mean maps
$\mathbb{R}$ onto the interior of the convex hull of the prior's support (\Cref{ec-mi-lemma: range}). The next prior
is either $\delta_0$ or a mass at $0$ plus a positive density (\Cref{ec-mi-lemma: structure}).}
\item \smrev{A signal with finite noise variance can save a user at most the loss $\truth^2\leq B^2$ of the
uninformative signal, while revealing a type drawn from an atomless population becomes infinitely costly as the noise
vanishes (\Cref{ec-mi-lemma: cost}(iii)). Hence every finite optimal noise variance is at least some $v_{\min}>0$,
which depends only on $(\truDist,\costParam,B)$ (\Cref{ec-mi-lemma: vmin}).}

\smrev{With a second ingredient, this bound also gives existence. By the bound, only the precisions $u=1/v$ in the
compact interval $[0,1/v_{\min}]$, where $u=0$ is the uninformative signal, can be optimal. The second ingredient is
that the loss is continuous in the precision on this interval; \Cref{ec-mi-lemma: selector} proves it, including at
$u=0$. So optimal signals exist. The smallest optimal precision $u$ gives the variance $v=1/u$, with $v=\infty$ when
$u=0$, and this is a measurable and even choice of an optimal variance for each type, that is, a tie-breaking rule
(\Cref{ec-mi-lemma: selector}).}
\item \smrev{There is $\varepsilon_0>0$, depending only on $(\truDist,\costParam,B)$, such that every symmetric prior
on $[-B,B]$ with a second moment below $\varepsilon_0$ makes every user strictly prefer the uninformative signal, so
the next prior is $\delta_0$ (\Cref{ec-mi-lemma: basin}). So, to show that the loop collapses in finite time, it
suffices to show that it reaches a prior with a second moment below $\varepsilon_0$.}
\item \smrev{If the loop never collapses, then at every round a mass at least $\alpha>0$ of users send informative
signals with noise variance at most some $\bar v$, where $\bar v$ and $\alpha$ depend only on
$(\truDist,\costParam,B)$ (\Cref{ec-mi-lemma: rescue}). Otherwise, at some round $t$, the next prior $\pi_{t+1}$ would
have a second moment below $\varepsilon_0$, and (c) would give $\pi_{t+2}=\delta_0$.}

\smrev{These users put mass near the center. Fix $c>0$. By \Cref{ec-mi-lemma: contraction}, a linear bound on the
posterior mean, the output of a user with a noise variance in $[v_{\min},\bar v]$ lies in $[-c,c]$ whenever her signal
is close enough to $0$. This event has a probability bounded below uniformly over types, variances and rounds, by
continuity and compactness. Hence every prior $\pi_t$, $t\geq1$, has a mass at least $\delta_c>0$ in $[-c,c]$, where
$\delta_c$ depends on $c$ and on $(\truDist,\costParam,B)$ (\Cref{ec-mi-lemma: central mass}).}
\item \smrev{For a symmetric prior $\rho$ on $[-B,B]$ and a cutoff $a\in(0,B)$, let}
\[
\smrev{\Lambda_a(\rho)\triangleq\int(x-a)_+\,\rho(dx)}
\]
\smrev{be the expected excess of the prior beyond $a$, where $(\cdot)_+$ and $[\cdot]_+$ denote the positive part.
Fact (e) says that, under a prior that has mass near the center and a small expected excess beyond $a$, an output
above $a$ is unlikely, with a probability that tends to $0$ faster than any power of the expected excess.}

\smrev{Precisely, let $\rho$ have a mass at least $\delta$ in $[-c,c]$, where $0<c<a$. There are constants
$\gamma>0$ and $C_0$, which depend only on $(B,c,v_{\min})$ and not on $a$, $\delta$ or $\rho$, such that a type
that uses a finite variance $v\geq v_{\min}$ produces an output above $a$ with probability at most
$\exp\{-\gamma[\ln((a-c)\delta/\Lambda_a(\rho))-C_0]_+^2\}$, and with probability $0$ when $\Lambda_a(\rho)=0$
(\Cref{ec-mi-lemma: edge}). For fixed $c$, $\delta$ and $a$, this bound tends to $0$ faster than any power of
$\Lambda_a(\rho)$ as $\Lambda_a(\rho)\to0$. This is the property that (f) uses, with $\delta=\delta_c$ from (d).}
\item \smrev{The proof of the theorem assumes that the loop never collapses and derives a contradiction. By (a), the
supports of the priors do not shrink. The proof therefore looks at the limit points of the priors, that is, the weak
limits of subsequences $(\pi_{t_n})$ with $t_n\to\infty$, which may be supported on a smaller interval. Let $R$ be
the smallest radius such that every limit point is supported on $[-R,R]$; it is one radius, common to all limit
points. The quantities of (f), unlike the constants of (b)--(e), depend on the trajectory of the loop: $R$, and with
it $c=R/2$, $\delta=\delta_c$, the constants $\gamma$ and $C_0$ of (e) for this $c$, and $K$ below. None of them
depends on the parameter $r$ below, which is what lets (f) take $r$ small.}

\smrev{\emph{Case $R=0$.} Every limit point is then $\delta_0$, so the whole sequence of priors converges to
$\delta_0$. Since all priors are supported on $[-B,B]$, where $x\mapsto x^2$ is bounded and continuous, their second
moments eventually fall below $\varepsilon_0$, and (c) makes the loop collapse.}

\smrev{\emph{Case $R>0$.} Set $c\triangleq R/2$ and $\delta\triangleq\delta_c$ from (d). For a small $r\in(0,R/4)$,
use the cutoff $a_r\triangleq R-r$, and let $M_r$ be the largest expected excess
$\Lambda_{a_r}(\pi)=\int(x-R+r)_+\,\pi(dx)$ beyond $R-r$ among the limit points $\pi$; the largest value is attained,
by compactness. We show that $M_r$ must be positive but cannot be positive when $r$ is small. First, $M_r>0$ by the
definition of $R$ and symmetry; and $M_r\leq r$, because the limit points are supported on $[-R,R]$.}

\smrev{Take rounds $t_n\to\infty$ along which $\pi_{t_n}$ converges to a limit point that attains $M_r$. Passing to a
further subsequence, the previous priors $\pi_{t_n-1}$ also converge, necessarily to a limit point. Apply (e) to the
round from $\pi_{t_n-1}$ to $\pi_{t_n}$, with the central mass from (d). To turn a probability into an expected
excess, cap the excess of each output at $r$, that is, use $\chi_r(y)\triangleq\min\{(y-a_r)_+,r\}$. Then each
output above $a_r$ contributes at most $r$, so $\int\chi_r\,d\pi_{t_n}$ is at most $r$ times the bound of (e),
evaluated at the expected excess $\Lambda_{a_r}(\pi_{t_n-1})$ of the previous prior. The cap is needed because
$\pi_{t_n}$ has support $[-B,B]$, where the excess beyond $a_r$ can exceed $r$ when $R<B$.}

\smrev{Now pass to the limit. On the left, the cap does not bind on the support $[-R,R]$ of the limit point, so
$\int\chi_r\,d\pi_{t_n}$ tends to $M_r$. On the right, the bound of (e) is continuous and nondecreasing in the
expected excess, and the previous prior's expected excess tends to that of a limit point, which is at most $M_r$.
Replacing $(a_r-c)\delta$ by its lower bound $K\triangleq R\delta/4$, which does not depend on $r$, can only increase
the bound. This gives}
\[
\smrev{M_r\leq r\exp\left\{-\gamma\left[\ln\frac{K}{M_r}-C_0\right]_+^2\right\}.}
\]

\smrev{The right side is $r$ times a factor that tends to $0$ faster than any power of $M_r$, so the inequality fails
for small $r$. Precisely, since $\gamma x^2-x\geq-1/(4\gamma)$ for every real $x$, no $M_r\in(0,r]$ satisfies it once
$r<Ke^{-C_0-1/(4\gamma)}$; Step 6 of the proof writes out this calculation. This contradiction rules out $R>0$.}
\end{enumerate}
\smrev{The paragraph ``Where the communication cost enters'', after the proof of \Cref{thm:bounded-atomless-collapse},
lists the properties of the communication cost that the argument uses, namely the five items of
\Cref{ec-mi-lemma: cost}, and where each of them is used.}

\paragraph{\smrev{Posterior means and the next prior.}}
\smrev{\Cref{ec-mi-lemma: derivative,ec-mi-lemma: range,ec-mi-lemma: structure} and
\Cref{ec-mi-cor: support preserved} prove fact (a).
\Cref{ec-mi-lemma: contraction}, a linear bound on the posterior mean, is not needed for (a). It is stated here
because it is also a fact about posterior means, and it is used later: in \Cref{ec-mi-lemma: selector}, to show
that the loss of a very noisy signal is close to that of the uninformative signal, and in (d). These results are
stated for priors with bounded support, because \Cref{thm:bounded-atomless-collapse} is a bounded-support result.}

\begin{lemma}[Derivative of the posterior mean]\label{ec-mi-lemma: derivative}
\smrev{Let $\pi$ be a distribution supported on $[-B,B]$, fix $v\in(0,\infty)$, and let
$g_v(\query)\triangleq\truthAIAlg{\pi}(\query,\sqrt{v})$. Then $g_v$ is continuously differentiable and}
\begin{equation}\smrev{\label{ec-mi-eq: posterior derivative}
g_v'(\query)=\frac1v\,\VV_\pi\left(X\,|\,X+\sqrt{v}Z=\query\right).}
\end{equation}
\smrev{Moreover, the maps $(\query,v)\mapsto g_v(\query)$} and \smrev{$(\query,v)\mapsto g_v'(\query)$ are continuous on
$\mathbb{R}\times(0,\infty)$. Consequently, $0\leq g_v'(\query)\leq B^2/v$ for all $\query$. If $\pi$ is nondegenerate,
then $g_v'(\query)>0$ for all $\query$. If $\pi$ is symmetric, then $g_v$ is odd and $g_v(0)=0$.}
\end{lemma}

\begin{proof}{Proof of \Cref{ec-mi-lemma: derivative}.}
\smrev{We write the posterior mean as a ratio of two integrals, differentiate under the integral sign, and recognize
the derivative as a posterior variance. The remaining claims are then proved one at a time: joint continuity from
the integrals, the bounds and strict positivity from the derivative formula, and oddness from a change of variables
in the integrals.}

\smrev{\emph{Weighted moments.} For $k=0,1,2$, let $N_k$ be the $k$-th moment of $\pi$ weighted by the Gaussian
likelihood $e^{-(\query-x)^2/(2v)}$ of the signal, from which we remove the factor $e^{-\query^2/(2v)}$, which does
not depend on $x$:}
\begin{equation}\smrev{\label{ec-mi-eq: tilted moments}
N_k(\query)\triangleq\int x^k\exp\left\{\frac{\query x}{v}-\frac{x^2}{2v}\right\}\pi(dx).}
\end{equation}
\smrev{They also depend on $v$ and $\pi$, which we fix and omit from the notation. Expanding
$(\query-x)^2=\query^2-2\query x+x^2$ in \Cref{ec-mi-eq: posterior mean}, the factor $e^{-\query^2/(2v)}$ cancels from
the numerator and the denominator, so $g_v(\query)=N_1(\query)/N_0(\query)$. For the same reason, the posterior
distribution of $X$ given $X+\sqrt vZ=\query$ is proportional to $\exp\{\query x/v-x^2/(2v)\}\pi(dx)$, and}
\[
\smrev{\frac{N_k(\query)}{N_0(\query)}=\EE{\pi}{X^k\,|\,X+\sqrt{v}Z=\query},\qquad k=0,1,2.}
\]

\smrev{\emph{Differentiation and continuity.} Since $\pi$ is supported on $[-B,B]$, the integrand of $N_k$ and its
derivative with respect to $\query$ are bounded, uniformly over $x\in[-B,B]$ and over $(\query,v)$ in any set where
$\query$ is bounded and $v$ lies in a compact subinterval of $(0,\infty)$. Dominated convergence therefore justifies
differentiation under the integral sign and shows that $N_0$, $N_1$, $N_2$ and the $\query$-derivatives of $N_0$ and
$N_1$
are jointly continuous in $(\query,v)$. Differentiating the exponent gives}
\[
\smrev{N_0'(\query)=\frac1v N_1(\query),
\qquad
N_1'(\query)=\frac1v N_2(\query).}
\]

\smrev{\emph{The derivative formula.} By the quotient rule and the two previous displays,}
\[
\smrev{\begin{aligned}
g_v'(\query)&=\frac{N_1'(\query)N_0(\query)-N_1(\query)N_0'(\query)}{N_0(\query)^2}
=\frac1v\left(\frac{N_2(\query)}{N_0(\query)}-\frac{N_1(\query)^2}{N_0(\query)^2}\right)\\
&=\frac1v\left(\EE{\pi}{X^2\,|\,X+\sqrt{v}Z=\query}-\EE{\pi}{X\,|\,X+\sqrt{v}Z=\query}^2\right),
\end{aligned}}
\]
\smrev{which is \Cref{ec-mi-eq: posterior derivative}.}

\smrev{\emph{Joint continuity.} Since $N_0(\query)>0$, the formulas $g_v=N_1/N_0$ and
$g_v'=(N_2N_0-N_1^2)/(vN_0^2)$ show that $g_v(\query)$ and $g_v'(\query)$ are jointly continuous in $(\query,v)$.}

\smrev{\emph{Bounds on the derivative.} The posterior is supported on $[-B,B]$, so its variance is nonnegative and at
most its second moment, which is at most $B^2$. By \Cref{ec-mi-eq: posterior derivative}, $0\leq g_v'\leq B^2/v$. If
$\pi$ is nondegenerate, the posterior is also nondegenerate, because it has a strictly positive density with respect
to $\pi$, so it has the same null sets as $\pi$. Its variance is then strictly positive, and $g_v'>0$.}

\smrev{\emph{Symmetry.} If $\pi$ is symmetric, replacing $x$ by $-x$ in \Cref{ec-mi-eq: tilted moments} shows that
$N_0(-\query)=N_0(\query)$ and $N_1(-\query)=-N_1(\query)$. Hence $g_v(-\query)=-g_v(\query)$, and in particular
$g_v(0)=0$.}
\Halmos
\end{proof}

\begin{lemma}[Range of the posterior mean]\label{ec-mi-lemma: range}
\smrev{Let $\pi$ be symmetric and nondegenerate,} with \smrev{$H_\pi=[-b,b]$ for some $0<b<\infty$. For every
$v\in(0,\infty)$, the map $g_v(\query)=\truthAIAlg{\pi}(\query,\sqrt{v})$ is a strictly increasing continuous bijection
from $\mathbb{R}$ onto $J_\pi=(-b,b)$.}
\end{lemma}

\begin{proof}{Proof of \Cref{ec-mi-lemma: range}.}
\smrev{We show that $g_v$ is continuous and strictly increasing, that it takes values in $(-b,b)$, and that it tends to
$b$ as $\query\to\infty$ and to $-b$ as $\query\to-\infty$.}

\smrev{\emph{Monotonicity and values.} \Cref{ec-mi-lemma: derivative}, applied with $b$ in place of $B$, shows that
$g_v$ is continuous with $g_v'>0$, hence strictly increasing. Write $\pi(\cdot\,|\,\query)$ for the posterior
distribution of $X$ given $X+\sqrt{v}Z=\query$. It is supported on $[-b,b]$, so its mean $g_v(\query)$ lies in $[-b,b]$.
If $g_v(\query)$ were equal to $b$, the
posterior would be a distribution on $[-b,b]$ with mean $b$, hence the point mass at $b$. This is impossible, because
the posterior has the same null sets as $\pi$ (block ``Bounds on the derivative'' of the proof of
\Cref{ec-mi-lemma: derivative}) and $\pi$ is nondegenerate.
The same argument excludes $-b$. Hence $g_v(\query)\in(-b,b)$ for all $\query$.}

\smrev{\emph{The limit at $+\infty$.} We show that $\liminf_{\query\to\infty}g_v(\query)\geq a$ for every
$a\in(-b,b)$. Fix such an $a$. Since $H_\pi=[-b,b]$, the endpoint $b$ is in the support of $\pi$, so there exists
$a'\in(a,b)$ such that $\pi([a',b])>0$. Write $e_{\query}(x)\triangleq \query x/v-x^2/(2v)$ for the exponent in
\Cref{ec-mi-eq: tilted moments}. For $\query\geq0$, $x\in[-b,a]$ and $y\in[a',b]$,}
\[
\smrev{e_{\query}(y)-e_{\query}(x)=\frac{\query(y-x)}{v}-\frac{y^2-x^2}{2v}
\geq\frac{\query(a'-a)}{v}-\frac{b^2}{2v},}
\]
\smrev{because $y-x\geq a'-a$ and $y^2-x^2\leq y^2\leq b^2$.}

\smrev{We now show that the posterior mass of $(-\infty,a]$, including any atom at $a$, vanishes as
$\query\to\infty$. The posterior is proportional to
$e^{e_{\query}(x)}\pi(dx)$, and $\pi$ puts no mass outside $[-b,b]$. In a ratio of two posterior masses, the
normalizing constant cancels, so}
\[
\smrev{\pi\left((-\infty,a]\,|\,\query\right)\leq
\frac{\pi\left((-\infty,a]\,|\,\query\right)}{\pi\left([a',b]\,|\,\query\right)}
=\frac{\int_{[-b,a]}e^{e_{\query}(x)}\,\pi(dx)}{\int_{[a',b]}e^{e_{\query}(y)}\,\pi(dy)},}
\]
\smrev{where the inequality holds because a posterior mass is at most $1$. Exponentiating the bound of the previous
display gives $e^{e_{\query}(x)}\leq e^{e_{\query}(y)}\exp\{-\query(a'-a)/v+b^2/(2v)\}$ for $\query\geq0$,
$x\in[-b,a]$ and $y\in[a',b]$. Integrating it against $\pi(dx)$ over $x\in[-b,a]$ and against $\pi(dy)$ over
$y\in[a',b]$, and using $\pi([-b,a])\leq1$, bounds the ratio:}
\[
\smrev{\frac{\int_{[-b,a]}e^{e_{\query}(x)}\,\pi(dx)}{\int_{[a',b]}e^{e_{\query}(y)}\,\pi(dy)}
\leq\frac{e^{b^2/(2v)}}{\pi([a',b])}\exp\left\{-\frac{\query(a'-a)}{v}\right\}
\longrightarrow0\qquad(\query\to\infty).}
\]

\smrev{We now bound the posterior mean. The posterior puts mass $\pi((-\infty,a]\,|\,\query)$ on $[-b,a]$, where
$x\geq-b$, and the rest on $(a,b]$, where $x>a$. Hence $g_v(\query)\geq a-(a+b)\,\pi((-\infty,a]\,|\,\query)$, and the
previous display gives $\liminf_{\query\to\infty}g_v(\query)\geq a$.}

\smrev{\emph{Conclusion.} Letting $a\uparrow b$ and using $g_v<b$ gives $\lim_{\query\to\infty}g_v(\query)=b$. Since
$g_v$ is odd (\Cref{ec-mi-lemma: derivative}), $\lim_{\query\to-\infty}g_v(\query)=-b$. A continuous, strictly
increasing function with these limits is a bijection from $\mathbb{R}$ onto $(-b,b)$, by the intermediate value
theorem.}
\Halmos
\end{proof}

\begin{lemma}[Linear bound under bounded support]\label{ec-mi-lemma: contraction}
\smrev{If $\pi$ is symmetric and supported on $[-B,B]$, then for every $v\in(0,\infty)$ and every
$\query\in\mathbb{R}$,}
\begin{equation}\smrev{\label{ec-mi-eq: contraction pointwise}
\left|\truthAIAlg{\pi}(\query,\sqrt{v})\right|\leq\frac{B^2}{v}|\query|.}
\end{equation}
\smrev{Consequently, if $\query=\truth+\sqrt{v}Z$ for a fixed $\truth$, then}
\begin{equation}\smrev{\label{ec-mi-eq: contraction second moment}
\EE{}{\truthAIAlg{\pi}(\query,\sqrt{v})^2}\leq\frac{B^4}{v^2}\left(\truth^2+v\right).}
\end{equation}
\end{lemma}

\begin{proof}{Proof of \Cref{ec-mi-lemma: contraction}.}
\smrev{Write $g_v(\query)=\truthAIAlg{\pi}(\query,\sqrt{v})$. By \Cref{ec-mi-lemma: derivative}, which applies to
every prior supported on $[-B,B]$, including $\delta_0$, we have $0\leq g_v'\leq B^2/v$, and
$g_v(0)=0$ because $\pi$ is symmetric. Hence}
\[
\smrev{|g_v(\query)|=\left|\int_0^{\query}g_v'(y)\,dy\right|\leq\frac{B^2}{v}|\query|,}
\]
\smrev{which is \Cref{ec-mi-eq: contraction pointwise}. Squaring, taking expectations and using
$\EE{}{(\truth+\sqrt{v}Z)^2}=\truth^2+v$ gives \Cref{ec-mi-eq: contraction second moment}.}
\Halmos
\end{proof}

\begin{lemma}[The next prior: $\delta_0$, or a mass at the default output plus a positive density]
\label{ec-mi-lemma: structure}
\smrev{Let $\pi$ be symmetric and nondegenerate,} with \smrev{$H_\pi=[-b,b]$ for some $0<b<\infty$. Let
$\hat v:\mathbb{R}\rightarrow(0,\infty]$ be any measurable variance rule, not necessarily optimal. For type $\truth$,
define the output}
\[
\smrev{Y^{\truth}_\pi\triangleq
\begin{cases}
\truthAIAlg{\pi}\left(\truth+\sqrt{\hat v(\truth)}Z,\sqrt{\hat v(\truth)}\right), & \hat v(\truth)<\infty,\\
0, & \hat v(\truth)=\infty,
\end{cases}}
\]
\smrev{which is \Cref{ec-mi-eq: type output} with $v^\star_\pi$ replaced by $\hat v$, and let $\pi_+$ be the
distribution of $Y^{\truth}_\pi$ when $\truth\sim\truDist$. If $\truDist(\{\truth:\hat v(\truth)<\infty\})=0$, then
$\pi_+=\delta_0$. Otherwise,}
\begin{equation}\smrev{\label{ec-mi-eq: one-step decomposition}
\pi_+=p_+\,\delta_0+f_+(z)\,dz,
\qquad
p_+=\truDist(\{\truth:\hat v(\truth)=\infty\}),}
\end{equation}
\smrev{where $f_+$ is a Borel function that gives the contribution of the users who send an informative signal: for
every Borel set $E$, $\int_Ef_+(z)\,dz$ is the probability that a user sends an informative signal and has an output
in $E$. So $f_+$ is not normalized: $\int_{\mathbb{R}}f_+(z)\,dz=1-p_+$. Moreover, $f_+(z)>0$ for every
$z\in J_\pi=(-b,b)$. The mass $p_+$ at $0$ may be zero.}
\end{lemma}

\begin{proof}{Proof of \Cref{ec-mi-lemma: structure}.}
\smrev{Users who send the uninformative signal output $0$; they give the mass $p_+$ at $0$. The proof deals with the
other users in four steps. We compute the density of the output of each of them, check that these densities can be
averaged over users, show that the average is a finite density, and show that it is positive on $J_\pi$.}

\smrev{\emph{Users who send an informative signal.} Let
$\mathcal{C}\triangleq\{\truth:\hat v(\truth)<\infty\}$. If $\truDist(\mathcal{C})=0$, almost every user sends the
uninformative signal and outputs $0$, so $\pi_+=\delta_0$. Assume from now on that $\truDist(\mathcal{C})>0$; note
that $p_+=\truDist(\mathcal{C}^c)=1-\truDist(\mathcal{C})$.}

\smrev{\emph{The output density of one user.} Fix $\truth\in\mathcal{C}$ and put $v=\hat v(\truth)$. Write
$g_v(\query)\triangleq\truthAIAlg{\pi}(\query,\sqrt{v})$. By \Cref{ec-mi-lemma: range}, $g_v$ is a strictly
increasing continuous bijection from $\mathbb{R}$ onto $J_\pi$, and by \Cref{ec-mi-lemma: derivative} it is
continuously differentiable with $g_v'>0$. Let $\vartheta_v:J_\pi\rightarrow\mathbb{R}$ be its inverse. Given type
$\truth$ and variance $v$, the signal $\query=\truth+\sqrt{v}Z$ has density $\phi_v(\query-\truth)$, where
$\phi_v(x)\triangleq v^{-1/2}\phi(x/\sqrt{v})$ and $\phi$ is the standard normal density. The change of variables
$z=g_v(\query)$ shows that the output $g_v(\query)$ has density}
\begin{equation}\smrev{\label{ec-mi-eq: conditional density}
k_{\truth,v}(z)\triangleq\frac{\phi_v(\vartheta_v(z)-\truth)}{g_v'(\vartheta_v(z))},\qquad z\in J_\pi,}
\end{equation}
\smrev{and we set $k_{\truth,v}(z)\triangleq0$ for $z\notin J_\pi$.}

\smrev{\emph{Measurability.} To average these densities over users, we need $(\truth,v,z)\mapsto k_{\truth,v}(z)$ to
be Borel. We first show that the inverse $(v,z)\mapsto\vartheta_v(z)$ is Borel, and then deduce it for $k$.}

\smrev{\emph{The inverse $\vartheta$ is Borel.} Since $g_v$ is continuous and strictly increasing,
$\{r:g_v(r)>z\}=(\vartheta_v(z),\infty)$ for $z\in J_\pi$. Hence $\vartheta_v(z)$ is the infimum of the rational
numbers $r$ with $g_v(r)>z$:}
\[
\smrev{\vartheta_v(z)=\inf_{r\in\mathbb{Q}}h_r(v,z),
\qquad\text{where}\qquad
h_r(v,z)\triangleq
\begin{cases}
r, & g_v(r)>z,\\
+\infty, & \text{otherwise}.
\end{cases}}
\]
\smrev{By \Cref{ec-mi-lemma: derivative}, $(r,v)\mapsto g_v(r)$ is continuous, so for each rational $r$ the set
$\{(v,z):g_v(r)>z\}$ is open and $h_r$ is Borel. Hence $(v,z)\mapsto\vartheta_v(z)$, an infimum of countably many
Borel functions, is Borel.}

\smrev{\emph{The density $k$ is Borel.} Also by \Cref{ec-mi-lemma: derivative}, $(\query,v)\mapsto g_v'(\query)$ is
continuous and strictly positive, and $(x,v)\mapsto\phi_v(x)$ is continuous. By \Cref{ec-mi-eq: conditional density},
$(\truth,v,z)\mapsto k_{\truth,v}(z)$ is therefore Borel on $\mathbb{R}\times(0,\infty)\times\mathbb{R}$. Since
$\hat v$ is measurable, so is $(\truth,z)\mapsto k_{\truth,\hat v(\truth)}(z)$ on $\mathcal{C}\times\mathbb{R}$.}

\smrev{\emph{Averaging over users.} For every Borel set $E\subseteq\mathbb{R}$, Tonelli's theorem gives}
\[
\smrev{\PP\left(Y^{\truth}_\pi\in E,\ \truth\in\mathcal{C}\right)
=\int_{\mathcal{C}}\int_E k_{\truth,\hat v(\truth)}(z)\,dz\,\truDist(d\truth)
=\int_E\widetilde f(z)\,dz,
\qquad\text{where}\qquad
\widetilde f(z)\triangleq\int_{\mathcal{C}} k_{\truth,\hat v(\truth)}(z)\,\truDist(d\truth).}
\]
\smrev{Thus $\widetilde f$ is a nonnegative Borel function whose integral over $E$ is the joint probability that a
user is in $\mathcal{C}$ and has an output in $E$. Taking $E=\mathbb{R}$ gives
$\int_{\mathbb{R}}\widetilde f(z)\,dz=\truDist(\mathcal{C})=1-p_+$, so $\widetilde f<\infty$ for Lebesgue-almost
every $z$. On the null set where $\widetilde f=\infty$, redefine it to be $1$. The resulting Borel
function, denoted $f_+$, is a finite version of the same density.}

\smrev{\emph{Positivity on $J_\pi$.} Fix $z\in J_\pi$. For every $\truth\in\mathcal{C}$, the density
$k_{\truth,\hat v(\truth)}(z)$ in \Cref{ec-mi-eq: conditional density} is strictly positive, because the Gaussian
density is strictly positive and the denominator $g_v'$ is positive and finite. Since $\truDist(\mathcal{C})>0$, the
integral defining
$\widetilde f(z)$ is strictly positive, possibly $+\infty$. Hence $f_+(z)>0$ for every $z\in J_\pi$.}

\smrev{\emph{Conclusion.} The part of $\pi_+$ that comes from users in $\mathcal{C}$ has the density $f_+$, so it
puts no mass at $0$ or at any other single point. Users outside $\mathcal{C}$ all output $0$. Therefore $\pi_+$ is the
sum of the mass
$\truDist(\mathcal{C}^c)=p_+$ at $0$, which is zero if every user sends an informative signal, and the part with
density $f_+$. This is \Cref{ec-mi-eq: one-step decomposition}.}
\Halmos
\end{proof}

\begin{corollary}[The support interval is preserved until collapse]\label{ec-mi-cor: support preserved}
\smrev{Let $H_{\truDist}=[-B,B]$. Along \Cref{alg:self-training-loop}, run with a tie-breaking rule at each round
(see the Setting), for every $t\geq1$, either $\pi_t=\delta_0$ or $\pi_t=p_t\,\delta_0+f_t(z)\,dz$ with $f_t(z)>0$ for
every $z\in J_{\truDist}=(-B,B)$. Here $p_t\triangleq\pi_t(\{0\})$ is the mass at the default output and $f_t$ is the
density of the rest. In particular, as long as the loop has not collapsed, $H_{\pi_t}=H_{\truDist}=[-B,B]$.}
\end{corollary}

\begin{proof}{Proof of \Cref{ec-mi-cor: support preserved}.}
\smrev{We prove by induction on $t\geq0$ that $\pi_t$ is symmetric and that either $\pi_t=\delta_0$ or
$H_{\pi_t}=[-B,B]$, with, for $t\geq1$, the decomposition of the statement. For $t=0$, $\pi_0=\truDist$ is symmetric, by
the symmetry of the population assumed in \Cref{sec:unique_key} and in \Cref{thm:bounded-atomless-collapse}, and
$H_{\pi_0}=[-B,B]$. Suppose that the claim holds for $\pi_{t-1}$. If $\pi_{t-1}=\delta_0$, then $\pi_t=\delta_0$,
because $\delta_0$ is absorbing. Assume from now on that $\pi_{t-1}$ is symmetric with $H_{\pi_{t-1}}=[-B,B]$. Since
$B>0$, $\pi_{t-1}$ is nondegenerate, and it is supported on $[-B,B]$, so
\Cref{ec-mi-lemma: derivative,ec-mi-lemma: range,ec-mi-lemma: structure} apply to it.}

\smrev{\emph{The next prior is symmetric.} The posterior mean under $\pi_{t-1}$ is odd in the signal
(\Cref{ec-mi-lemma: derivative}), and the tie-breaking rule gives $v^\star_{t-1}(-\truth)=v^\star_{t-1}(\truth)$.
Compare a draw $(\truth,Z)$ of a type and its noise with the reflected draw $(-\truth,-Z)$, and let
$v=v^\star_{t-1}(\truth)=v^\star_{t-1}(-\truth)$. If $v<\infty$, the reflected draw gives the signal
$-(\truth+\sqrt{v}Z)$ and the output $-\truthAIAlg{\pi_{t-1}}(\truth+\sqrt{v}Z,\sqrt{v})$, the negative of the
output of the original draw. If $v=\infty$, both draws give the output $0$. Since $(-\truth,-Z)$ has the same
distribution as $(\truth,Z)$, the output
$\outcome_{t-1}$ has the same distribution as $-\outcome_{t-1}$, so $\pi_t$ is symmetric.}

\smrev{\emph{The decomposition.} Apply \Cref{ec-mi-lemma: structure} with $\pi=\pi_{t-1}$, $b=B$ and
$\hat v=v^\star_{t-1}$, extended by $\infty$ outside the support of $\truDist$ as in the Setting. Either
$\pi_t=\delta_0$, or $\pi_t=p_t\,\delta_0+f_t(z)\,dz$ with $f_t>0$ on $(-B,B)$.}

\smrev{\emph{The hull.} In the second case, $f_t>0$ on $(-B,B)$ gives positive mass to every open subinterval of
$(-B,B)$, so the support of $\pi_t$ contains $[-B,B]$. Conversely, every output of an informative signal lies in
$J_{\pi_{t-1}}=(-B,B)$, by \Cref{ec-mi-lemma: range}, and every output of the uninformative signal is $0$. Hence the
support of $\pi_t$ is contained in $[-B,B]$, and $H_{\pi_t}=[-B,B]$. This completes the induction step and the
proof.}
\Halmos
\end{proof}

\paragraph{\smrev{The communication cost of a bounded atomless population.}} \smrev{In this part,
$\truth\sim\truDist$ is symmetric, atomless, and supported on $[-B,B]$, and $\truSD^2=\EE{}{\truth^2}$ and
$m_1=\EE{}{|\truth|}$ are as in the Setting. Atomlessness excludes, in particular, the point mass
$\delta_0$, so $\truth$ is not almost surely equal to $0$. Since $\truDist$ is symmetric around zero and bounded,
$0<\truSD^2\leq B^2$ and $m_1>0$; the latter makes the lower bound of
\Cref{ec-mi-lemma: cost}(iv) below strictly positive.}

\begin{lemma}[Population mutual information]\label{ec-mi-lemma: cost}
\smrev{The cost $\costEffort(\sqrt{v})=I(\truth;\truth+\sqrt{v}Z)$, $v\in(0,\infty)$, with $\costEffort(\infty)=0$, has
the following properties.}
\begin{enumerate}[label=\textup{(\roman*)}]
\item \smrev{For every finite $v>0$,}
\[
\smrev{0\leq\costEffort(\sqrt v)\leq
\frac12\ln\left(1+\frac{\truSD^2}{v}\right)<\infty,}
\]
\smrev{and $\costEffort(\sqrt v)\to0$ as $v\to\infty$.
The displayed upper bound holds for every population with finite variance.}
\item \smrev{$\costEffort(\sqrt{v})$ is nonincreasing in $v$}.
\item \smrev{$\costEffort(\sqrt{v})\rightarrow\infty$ as $v\downarrow0$.}
\item \smrev{For every finite $v$,}
\begin{equation}\smrev{\label{ec-mi-eq: cost lower bound}
\costEffort(\sqrt{v})\geq\frac{m_1^2}{2(v+\truSD^2)}.}
\end{equation}
\item \smrev{For every $\bar u<\infty$, the function $u\mapsto\costEffort(1/\sqrt{u})$ of the precision
$u\in[0,\bar u]$, with value $\costEffort(\infty)=0$ at $u=0$, is continuous.}
\end{enumerate}
\end{lemma}

\begin{proof}{Proof of \Cref{ec-mi-lemma: cost}.}
\smrev{We prove the five items in order, one block each. Items (i) and (ii) use two standard facts about
information: the normal distribution maximizes differential entropy at a given variance, and adding independent
noise cannot increase information. Item (iii) applies Fano's inequality, in the form stated there, to a
discretization of $\truth$, item (iv) bounds the information carried by the sign of $\truth$, and item (v) uses
dominated convergence. Throughout, let $\query_v\triangleq\truth+\sqrt{v}Z$ be the signal with noise variance
$v\in(0,\infty)$, let $\mathcal{H}(\cdot)$ denote differential entropy, as in \Cref{sec:model}, and let all logarithms
be natural.}

\smrev{\emph{Item (i): the Gaussian upper bound.} We show that
$\costEffort(\sqrt v)\leq\frac12\ln(1+\truSD^2/v)$, which also gives finiteness and the limit as $v\to\infty$. The
signal $\query_v$ has a density, because it is the sum of $\truth$ and an independent Gaussian variable, and its
variance is $\truSD^2+v$. Its differential entropy is finite, for two reasons. It is at least
$\mathcal{H}(\query_v\,|\,\truth)=\mathcal{H}(\sqrt{v}Z)>-\infty$, because conditioning cannot increase entropy. It is
at most the differential entropy of the normal distribution with variance $\truSD^2+v$, because the normal
distribution maximizes differential entropy at a fixed variance. Since the noise is independent of $\truth$, the
mutual information is therefore $\costEffort(\sqrt{v})=\mathcal{H}(\query_v)-\mathcal{H}(\sqrt{v}Z)$, and the same
maximum-entropy bound gives}
\[
\smrev{\costEffort(\sqrt{v})\leq\frac12\ln\left(2\pi e(\truSD^2+v)\right)-\frac12\ln(2\pi e v)
=\frac12\ln\left(1+\frac{\truSD^2}{v}\right).}
\]
\smrev{Mutual information is nonnegative, and the right side tends to $0$ as $v\to\infty$, so
$\costEffort(\sqrt{v})\to0$. The argument uses only that $\truSD^2$ is finite. This proves item (i).}

\smrev{\emph{Item (ii): monotonicity.} We show that $\costEffort(\sqrt{v_2})\leq\costEffort(\sqrt{v_1})$ for
$v_2>v_1$. Let $Z_1$ and $Z_2$ be independent standard normals, independent of $\truth$. Then $\query_{v_2}$ has the
same joint distribution with $\truth$ as $\truth+\sqrt{v_1}Z_1+\sqrt{v_2-v_1}Z_2$. So the signal with variance $v_2$
is a garbling of the signal with variance $v_1$: it is obtained from it by adding independent noise. By the
data-processing inequality, $\costEffort(\sqrt{v_2})\leq\costEffort(\sqrt{v_1})$. This proves item (ii).}

\smrev{\emph{Item (iii): divergence as $v\downarrow0$.} We show that, for every integer $N\geq2$,
$\liminf_{v\downarrow0}\costEffort(\sqrt{v})\geq\ln N$; since $N$ is arbitrary, this gives item (iii). Fix
$N\geq2$. Because $\truDist$ is atomless, its distribution function is continuous, so there exist thresholds
$-B=x_0<x_1<\cdots<x_N=B$ such that $\truDist((x_{i-1},x_i])=1/N$ for $i=1,\ldots,N$; how the endpoints are assigned
does not matter, since they have probability zero. Let $K_N\in\{1,\ldots,N\}$ be the index of the bin that contains
$\truth$. Then $K_N$ is uniform, and its (Shannon) entropy is $\ln N$.}

\smrev{We estimate $K_N$ from the signal. Let $\widehat K_v$ be the index of the bin that contains $\query_v$, with
values below $x_1$ assigned to bin $1$ and values above $x_{N-1}$ assigned to bin $N$. For any $\eta>0$, the estimate
can be wrong only if $\truth$ lies within distance $\eta$ of one of the internal thresholds $x_1,\ldots,x_{N-1}$ or if
the noise satisfies $|\sqrt{v}Z|>\eta$. Therefore}
\begin{equation}\smrev{\label{ec-mi-eq: fano error}
P_e(v)\triangleq\PP(\widehat K_v\neq K_N)\leq\truDist\left(\bigcup_{i=1}^{N-1}[x_i-\eta,x_i+\eta]\right)
+\PP(|\sqrt{v}Z|>\eta).}
\end{equation}
\smrev{We show that $P_e(v)\to0$ as $v\downarrow0$. Let $\epsilon>0$. As $\eta\downarrow0$, the first term in
\Cref{ec-mi-eq: fano error} tends to $\truDist(\{x_1,\ldots,x_{N-1}\})=0$, because $\truDist$ has no atoms, so we can
fix $\eta>0$ that makes it smaller than $\epsilon/2$. For this $\eta$, the second term tends to zero as
$v\downarrow0$, so it is smaller than $\epsilon/2$ for all small $v$. Hence $P_e(v)<\epsilon$ for all small $v$.}

\smrev{Fano's inequality bounds the information that the signal carries about $K_N$ from below:}
\[
\smrev{I(K_N;\query_v)\geq\ln N-H_b(P_e(v))-P_e(v)\ln(N-1),}
\]
\smrev{where $H_b(p)\triangleq-p\ln p-(1-p)\ln(1-p)$ is the binary entropy, which is continuous with $H_b(0)=0$.
Since $K_N$ is a measurable function of $\truth$, the data-processing inequality gives
$\costEffort(\sqrt{v})=I(\truth;\query_v)\geq I(K_N;\query_v)$. As $P_e(v)\to0$, we obtain
$\liminf_{v\downarrow0}\costEffort(\sqrt{v})\geq\ln N$. Because $N$ is arbitrary,
$\costEffort(\sqrt{v})\to\infty$ as $v\downarrow0$. This proves item (iii).}

\smrev{\emph{Item (iv): the lower bound \Cref{ec-mi-eq: cost lower bound}.} We bound the information that the signal
carries about the sign of $\truth$. Let $\xi\triangleq\operatorname{sgn}(\truth)\in\{-1,1\}$, which is defined almost
surely because $\truDist(\{0\})=0$, and let $\bar\xi(\query_v)\triangleq\EE{}{\xi\,|\,\query_v}$ be its conditional
mean given the signal. The bound follows from two inequalities, proved below:}
\[
\smrev{I(\xi;\query_v)\geq\frac12\EE{}{\bar\xi(\query_v)^2}
\qquad\text{and}\qquad
\EE{}{\bar\xi(\query_v)^2}\geq\frac{m_1^2}{v+\truSD^2}.}
\]

\smrev{By symmetry, $\xi$ takes each value with probability $1/2$. By the data-processing inequality,
$\costEffort(\sqrt{v})\geq I(\xi;\query_v)$. By definition of $\bar\xi$,
$\PP(\xi=1\,|\,\query_v)=(1+\bar\xi(\query_v))/2$. For a variable that takes two values with probability $1/2$ each,
the mutual information is $\ln2$ minus the expected entropy of its conditional distribution. Hence
$I(\xi;\query_v)=\EE{}{\psi(\bar\xi(\query_v))}$, where, with the convention $0\ln0=0$,}
\[
\smrev{\psi(r)\triangleq\ln2-H_b\left(\frac{1+r}2\right)=\frac{1+r}{2}\ln(1+r)+\frac{1-r}{2}\ln(1-r),
\qquad r\in[-1,1],}
\]

\smrev{so that $\psi(\pm1)=\ln2$, the value given by continuity. We now bound $\psi$ from below by a
quadratic. The function $\psi$ satisfies $\psi(0)=\psi'(0)=0$ and
$\psi''(r)=1/(1-r^2)\geq1$ for $r\in(-1,1)$, and it is continuous at the endpoints. Therefore $\psi(r)\geq r^2/2$, and}
\[
\smrev{\costEffort(\sqrt{v})\geq I(\xi;\query_v)\geq\frac12\EE{}{\bar\xi(\query_v)^2}.}
\]

\smrev{It remains to bound $\EE{}{\bar\xi(\query_v)^2}$ from below. Since $\bar\xi(\query_v)$ is the conditional mean
of $\xi$ given $\query_v$, and $Z$ is independent of $\truth$ with mean zero,}
\[
\smrev{\EE{}{\query_v\,\bar\xi(\query_v)}=\EE{}{\query_v\,\xi}=\EE{}{\truth\operatorname{sgn}(\truth)}=m_1.}
\]
\smrev{By the Cauchy--Schwarz inequality,
$m_1^2\leq\EE{}{\query_v^2}\,\EE{}{\bar\xi(\query_v)^2}=(v+\truSD^2)\,\EE{}{\bar\xi(\query_v)^2}$. Combining this with
the previous display proves \Cref{ec-mi-eq: cost lower bound}, and item (iv).}

\smrev{\emph{Item (v): continuity in the precision.} We show that $u\mapsto\costEffort(1/\sqrt u)$ is continuous on
$[0,\bar u]$, including at $u=0$. We write the mutual information as an integral against a fixed Gaussian density,
bound its integrand uniformly in $u$, and apply dominated convergence.}

\smrev{We use the rescaled signal $\tilde{\query}_u\triangleq\sqrt{u}\,\truth+Z$, $u\in[0,\bar u]$. For $u>0$,
$\tilde{\query}_u=\sqrt u\,\query_v$ with $v=1/u$, an invertible function of $\query_v$, so
$I(\truth;\tilde{\query}_u)=\costEffort(1/\sqrt{u})$. At $u=0$, $\tilde{\query}_0=Z$ is independent of $\truth$, so
$I(\truth;\tilde{\query}_0)=0=\costEffort(\infty)$, as required. Let $\phi$ be the standard normal density and
$q_u(y)\triangleq\int\phi(y-\sqrt{u}x)\,\truDist(dx)$ the density of $\tilde{\query}_u$. Given $\truth=x$, the signal
$\tilde{\query}_u$ has density $\phi(y-\sqrt ux)$, so}
\begin{equation}\smrev{\label{ec-mi-eq: continuity integral}
I(\truth;\tilde{\query}_u)=\int\!\!\int\phi(y-\sqrt{u}x)\ln\frac{\phi(y-\sqrt{u}x)}{q_u(y)}\,dy\,\truDist(dx)
=\int\!\!\int\phi(z)\ln\frac{\phi(z)}{q_u(z+\sqrt{u}x)}\,dz\,\truDist(dx),}
\end{equation}
\smrev{where the second expression uses the change of variables $z=y-\sqrt{u}x$. In the second expression, the
integration measure $\phi(z)\,dz\,\truDist(dx)$ does not depend on $u$.}

\smrev{\emph{A bound on the logarithm.} Set $\bar a\triangleq B\sqrt{\bar u}$. For $u\in[0,\bar u]$ and $|x|\leq B$,
the shift $\sqrt{u}x$ belongs to $[-\bar a,\bar a]$; below, $a$ and $b$ denote such shifts. Hence, for every $y$,
$q_u(y)$ is an average of values
$\phi(y-a)$ with $|a|\leq\bar a$, and}
\[
\smrev{\min_{|a|\leq\bar a}\phi(y-a)\leq q_u(y)\leq\max_{|a|\leq\bar a}\phi(y-a).}
\]
\smrev{So the ratio $\phi(y-\sqrt ux)/q_u(y)$ lies between the smallest and the largest of the ratios
$\phi(y-\sqrt ux)/\phi(y-a)$ over $|a|\leq\bar a$. Each of these ratios is controlled by the following identity and
bound, for $a,b\in[-\bar a,\bar a]$:}
\[
\smrev{\left|\ln\frac{\phi(y-a)}{\phi(y-b)}\right|=\frac{|(a-b)(2y-a-b)|}{2}\leq2\bar a|y|+2\bar a^2.}
\]
\smrev{Hence the logarithm in the first expression of \Cref{ec-mi-eq: continuity integral} is at most
$2\bar a|y|+2\bar a^2$ in absolute value. In the second expression, $y=z+\sqrt ux$, so $|y|\leq|z|+\bar a$, and the
logarithm is at most $2\bar a|z|+4\bar a^2$ in absolute value. This bound is integrable with respect to
$\phi(z)\,dz\,\truDist(dx)$ and does not depend on $u\in[0,\bar u]$.}

\smrev{\emph{Dominated convergence.} It remains to check that the integrand of the second expression in
\Cref{ec-mi-eq: continuity integral} is continuous in $u$ for fixed $(z,x)$. Its argument $z+\sqrt ux$ moves with
$u$, so this requires joint continuity of $(u,y)\mapsto q_u(y)$.}

\smrev{The map $(u,y)\mapsto q_u(y)=\int\phi(y-\sqrt{u}x)\,\truDist(dx)$ is jointly continuous by dominated convergence:
its integrand is continuous in $(u,y)$, and it is bounded by the constant $1/\sqrt{2\pi}$, which is integrable against
the probability measure $\truDist$. Since $q_u>0$, the integrand $\phi(z)\ln\{\phi(z)/q_u(z+\sqrt ux)\}$ is therefore
continuous in $u$.}

\smrev{Dominated convergence, applied to the second expression in \Cref{ec-mi-eq: continuity integral} with the bound
of the previous block, then proves that $u\mapsto I(\truth;\tilde{\query}_u)$ is continuous on $[0,\bar u]$. This
proves item (v).}
\Halmos
\end{proof}

\paragraph{\smrev{Optimal finite signals cannot be too precise.}} \smrev{From here until the end of the
bounded-support part, that is, through the paragraph ``Where the communication cost enters'', $\truDist$ satisfies all
the assumptions of \Cref{thm:bounded-atomless-collapse}. Compared with the paragraph ``The communication cost of a
bounded atomless population'', where $B$ was only a bound on the support, $B$ is now the endpoint of the hull of the
support, $H_{\truDist}=[-B,B]$ with $B>0$.}

\emph{\smrev{The priors of the loop.}} \smrev{Every prior $\pi_t$ of the loop is symmetric and supported on
$[-B,B]$: for $t=0$ because $\pi_0=\truDist$, and for $t\geq1$ by \Cref{ec-mi-cor: support preserved} and its proof.
Symmetry is proved within the induction of that proof, which also covers the case $\pi_t=\delta_0$ after a collapse. The
lemmas and the proof below cite this paragraph by its name when they use these two properties of $\pi_t$.}

\begin{lemma}[Uniform lower bound on finite optimal variances]\label{ec-mi-lemma: vmin}
\smrev{There exists $v_{\min}>0$, depending only on $(\truDist,\costParam,B)$,} such that\smrev{, for any symmetric
prior $\pi$ supported on $[-B,B]$, any type $\truth\in[-B,B]$, and any finite optimal variance $v$ for type $\truth$
under $\pi$, we have $v\geq v_{\min}$.}
\end{lemma}

\begin{proof}{Proof of \Cref{ec-mi-lemma: vmin}.}
\smrev{The idea is that a signal with finite noise variance can save a user at most the loss $\truth^2\leq B^2$ of
the uninformative signal, while a very precise signal costs more than that. By \Cref{ec-mi-lemma: cost}(iii),
$\costEffort(\sqrt{v})\rightarrow\infty$ as $v\downarrow0$. Hence there exists $v_{\min}>0$, which depends only on
$\truDist$ (through $\costEffort$), $\costParam$ and $B$, such that}
\begin{equation}\smrev{\label{ec-mi-eq: vmin threshold}
\costParam\costEffort(\sqrt{v})>B^2\qquad\text{for every }v\in(0,v_{\min}).}
\end{equation}
\smrev{Now let $v$ be a finite optimal variance for type $\truth\in[-B,B]$ under $\pi$. Comparing it with the
uninformative signal gives $\obj_\pi(\truth,\sqrt{v})\leq\obj_\pi(\truth,\infty)=\truth^2\leq B^2$. The fidelity
error in \Cref{ec-mi-eq: user loss} is nonnegative, so $\costParam\costEffort(\sqrt{v})\leq B^2$. By
\Cref{ec-mi-eq: vmin threshold}, $v\geq v_{\min}$.}
\Halmos
\end{proof}

\begin{lemma}[Optimal signals and a tie-breaking rule exist]\label{ec-mi-lemma: selector}
\smrev{Let $\pi$ be symmetric and supported on $[-B,B]$. For every $\truth\in[-B,B]$, the loss
$\obj_\pi(\truth,\sqrt{v})$ attains its minimum over $v\in(0,\infty]$. Moreover, a tie-breaking rule for
$(\pi,\costParam)$ exists: there is a choice $v_\pi^\star(\truth)$ of optimal variances, $\truth\in[-B,B]$, that is
measurable and even in $\truth$.}
\end{lemma}

\begin{proof}{Proof of \Cref{ec-mi-lemma: selector}.}
\smrev{The proof has six steps. We first treat the prior $\delta_0$, for which the uninformative signal is the only
optimum. For any other prior, we write the loss as a function $F(\truth,u)$ of the precision $u=1/v$, where $u=0$
stands for the uninformative signal, and reduce the minimization to a compact interval of precisions. We then show
that $F$ is jointly continuous, separately where $u>0$ and at $u=0$, so that minima exist. Finally, we show that the
smallest minimizer is measurable and even, which gives the tie-breaking rule.}

\smrev{\emph{Step 1: the prior $\delta_0$.} If $\pi=\delta_0$, then $\truthAIAlg{\pi}(\cdot,\sqrt{v})\equiv0$ for every
finite $v$, so every finite $v$ has loss $\truth^2+\costParam\costEffort(\sqrt{v})$. This is strictly larger than the
loss $\truth^2$ of the uninformative signal, because $\costEffort(\sqrt{v})>0$ by \Cref{ec-mi-lemma: cost}(iv). Thus
$v_\pi^\star\equiv\infty$ is the only optimal choice, and it is measurable and even. Assume from now on that
$\pi\neq\delta_0$.}

\smrev{\emph{Step 2: reduction to a compact set of precisions.} For $\truth\in[-B,B]$ and $u\in[0,\infty)$, define}
\[
\smrev{F(\truth,u)\triangleq\obj_\pi(\truth,1/\sqrt{u})\quad(u>0),\qquad F(\truth,0)\triangleq\truth^2,}
\]
\smrev{so that $F(\truth,0)$ is the loss of the uninformative signal and minimizing $F(\truth,\cdot)$ over
$[0,\infty)$ is the same as minimizing $\obj_\pi(\truth,\sqrt v)$ over $v\in(0,\infty]$. Let
$u_{\max}\triangleq1/v_{\min}$, with $v_{\min}$ from \Cref{ec-mi-lemma: vmin}. Every precision $u>u_{\max}$, that is,
every variance $v\in(0,v_{\min})$, is strictly worse than the uninformative signal, because by
\Cref{ec-mi-eq: vmin threshold} its loss is at least $\costParam\costEffort(\sqrt{v})>B^2\geq\truth^2=F(\truth,0)$.
Hence the minimizers of $F(\truth,\cdot)$ over $[0,u_{\max}]$, if any, are exactly its minimizers over $[0,\infty)$.}

\smrev{\emph{Step 3: continuity where $u>0$.} We show that $F$ is continuous on $[-B,B]\times(0,u_{\max}]$. For
$u>0$, write $\hat\theta_{\truth,u}\triangleq\truthAIAlg{\pi}(\truth+Z/\sqrt{u},1/\sqrt{u})$ for the output when type
$\truth$ uses the variance $1/u$; it is random through $Z$. The fidelity error is
$\EE{}{(\hat\theta_{\truth,u}-\truth)^2}$. For each value of $Z$, the integrand is continuous in $(\truth,u)$, because
$(\query,v)\mapsto\truthAIAlg{\pi}(\query,\sqrt{v})$ is continuous (\Cref{ec-mi-lemma: derivative}). It is bounded by
$4B^2$, because the output and $\truth$ both lie in $[-B,B]$. By dominated convergence, the fidelity error is
continuous. The communication cost $\costEffort(1/\sqrt{u})$ is continuous by \Cref{ec-mi-lemma: cost}(v).}

\smrev{\emph{Step 4: continuity at $u=0$.} We show that $F(\truth,u)\to\truth^2$ as $u\downarrow0$, uniformly in
$|\truth|\leq B$. By \Cref{ec-mi-eq: contraction second moment} of \Cref{ec-mi-lemma: contraction}, with $v=1/u$,
uniformly in $|\truth|\leq B$,}
\[
\smrev{\EE{}{\hat\theta_{\truth,u}^2}\leq B^4u^2\left(\truth^2+\frac1u\right)
\leq B^4\left(B^2u^2+u\right)\rightarrow0\qquad(u\downarrow0).}
\]
\smrev{The fidelity error equals $\EE{}{\hat\theta_{\truth,u}^2}-2\truth\,\EE{}{\hat\theta_{\truth,u}}+\truth^2$. By the
Cauchy--Schwarz inequality, its cross term satisfies
$|2\truth\,\EE{}{\hat\theta_{\truth,u}}|\leq2B(\EE{}{\hat\theta_{\truth,u}^2})^{1/2}\rightarrow0$, uniformly in
$|\truth|\leq B$. By
\Cref{ec-mi-lemma: cost}(i), $\costEffort(1/\sqrt{u})\rightarrow0$ as $u\downarrow0$. Therefore
$F(\truth,u)\rightarrow\truth^2=F(\truth,0)$ uniformly in $\truth$. Together with
$|F(\truth',u)-F(\truth,0)|\leq|F(\truth',u)-\truth'^2|+|\truth'^2-\truth^2|$, this shows that $F$ is continuous at
every point $(\truth,0)$. With Step 3, $F$ is continuous on $[-B,B]\times[0,u_{\max}]$.}

\smrev{\emph{Step 5: existence and measurability.} For each $\truth$, the continuous function $F(\truth,\cdot)$ attains
its minimum on the compact set $[0,u_{\max}]$. So the set of minimizers
$\mathcal{M}(\truth)\triangleq\argmin_{u\in[0,u_{\max}]}F(\truth,u)$ is nonempty and closed, and, by Step 2, it is the
set of optimal precisions. Let $u^\star(\truth)\triangleq\min\mathcal{M}(\truth)$ be the smallest optimal precision.
We show that $u^\star$ is lower semicontinuous, hence Borel measurable, by a direct sequence argument.}

\smrev{Let $\truth_n\to\truth$ in $[-B,B]$, and pass to a subsequence, which we relabel $(\truth_n)$, along which
$u^\star(\truth_n)$ converges to
$\liminf_nu^\star(\truth_n)\triangleq u_\infty\in[0,u_{\max}]$. For every $u\in[0,u_{\max}]$ and every $n$,
$F(\truth_n,u^\star(\truth_n))\leq F(\truth_n,u)$. Letting $n\to\infty$ and using the continuity of $F$ gives
$F(\truth,u_\infty)\leq F(\truth,u)$ for every $u$, so $u_\infty\in\mathcal{M}(\truth)$ and
$u_\infty\geq u^\star(\truth)$. Hence $\liminf_nu^\star(\truth_n)\geq u^\star(\truth)$, which is lower
semicontinuity.}

\smrev{\emph{Step 6: evenness and the tie-breaking rule.} Since $\pi$ is symmetric, $\truthAIAlg{\pi}(\cdot,\sqrt{v})$
is odd (\Cref{ec-mi-lemma: derivative}), and $-\truth+\sqrt{v}Z$ has the same distribution as $-(\truth+\sqrt{v}Z)$.
Hence the error $(\truthAIAlg{\pi}(-\truth+\sqrt vZ,\sqrt v)+\truth)^2$ of type $-\truth$ has the same distribution as
the error of type $\truth$, so $F(-\truth,u)=F(\truth,u)$. It follows that $\mathcal{M}(-\truth)=\mathcal{M}(\truth)$
and that $u^\star$ is even. Set $v_\pi^\star(\truth)\triangleq1/u^\star(\truth)$ when $u^\star(\truth)>0$, and
$v_\pi^\star(\truth)\triangleq\infty$ when $u^\star(\truth)=0$. Then $v_\pi^\star(\truth)$ is optimal for every
$\truth\in[-B,B]$, measurable and even, so it is a tie-breaking rule for $(\pi,\costParam)$.}
\Halmos
\end{proof}

\paragraph{\smrev{Priors with a small second moment lead to collapse.}}

\begin{lemma}[A prior with a small second moment makes every user send an uninformative signal]
\label{ec-mi-lemma: basin}
\smrev{There exists $\varepsilon_0>0$, depending only on $(\truDist,\costParam,B)$, such that the following holds. If
$\pi$ is symmetric, supported on $[-B,B]$, and has second moment $m_2(\pi)<\varepsilon_0$, then every type
$|\truth|\leq B$ strictly prefers the uninformative signal: for every finite $v$,
$\obj_\pi(\truth,\sqrt{v})>\obj_\pi(\truth,\infty)=\truth^2$. Consequently, under every tie-breaking rule, the next
prior is $\delta_0$.}
\end{lemma}

\begin{proof}{Proof of \Cref{ec-mi-lemma: basin}.}
\smrev{We first treat the prior $\delta_0$. For any other prior, we call the gain of a finite signal the amount by
which its fidelity error falls below $\truth^2$, the loss of the uninformative signal; its communication cost is
not subtracted. We first bound the slope of the posterior mean by a multiple of $m_2(\pi)$. Integrating the slope gives
the user's mean output, and expanding the fidelity error shows that the gain is at most twice the type times the mean
output. So the gain is bounded above by a multiple of $m_2(\pi)$. We bound the communication cost from below by a
quantity that does not depend on $\pi$. We compare the two on three ranges of noise variances: $v<v_{\min}$, where the
cost alone exceeds $B^2$ (\Cref{ec-mi-lemma: vmin}); a compact range $[v_{\min},\bar v_1]$, with $\bar v_1$ defined
below, where the cost is at least a positive constant; and $v\geq\bar v_1$, where the cost is at least of order $1/v$
and the gain at most of order $m_2(\pi)/v$.}

\smrev{\emph{The prior $\delta_0$.} If $\pi=\delta_0$, Step 1 of the proof of \Cref{ec-mi-lemma: selector} shows
that every finite $v$ has a loss strictly larger than $\truth^2$. We now assume that $\pi$ is nondegenerate.}

\smrev{\emph{A bound on the slope of the posterior mean.} Fix a finite $v$ and write
$g_v(\query)\triangleq\truthAIAlg{\pi}(\query,\sqrt{v})$. With the weighted moments $N_0$ and $N_2$ of
\Cref{ec-mi-eq: tilted moments}, \Cref{ec-mi-eq: posterior derivative} gives
$g_v'(\query)\leq\EE{\pi}{X^2\,|\,X+\sqrt vZ=\query}/v=N_2(\query)/(vN_0(\query))$, because a variance is at most the
second moment. We bound the denominator $N_0$ from below and the numerator $N_2$ from above. Since $|x|\leq B$ on the
support of
$\pi$, we have $\query x/v\geq-B|\query|/v$ and
$x^2/(2v)\leq B^2/(2v)$ in the exponent, and $x^2e^{\query x/v-x^2/(2v)}\leq x^2e^{B|\query|/v}$. Hence}
\[
\smrev{N_0(\query)\geq\exp\left\{-\frac{B|\query|}{v}-\frac{B^2}{2v}\right\},
\qquad
N_2(\query)\leq m_2(\pi)\exp\left\{\frac{B|\query|}{v}\right\},}
\]
\smrev{and therefore}
\begin{equation}\smrev{\label{ec-mi-eq: basin derivative}
0\leq g_v'(\query)\leq\frac{N_2(\query)}{vN_0(\query)}
\leq\frac{m_2(\pi)}{v}\exp\left\{\frac{2B|\query|}{v}+\frac{B^2}{2v}\right\}.}
\end{equation}

\smrev{\emph{The gain from a finite signal.} The loss is even in $\truth$ (Step 6 of the proof of
\Cref{ec-mi-lemma: selector}), so it suffices to consider $0\leq\truth\leq B$. Let
$h_v(\truth)\triangleq\EE{}{g_v(\truth+\sqrt{v}Z)}$ be the mean output of type $\truth$. The bound
$0\leq g_v'\leq B^2/v$ of \Cref{ec-mi-lemma: derivative} permits differentiation under the expectation, and
$h_v(0)=0$ because $g_v$ is odd and $Z$ is symmetric. Hence}
\begin{equation}\smrev{\label{ec-mi-eq: hv integral}
h_v(\truth)=\int_0^{\truth}\EE{}{g_v'(y+\sqrt{v}Z)}\,dy.}
\end{equation}
\smrev{Expanding the square in the fidelity error, the difference between the loss of a finite signal and the loss of
the uninformative signal is}
\begin{equation}\smrev{\label{ec-mi-eq: basin loss difference}
\obj_\pi(\truth,\sqrt{v})-\truth^2
=\EE{}{g_v(\truth+\sqrt{v}Z)^2}-2\truth h_v(\truth)+\costParam\costEffort(\sqrt{v})
\geq\costParam\costEffort(\sqrt{v})-2\truth h_v(\truth).}
\end{equation}
\smrev{The inequality drops the second moment of the output, which is nonnegative. So $2\truth h_v(\truth)$ is an upper
bound on what a finite signal can gain over the uninformative signal.}

\smrev{It remains to show that the loss difference is positive for every finite $v$ when $m_2(\pi)$ is small. For
$v<v_{\min}$, we compare the losses directly. For $v\geq v_{\min}$, we show that
the right side of \Cref{ec-mi-eq: basin loss difference} is positive. There, we bound $h_v$ through
\Cref{ec-mi-eq: hv integral}, by taking the expectation of the bound \Cref{ec-mi-eq: basin derivative} over the
Gaussian noise. The next block defines the constants that bound this expectation.}

\smrev{\emph{Constants that do not depend on $\pi$.} Let $\bar v_1\triangleq\max\{1,v_{\min},\truSD^2\}$. Both
constants below bound the expectation of the exponential in \Cref{ec-mi-eq: basin derivative}, but they differ by the
factor $1/v$. The constant $K_1$ keeps this factor, which is bounded on the compact range. The constant $K_2$ leaves it
out, because on the large range this factor is what makes the gain of order $1/v$:}
\begin{align*}
\smrev{K_1}&\smrev{\triangleq\sup_{v\in[v_{\min},\bar v_1],\ 0\leq y\leq B}\frac1v\,
\EE{}{\exp\left\{\frac{2B|y+\sqrt{v}Z|}{v}+\frac{B^2}{2v}\right\}},}\\
\smrev{K_2}&\smrev{\triangleq\sup_{v\geq\bar v_1,\ 0\leq y\leq B}
\EE{}{\exp\left\{\frac{2B|y+\sqrt{v}Z|}{v}+\frac{B^2}{2v}\right\}}.}
\end{align*}
\smrev{In the definition of $\bar v_1$, the term $v_{\min}$ makes $[v_{\min},\bar v_1]$ a nonempty interval, the
term $1$ is used below to show that $K_2$ is finite, and the term $\truSD^2$ gives $v+\truSD^2\leq2v$ in the range
$v\geq\bar v_1$.}

\smrev{Both constants are finite. For $K_1$, the expectation is a continuous function of $(v,y)$ on the compact set
$[v_{\min},\bar v_1]\times[0,B]$, by dominated convergence with the integrable bound
$\exp\{2B(B+\sqrt{\bar v_1}|Z|)/v_{\min}+B^2/(2v_{\min})\}$. For $K_2$, because $v\geq\bar v_1\geq1$ and
$0\leq y\leq B$, we have $|y+\sqrt{v}Z|/v\leq B+|Z|$ and $B^2/(2v)\leq B^2/2$, so the exponential is at most
$e^{5B^2/2}e^{2B|Z|}$, which is integrable.}

\smrev{\emph{Range $v<v_{\min}$.} Here the cost alone exceeds the loss of the uninformative signal. As in Step 2 of
the proof of \Cref{ec-mi-lemma: selector}, since the fidelity error is nonnegative, \Cref{ec-mi-eq: vmin threshold}
gives}
\[
\smrev{\obj_\pi(\truth,\sqrt{v})\geq\costParam\costEffort(\sqrt{v})>B^2\geq\truth^2
\qquad\text{for every }v\in(0,v_{\min}),}
\]
\smrev{whatever the value of $m_2(\pi)$.}

\smrev{\emph{Range $v\in[v_{\min},\bar v_1]$.} By \Cref{ec-mi-eq: basin derivative} and the definition of $K_1$,
$\EE{}{g_v'(y+\sqrt{v}Z)}\leq K_1m_2(\pi)$ for $0\leq y\leq B$, so \Cref{ec-mi-eq: hv integral} gives
$h_v(\truth)\leq K_1m_2(\pi)\truth$. Since $0\leq\truth\leq B$, the gain satisfies}
\[
\smrev{2\truth h_v(\truth)\leq2K_1m_2(\pi)\truth^2\leq2K_1m_2(\pi)B^2.}
\]
\smrev{By monotonicity of the communication cost (\Cref{ec-mi-lemma: cost}(ii)),
$\costParam\costEffort(\sqrt{v})\geq c_1\triangleq\costParam\costEffort(\sqrt{\bar v_1})$, and $c_1>0$ by
\Cref{ec-mi-eq: cost lower bound}. With the previous display, \Cref{ec-mi-eq: basin loss difference} gives}
\begin{equation}\smrev{\label{ec-mi-eq: compact v bound}
\obj_\pi(\truth,\sqrt{v})-\truth^2\geq c_1-2K_1m_2(\pi)B^2.}
\end{equation}

\smrev{\emph{Range $v\geq\bar v_1$.} By \Cref{ec-mi-eq: basin derivative} and the definition of $K_2$,
$\EE{}{g_v'(y+\sqrt{v}Z)}\leq K_2m_2(\pi)/v$ for $0\leq y\leq B$, so \Cref{ec-mi-eq: hv integral} gives
$h_v(\truth)\leq K_2m_2(\pi)\truth/v$. Since $0\leq\truth\leq B$, the gain satisfies}
\[
\smrev{2\truth h_v(\truth)\leq\frac{2K_2m_2(\pi)\truth^2}{v}\leq\frac{2K_2m_2(\pi)B^2}{v}.}
\]
\smrev{By \Cref{ec-mi-eq: cost lower bound} and $v\geq\bar v_1\geq\truSD^2$,}
\[
\smrev{\costEffort(\sqrt{v})\geq\frac{m_1^2}{2(v+\truSD^2)}\geq\frac{m_1^2}{4v}.}
\]
\smrev{With the two previous displays, \Cref{ec-mi-eq: basin loss difference} gives}
\begin{equation}\smrev{\label{ec-mi-eq: large v bound}
\obj_\pi(\truth,\sqrt{v})-\truth^2\geq\frac1v\left(\frac{\costParam m_1^2}{4}-2K_2m_2(\pi)B^2\right).}
\end{equation}

\smrev{\emph{Choice of $\varepsilon_0$.} The constants $v_{\min}$, $\bar v_1$, $K_1$, $K_2$, $c_1$, $m_1$ and $\truSD^2$
depend only on $(\truDist,\costParam,B)$. Choose $\varepsilon_0>0$ such that}
\begin{equation}\smrev{\label{ec-mi-eq: basin choice}
\varepsilon_0<\min\left\{\frac{c_1}{2K_1B^2},\ \frac{\costParam m_1^2}{8K_2B^2}\right\}.}
\end{equation}

\smrev{\emph{Conclusion.} Suppose that $m_2(\pi)<\varepsilon_0$. We combine the three ranges. For $v<v_{\min}$, the
range block above shows that $\obj_\pi(\truth,\sqrt{v})>\truth^2$. By \Cref{ec-mi-eq: basin choice}, the right side of
\Cref{ec-mi-eq: compact v bound} is strictly positive for all $v\in[v_{\min},\bar v_1]$, and that of
\Cref{ec-mi-eq: large v bound} is strictly positive for all $v\geq\bar v_1$. Thus every finite variance is strictly
worse than the uninformative signal for every type in $[-B,B]$. Every tie-breaking rule then chooses $v=\infty$ for
every type, every output is $0$, and the next prior is $\delta_0$.}
\Halmos
\end{proof}

\paragraph{\smrev{A loop that never collapses keeps a mass of users with bounded noise.}} \smrev{For
$\bar v\in(0,\infty)$, we sort the users at round $t$ into three groups by the noise variance that the tie-breaking
rule gives them: those who send the uninformative signal, those in $\mathcal{C}_t(\bar v)$ and those in
$\mathcal{W}_t(\bar v)$, where}
\[
\smrev{\mathcal{C}_t(\bar v)\triangleq\{\truth:v_t^\star(\truth)\leq\bar v\},
\qquad
\mathcal{W}_t(\bar v)\triangleq\{\truth:\bar v<v_t^\star(\truth)<\infty\},
\qquad
\alpha_t(\bar v)\triangleq\truDist(\mathcal{C}_t(\bar v)).}
\]
\smrev{The split is useful because we bound the outputs of the two informative groups
differently: for users in
$\mathcal{C}_t(\bar v)$ we only use that the output lies in $[-B,B]$, while for users in $\mathcal{W}_t(\bar v)$,
\Cref{ec-mi-lemma: contraction} gives a second moment that vanishes as $\bar v$ grows.}

\begin{lemma}[If the loop never collapses, a positive mass of users send informative signals with bounded noise at
every round]
\label{ec-mi-lemma: rescue}
\smrev{Suppose the loop never collapses. Then there exist a finite $\bar v\geq v_{\min}$ and $\alpha>0$, depending
only on $(\truDist,\costParam,B)$,} such that \smrev{$\alpha_t(\bar v)\geq\alpha$ for every $t\geq0$.}
\end{lemma}

\begin{proof}{Proof of \Cref{ec-mi-lemma: rescue}.}
\smrev{We bound the second moment of the next prior by a multiple of $\alpha_t(\bar v)$ plus a term that is small when
$\bar v$ is large. We first choose $\bar v$ to make the second term small, and then $\alpha$ so that
$\alpha_t(\bar v)<\alpha$ would make the next prior's second moment smaller than $\varepsilon_0$; then
\Cref{ec-mi-lemma: basin} would make the loop collapse.}

\smrev{\emph{The second moment of the next prior.} Fix $t\geq0$ and $\bar v\in[v_{\min},\infty)$. The prior $\pi_t$ is
symmetric and supported on $[-B,B]$, by the paragraph ``The priors of the loop'' before \Cref{ec-mi-lemma: vmin}.
For each of the three groups, we bound
the
second moment of the output $Y_{\pi_t}^{\truth}$ of \Cref{ec-mi-eq: type output} for a fixed type $\truth$; the
expectation is over the signal noise. Users in $\mathcal{C}_t(\bar v)$ receive a posterior mean
under a prior on $[-B,B]$, which lies in $[-B,B]$, so the square of their output is at most $B^2$. For a user $\truth$
in $\mathcal{W}_t(\bar v)$, with variance $v=v_t^\star(\truth)>\bar v$, \Cref{ec-mi-eq: contraction second moment}
gives}
\[
\smrev{\EE{}{\left(Y_{\pi_t}^{\truth}\right)^2}\leq B^4\left(\frac{\truth^2}{v^2}+\frac1v\right)
\leq B^4\left(\frac{\truth^2}{\bar v^2}+\frac1{\bar v}\right).}
\]
\smrev{Now average over $\truth\sim\truDist$, one group at a time. The users who send the uninformative signal
output $0$, so they contribute $0$. The users in $\mathcal{C}_t(\bar v)$ contribute at most $B^2\alpha_t(\bar v)$,
because their mass is $\alpha_t(\bar v)$. The users in $\mathcal{W}_t(\bar v)$ contribute at most the average of the
previous bound over all users, because that bound is nonnegative; this average is $B^4(\truSD^2/\bar v^2+1/\bar v)$,
because $\EE{}{\truth^2}=\truSD^2$. Hence}
\begin{equation}\smrev{\label{ec-mi-eq: rescue bound}
m_2(\pi_{t+1})\leq B^2\alpha_t(\bar v)+B^4\left(\frac{\truSD^2}{\bar v^2}+\frac1{\bar v}\right).}
\end{equation}

\smrev{\emph{Choice of $\bar v$ and $\alpha$.} Let $\varepsilon_0$ be as in \Cref{ec-mi-lemma: basin}. Choose
$\bar v\geq v_{\min}$ so large that $B^4\left(\truSD^2/\bar v^2+1/\bar v\right)<\varepsilon_0/2$, and set
$\alpha\triangleq\varepsilon_0/(2B^2)>0$.}

\smrev{\emph{Conclusion.} Suppose that $\alpha_t(\bar v)<\alpha$ for some $t$. Then \Cref{ec-mi-eq: rescue bound}
gives $m_2(\pi_{t+1})<\varepsilon_0/2+\varepsilon_0/2=\varepsilon_0$. The prior $\pi_{t+1}$ is symmetric and
supported on $[-B,B]$, so, by \Cref{ec-mi-lemma: basin}, all users strictly prefer the uninformative signal under
$\pi_{t+1}$, and $\pi_{t+2}=\delta_0$. This contradicts the assumption that the loop never collapses. Hence
$\alpha_t(\bar v)\geq\alpha$ for all $t\geq0$.}
\Halmos
\end{proof}

\begin{lemma}[The priors keep mass near the center]\label{ec-mi-lemma: central mass}
\smrev{Suppose the loop never collapses. For every $c>0$, there exists $\delta_c>0$ such that
$\pi_t([-c,c])\geq\delta_c$ for every $t\geq1$.}
\end{lemma}

\begin{proof}{Proof of \Cref{ec-mi-lemma: central mass}.}
\smrev{The prior $\pi_t$ is the distribution of the outputs produced at round $t-1$. At every round $t-1\geq0$, the
users in $\mathcal{C}_{t-1}(\bar v)$ of \Cref{ec-mi-lemma: rescue} have mass at least $\alpha$. We show that each of
them produces an output in $[-c,c]$ with a probability bounded below by a
constant $q_c>0$ that is the same for all types in $[-B,B]$, all variances in $[v_{\min},\bar v]$ and all rounds.}

\smrev{\emph{One user.} Let $\bar v$ and $\alpha$ be given by \Cref{ec-mi-lemma: rescue}, and fix $c>0$ and
$t\geq1$. The prior $\pi_{t-1}$ is symmetric and supported on $[-B,B]$, by ``The priors of the loop'' before
\Cref{ec-mi-lemma: vmin}. Let
$\truth\in[-B,B]$ be a user in $\mathcal{C}_{t-1}(\bar v)$, and let $v=v^\star_{t-1}(\truth)$. Then
$v\in[v_{\min},\bar v]$: the upper bound holds by the definition of $\mathcal{C}_{t-1}(\bar v)$, and the lower bound by
\Cref{ec-mi-lemma: vmin}. By \Cref{ec-mi-eq: contraction pointwise}, the output has absolute value at most
$B^2|\truth+\sqrt{v}Z|/v$, so it lies in $[-c,c]$ whenever $|\truth+\sqrt{v}Z|\leq cv/B^2$.}

\smrev{\emph{A uniform lower bound.} The probability of this event is bounded below by}
\begin{equation}\smrev{\label{ec-mi-eq: qc}
q_c\triangleq\inf_{|\truth|\leq B,\ v\in[v_{\min},\bar v]}\PP\left(|\truth+\sqrt{v}Z|\leq\frac{cv}{B^2}\right).}
\end{equation}
\smrev{This infimum is positive, $q_c>0$, because the displayed probability is the Gaussian probability of an interval
of positive length, hence strictly positive, and it is continuous in $(\truth,v)$, so its minimum over the compact set
$[-B,B]\times[v_{\min},\bar v]$ is attained and positive.}

\smrev{\emph{Conclusion.} Almost every user lies in $[-B,B]$, so the bound of the previous step applies to almost
every user in $\mathcal{C}_{t-1}(\bar v)$. Hence, averaging over the users in
$\mathcal{C}_{t-1}(\bar v)$, whose mass is $\alpha_{t-1}(\bar v)\geq\alpha$, gives
$\pi_t([-c,c])\geq\alpha q_c$. The claim holds with $\delta_c\triangleq\alpha q_c>0$, which does not depend on $t$.}
\Halmos
\end{proof}

\paragraph{\smrev{Mass near the edge of the support.}} \smrev{For a symmetric distribution $\rho$ on $[-B,B]$ and a
cutoff $a\in(0,B)$, recall from item (e) of the plan the expected excess $\Lambda_a(\rho)=\int(x-a)_+\,\rho(dx)$
of $X\sim\rho$ beyond $a$, which the next lemma uses: it weights the mass to the right of $a$ by its distance to $a$, so
it is not a probability. By symmetry, it also measures the excess below $-a$. We use it rather than the probability
$\rho((a,B])$ for two reasons: it is continuous under weak convergence, because $x\mapsto(x-a)_+$ is bounded and
continuous on $[-B,B]$, and it still detects the support, because $\Lambda_a(\rho)=0$ exactly when $\rho((a,B])=0$. The
next lemma shows that, when $\rho$ also has mass near the
center and its expected excess beyond $a$ is small, one round of the loop sends an output beyond $a$ only with a small
probability: for fixed $c$, $\delta$ and $a$, the bound
\Cref{ec-mi-eq: edge bound} becomes smaller than any power of $\Lambda_a(\rho)$ as $\Lambda_a(\rho)\to0$.}

\begin{lemma}[Few outputs land near the edge of the support]\label{ec-mi-lemma: edge}
\smrev{Let $\rho$ be symmetric and supported on $[-B,B]$, and suppose $\rho([-c,c])\geq\delta$ for some $0<c<B$ and
$\delta>0$. For $a\in(c,B)$, let $\Lambda_a(\rho)\triangleq\int(x-a)_+\,\rho(dx)$. Let}
\[
\smrev{C_0\triangleq\frac{(B+c)B+c^2/2}{v_{\min}},
\qquad
\gamma\triangleq\frac{v_{\min}}{2(B+c)^2},}
\]
\smrev{which depend only on $(B,c,v_{\min})$. If a type $\truth\in[-B,B]$ communicates with a finite variance
$v\geq v_{\min}$, then}
\begin{equation}\smrev{\label{ec-mi-eq: edge bound}
\PP\left(\truthAIAlg{\rho}(\truth+\sqrt{v}Z,\sqrt{v})>a\right)
\leq\exp\left\{-\gamma\left[\ln\frac{(a-c)\delta}{\Lambda_a(\rho)}-C_0\right]_+^2\right\},}
\end{equation}
\smrev{with the convention that the right-hand side is zero if $\Lambda_a(\rho)=0$.}
\end{lemma}

\begin{proof}{Proof of \Cref{ec-mi-lemma: edge}.}
\smrev{The first block treats the case $\Lambda_a(\rho)=0$: then $\rho$ has no mass beyond $a$, and the probability is
$0$. The remaining blocks treat the case $\Lambda_a(\rho)>0$. There, we first show that only positive signals can
produce an output above $a$. We then turn the event that the output exceeds $a$ into a comparison between the
posterior's expected excess beyond $a$ and its mass on $[-c,c]$, and then into a lower bound on the signal. A Gaussian
tail bound on the noise then gives \Cref{ec-mi-eq: edge bound}.}

\smrev{\emph{No mass beyond $a$.} If $\Lambda_a(\rho)=0$, then $\rho((a,B])=0$, so every posterior under $\rho$ is
supported on $[-B,a]$. Its mean cannot exceed $a$, and the probability in \Cref{ec-mi-eq: edge bound} is $0$. Assume
from now on that $\Lambda_a(\rho)>0$.}

\smrev{\emph{Only positive signals matter.} Write $g_v(\query)\triangleq\truthAIAlg{\rho}(\query,\sqrt{v})$. Because
$\rho$ is symmetric, $g_v$ is odd and nondecreasing (\Cref{ec-mi-lemma: derivative}), so $g_v(\query)\leq g_v(0)=0$
for $\query\leq0$. Since $a>c>0$, the event $\{g_v(\query)>a\}$ implies $\query>0$.}

\smrev{\emph{What a posterior mean above $a$ requires.} For a signal realization $\query$, let
$\rho(\cdot\,|\,\query)$ be the posterior distribution under the prior $\rho$ and variance $v$. If $g_v(\query)>a$,
then $\int(x-a)\,\rho(dx\,|\,\query)>0$. Since
$\int(x-a)\,\rho(dx\,|\,\query)=\int(x-a)_+\,\rho(dx\,|\,\query)-\int(a-x)_+\,\rho(dx\,|\,\query)$, and
$(a-x)_+\geq a-c$ for $x\in[-c,c]$, this implies}
\begin{equation}\smrev{\label{ec-mi-eq: posterior edge central}
\int(x-a)_+\,\rho(dx\,|\,\query)>(a-c)\,\rho\left([-c,c]\,|\,\query\right).}
\end{equation}

\smrev{\emph{A lower bound on the signal.} As in the proof of \Cref{ec-mi-lemma: derivative}, the posterior is
proportional to $\varpi_{\query}(x)\,\rho(dx)$, where $\varpi_{\query}(x)\triangleq\exp\{\query x/v-x^2/(2v)\}$ (it also
depends on the fixed $v$), and
the normalizing constant cancels from the two sides of \Cref{ec-mi-eq: posterior edge central}. For $\query>0$,
$\varpi_{\query}(x)\leq e^{\query x/v}\leq e^{\query B/v}$ for $x\leq B$, and
$\varpi_{\query}(x)\geq\exp\{-\query c/v-c^2/(2v)\}$ for $|x|\leq c$. Hence}
\[
\smrev{\int(x-a)_+\,\varpi_{\query}(x)\,\rho(dx)\leq\Lambda_a(\rho)\,e^{\query B/v},
\qquad\text{whereas}\qquad
\int_{[-c,c]}\varpi_{\query}(x)\,\rho(dx)\geq\delta\exp\left\{-\frac{\query c}{v}-\frac{c^2}{2v}\right\}.}
\]
\smrev{Therefore \Cref{ec-mi-eq: posterior edge central} implies
$\Lambda_a(\rho)e^{\query B/v}>(a-c)\delta\exp\{-\query c/v-c^2/(2v)\}$. Taking logarithms, with
$L\triangleq\ln\left((a-c)\delta/\Lambda_a(\rho)\right)$, this inequality reads $\query(B+c)/v>L-c^2/(2v)$, that
is,}
\begin{equation}\smrev{\label{ec-mi-eq: s threshold}
\query>\frac{v}{B+c}L-\frac{c^2}{2(B+c)}.}
\end{equation}

\smrev{\emph{The probability of a large signal.} If $L\leq C_0$, the right-hand side of \Cref{ec-mi-eq: edge bound}
equals one, and the inequality holds trivially. Suppose that $L>C_0$. Now $\query=\truth+\sqrt{v}Z$ with
$\truth\leq B$, so $\sqrt vZ\geq \query-B$, and \Cref{ec-mi-eq: s threshold} implies}
\[
\smrev{Z>\frac{1}{\sqrt{v}}\left(\frac{v}{B+c}L-\frac{c^2}{2(B+c)}-B\right)\geq\frac{\sqrt{v}}{B+c}(L-C_0).}
\]
\smrev{The last inequality is equivalent to $\{(B+c)B+c^2/2\}/v\leq C_0$, which holds because $v\geq v_{\min}$. So
the constant $C_0$ absorbs the shift due to $\truth$ and $c$, and $\gamma$ is the resulting Gaussian tail rate. The
Gaussian tail bound $\PP(Z>x)\leq e^{-x^2/2}$ for $x\geq0$, together with $v\geq v_{\min}$, then gives}
\[
\smrev{\PP\left(g_v(\truth+\sqrt{v}Z)>a\right)\leq\exp\left\{-\frac{v}{2(B+c)^2}(L-C_0)^2\right\}
\leq\exp\left\{-\frac{v_{\min}}{2(B+c)^2}(L-C_0)^2\right\}=\exp\left\{-\gamma(L-C_0)^2\right\}.}
\]
\smrev{This is \Cref{ec-mi-eq: edge bound} when $L>C_0$.}
\Halmos
\end{proof}

\paragraph{\smrev{Finite-time collapse.}}

\begin{proof}{Proof of \Cref{thm:bounded-atomless-collapse}.}
\smrev{We follow fact (f) of the plan, which gives the idea of the argument; this proof gives the details. Assume,
toward a contradiction, that the loop never collapses. Three preliminary blocks fix the tie-breaking rules, define
the set $\Omega$ of limit points of the priors, and define their common radius $R$. We then treat the cases $R=0$ and
$R>0$, which are exhaustive, and show that each contradicts the assumption. The case $R>0$ takes six steps, listed
at its start.}

\smrev{\emph{Tie-breaking rules.} Fix $\costParam>0$ and, at each round, any tie-breaking rule. The argument below
does not depend on which rules are used. Such rules exist at every round, as shown in the paragraph ``Existence of
tie-breaking rules'' of the Setting.}

\smrev{\emph{Limit points of the priors.} Every $\pi_t$ is symmetric and supported on $[-B,B]$, by
``The priors of the loop'' before \Cref{ec-mi-lemma: vmin}. The space of probability distributions on $[-B,B]$
is compact
under weak convergence, by Prokhorov's theorem, and weak convergence on it is metrizable. Hence the set of limit points}
\[
\smrev{\Omega\triangleq\left\{\pi:\text{there exist }t_n\rightarrow\infty\text{ with }\pi_{t_n}\Rightarrow\pi\right\}}
\]
\smrev{is nonempty, by compactness, and closed, because a limit of limit points is a limit point (choose the times
along a diagonal sequence, using a metric for weak convergence). Hence $\Omega$ is compact. Weak limits of symmetric
distributions are symmetric, so every element of $\Omega$ is symmetric.}

\smrev{\emph{The common radius.} Define the common radius of the limit points by}
\begin{equation}\smrev{\label{ec-mi-eq: effective radius}
R\triangleq\inf\bigl\{r\geq0:\pi([-r,r])=1\text{ for every }\pi\in\Omega\bigr\}.}
\end{equation}
\smrev{The set in \Cref{ec-mi-eq: effective radius} contains $B$, and it contains every number larger than one of its
elements. Hence, for every $\eta>0$, every $\pi\in\Omega$ is supported on $[-R-\eta,R+\eta]$. Letting
$\eta\downarrow0$ shows that every $\pi\in\Omega$ is supported on $[-R,R]$, so the infimum is attained and $R$ is the
smallest such radius. Also, $0\leq R\leq B$.}

\smrev{\emph{Case $R=0$.} We show that this case contradicts the assumption that the loop never collapses. Since every
limit point is supported on $[-R,R]=\{0\}$, $\Omega=\{\delta_0\}$.
The full sequence $(\pi_t)$ then converges weakly to $\delta_0$. Otherwise, some subsequence would stay outside a weak
neighborhood of $\delta_0$; by compactness, it would have a further weakly convergent subsequence, whose limit would be
in $\Omega$ but differ from $\delta_0$. Since all $\pi_t$ are supported on $[-B,B]$, where $x\mapsto x^2$ is bounded
and continuous, $m_2(\pi_t)\rightarrow0$. For all sufficiently large $t$, $m_2(\pi_t)<\varepsilon_0$, with
$\varepsilon_0$ from \Cref{ec-mi-lemma: basin}. By that lemma, all users strictly prefer the uninformative signal under
the prior $\pi_t$, so $\pi_{t+1}=\delta_0$. This contradicts the assumption that the loop never collapses, so
$R>0$.}

\smrev{\emph{Case $R>0$.} We derive a contradiction in six steps. Step 1 fixes $c$ and $\delta$ and recalls the
central mass of \Cref{ec-mi-lemma: central mass} along the loop. Step 2 restates \Cref{ec-mi-lemma: edge} for the
cutoffs $R-r$, with constants that do not depend on $r$. Step 3 defines $M_r$ and shows that $0<M_r\leq r$. Steps 4 and
5 bound one round of the loop and pass to the limit, which gives \Cref{ec-mi-eq: scalar inequality}. Step 6 shows that
\Cref{ec-mi-eq: scalar inequality} fails for small $r$.}

\smrev{\emph{Step 1: central mass along the loop.} Set $c\triangleq R/2$. By
\Cref{ec-mi-lemma: central mass}, with $\delta\triangleq\delta_c>0$,}
\begin{equation}\smrev{\label{ec-mi-eq: orbit central}
\pi_t([-c,c])\geq\delta\qquad\text{for every }t\geq1.}
\end{equation}
\smrev{Step 4 applies \Cref{ec-mi-lemma: edge} to the priors $\pi_{t_n-1}$ of the loop, so this bound along the loop
is the one we need.}

\smrev{\emph{Step 2: the edge bound with constants that do not depend on $r$.} The cutoffs will be $a_r\triangleq R-r$
for $r\in(0,R/4)$. We apply \Cref{ec-mi-lemma: edge} with $a=a_r$ and replace the factor $(a_r-c)\delta$ in its bound
by a lower bound $K$ that does not depend on $r$. Since $0<c=R/2<B$, let $\gamma$ and $C_0$ be the constants of
\Cref{ec-mi-lemma: edge} for this $c$, and let $K\triangleq R\delta/4>0$. The right side of
\Cref{ec-mi-eq: edge bound}, with $(a-c)\delta$ replaced by $K$, is the function}
\[
\smrev{\Xi_K(\Lambda)\triangleq
\begin{cases}
0, & \Lambda=0,\\
\exp\left\{-\gamma\left[\ln(K/\Lambda)-C_0\right]_+^2\right\}, & \Lambda>0.
\end{cases}}
\]
\smrev{The function $\Xi_K$ is nondecreasing in $\Lambda$, because $\ln(K/\Lambda)$ decreases as $\Lambda$ grows, and
it is continuous on $[0,\infty)$, including at $0$, because $\ln(K/\Lambda)\to\infty$ as $\Lambda\downarrow0$.
Monotonicity will let us replace $\Lambda$ by an upper bound, and continuity will let us pass to weak limits.}

\smrev{Let $r\in(0,R/4)$. We first check that the cutoff $a_r$ is admissible in \Cref{ec-mi-lemma: edge} and that
$K$ is a lower bound on $(a_r-c)\delta$. First, $c<a_r<B$, because $a_r>3R/4>R/2$ and $a_r<R\leq B$. Second,
$a_r-c=R/2-r>R/4$, so}
\[
\smrev{(a_r-c)\delta>\frac{R\delta}4=K.}
\]

\smrev{Now let $\rho$ be a symmetric prior supported on $[-B,B]$ with $\rho([-c,c])\geq\delta$, and let a type
$\truth\in[-B,B]$ communicate with a finite variance $v\geq v_{\min}$. If $\Lambda_{a_r}(\rho)>0$,
\Cref{ec-mi-lemma: edge} with $a=a_r$ and the inequality $\ln((a_r-c)\delta/\Lambda)\geq\ln(K/\Lambda)$ for
$\Lambda>0$ give the bound below. If $\Lambda_{a_r}(\rho)=0$, both sides are $0$, by the convention of
\Cref{ec-mi-lemma: edge} and the definition of $\Xi_K$.}
\begin{equation}\smrev{\label{ec-mi-eq: edge bound K}
\PP\left(\truthAIAlg{\rho}(\truth+\sqrt{v}Z,\sqrt{v})>a_r\right)\leq\Xi_K\left(\Lambda_{a_r}(\rho)\right).}
\end{equation}
\smrev{The constants $c$, $\delta$, $\gamma$, $C_0$ and $K$ do not depend on $r$.}

\smrev{\emph{Step 3: the largest expected excess among the limit points.} Fix $r\in(0,R/4)$ and let}
\[
\smrev{M_r\triangleq\max_{\pi\in\Omega}\Lambda_{a_r}(\pi),\qquad\Lambda_{a_r}(\pi)=\int(x-a_r)_+\,\pi(dx).}
\]
\smrev{The maximum exists because $\Omega$ is compact and $\pi\mapsto\Lambda_{a_r}(\pi)$ is continuous under weak
convergence, since $x\mapsto(x-a_r)_+$ is bounded and continuous on $[-B,B]$. Since every $\pi\in\Omega$ is supported
on $[-R,R]$, where $(x-a_r)_+\leq r$,}
\begin{equation}\smrev{\label{ec-mi-eq: Mr range}
0\leq M_r\leq r.}
\end{equation}
\smrev{Moreover, $M_r>0$. If $M_r=0$, then no $\pi\in\Omega$ has mass above $R-r$; by symmetry, no $\pi\in\Omega$ has
mass below $-(R-r)$ either. Then every element of $\Omega$ is supported on $[-(R-r),R-r]$, which contradicts the
definition \Cref{ec-mi-eq: effective radius} of $R$.}

\smrev{\emph{Step 4: one round of the loop.} Let $\chi_r(y)\triangleq\min\{(y-a_r)_+,r\}$ be the excess of an output
$y$ beyond $a_r$, capped at $r$ for each output. Step 4 bounds $\int\chi_r\,d\pi_{t_n}$ by $r$ times the bound
\Cref{ec-mi-eq: edge bound K} evaluated at the previous prior $\pi_{t_n-1}$; this is \Cref{ec-mi-eq: ell bound}
below. It does so by bounding, for each user, the probability that
her output exceeds $a_r$.}

\smrev{We first choose the rounds. Choose $\nu_r\in\Omega$ with $\Lambda_{a_r}(\nu_r)=M_r$, and a sequence
$t_n\rightarrow\infty$ such that $\pi_{t_n}\Rightarrow\nu_r$. Passing to a subsequence if necessary, compactness gives
$\pi_{t_n-1}\Rightarrow\rho_r$ for some $\rho_r$, which belongs to $\Omega$ because $t_n-1\to\infty$.}

\smrev{The cap in $\chi_r$ is needed because the priors $\pi_{t_n}$ put mass on all of $(-B,B)$, where $(y-a_r)_+$
can be as large as $B-a_r$, which exceeds $r$ when $R<B$. When $R=B$, the cap never binds, and it does no harm. With
the cap,
$0\leq\chi_r(y)\leq r\,\1{y>a_r}$, so the integral of $\chi_r$ is at most $r$ times the probability of an output above
$a_r$. Since $\nu_r$ is supported on $[-R,R]$, where $(y-a_r)_+\leq r$, the cap does not bind on its support, and}
\begin{equation}\smrev{\label{ec-mi-eq: ell equals lambda}
\int\chi_r(y)\,\nu_r(dy)=\Lambda_{a_r}(\nu_r)=M_r.}
\end{equation}

\smrev{We now bound $\int\chi_r\,d\pi_{t_n}$ through the transition from $\pi_{t_n-1}$ to $\pi_{t_n}$. Take $n$ large
enough that $t_n-1\geq1$, so that \Cref{ec-mi-eq: orbit central} applies to the prior $\pi_{t_n-1}$. This prior is
also symmetric and supported on $[-B,B]$, by ``The priors of the loop'' before \Cref{ec-mi-lemma: vmin}, so
\Cref{ec-mi-eq: edge bound K} applies with $\rho=\pi_{t_n-1}$.}

\smrev{We bound the probability that the output $Y^{\truth}_{\pi_{t_n-1}}$ of a type $\truth\in[-B,B]$, given by
\Cref{ec-mi-eq: type output} with the rule $v^\star_{t_n-1}$, exceeds $a_r$. A user who sends the uninformative
signal outputs $0$, which does not exceed $a_r$ because $a_r>0$. A user who sends an informative signal has a finite
optimal variance of at least $v_{\min}$ (\Cref{ec-mi-lemma: vmin}), so \Cref{ec-mi-eq: edge bound K} applies to her.
In both cases,}
\[
\smrev{\PP\left(Y^{\truth}_{\pi_{t_n-1}}>a_r\right)\leq\Xi_K\left(\Lambda_{a_r}(\pi_{t_n-1})\right)
\qquad\text{for every }\truth\in[-B,B].}
\]
\smrev{The prior $\pi_{t_n}$ is the distribution of this output when $\truth\sim\truDist$. Integrating the previous
display over $\truth$ and using $0\leq\chi_r(y)\leq r\,\1{y>a_r}$ gives}
\begin{equation}\smrev{\label{ec-mi-eq: ell bound}
\int\chi_r\,d\pi_{t_n}\leq r\,\Xi_K\left(\Lambda_{a_r}(\pi_{t_n-1})\right).}
\end{equation}

\smrev{\emph{Step 5: passing to the limit.} As $n\to\infty$, the left side of \Cref{ec-mi-eq: ell bound} converges to
$\int\chi_r\,d\nu_r=M_r$, by \Cref{ec-mi-eq: ell equals lambda} and because $\chi_r$ is bounded and continuous. On the
right side, $\Lambda_{a_r}(\pi_{t_n-1})\to\Lambda_{a_r}(\rho_r)$, and $\Xi_K$ is continuous. Hence
$M_r\leq r\,\Xi_K(\Lambda_{a_r}(\rho_r))$. Since $\rho_r\in\Omega$, we have $\Lambda_{a_r}(\rho_r)\leq M_r$, and since
$\Xi_K$ is nondecreasing,}
\begin{equation}\smrev{\label{ec-mi-eq: Mr inequality}
M_r\leq r\,\Xi_K(M_r).}
\end{equation}
\smrev{Since $M_r>0$ (Step 3), the definition of $\Xi_K$ turns \Cref{ec-mi-eq: Mr inequality} into the inequality}
\begin{equation}\smrev{\label{ec-mi-eq: scalar inequality}
M_r\leq r\exp\left\{-\gamma\left[\ln\frac{K}{M_r}-C_0\right]_+^2\right\},}
\end{equation}
\smrev{which holds for every $r\in(0,R/4)$, with $K$, $\gamma$ and $C_0$ independent of $r$.}

\smrev{\emph{Step 6: the inequality \Cref{ec-mi-eq: scalar inequality} fails for small $r$.} We use the elementary
fact that $\gamma x^2-x\geq-1/(4\gamma)$ for every real $x$, with the minimum attained at $x=1/(2\gamma)$. Choose}
\[
\smrev{0<r<\min\left\{\frac R4,\ Ke^{-C_0-1/(4\gamma)}\right\},}
\]
\smrev{which implies $r<Ke^{-C_0}$. By \Cref{ec-mi-eq: Mr range}
and Step 3, $0<M_r\leq r<Ke^{-C_0}$, so
$x_r\triangleq\ln(K/M_r)-C_0>0$ and the positive part in \Cref{ec-mi-eq: scalar inequality} equals $x_r$. We rewrite
\Cref{ec-mi-eq: scalar inequality} in terms of $x_r$, so that the elementary fact applies. Since
$M_r=Ke^{-C_0}e^{-x_r}$, \Cref{ec-mi-eq: scalar inequality} reads $Ke^{-C_0}e^{-x_r}\leq r\,e^{-\gamma x_r^2}$.
Taking logarithms and rearranging,}
\[
\smrev{\gamma x_r^2-x_r+\ln\frac{K}{re^{C_0}}\leq0.}
\]
\smrev{However, $\gamma x_r^2-x_r\geq-1/(4\gamma)$ by the elementary fact above, while the choice of $r$ gives
$\ln\left(K/(re^{C_0})\right)>1/(4\gamma)$. So the left side is strictly positive, a
contradiction. Hence the case $R>0$ is impossible as well.}

\smrev{\emph{Conclusion.} Therefore the loop collapses in finite time: $\pi_T=\delta_0$ for some finite $T$.}
\Halmos
\end{proof}

\paragraph{\smrev{Where the communication cost enters.}} \smrev{Apart from its definition, which makes the cost of a
signal depend only on its variance and not on the user or the round, the proof uses the population
mutual-information cost only through the five properties proved in \Cref{ec-mi-lemma: cost}. We list the uses of
each.}
\begin{itemize}
\item \smrev{Item (i), finite cost that tends to $0$ as the noise grows. Used in \Cref{ec-mi-lemma: selector}, for
continuity of the loss at the uninformative signal.}
\item \smrev{Item (ii), monotonicity in the noise variance. Used in \Cref{ec-mi-lemma: basin}, to bound the cost from
below on the compact range $[v_{\min},\bar v_1]$.}
\item \smrev{Item (iii), divergence of the cost as the noise vanishes, for an atomless population. Used in
\Cref{ec-mi-lemma: vmin}, for the uniform lower bound $v_{\min}$ on finite optimal noise variances, which
\Cref{ec-mi-lemma: selector,ec-mi-lemma: basin,ec-mi-lemma: rescue,ec-mi-lemma: central mass,ec-mi-lemma: edge}
and Step 4 of the proof of \Cref{thm:bounded-atomless-collapse} all use.}
\item \smrev{Item (iv), a lower bound of order $1/v$ on the cost of noisy signals, which holds because even the sign
of a type drawn from a symmetric atomless population carries information. Used in three ways: in
\Cref{ec-mi-lemma: selector,ec-mi-lemma: basin}, for the strict positivity of the cost, which makes every finite
signal worse than the uninformative signal when the AI's prior is $\delta_0$ (the cost itself is computed under
$\truDist$);
in \Cref{ec-mi-lemma: basin}, for $c_1>0$ on the compact range $[v_{\min},\bar v_1]$; and in
\Cref{ec-mi-lemma: basin}, for the lower bound of order $1/v$ on the range $v\geq\bar v_1$.}
\item \smrev{Item (v), continuity in the precision. Used in \Cref{ec-mi-lemma: selector}, together with item (i) and
the bound $v_{\min}$, for the existence of optimal signals and of a tie-breaking rule.}
\end{itemize}
}

\rev{
\paragraph{\smrev{Populations with unbounded support: setting.}} \smrev{The rest of this subsection proves
\Cref{thm:unbounded-finite-variance-no-collapse,cor:two-sided-polynomial-tails}, whose populations may have
unbounded support. We keep the setting and notation stated at the start of \Cref{ec-sec: proofs unique key},
except for the changes and additions in the labelled paragraphs below. The two changes concern the model: the
support of the population may be unbounded, and the interaction cost may vary across rounds. The additions concern
notation and conventions: the optimal loss of a user, the communication cost as a function of the precision of the
signal, a refinement of the tie-breaking convention, and the existence of optimal signals.}

\smrev{\emph{First change: the support may be unbounded.} The population distribution $\truDist$ is symmetric, but
its support need not be bounded. We keep the notation $m_1=\EE{}{|\truth|}$ for its first absolute moment and
$\truSD^2=\EE{}{\truth^2}$ for its second moment, which is its variance when it is finite (the mean is then
$0$, by symmetry). Each result below states which of them it assumes finite.}

\smrev{\emph{Second change: the interaction cost may vary across rounds.} At round $t$, the AI's prior is $\pi_t$ and
the interaction cost is a number $\costParam_t\in(0,\infty)$. A user of type $\truth$ chooses a noise variance
$v\in(0,\infty]$ that minimizes her utility loss}
\begin{equation}\smrev{\label{ec-mm-eq: round loss}
\obj_{\pi_t,\costParam_t}(\truth,\sqrt{v})\triangleq
\underbrace{\costQual_{\pi_t}(\truth,\sqrt{v})}_{\text{fidelity error}}
+\costParam_t\underbrace{\costEffort(\sqrt{v})}_{\text{communication cost}},
\qquad
\obj_{\pi_t,\costParam_t}(\truth,\infty)\triangleq\truth^2.}
\end{equation}
\smrev{This is the loss of \Cref{ec-mi-eq: user loss} with $\costParam$ replaced by $\costParam_t$. The loop of
\Cref{sec:unique_key} is the case $\costParam_t\equiv\costParam$. The proofs below cover varying costs without
extra work, and \Cref{ec-sec: proofs mit-efficiency} uses this generality.}

\smrev{\emph{Optimal loss.} For a prior $\pi$ and a cost $\costParam\in(0,\infty)$ fixed by the context (along the
loop, the cost of the current round), we write}
\begin{equation}\smrev{\label{ec-mm-eq: optimal loss}
V_\pi(\truth)\triangleq\inf_{v\in(0,\infty]}\obj_{\pi,\costParam}(\truth,\sqrt{v})}
\end{equation}
\smrev{for the uncapped optimal AI loss of user $\truth$. It also depends on $\costParam$, which we omit from the
notation. It is not capped by a manual cost, because $\costHumanOnly=+\infty$ throughout this subsection. When
optimal signals exist (see the last paragraph of this setting), the infimum is a minimum.}

\smrev{\emph{New notation: the communication cost as a function of the precision.} We often write the
communication cost in terms of the precision $u=1/v$ of the signal, with $u=0$ standing for the uninformative
signal, because the continuity and divergence of the cost are proved in this variable
(\Cref{ec-mm-lemma: cost}):}
\[
\smrev{\mathcal{J}(u)\triangleq\costEffort(1/\sqrt{u})=I(\truth;\sqrt{u}\,\truth+Z)\quad(u>0),
\qquad\mathcal{J}(0)\triangleq0.}
\]
\smrev{Thus $\costEffort(\sqrt{v})=\mathcal{J}(1/v)$ for every finite $v>0$. The second equality in the display is
the rescaling used in the proof of \Cref{ec-mi-lemma: cost}(v).}

\smrev{\emph{Tie-breaking.} When several noise variances are optimal, we use the tie-breaking convention stated in
the Setting of \Cref{ec-sec: proofs unique key}. At round $t$, a tie-breaking rule is one for the pair
$(\pi_t,\costParam_t)$. We refine that convention in one respect: the chosen variance need only be optimal for
$\truDist$-almost every $\truth$, instead of for every type in the support of $\truDist$. On the remaining
$\truDist$-null set of types, the rule may choose any variance in $(0,\infty]$, as long as it stays measurable and
treats $\truth$ and $-\truth$ alike. Conclusions about the chosen variance of individual types are therefore stated
for $\truDist$-almost every $\truth$.}

\smrev{For an atomless population with $m_1<\infty$, the refinement is not needed for existence:
\Cref{ec-mm-lemma: wellposed} constructs rules that are optimal
for every type. It widens the class of rules that the results cover, at no cost. Indeed, the arguments below use
optimality only for $\truDist$-almost every type, and changing the chosen variances on a $\truDist$-null set of types
does not change the distribution of the output.}

\smrev{\emph{Existence of optimal signals.} For an atomless population with $m_1<\infty$,
\Cref{ec-mm-lemma: wellposed} below shows that, under any symmetric prior with a finite first moment, optimal
signals exist for every type and a tie-breaking rule exists. Part (ii) of that lemma shows that every prior along the
loop has these two properties, so a tie-breaking rule exists at every round. For a population with atoms, optimal
signals need not exist; when they fail to exist for a set of types of positive probability, no tie-breaking rule
exists (an example closes \Cref{ec-sec: proofs mit-efficiency}). So results that cover populations with atoms assume
that a tie-breaking rule exists at each round.}

\paragraph{\smrev{Plan of the proofs.}} \smrev{We first extend three tools of the bounded case to unbounded
support.}
\begin{itemize}
\item \smrev{\emph{The communication cost.} \Cref{ec-mm-lemma: cost} extends \Cref{ec-mi-lemma: cost}. When the
population is nondegenerate with a finite logarithmic moment, $\EE{}{\ln(1+|\truth|)}<\infty$, the communication
cost is finite, nondecreasing and continuous in the precision of the signal. When the population is also atomless,
the cost tends to infinity as the precision tends to infinity, that is, as the noise variance tends to zero.
\Cref{ec-mm-cor: heavy cost} checks these assumptions for the population of \Cref{cor:two-sided-polynomial-tails},
whose variance is infinite.}
\item \smrev{\emph{The AI's output.} \Cref{ec-mm-lemma: posterior} extends
\Cref{ec-mi-lemma: derivative,ec-mi-lemma: range,ec-mi-lemma: structure}. Under any prior and for every noise
variance $v\in(0,\infty)$, the AI's output is continuously differentiable in the signal. Under a symmetric
nondegenerate prior, it is odd and a strictly increasing bijection onto $J_\pi$, and the next prior has the form
described in \Cref{ec-mi-lemma: structure}. Under a prior with a finite first moment, the output grows at most
linearly in the signal, and its fidelity error vanishes as the noise variance tends to zero, for every user whose
preference is in the support of the prior.}
\item \smrev{\emph{Optimal signals.} \Cref{ec-mm-lemma: wellposed} extends \Cref{ec-mi-lemma: selector}: optimal
signals exist for every type, and tie-breaking rules exist. It also shows that the output of every optimal signal
has an expected squared distance to the user's preference of at most $\truth^2$, the fidelity error of the
uninformative signal.}
\end{itemize}

\smrev{\emph{The route to \Cref{thm:unbounded-finite-variance-no-collapse}.} We prove a stronger result,
\Cref{ec-mm-prop: noncollapse}: it needs only a finite first moment, and it allows the interaction cost to vary
across rounds. Its proof fixes one signal, with noise variance $1$. The communication cost of this signal is finite,
by \Cref{ec-mm-lemma: cost}. Its fidelity gain is the fidelity error $\truth^2$ of the uninformative signal minus the
fidelity error of this signal. Under any symmetric prior with unbounded support, this gain grows faster than
$|\truth|$ (\Cref{ec-mm-lemma: separation,ec-mm-lemma: gain}). In particular, it tends to infinity, which is what
the proof uses: users with large $|\truth|$ gain more from this signal than it costs them, so they send an
informative signal. By \Cref{ec-mm-lemma: posterior}(iii), once a positive mass of users sends an informative
signal, the support of the next prior is the closed convex hull of the support of the current prior. So the support
stays unbounded, and no round collapses.}

\smrev{\emph{The route to \Cref{cor:two-sided-polynomial-tails}.} We prove a stronger result,
\Cref{ec-mm-prop: heavy tails}: it allows the interaction cost to vary across rounds, and it also shows that the
output has a finite first moment and mean zero. Its proof is an induction over rounds on a class $\mathcal{T}_\zeta$
of symmetric priors, defined by two tail bounds in \Cref{ec-mm-eq: tail class}. The first bound says that the
probability of $\{|y|>x\}$ is at most a constant times $x^{-\zeta}$. The second says that every window of half-width
proportional to $\sqrt{\ln x}$ around a large point $x$ has probability at least a constant times $x^{-\zeta-1}$; it
keeps the denominator of the AI's posterior mean from being too small. Under a prior in $\mathcal{T}_\zeta$, the
AI's output differs from the signal by at most a constant times $\sqrt{\ln(e+|\query|)}$, uniformly over noise
variances in a compact subinterval of $(0,\infty)$ (\Cref{ec-mm-lemma: localization}).}

\smrev{The optimal signal has a utility loss at most that of the signal with noise variance $1$, whose communication
cost is finite. Averaging the bound of \Cref{ec-mm-lemma: localization} over the noise bounds the fidelity error of
the signal with noise variance $1$ by a constant times $\ln(e+|\truth|)$, and the comparison of utility losses then
gives the same type of bound for the optimal signal. So the root-mean-square distance between the output of user
$\truth$ and her preference is of order at most $\sqrt{\ln(e+|\truth|)}$, which is small relative to $|\truth|$.
Three consequences follow, each proved separately: the output has infinite variance, its distribution satisfies the
first tail bound, and it satisfies the second. So the next prior is again in $\mathcal{T}_\zeta$, and the induction
goes through.}

\paragraph{\smrev{The communication cost without bounded support.}} \smrev{The next lemma extends
\Cref{ec-mi-lemma: cost} to populations without bounded support. It assumes only a logarithmic moment, because the
population of \Cref{cor:two-sided-polynomial-tails} has infinite variance. Each item has a later use. The continuity
of $\mathcal{J}$, which item~(i) derives from its increment bound, is used for the existence of optimal signals
(\Cref{ec-mm-lemma: wellposed}). The divergence in item~(ii) keeps the optimal precisions bounded in that same
lemma. The bounds in item~(iii) are used when $\truSD^2<\infty$, in \Cref{ec-sec: proofs mit-efficiency}.}

\begin{lemma}[Population mutual information under a logarithmic moment]\label{ec-mm-lemma: cost}
\smrev{Suppose that $\truDist$ is nondegenerate and $\EE{}{\ln(1+|\truth|)}<\infty$; for example, $m_1<\infty$.}
\begin{enumerate}[label=\textup{(\roman*)}]
\item \smrev{$\mathcal{J}$ is finite and nondecreasing on $[0,\infty)$, with $\mathcal{J}(0)=0$ and $\mathcal{J}(u)>0$
for
$u>0$. For all $u,\delta\geq0$,}
\begin{equation}\smrev{\label{ec-mm-eq: cost uniform continuity}
0\leq\mathcal{J}(u+\delta)-\mathcal{J}(u)\leq\mathcal{J}(\delta),}
\end{equation}
\smrev{and $\mathcal{J}(\delta)\rightarrow0$ as $\delta\downarrow0$. Hence $\mathcal{J}$ is uniformly continuous on
$[0,\infty)$, and $\costEffort(\sqrt{v})\rightarrow0$ as $v\rightarrow\infty$.}
\item \smrev{If $\truDist$ is atomless, then $\mathcal{J}(u)\rightarrow\infty$ as $u\rightarrow\infty$, that is,
$\costEffort(\sqrt{v})\rightarrow\infty$ as $v\downarrow0$.}
\item \smrev{If $\truDist$ is symmetric and $\truSD^2<\infty$, then, for every $v\in(0,\infty)$},
\begin{equation}\smrev{\label{ec-mm-eq: cost bounds}
\frac{m_1^2}{2(v+\truSD^2)}\leq\costEffort(\sqrt{v})\leq\frac12\ln\left(1+\frac{\truSD^2}{v}\right).}
\end{equation}
\smrev{The lower bound allows an atom at $0$, which the proof of \Cref{ec-mi-lemma: cost}(iv) excludes.}
\end{enumerate}
\end{lemma}

\begin{proof}{Proof of \Cref{ec-mm-lemma: cost}.}
\smrev{We prove item~(i) in four steps: the communication cost is finite, it is nondecreasing in the precision and
positive, it vanishes as the precision goes to zero, and it satisfies the increment bound
\Cref{ec-mm-eq: cost uniform continuity}. We then prove items~(ii) and~(iii), mostly by reusing the arguments of
the proof of \Cref{ec-mi-lemma: cost}. Throughout, $\mathcal{H}(\cdot)$ denotes differential entropy, as in
\Cref{sec:model}, and for a finite variance $v>0$ we write $\query_v\triangleq\truth+\sqrt{v}Z$ for the signal and
$\tilde{\query}_u\triangleq\sqrt{u}\,\truth+Z$ for its rescaled version with precision $u=1/v$, so that
$\mathcal{J}(u)=I(\truth;\tilde{\query}_u)$.}

\smrev{\emph{Item (i), step 1: the cost is finite.} We show that $\costEffort(\sqrt{v})<\infty$ for every finite
$v>0$. Without bounded support or finite variance, the maximum-entropy bound used in the proof of
\Cref{ec-mi-lemma: cost}(i) is not available. So we first bound the differential entropy of any random variable
with a bounded density by its logarithmic moment, and then apply this bound to the signal.}

\smrev{Let $W$ be a real random variable with a density $f_W\leq\bar f<\infty$ and with
$\EE{}{\ln(1+|W|)}<\infty$. We show that $\mathcal{H}(W)=-\int f_W\ln f_W$ is finite, with}
\begin{equation}\smrev{\label{ec-mm-eq: reference entropy}
-\ln\bar f\leq\mathcal{H}(W)\leq\ln2+2\,\EE{}{\ln(1+|W|)}.}
\end{equation}
\smrev{We prove the two bounds in turn. For $\mathcal{H}(W)$ to be finite, both the positive and the negative part
of $f_W\ln f_W$ must be integrable. Since $\mathcal{H}(W)=-\int f_W\ln f_W$, the proof of the lower bound shows
along the way that the positive part is integrable, and the proof of the upper bound does the same for the negative
part.}

\smrev{\emph{Lower bound.} Because $f_W\leq\bar f$, the positive part of $f_W\ln f_W$ is at most $f_W\ln^+\bar f$,
where $\ln^+x\triangleq\max\{\ln x,0\}$; it is therefore integrable. Moreover, $\ln f_W\leq\ln\bar f$ wherever
$f_W>0$, so $\mathcal{H}(W)\geq-\int f_W\ln\bar f=-\ln\bar f$.}

\smrev{\emph{Upper bound.} We compare $f_W$ with the reference density $f_0(y)\triangleq\frac12(1+|y|)^{-2}$, a
probability density on $\mathbb{R}$ with $\ln f_0\leq0$. For $\alpha\geq0$ and $\beta>0$, the elementary inequality
$\alpha\ln(\alpha/\beta)-\alpha+\beta\geq0$ holds. Applied with $\alpha=f_W(y)$ and $\beta=f_0(y)$, and using $\ln
f_0\leq0$, it gives}
\[
\smrev{-f_W(y)\ln f_W(y)\leq f_W(y)\,|\ln f_0(y)|-f_W(y)+f_0(y).}
\]
\smrev{The negative part of $f_W\ln f_W$ is $\max\{-f_W\ln f_W,0\}$. By the display, it is at most
$\max\{f_W|\ln f_0|-f_W+f_0,0\}\leq f_W|\ln f_0|+f_0$. This function is integrable, since
$\int f_W|\ln f_0|=\ln2+2\,\EE{}{\ln(1+|W|)}$ and $\int f_0=1$. Hence $\mathcal{H}(W)$ is finite. Integrating the
pointwise inequality $-f_W\ln f_W\leq f_W|\ln f_0|-f_W+f_0$, and using $\int f_W=\int f_0=1$, gives the upper
bound in \Cref{ec-mm-eq: reference entropy}.}

\smrev{We apply this to $W=\query_v$. The signal $\query_v$ has a density bounded by $(2\pi v)^{-1/2}$, because its
density is an average of normal densities with variance $v$. Moreover, $1+|x+y|\leq(1+|x|)(1+|y|)$ for all real
$x,y$, so}
\[
\smrev{\EE{}{\ln(1+|\query_v|)}\leq\EE{}{\ln(1+|\truth|)}+\EE{}{\ln(1+\sqrt{v}|Z|)}<\infty.}
\]
\smrev{Hence $\mathcal{H}(\query_v)$ is finite. The conditional differential entropy of $\query_v$ given $\truth$ is
that of the noise, $\mathcal{H}(\sqrt{v}Z)=\frac12\ln(2\pi ev)$, which is also finite. Therefore the usual
identity for additive noise applies, and}
\[
\smrev{\costEffort(\sqrt{v})=\mathcal{H}(\query_v)-\tfrac12\ln(2\pi ev)<\infty.}
\]
\smrev{This proves that $\mathcal{J}$ is finite on $[0,\infty)$; recall that $\mathcal{J}(0)=0$ by definition.}

\smrev{\emph{Item (i), step 2: the cost is nondecreasing in the precision and positive.} The cost
$\costEffort(\sqrt{v})$ is nonincreasing in $v$ by the argument in the proof of \Cref{ec-mi-lemma: cost}(ii): a
signal with a larger noise variance has the same distribution as a signal with a smaller noise variance plus
independent normal noise, and the data-processing inequality applies. That argument does not use bounded support.
Hence $\mathcal{J}$ is nondecreasing.}

\smrev{We now show that $\mathcal{J}(u)>0$ for $u>0$. We use characteristic functions because they need no
moment of $\truth$. Suppose instead that $\costEffort(\sqrt{v})=0$ for some
finite $v$. Then $\query_v$ is independent of $\truth$. For every real $\omega$, the conditional characteristic
function of $\query_v$ given $\truth$ is the characteristic function of the normal noise $\sqrt{v}Z$, shifted by
$\truth$:}
\[
\smrev{\EE{}{e^{\mathrm{i}\omega{\query}_v}\,\big|\,\truth}=e^{\mathrm{i}\omega\truth-v\omega^2/2},}
\]
\smrev{and independence makes it equal to the unconditional characteristic function, almost surely. Since
$e^{-v\omega^2/2}\neq0$, this gives our first conclusion: for every rational $\omega$, $e^{\mathrm{i}\omega\truth}$
is almost surely constant. We use rational frequencies because there are countably many, so that one event of
probability one serves all of them.}

\smrev{We next show that $\truth$ coincides almost surely with an independent copy of itself. Let $\truth'$ be an
independent copy of $\truth$. For each rational $\omega$, $e^{\mathrm{i}\omega\truth}$ and
$e^{\mathrm{i}\omega\truth'}$ are almost surely equal to the same constant. Intersecting countably many events of
probability one, we get}
\[
\smrev{e^{\mathrm{i}\omega(\truth-\truth')}=1\quad\text{for every rational }\omega,\qquad\text{almost surely.}}
\]
\smrev{On this event, suppose that $d\triangleq\truth-\truth'\neq0$. Then $\omega=1/n$ with an integer
$n>|d|/(2\pi)$ gives $0<|\omega d|<2\pi$, so $e^{\mathrm{i}\omega d}\neq1$, a contradiction. Hence
$\truth=\truth'$ almost surely.}

\smrev{Finally, we show that this forces $\truth$ to be constant. If $\truth=\truth'$ almost surely, then
$\PP(\truth\in E)=\PP(\truth\in E,\truth'\in E)=\PP(\truth\in E)^2$ for every Borel set $E$, so every such
probability is $0$ or $1$, and $\truth$ is almost surely constant. This contradicts the nondegeneracy of
$\truDist$. Therefore $\mathcal{J}(u)>0$ for every $u>0$.}

\smrev{\emph{Item (i), step 3: the cost vanishes as the precision goes to zero.} We show that
$\mathcal{J}(u)\rightarrow0$ as $u\downarrow0$. The idea is to split the types into a bounded part, for which the
normal maximum-entropy bound applies, and a tail part of small probability, for which the logarithmic-moment bound
\Cref{ec-mm-eq: reference entropy} applies.}

\smrev{For $M>0$, let $\chi_M\triangleq\1{|\truth|>M}$ and $p_M\triangleq\PP(|\truth|>M)$, with $M$ large enough
that $p_M<1$. If $p_M=0$, the terms conditional on $\chi_M=1$ below are absent (they are multiplied by $p_M=0$).
Fix $0<u\leq1$. Since $\chi_M$ is a function of $\truth$, the chain rule for mutual information gives}
\[
\smrev{\mathcal{J}(u)=I(\truth,\chi_M;\tilde{\query}_u)
=I(\chi_M;\tilde{\query}_u)+I(\truth;\tilde{\query}_u\,|\,\chi_M)
\leq H_b(p_M)+(1-p_M)\,I(\truth;\tilde{\query}_u\,|\,\chi_M=0)+p_M\,I(\truth;\tilde{\query}_u\,|\,\chi_M=1),}
\]
\smrev{where $H_b$ is the binary entropy, which bounds $I(\chi_M;\tilde{\query}_u)$.}

\smrev{Given $\chi_M=0$, we have $|\truth|\leq M$, so the second moment of $\tilde{\query}_u$ is at most $1+uM^2$. The
normal distribution with mean zero and that second moment has the largest differential entropy among all
distributions whose second moment is at most $1+uM^2$, so}
\[
\smrev{I(\truth;\tilde{\query}_u\,|\,\chi_M=0)\leq\tfrac12\ln\bigl(2\pi e(1+uM^2)\bigr)-\tfrac12\ln(2\pi e)
=\tfrac12\ln(1+uM^2).}
\]

\smrev{Given $\chi_M=1$ (when $p_M>0$), we apply \Cref{ec-mm-eq: reference entropy} under the
conditional law given $\chi_M=1$, under which $\tilde{\query}_u$ still has a density bounded by $(2\pi)^{-1/2}$.
Since $u\leq1$, we have $1+|\tilde{\query}_u|\leq(1+|\truth|)(1+|Z|)$. Hence the conditional
differential entropy of $\tilde{\query}_u$ is at most
$\ln2+2\,\EE{}{\ln(1+|\truth|)\,|\,\chi_M=1}+2\,\EE{}{\ln(1+|Z|)}$, and subtracting the entropy
$\frac12\ln(2\pi e)$ of the noise $Z$ gives}
\[
\smrev{p_M\,I(\truth;\tilde{\query}_u\,|\,\chi_M=1)\leq C_Zp_M+2\,\EE{}{\ln(1+|\truth|)\1{|\truth|>M}},
\qquad
C_Z\triangleq\left(\ln2+2\,\EE{}{\ln(1+|Z|)}-\tfrac12\ln(2\pi e)\right)_+,}
\]
\smrev{where $(x)_+\triangleq\max\{x,0\}$.}

\smrev{Combining the three bounds and letting $u\downarrow0$, so that $\frac12\ln(1+uM^2)\rightarrow0$, we get, for
every $M$ with $p_M<1$,}
\[
\smrev{\limsup_{u\downarrow0}\mathcal{J}(u)\leq H_b(p_M)+C_Zp_M+2\,\EE{}{\ln(1+|\truth|)\1{|\truth|>M}}.}
\]
\smrev{As $M\rightarrow\infty$, $p_M\rightarrow0$, so $H_b(p_M)\rightarrow0$, and the last term tends to zero by
dominated convergence, because $\EE{}{\ln(1+|\truth|)}<\infty$. Hence $\mathcal{J}(u)\rightarrow0$ as
$u\downarrow0$, that is, $\costEffort(\sqrt{v})\rightarrow0$ as $v\rightarrow\infty$.}

\smrev{\emph{Item (i), step 4: the increment bound.} We prove \Cref{ec-mm-eq: cost uniform continuity}. If $u=0$
or $\delta=0$, it holds trivially, because $\mathcal{J}(0)=0$ and $\mathcal{J}$ is nondecreasing. Fix $u,\delta>0$.
The idea is that a signal of precision $u+\delta$ carries the same information as two signals of precisions $u$
and $\delta$ whose noises are independent, so that the two signals are conditionally independent given $\truth$.}

\smrev{Let $\tilde{\query}_u=\sqrt{u}\,\truth+Z_1$ and $\tilde{\query}'_\delta\triangleq\sqrt{\delta}\,\truth+Z_2$,
with $Z_1,Z_2$ independent standard normal variables, independent of $\truth$. The two combinations}
\[
\smrev{\frac{\sqrt{u}\,\tilde{\query}_u+\sqrt{\delta}\,\tilde{\query}'_\delta}{\sqrt{u+\delta}}
=\sqrt{u+\delta}\,\truth+Z',
\qquad
\frac{\sqrt{\delta}\,\tilde{\query}_u-\sqrt{u}\,\tilde{\query}'_\delta}{\sqrt{u+\delta}}
=\frac{\sqrt{\delta}\,Z_1-\sqrt{u}\,Z_2}{\sqrt{u+\delta}}\triangleq Z'',}
\]
\smrev{with $Z'\triangleq(\sqrt{u}\,Z_1+\sqrt{\delta}\,Z_2)/\sqrt{u+\delta}$, are an invertible linear function of
$(\tilde{\query}_u,\tilde{\query}'_\delta)$. The first is a signal of precision $u+\delta$. The second, $Z''$, is pure
noise: a standard
normal variable that is uncorrelated with $Z'$, hence independent of $(\truth,Z')$, because $(Z_1,Z_2)$ is jointly
normal and independent of $\truth$. So the pair $(\tilde{\query}_u,\tilde{\query}'_\delta)$ carries exactly the
information about $\truth$ of a signal of precision $u+\delta$, and the chain rule gives}
\[
\smrev{\mathcal{J}(u+\delta)=I(\truth;\tilde{\query}_u,\tilde{\query}'_\delta)
=\mathcal{J}(u)+I(\truth;\tilde{\query}'_\delta\,|\,\tilde{\query}_u).}
\]
\smrev{It remains to bound the last term by $\mathcal{J}(\delta)$. The signals $\tilde{\query}_u$ and
$\tilde{\query}'_\delta$ are conditionally independent given $\truth$, so
$I(\tilde{\query}_u;\tilde{\query}'_\delta\,|\,\truth)=0$, and two applications of the chain rule to
$I(\truth,\tilde{\query}_u;\tilde{\query}'_\delta)$ give}
\[
\smrev{0\leq I(\truth;\tilde{\query}'_\delta\,|\,\tilde{\query}_u)
=I(\truth;\tilde{\query}'_\delta)-I(\tilde{\query}_u;\tilde{\query}'_\delta)\leq\mathcal{J}(\delta).}
\]
\smrev{All these quantities are finite: the single-signal terms by step~1, and the two-signal term because it equals
$\mathcal{J}(u+\delta)$. This proves \Cref{ec-mm-eq: cost uniform continuity}. Since
$\mathcal{J}(\delta)\rightarrow0$ as $\delta\downarrow0$ by step~3, the increment bound makes $\mathcal{J}$ uniformly
continuous on $[0,\infty)$. This proves item~(i).}

\smrev{\emph{Item (ii).} Item~(ii) states that $\mathcal{J}(u)\rightarrow\infty$ as $u\rightarrow\infty$ when
$\truDist$ is atomless. The Fano argument in the proof of \Cref{ec-mi-lemma: cost}(iii) applies verbatim, with the
outer thresholds $x_0=-\infty$ and $x_N=+\infty$. The outer thresholds play no role in that argument: the
estimate assigns signals below $x_1$ to the first cell and signals above $x_{N-1}$ to the last, and the error bound
\Cref{ec-mi-eq: fano error} involves only the internal thresholds. The internal thresholds
$x_i\triangleq\inf\{x:\truDist((-\infty,x])\geq i/N\}$, $1\leq i\leq N-1$, are finite. Because $\truDist$ is
atomless, its distribution function is continuous, so these thresholds are strictly increasing and cut
$\mathbb{R}$ into $N$ cells of probability $1/N$ each. That argument uses no moment of $\truDist$. This proves
item~(ii).}

\smrev{\emph{Item (iii).} Item~(iii) states the two bounds of \Cref{ec-mm-eq: cost bounds} when $\truDist$ is
symmetric with $\truSD^2<\infty$. The upper bound follows from the normal maximum-entropy inequality, as in the proof
of \Cref{ec-mi-lemma: cost}(i), whose upper bound holds for every population with finite variance.}

\smrev{For the lower bound, we follow the proof of \Cref{ec-mi-lemma: cost}(iv). The only change is that
$\truDist$ may now have an atom at $0$, where $\operatorname{sgn}(\truth)=0$ is not a fair $\pm1$ sign. Let
$\xi\triangleq\operatorname{sgn}(\truth)$ when $\truth\neq0$, and when $\truth=0$, let $\xi$ be the outcome of
a fair coin with values $\pm1$, tossed independently of $(\truth,Z)$. Then $\xi$ is equiprobable, because $\truDist$
is symmetric. Given
$\truth$, the sign $\xi$ and the
signal $\query_v$ are independent, so $\xi\rightarrow\truth\rightarrow{\query}_v$ is a Markov chain, and data processing
gives $\costEffort(\sqrt{v})\geq I(\xi;\query_v)$. Finally, $\truth\xi=|\truth|$ and $Z$ is independent of
$(\truth,\xi)$, so $\EE{}{\query_v\,\xi}=\EE{}{|\truth|}=m_1$. The rest of that proof applies unchanged and gives the
lower bound. This proves item~(iii).}
\Halmos
\end{proof}

\begin{corollary}[The communication cost under two-sided polynomial tails]\label{ec-mm-cor: heavy cost}
\smrev{Suppose $\truDist$ satisfies the assumptions of \Cref{cor:two-sided-polynomial-tails}. Then
$\truDist(\{|\truth|>x\})\leq(2c_+/\zeta)x^{-\zeta}$ for $x\geq A$, $m_1<\infty$, and $\EE{}{\truth^2}=\infty$.
Moreover, $0<\mathcal{J}(u)<\infty$ for $u>0$, $\mathcal{J}$ is continuous on $[0,\infty)$, and
$\mathcal{J}(u)\rightarrow\infty$ as $u\rightarrow\infty$. In terms of the noise variance,
$\costEffort(\sqrt{v})$ is finite at every finite $v>0$, tends to $0$ as $v\rightarrow\infty$, and diverges as
$v\downarrow0$}.
\end{corollary}

\begin{proof}{Proof of \Cref{ec-mm-cor: heavy cost}.}
\smrev{We first prove the three moment and tail claims from the density bounds of
\Cref{cor:two-sided-polynomial-tails}, and then obtain the claims on the cost from \Cref{ec-mm-lemma: cost}.}

\smrev{\emph{Tail bound.} For $x\geq A$, the upper density bound and the symmetry of $\truDist$ give}
\[
\smrev{\truDist(\{|\truth|>x\})=2\int_x^\infty p(\truth)\,d\truth\leq2c_+\int_x^\infty\truth^{-\zeta-1}\,d\truth
=\frac{2c_+}{\zeta}\,x^{-\zeta}.}
\]

\smrev{\emph{Finite first moment.} Writing the first absolute moment as the integral of the tail, and bounding the
tail by $1$ on $[0,A]$ and by the tail bound beyond $A$,}
\[
\smrev{m_1=\int_0^\infty\truDist(\{|\truth|>x\})\,dx\leq A+\frac{2c_+}{\zeta}\int_A^\infty x^{-\zeta}\,dx<\infty,}
\]
\smrev{because $\zeta>1$. In particular, $\EE{}{\ln(1+|\truth|)}\leq m_1<\infty$.}

\smrev{\emph{Infinite variance.} The lower density bound gives}
\[
\smrev{\EE{}{\truth^2}\geq2\int_A^\infty\truth^2\,c_-\truth^{-\zeta-1}\,d\truth
=2c_-\int_A^\infty\truth^{1-\zeta}\,d\truth
=\infty,}
\]
\smrev{because $\zeta<2$.}

\smrev{\emph{The communication cost.} Since $\truDist$ has a density, it is atomless, hence nondegenerate. Together
with the finite logarithmic moment, this gives the assumptions of \Cref{ec-mm-lemma: cost}, whose items~(i) and~(ii)
give the claims on $\mathcal{J}$.}
\Halmos
\end{proof}

\paragraph{\smrev{Posterior means without bounded support.}} \smrev{The next lemma extends
\Cref{ec-mi-lemma: derivative,ec-mi-lemma: range,ec-mi-lemma: structure} to priors whose support may be unbounded.
This is needed because, under the populations of
\Cref{thm:unbounded-finite-variance-no-collapse,cor:two-sided-polynomial-tails},
every prior along the loop has unbounded support, as the proofs of
\Cref{ec-mm-prop: noncollapse,ec-mm-prop: heavy tails}
show.}

\begin{lemma}[Posterior means for general priors]\label{ec-mm-lemma: posterior}
\smrev{Let $\pi$ be a distribution on $\mathbb{R}$, and write $g_v(\query)\triangleq\truthAIAlg{\pi}(\query,\sqrt{v})$
for
$v\in(0,\infty)$.}
\begin{enumerate}[label=\textup{(\roman*)}]
\item \smrev{The posterior mean $g_v(\query)$ is well defined and continuously differentiable in $\query$, the maps
$(\query,v)\mapsto g_v(\query)$ and $(\query,v)\mapsto g_v'(\query)$ are continuous on $\mathbb{R}\times(0,\infty)$, and
\Cref{ec-mi-eq: posterior derivative} holds. If $\pi$ is symmetric and nondegenerate, then $g_v$ is odd and is a
strictly increasing continuous bijection from $\mathbb{R}$ onto $J_\pi$, which may be all of $\mathbb{R}$. No moment
of $\pi$ is needed.}
\item \smrev{If $m_\pi\triangleq\int|x|\,\pi(dx)<\infty$, then, for any $R\geq0$ with $q\triangleq\pi([-R,R])>0$,}
\begin{equation}\smrev{\label{ec-mm-eq: growth}
|g_v(\query)|\leq2|\query|+R+m_\pi/q\qquad(\query\in\mathbb{R},\ v\in(0,\infty)).}
\end{equation}
\smrev{Moreover, $\costQual_\pi(\truth,\sqrt{v})\rightarrow0$ as $v\downarrow0$ for every $\truth$ in the support of
$\pi$.}
\item \smrev{If $\pi$ is symmetric and nondegenerate, the conclusions of \Cref{ec-mi-lemma: structure} hold with $J_\pi$
possibly unbounded: for any measurable variance rule $\hat v$ with $\truDist(\hat v<\infty)>0$, the next distribution
is $\pi_+=p_+\,\delta_0+f_+(z)\,dz$ with $p_+=\truDist(\hat v=\infty)$, where $f_+$ is a finite Borel density that is
strictly positive on $J_\pi$ and zero outside it. Users who send an informative signal output $0$ with probability
$0$, and the support of $\pi_+$ is $H_\pi$.}
\end{enumerate}
\end{lemma}

\begin{proof}{Proof of \Cref{ec-mm-lemma: posterior}.}
\smrev{We prove the three items in order. Item~(i) repeats the computation of the proof of
\Cref{ec-mi-lemma: derivative}, with a different justification for differentiating under the integral sign, and
identifies the range of the AI's output through a convexity argument. Item~(ii) has two parts: the growth bound
comes from splitting the posterior first absolute moment, and the limit of the fidelity error comes from showing that
the posterior concentrates near the preference and then applying dominated convergence with the growth bound.
Item~(iii) checks that the proof of
\Cref{ec-mi-lemma: structure} goes through.}

\smrev{\emph{Item (i): regularity and the derivative.} Item~(i) first states that $g_v$ is well defined,
continuously differentiable and jointly continuous with its derivative, and that
\Cref{ec-mi-eq: posterior derivative} holds; the range claim is proved in the next block. As in the
proof of \Cref{ec-mi-lemma: derivative}, the factor $e^{-\query^2/(2v)}$ cancels from the numerator and the
denominator of \Cref{ec-mi-eq: posterior mean}, so that $g_v=N_1/N_0$, where, for the fixed prior $\pi$ and
variance $v$ (we omit $v$ from the notation $N_k$, although the argument below varies it),}
\[
\smrev{N_k(\query)\triangleq\int x^k\exp\left\{\frac{\query x}{v}-\frac{x^2}{2v}\right\}\pi(dx),\qquad k=0,1,2.}
\]
\smrev{The support of $\pi$ need not be bounded, so we justify differentiation under the integral sign differently. On a
compact set of pairs $(\query,v)$ with $v$ bounded away from zero, the function
$x\mapsto|x|^k\exp\{\query x/v-x^2/(2v)\}$ is bounded uniformly in $x$ and in $(\query,v)$, because the term
$-x^2/(2v)$ dominates. Since $\pi$ is a probability distribution, dominated convergence (with a constant dominating
function) shows that each $N_k$ is finite, that $N_0'=N_1/v$ and $N_1'=N_2/v$, and that $N_0,N_1,N_2$ are jointly
continuous in $(\query,v)$. Since $N_0>0$, the computation in the proof of \Cref{ec-mi-lemma: derivative} then gives
\Cref{ec-mi-eq: posterior derivative}, that is, $g_v'(\query)=\VV_\pi(X\,|\,X+\sqrt{v}Z=\query)/v$, and the joint
continuity of $g_v$ and $g_v'$. No moment of $\pi$ is needed here: the posterior moments $N_k/N_0$, $k\leq2$, are
finite by the bound above, even if $\pi$ has no mean. (For the uninformative signal, the AI's output is $0$, by the
convention for symmetric priors stated in the Setting of \Cref{ec-sec: proofs unique key}; when $\pi$ has no mean,
this is a convention, and the priors along the loop below all have a finite mean.)}

\smrev{The posterior distribution has a positive density relative to $\pi$, so it is equivalent to $\pi$ (it has the
same null sets). If $\pi$ is
nondegenerate, the posterior is nondegenerate too, its variance is positive, and $g_v'>0$ by
\Cref{ec-mi-eq: posterior derivative}; thus $g_v$ is strictly increasing. If $\pi$ is symmetric, replacing $x$ by $-x$
shows that $N_0$ is even and $N_1$ is odd, so $g_v$ is odd.}

\smrev{\emph{Item (i): the range.} Item~(i) also states that, if $\pi$ is symmetric and nondegenerate, then $g_v$
is a strictly increasing continuous bijection from $\mathbb{R}$ onto $J_\pi$. Let $\pi$ be symmetric and
nondegenerate, and let $h_\pi\triangleq\sup H_\pi$, with $h_\pi\in(0,\infty]$. If $h_\pi<\infty$, then
$H_\pi=[-h_\pi,h_\pi]$ and $J_\pi=(-h_\pi,h_\pi)$; if $h_\pi=\infty$, then $H_\pi=J_\pi=\mathbb{R}$, and we write
$(-h_\pi,h_\pi)=\mathbb{R}$. Since $g_v$ is odd and strictly increasing by the previous block, it suffices to show
that}
\[
\smrev{g_v(\query)\rightarrow h_\pi\qquad\text{as }\query\rightarrow\infty.}
\]

\smrev{To do so, we write $g_v$ as the derivative of a convex function. Let}
\[
\smrev{\mathcal{A}_v(\omega)\triangleq\ln\int e^{\omega x-x^2/(2v)}\,\pi(dx),\qquad\omega\in\mathbb{R},}
\]
\smrev{so that $\mathcal{A}_v(\query/v)=\ln N_0(\query)$. By the same domination argument, $\mathcal{A}_v$ is finite
and smooth. Its derivative $\mathcal{A}_v'(\omega)$ is the mean of the distribution proportional to
$e^{\omega x-x^2/(2v)}\pi(dx)$, and its second derivative is the variance of that distribution, so
$\mathcal{A}_v$ is convex and $\mathcal{A}_v'$ is nondecreasing. At $\omega=\query/v$, this distribution is the
posterior given the signal $\query$, so $g_v(\query)=\mathcal{A}_v'(\query/v)$. It therefore suffices to show
that $\mathcal{A}_v'(\omega)\rightarrow h_\pi$ as $\omega\rightarrow\infty$.}

\smrev{We first bound $\mathcal{A}_v$ from below. Fix $r\in(0,h_\pi)$. Because $r<\sup H_\pi$, there is $r'>r$
such that the bounded interval $[r,r']$ has positive probability under $\pi$. For $\omega\geq0$ and $x\in[r,r']$,
$\omega x-x^2/(2v)\geq\omega r-r'^2/(2v)$, so}
\begin{equation}\smrev{\label{ec-mm-eq: log-mgf lower}
\mathcal{A}_v(\omega)\geq\omega r-c_r\quad(\omega\geq0),\qquad c_r\triangleq\frac{r'^2}{2v}-\ln\pi([r,r'])\quad
\text{(a constant that depends on $\pi$, $r$, $r'$ and $v$)},}
\end{equation}
\smrev{and therefore $\liminf_{\omega\rightarrow\infty}\mathcal{A}_v(\omega)/\omega\geq r$. This holds for every
$r\in(0,h_\pi)$.}

\smrev{We now show that the limit of $\mathcal{A}_v'$ is at least $h_\pi$. Since $\mathcal{A}_v'$ is nondecreasing,
it has a limit $\mathcal{A}_v'(\infty)\in(-\infty,\infty]$ as $\omega\rightarrow\infty$, and
$\mathcal{A}_v'\leq\mathcal{A}_v'(\infty)$ everywhere. Suppose that $\mathcal{A}_v'(\infty)<h_\pi$, and choose
$r\in(\max\{\mathcal{A}_v'(\infty),0\},h_\pi)$. Then $\mathcal{A}_v(\omega)\leq\mathcal{A}_v(0)
+\mathcal{A}_v'(\infty)\,\omega$ for $\omega\geq0$, so}
\[
\smrev{\limsup_{\omega\rightarrow\infty}\frac{\mathcal{A}_v(\omega)}{\omega}\leq\mathcal{A}_v'(\infty)<r,}
\]
\smrev{which contradicts \Cref{ec-mm-eq: log-mgf lower}. Hence $\mathcal{A}_v'(\infty)\geq h_\pi$.}

\smrev{It remains to bound the limit from above. If $h_\pi=\infty$, there is nothing to prove. If $h_\pi<\infty$,
then $\mathcal{A}_v'\leq h_\pi$, because $\mathcal{A}_v'(\omega)$ is the mean of a distribution supported in
$H_\pi=[-h_\pi,h_\pi]$. In both cases, $\mathcal{A}_v'(\infty)=h_\pi$, so $g_v(\query)\rightarrow h_\pi$ as
$\query\rightarrow\infty$, and, since $g_v$ is odd, $g_v(\query)\rightarrow-h_\pi$ as $\query\rightarrow-\infty$.}

\smrev{Since $g_v$ is continuous and strictly increasing, with limits $\pm h_\pi$ at $\pm\infty$, its values lie
strictly between $-h_\pi$ and $h_\pi$, and the intermediate value theorem shows that $g_v$ is a bijection from
$\mathbb{R}$ onto $(-h_\pi,h_\pi)=J_\pi$. This proves item~(i).}

\smrev{\emph{Item (ii): the growth bound.} Item~(ii) first states the growth bound
$|g_v(\query)|\leq2|\query|+R+m_\pi/q$ of \Cref{ec-mm-eq: growth}, when $m_\pi<\infty$ and $q=\pi([-R,R])>0$. Fix
$\query$ and $v$. Since $|\query-x|\leq|\query|+R$ for $|x|\leq R$, the
denominator of \Cref{ec-mi-eq: posterior mean} satisfies}
\[
\smrev{\int e^{-(\query-x)^2/(2v)}\,\pi(dx)\geq q\,e^{-(|\query|+R)^2/(2v)}.}
\]
\smrev{On the set $\{|x|>2|\query|+R\}$, we have $|x-\query|\geq|x|-|\query|>|\query|+R$, so the Gaussian kernel
$e^{-(\query-x)^2/(2v)}$ is at most $e^{-(|\query|+R)^2/(2v)}$ there. We split the posterior first absolute moment
at $|x|=2|\query|+R$. The part with $|x|\leq2|\query|+R$ contributes at most $2|\query|+R$. The part with
$|x|>2|\query|+R$ contributes at most}
\[
\smrev{\frac{\int|x|\,e^{-(|\query|+R)^2/(2v)}\,\pi(dx)}{q\,e^{-(|\query|+R)^2/(2v)}}\leq\frac{m_\pi}{q}.}
\]
\smrev{Since $|g_v(\query)|$ is at most the posterior first absolute moment, \Cref{ec-mm-eq: growth} follows. The
same proof gives a common constant for a family of priors with a common first-moment bound and a common positive
lower bound on the probability of a fixed interval $[-R,R]$; \Cref{ec-mm-lemma: compact recovery}, which needs bounds
that are uniform
over such a family of priors, uses this.}

\smrev{\emph{Item (ii): the fidelity error vanishes for precise signals.} Item~(ii) also states that
$\costQual_\pi(\truth,\sqrt{v})\rightarrow0$ as $v\downarrow0$ for every $\truth$ in the support of $\pi$. We first
show that the AI's output converges to $\truth$ for each fixed value $z$ of the standardized noise $Z$, and then
apply dominated convergence.}

\smrev{\emph{The output converges for each noise value.} Fix $\truth$ in the support of $\pi$, $\varepsilon>0$ and
$z\in\mathbb{R}$, and let $\query=\truth+\sqrt{v}z$. We bound $|g_v(\query)-\truth|$ for small $v$. For all
sufficiently small $v$,}
\[
\smrev{|\query-\truth|=\sqrt{v}\,|z|\leq\varepsilon/8,}
\]
\smrev{and we assume this from now on.}

\smrev{We first bound the denominator of \Cref{ec-mi-eq: posterior mean} from below. Because $\truth$ is in the
support of $\pi$, the interval $[\truth-\varepsilon/4,\truth+\varepsilon/4]$ has probability
$p_{\truth,\varepsilon}\triangleq\pi([\truth-\varepsilon/4,\truth+\varepsilon/4])>0$, which depends on $\truth$ and
$\varepsilon$ but not on $v$. On this interval, $|x-\query|\leq\varepsilon/4+\varepsilon/8=3\varepsilon/8$, so}
\[
\smrev{\int e^{-(\query-x)^2/(2v)}\,\pi(dx)\geq p_{\truth,\varepsilon}\,e^{-(3\varepsilon/8)^2/(2v)}.}
\]

\smrev{We then split the numerator. We write $g_v(\query)-\truth$ as the posterior mean of $X-\truth$ and split it
at $|x-\truth|=\varepsilon$. The part with $|x-\truth|\leq\varepsilon$ contributes at most $\varepsilon$ in absolute
value. On the far set $\{|x-\truth|>\varepsilon\}$, we have $|x-\query|\geq\varepsilon-\varepsilon/8=7\varepsilon/8$,
so the kernel is bounded there:}
\[
\smrev{e^{-(\query-x)^2/(2v)}\leq e^{-(7\varepsilon/8)^2/(2v)}\qquad(|x-\truth|>\varepsilon).}
\]
\smrev{Since $\int|x-\truth|\,\pi(dx)\leq m_\pi+|\truth|$, the far part contributes, in absolute value, at most
$(m_\pi+|\truth|)\,e^{-(7\varepsilon/8)^2/(2v)}$ divided by the lower bound on the denominator. Since
$(7/8)^2-(3/8)^2=5/8$, the two parts together give}
\[
\smrev{|g_v(\truth+\sqrt{v}z)-\truth|\leq\varepsilon+\frac{m_\pi+|\truth|}{p_{\truth,\varepsilon}}
\exp\left\{-\frac{5\varepsilon^2}{16v}\right\}.}
\]
\smrev{Letting $v\downarrow0$ and then $\varepsilon\downarrow0$ shows that $g_v(\truth+\sqrt{v}z)\rightarrow\truth$
for every $z$.}

\smrev{For the domination, fix any $R\geq0$ with $q=\pi([-R,R])>0$. For $v\leq1$, \Cref{ec-mm-eq: growth} and
$|\truth+\sqrt{v}Z|\leq|\truth|+|Z|$ give}
\[
\smrev{|g_v(\truth+\sqrt{v}Z)-\truth|\leq3|\truth|+2|Z|+R+m_\pi/q,}
\]
\smrev{whose square is integrable. Dominated convergence gives
$\costQual_\pi(\truth,\sqrt{v})=\EE{}{(g_v(\truth+\sqrt{v}Z)-\truth)^2}\rightarrow0$ as $v\downarrow0$. This proves
item~(ii).}

\smrev{\emph{Item (iii).} Item~(iii) states that \Cref{ec-mi-lemma: structure} holds for a symmetric nondegenerate
prior whose support may be unbounded. The proof of \Cref{ec-mi-lemma: structure} uses bounded support only through
\Cref{ec-mi-lemma: range} (the AI's output is a strictly increasing continuous bijection onto $J_\pi$) and through
the continuity and positivity statements of \Cref{ec-mi-lemma: derivative}. Item~(i) supplies all of these, with
$J_\pi$ possibly equal to $\mathbb{R}$, so the rest of that proof applies verbatim and gives every conclusion of
item~(iii) except the last one, on the support of $\pi_+$.}

\smrev{Finally, every output lies in $J_\pi\cup\{0\}\subseteq H_\pi$, and the density $f_+$ is positive on $J_\pi$,
whose closure is $H_\pi$. Hence the support of $\pi_+$ is the whole closed convex hull $H_\pi$ of the support of
$\pi$, even where $\pi$ itself has gaps. This proves item~(iii).}
\Halmos
\end{proof}

\paragraph{\smrev{Optimal signals without bounded support.}} \smrev{The next lemma extends
\Cref{ec-mi-lemma: selector}: for an atomless population with $m_1<\infty$ and a symmetric prior with a finite
first moment, optimal signals exist for every type and a tie-breaking rule exists. It also records the bounds on the
output that
every optimal signal satisfies; these bounds do not need atomlessness.}

\begin{lemma}[Existence of optimal signals, and bounds on the output]\label{ec-mm-lemma: wellposed}
\smrev{Suppose $\truDist$ is atomless with $m_1<\infty$, the prior $\pi$ is symmetric with $\int|x|\,\pi(dx)<\infty$,
and
$0<\costParam<\infty$. Let $F$ be the loss as a function of the precision, as in the proof of
\Cref{ec-mi-lemma: selector}, now with a prior $\pi$ that may have unbounded support:
$F(\truth,u)=\obj_{\pi,\costParam}(\truth,1/\sqrt{u})=\costQual_\pi(\truth,1/\sqrt{u})+\costParam\mathcal{J}(u)$ for
$u>0$, and $F(\truth,0)=\truth^2$.}
\begin{enumerate}[label=\textup{(\roman*)}]
\item \smrev{$F$ is continuous on $\mathbb{R}\times(0,\infty)$ and lower semicontinuous on
$\mathbb{R}\times[0,\infty)$. For
every $\truth$, the loss $v\mapsto\obj_{\pi,\costParam}(\truth,\sqrt{v})$ attains its minimum over $v\in(0,\infty]$.
The smallest
optimal precision is a Borel and even function of $\truth$, so a tie-breaking rule exists.}
\item \smrev{This item drops atomlessness. Keep $\pi$ and $\costParam$ as above, but let $\truDist$ be any symmetric
nondegenerate population with $m_1<\infty$, and let $v^\star$ be any tie-breaking rule for $(\pi,\costParam)$. For
$\truDist$-almost every $\truth$, the output $Y^\truth_\pi$
of \Cref{ec-mi-eq: type output} satisfies}
\begin{equation}\smrev{\label{ec-mm-eq: pointwise bounds}
\EE{}{(Y^\truth_\pi-\truth)^2\,|\,\truth}\leq\truth^2,\qquad
\EE{}{|Y^\truth_\pi|\,|\,\truth}\leq2|\truth|,\qquad
\EE{}{(Y^\truth_\pi)^2\,|\,\truth}\leq4\truth^2.}
\end{equation}
\smrev{The next prior is symmetric, its first absolute moment is at most $2m_1$, and its second moment is at most
$4\truSD^2$ if $\truSD^2<\infty$. The prior $\delta_0$ is absorbing: under it, the uninformative signal is the unique
optimum of every type.}
\end{enumerate}
\end{lemma}

\begin{proof}{Proof of \Cref{ec-mm-lemma: wellposed}.}
\smrev{For item~(i), we prove that $F$ is continuous where the precision is positive and lower semicontinuous at
zero precision. Lower semicontinuity is enough for a minimum to exist on a compact set. For measurability, the
proof of \Cref{ec-mi-lemma: selector} used the continuity of $F$; here we replace it by the continuity of the
optimal loss $V_\pi$, which makes the graph of the set of minimizers closed. Since the minimizers for types in a
bounded set lie in one compact interval, closedness makes the smallest minimizer lower semicontinuous, hence
Borel. For item~(ii), we compare the optimal signal with the uninformative
signal.}

\smrev{\emph{Item (i), continuity for $u>0$.} We show that $F$ is continuous on $\mathbb{R}\times(0,\infty)$. Since
$\pi$ has a finite first moment, we may fix $R\geq0$ with $q\triangleq\pi([-R,R])>0$ and apply
\Cref{ec-mm-eq: growth}, with $m_\pi=\int|x|\,\pi(dx)$. Fix a bound $M>0$ on $|\truth|$ and $0<u_1<u_2<\infty$. For
$|\truth|\leq M$,
$u\in[u_1,u_2]$ and a noise value
$z$, the signal is $\truth+z/\sqrt{u}$, with noise standard deviation $1/\sqrt{u}$, and \Cref{ec-mm-eq: growth} gives}
\[
\smrev{\left|\truthAIAlg{\pi}(\truth+z/\sqrt{u},1/\sqrt{u})-\truth\right|
\leq2\left|\truth+z/\sqrt{u}\right|+R+\frac{m_\pi}{q}+|\truth|
\leq3M+\frac{2|z|}{\sqrt{u_1}}+R+\frac{m_\pi}{q}.}
\]
\smrev{The square of the right-hand side is integrable against the standard normal density and does not depend on
$(\truth,u)$. The AI's output is jointly continuous in the signal and the variance, by
\Cref{ec-mm-lemma: posterior}(i). Dominated convergence therefore makes the fidelity error
$\costQual_\pi(\truth,1/\sqrt{u})$ continuous on $[-M,M]\times[u_1,u_2]$. The communication cost $\mathcal{J}$ is
continuous by \Cref{ec-mm-lemma: cost}(i), whose assumptions hold because $\truDist$ is atomless, hence
nondegenerate, and $\EE{}{\ln(1+|\truth|)}\leq m_1<\infty$. Since $M$, $u_1$ and $u_2$ are arbitrary, $F$ is
continuous on $\mathbb{R}\times(0,\infty)$.}

\smrev{\emph{Item (i), lower semicontinuity at $u=0$.} We show that
$\liminf_nF(\truth_n,u_n)\geq F(\truth,0)=\truth^2$ whenever $(\truth_n,u_n)\rightarrow(\truth,0)$. Terms with
$u_n=0$ equal $\truth_n^2\rightarrow\truth^2$, so it suffices to consider $u_n>0$. Fix a noise value $z$. The signal
$\truth_n+z/\sqrt{u_n}$ has noise variance $1/u_n$. Its posterior has density relative to $\pi$ proportional to
$e^{-u_n(\truth_n+z/\sqrt{u_n}-x)^2/2}$. Expanding the square and dropping the factor that does not depend on $x$,
this density is proportional to}
\[
\smrev{\varpi_n(x)\triangleq\exp\left\{\sqrt{u_n}\left(z+\sqrt{u_n}\,\truth_n\right)x-\frac{u_nx^2}{2}\right\}.}
\]

\smrev{We show that, as the precision vanishes, the AI's output tends to the prior mean $0$:}
\[
\smrev{\truthAIAlg{\pi}(\truth_n+z/\sqrt{u_n},1/\sqrt{u_n})=\frac{\int x\,\varpi_n(x)\,\pi(dx)}
{\int\varpi_n(x)\,\pi(dx)}\longrightarrow0\qquad\text{for every }z.}
\]
\smrev{As $n\rightarrow\infty$, $\varpi_n(x)\rightarrow1$ for every $x$. Moreover, maximizing the exponent over $x$
gives $\varpi_n(x)\leq e^{(z+\sqrt{u_n}\,\truth_n)^2/2}$, which is bounded in $n$. Since $\int|x|\,\pi(dx)<\infty$,
dominated convergence gives $\int x\,\varpi_n(x)\,\pi(dx)\rightarrow\int x\,\pi(dx)=0$, where the last equality uses
the symmetry of $\pi$, and $\int\varpi_n(x)\,\pi(dx)\rightarrow1$. This proves the display.}

\smrev{We conclude with Fatou's lemma. By the display, the squared error
$\bigl(\truthAIAlg{\pi}(\truth_n+z/\sqrt{u_n},1/\sqrt{u_n})-\truth_n\bigr)^2$ tends to $\truth^2$ for every $z$.
Fatou's lemma gives $\liminf_n\costQual_\pi(\truth_n,1/\sqrt{u_n})\geq\truth^2$. Since the communication cost is
nonnegative, $\liminf_nF(\truth_n,u_n)\geq\truth^2$. Together with the previous step, $F$ is lower semicontinuous
on $\mathbb{R}\times[0,\infty)$.}

\smrev{\emph{Item (i), existence of a minimum.} Fix $\truth$. Since $F(\truth,0)=\truth^2$ and the fidelity error is
nonnegative, every $u$ with $F(\truth,u)\leq F(\truth,0)$ lies in the set
$\{u\geq0:\costParam\mathcal{J}(u)\leq\truth^2\}$. By \Cref{ec-mm-lemma: cost}(i)--(ii), $\mathcal{J}$ is continuous,
nondecreasing and tends to infinity. Continuity makes this set closed, monotonicity makes it an interval that
contains $0$, and divergence makes it bounded, since $\mathcal{J}(u)$ exceeds $\truth^2/\costParam$ for all large
$u$. So the set is a compact interval $[0,\bar u(\truth)]$. If $|\truth|\leq M$, the same argument with the bound
$M^2/\costParam$ shows that all these intervals lie in one compact interval $[0,\bar u_M]$. The lower semicontinuous
function $F(\truth,\cdot)$ attains its minimum on $[0,\bar u(\truth)]$, and
this is its minimum over $[0,\infty)$,
because $F(\truth,u)>\truth^2$ outside that interval. In terms of the noise variance, the loss
$v\mapsto\obj_{\pi,\costParam}(\truth,\sqrt{v})$ attains its minimum over $v\in(0,\infty]$, and the uncapped optimal AI
loss
of \Cref{ec-mm-eq: optimal loss} is $V_\pi(\truth)=\min_{u\geq0}F(\truth,u)$.}

\smrev{\emph{Item (i), continuity of the optimal loss.} We show that $V_\pi$ is continuous; the next block uses
this to show that the graph of the set of minimizers is closed.}

\smrev{\emph{Upper semicontinuity.} The function $V_\pi$ is upper semicontinuous, as an infimum over $u$ of the
functions $\truth\mapsto F(\truth,u)$, each of which is continuous. For $u>0$, this is the first step; for $u=0$,
it holds because $F(\truth,0)=\truth^2$.}

\smrev{\emph{Lower semicontinuity.} Let $\truth_n\rightarrow\truth$, and let $u_n$ be a minimizer at $\truth_n$.
Pass to a subsequence, still indexed by $n$, along which $V_\pi(\truth_n)$ tends to $\liminf_nV_\pi(\truth_n)$.
Let $M\geq\sup_n|\truth_n|$, which is finite because the sequence converges. By the existence step, all $u_n$ lie
in the compact interval $[0,\bar u_M]$, so we may pass to a further subsequence, again indexed by $n$, along which
$u_n$ converges to some $u'$. Joint lower semicontinuity of $F$ then gives}
\[
\smrev{\liminf_nV_\pi(\truth_n)=\lim_nF(\truth_n,u_n)\geq F(\truth,u')\geq V_\pi(\truth).}
\]
\smrev{Hence $V_\pi$ is lower semicontinuous, and so continuous.}

\smrev{\emph{Item (i), a measurable and even choice.} Since $F$ is lower semicontinuous and $V_\pi$ is continuous,
the set of pairs $\{(\truth,u):F(\truth,u)\leq V_\pi(\truth)\}$, which is the graph of the set of minimizers, is
closed. Let $u^\star(\truth)$ be the smallest minimizer at $\truth$; it exists because the set of minimizers at
$\truth$ is closed and bounded. We show that $u^\star$ is lower semicontinuous. Let $\truth_n\rightarrow\truth$, and
pass to a subsequence, still indexed by $n$, along which $u^\star(\truth_n)$ tends to $\liminf_nu^\star(\truth_n)$.
This limit is finite,
because all $u^\star(\truth_n)$ lie in the compact interval $[0,\bar u_M]$ of the existence step, with
$M\geq\sup_n|\truth_n|$. Since the graph is closed, the limit is a minimizer at $\truth$, so it is at least
$u^\star(\truth)$. Hence $u^\star$ is lower semicontinuous, and therefore Borel.}

\smrev{We now show that $u^\star$ is even. The prior $\pi$ is symmetric, so the AI's output is odd in the signal.
This follows from the proof of \Cref{ec-mm-lemma: posterior}(i), which uses only the symmetry of $\pi$ for this
claim; for $\pi=\delta_0$, the output is identically $0$. Replacing $\truth$ by $-\truth$ and $Z$ by $-Z$ therefore
changes the sign of both the AI's output and the type, which leaves the squared error unchanged. Hence
$F(-\truth,\cdot)=F(\truth,\cdot)$, and $u^\star$ is even.}

\smrev{Set $v^\star(\truth)\triangleq1/u^\star(\truth)$, with $1/0=\infty$. This is a measurable choice of an optimal
variance for every $\truth$ that treats $\truth$ and $-\truth$ alike. Restricted to the support of $\truDist$ and
extended by $\infty$ outside it, as in the Setting, it is a tie-breaking rule. The restriction is still measurable
and even, because the support of a symmetric distribution is closed and symmetric. This proves item~(i).}

\smrev{\emph{Item (ii), the bounds \Cref{ec-mm-eq: pointwise bounds}.} Item~(ii) states the three bounds of
\Cref{ec-mm-eq: pointwise bounds} for $\truDist$-almost every $\truth$, under any tie-breaking rule, without
assuming that $\truDist$ is atomless. Let $v^\star$ be a tie-breaking rule for $(\pi,\costParam)$, and fix a type
$\truth$ at which $v^\star(\truth)$ is optimal; by the Tie-breaking paragraph of the Setting above, this holds for
$\truDist$-almost every $\truth$.}

\smrev{\emph{Measurability.} Because $v^\star$ is Borel and the posterior mean is jointly continuous in
$(\query,v)$ (\Cref{ec-mm-lemma: posterior}(i)), $Y^\truth_\pi$ is a Borel function of $(\truth,Z)$ on the Borel
set $\{v^\star<\infty\}$. On its complement, $Y^\truth_\pi=0$. So $Y^\truth_\pi$ is a Borel function of
$(\truth,Z)$, and the conditional expectations below are its integrals over $Z$.}

\smrev{\emph{The three bounds.} Compare the optimal signal with the uninformative signal, whose loss
is $\truth^2$. If $v^\star(\truth)<\infty$, then, since the communication cost is nonnegative,}
\[
\smrev{\EE{}{(Y^\truth_\pi-\truth)^2\,|\,\truth}=\costQual_\pi\bigl(\truth,\sqrt{v^\star(\truth)}\bigr)
\leq\obj_{\pi,\costParam}\bigl(\truth,\sqrt{v^\star(\truth)}\bigr)\leq\truth^2.}
\]
\smrev{If $v^\star(\truth)=\infty$, then $Y^\truth_\pi=0$ and the squared error equals $\truth^2$. This is the first
bound. The triangle inequality in conditional $L^2$ gives
$\EE{}{(Y^\truth_\pi)^2\,|\,\truth}^{1/2}\leq\EE{}{(Y^\truth_\pi-\truth)^2\,|\,\truth}^{1/2}+|\truth|\leq2|\truth|$,
which is the third bound. The Cauchy--Schwarz inequality,
$\EE{}{|Y^\truth_\pi|\,|\,\truth}\leq\EE{}{(Y^\truth_\pi)^2\,|\,\truth}^{1/2}$, gives the second.}

\smrev{\emph{Item (ii), the next prior.} The next prior is the distribution of $Y^\truth_\pi$ when
$\truth\sim\truDist$. It is symmetric. Indeed, the map $(\truth,Z)\mapsto(-\truth,-Z)$ preserves the joint
distribution of $(\truth,Z)$, because $\truDist$ and the noise are symmetric and independent. It also negates the
output, because $v^\star(-\truth)=v^\star(\truth)$ and the AI's output is odd in the signal (as in item~(i)
above). Integrating the second and
third bounds of \Cref{ec-mm-eq: pointwise bounds} over $\truth\sim\truDist$ gives a first absolute moment of at most
$2m_1$ and, if $\truSD^2<\infty$, a second moment of at most $4\truSD^2$.}

\smrev{\emph{Item (ii), the prior $\delta_0$ is absorbing.} Under the prior $\delta_0$, the AI's output is $0$ for
every signal. A finite variance $v$ therefore has loss $\truth^2+\costParam\costEffort(\sqrt{v})$, which exceeds
the loss $\truth^2$ of the uninformative signal, because $\costEffort(\sqrt{v})>0$ by \Cref{ec-mm-lemma: cost}(i).
Hence the uninformative signal is the unique optimum of every type, and the next prior is again $\delta_0$.}

\smrev{The steps of item~(ii) use only that $\truDist$ is symmetric and nondegenerate with $m_1<\infty$, which
implies the assumptions of \Cref{ec-mm-lemma: cost}(i), and not that $\truDist$ is atomless. This proves item~(ii).}
\Halmos
\end{proof}

\paragraph{\smrev{Unbounded support: no finite-time collapse.}} \smrev{We prove
\Cref{thm:unbounded-finite-variance-no-collapse} in a stronger form, \Cref{ec-mm-prop: noncollapse} below: a finite
first absolute moment $m_1$ suffices, the interaction costs may be any sequence of positive numbers, and the
conclusion holds for every sequence of tie-breaking rules.}

\smrev{The next two lemmas bound the fidelity gain of a signal with a fixed noise variance, which the
fixed-signal argument outlined in the Plan of the proofs needs. To bound the fidelity gain, we write it as the
expectation of $g_v(\query)\,(2\query-g_v(\query))$, where
$\query$ is the signal and $g_v$ the AI's output, minus the noise correction $2\sqrt{v}\,\EE{}{Zg_v(\query)}$,
which grows at most linearly in $|\truth|$ (\Cref{ec-mm-eq: whole risk} below). For large positive $\query$, both
factors of the product are positive
and sum to
$2\query$, so the product is at least $\query$ times the smaller factor; it grows faster than linearly as
$\query\rightarrow\infty$ if both factors tend to infinity. So we
first show that both factors tend to infinity as $\query\rightarrow\infty$, and that $g_v$ grows at most linearly,
which makes all the expectations finite (\Cref{ec-mm-lemma: separation}).
Neither lemma assumes a moment condition on the prior.}

\begin{lemma}[The posterior mean stays increasingly far below twice the signal]\label{ec-mm-lemma: separation}
\smrev{Let $\pi$ be symmetric with unbounded support (hence nondegenerate), fix $v\in(0,\infty)$, and write
$g_v(\query)\triangleq\truthAIAlg{\pi}(\query,\sqrt{v})$. Then}
\begin{equation}\smrev{\label{ec-mm-eq: separation}
g_v(\query)\rightarrow\infty\qquad\text{and}\qquad2\query-g_v(\query)\rightarrow\infty\qquad\text{as
}\query\rightarrow\infty.}
\end{equation}
\smrev{Consequently, there is a finite constant $C_{\pi,v}$ such that}
\begin{equation}\smrev{\label{ec-mm-eq: linear}
|g_v(\query)|\leq2|\query|+C_{\pi,v}\qquad(\query\in\mathbb{R}).}
\end{equation}
\end{lemma}

\begin{proof}{Proof of \Cref{ec-mm-lemma: separation}.}
\smrev{We prove the two limits in turn, and then the linear bound. For the second limit, we fix $d>0$ and use a
prior interval $[r_1,r_2]$ beyond $d$. Its mass bounds the denominator of \Cref{ec-mi-eq: posterior mean} from
below by a term that decays more slowly in $\query$ than the part of the numerator above $2\query-d$. This makes the
posterior expected excess
$\EE{\pi}{(X-(2\query-d))_+\,|\,X+\sqrt{v}Z=\query}$ vanish as $\query\rightarrow\infty$. Hence
$g_v(\query)\leq2\query-d+o(1)$, and since $d$ is arbitrary, $2\query-g_v(\query)\rightarrow\infty$. The linear
bound then follows by oddness and continuity.}

\smrev{\emph{First limit.} Since $\pi$ is symmetric with unbounded support, $H_\pi=\mathbb{R}$ and
$J_\pi=\mathbb{R}$. By \Cref{ec-mm-lemma: posterior}(i), $g_v$ is a strictly increasing bijection from
$\mathbb{R}$ onto $\mathbb{R}$, so $g_v(\query)\rightarrow\infty$ as $\query\rightarrow\infty$.}

\smrev{\emph{Second limit.} We show that $\liminf_{\query\rightarrow\infty}(2\query-g_v(\query))\geq d$ for every
$d>0$. Fix $d>0$. Because $\pi$ is symmetric with unbounded support, it gives positive probability to
$(d,\infty)$, and hence to some bounded interval $[r_1,r_2]$ with $d<r_1\leq r_2<\infty$. Let
$p_d\triangleq\pi([r_1,r_2])>0$. For
$\query\geq r_2$ and $x\in[r_1,r_2]$, we have $|\query-x|\leq\query-r_1$, so the denominator of
\Cref{ec-mi-eq: posterior mean} satisfies}
\begin{equation}\smrev{\label{ec-mm-eq: denominator}
\int e^{-(\query-x)^2/(2v)}\,\pi(dx)\geq p_d\,e^{-(\query-r_1)^2/(2v)}\qquad(\query\geq r_2).}
\end{equation}

\smrev{We now bound the posterior expected excess $\EE{\pi}{(X-(2\query-d))_+\,|\,X+\sqrt{v}Z=\query}$. Fix $\query\geq
r_2$; then $\query>d$. For a prior point
$x=2\query-d+y$ with $y>0$, we have $\query-x=-(\query-d+y)$. Expanding the square and dropping $y^2\geq0$,}
\[
\smrev{\bigl(x-(2\query-d)\bigr)\,e^{-(\query-x)^2/(2v)}=y\,e^{-(\query-d+y)^2/(2v)}
\leq e^{-(\query-d)^2/(2v)}\,y\,e^{-(\query-d)y/v}
\leq\frac{v}{e(\query-d)}\,e^{-(\query-d)^2/(2v)},}
\]
\smrev{where the last step uses $\sup_{y\geq0}ye^{-\kappa y}=1/(e\kappa)$ with $\kappa=(\query-d)/v>0$. Integrating
this bound against
$\pi$ and dividing by \Cref{ec-mm-eq: denominator}, we get, using
$(\query-r_1)^2-(\query-d)^2=-2(r_1-d)\query+r_1^2-d^2$,}
\begin{equation}\smrev{\label{ec-mm-eq: overshoot}
\EE{\pi}{\bigl(X-(2\query-d)\bigr)_+\,\big|\,X+\sqrt{v}Z=\query}
\leq\frac{v}{e\,p_d\,(\query-d)}\exp\left\{-\frac{(r_1-d)\query}{v}+\frac{r_1^2-d^2}{2v}\right\}
\longrightarrow0\qquad(\query\rightarrow\infty),}
\end{equation}
\smrev{because $r_1>d$. The posterior moments used here are finite, because, for fixed $\query$ and $v$, a polynomial in
$x$ times $e^{-(\query-x)^2/(2v)}$ is a bounded function of $x$.}

\smrev{Since $X\leq(2\query-d)+(X-(2\query-d))_+$, taking posterior means and using \Cref{ec-mm-eq: overshoot}
gives $g_v(\query)\leq2\query-d+o(1)$ as $\query\rightarrow\infty$. Hence
$\liminf_{\query\rightarrow\infty}(2\query-g_v(\query))\geq d$. Since $d>0$ is arbitrary, the second limit in
\Cref{ec-mm-eq: separation} follows.}

\smrev{\emph{The linear bound \Cref{ec-mm-eq: linear}.} By \Cref{ec-mm-lemma: posterior}(i), $g_v$ is odd, continuous
and strictly increasing, so $g_v>0$ on $(0,\infty)$. By the second limit, there is $\query_0>0$ such that
$0<g_v(\query)\leq2\query$ for all $\query\geq \query_0$. Since $g_v$ is odd, $|g_v(\query)|\leq2|\query|$ also for
$\query\leq-\query_0$. On the compact interval $[-\query_0,\query_0]$, the continuous function $g_v$ is bounded in
absolute value, say by
$C_{\pi,v}$. Together, these give \Cref{ec-mm-eq: linear}.}
\Halmos
\end{proof}

\begin{lemma}[A fixed signal variance gives fidelity gains that grow faster than linearly]\label{ec-mm-lemma: gain}
\smrev{Let $\pi$ be symmetric with unbounded support (hence nondegenerate), and fix $v\in(0,\infty)$. Let
$\Delta_{\pi,v}(\truth)\triangleq\truth^2-\costQual_\pi(\truth,\sqrt{v})$ be the fidelity gain of the signal with
variance $v$ over the uninformative signal. Then $\Delta_{\pi,v}$ is even and}
\begin{equation}\smrev{\label{ec-mm-eq: superlinear gain}
\frac{\Delta_{\pi,v}(\truth)}{|\truth|}\longrightarrow\infty\qquad\text{as }|\truth|\rightarrow\infty.}
\end{equation}
\end{lemma}

\begin{proof}{Proof of \Cref{ec-mm-lemma: gain}.}
\smrev{We first show that $\Delta_{\pi,v}$ is even, so that it suffices to study $\truth\rightarrow\infty$. We then
write the fidelity gain as the expectation of a function of the signal, $\Upsilon$ below, minus a correction due to
the noise. We show that $\Upsilon$ grows faster than linearly, using \Cref{ec-mm-lemma: separation}. Because the
expectation runs over all signals, we turn this growth into an affine lower bound on $\Upsilon$ that holds for every
signal, \Cref{ec-mm-eq: affine bound}. Finally, we show that the correction grows at most linearly. Let
$g_v(\query)\triangleq\truthAIAlg{\pi}(\query,\sqrt{v})$, as in
\Cref{ec-mm-lemma: separation}, and}
\[
\smrev{\Upsilon(\query)\triangleq g_v(\query)\bigl(2\query-g_v(\query)\bigr),\qquad{\query}\in\mathbb{R}.}
\]

\smrev{\emph{The fidelity gain is even.} We show that $\Delta_{\pi,v}(-\truth)=\Delta_{\pi,v}(\truth)$. Replacing
$\truth$ by $-\truth$ and $Z$ by $-Z$ negates both the type and, because $g_v$ is odd, the AI's output, so the
squared error is unchanged. Since $Z$ has the same distribution as $-Z$, the fidelity error is even, and so is
$\Delta_{\pi,v}$.}

\smrev{\emph{The function $\Upsilon$ grows faster than linearly.} The function $\Upsilon$ is continuous, and it is
even because $g_v$ is odd. By \Cref{ec-mm-lemma: separation}, both factors $g_v(\query)$ and $2\query-g_v(\query)$
are positive for large $\query$ and tend to infinity as $\query\rightarrow\infty$. Their sum is $2\query$, so the
larger of the two is at least $\query$, and}
\[
\smrev{\frac{\Upsilon(\query)}{\query}\geq\min\left\{g_v(\query),\,2\query-g_v(\query)\right\}\longrightarrow\infty
\qquad(\query\rightarrow\infty).}
\]
\smrev{By evenness, $\Upsilon(\query)/|\query|\rightarrow\infty$ as $|\query|\rightarrow\infty$. We turn this into an
inequality valid for every signal. For each $M>0$, there is $C_M<\infty$, which depends on $M$ and on the fixed
$\pi$ and $v$, such that}
\begin{equation}\smrev{\label{ec-mm-eq: affine bound}
\Upsilon(\query)\geq M|\query|-C_M\qquad(\query\in\mathbb{R}).}
\end{equation}
\smrev{Indeed, $\Upsilon(\query)\geq M|\query|$ outside some compact interval, and on that interval the continuous
function $M|\query|-\Upsilon(\query)$ is bounded above; let $C_M$ be the larger of $0$ and that bound.}

\smrev{\emph{The fidelity gain in terms of $\Upsilon$.} Fix $\truth>0$, and let $\query=\truth+\sqrt{v}Z$ be the signal.
By \Cref{ec-mm-eq: linear}, $|g_v(\query)|\leq2\truth+2\sqrt{v}|Z|+C_{\pi,v}$, so every expectation below is finite.
The fidelity error is $\costQual_\pi(\truth,\sqrt{v})=\EE{}{(g_v(\query)-\truth)^2}$. Expanding the square, and then
substituting $\truth=\query-\sqrt{v}Z$ in the cross term, gives}
\begin{equation}\smrev{\label{ec-mm-eq: whole risk}
\Delta_{\pi,v}(\truth)=\EE{}{2\truth g_v(\query)-g_v(\query)^2}=\EE{}{\Upsilon(\query)}-2\sqrt{v}\,\EE{}{Zg_v(\query)}.}
\end{equation}

\smrev{\emph{The noise correction grows at most linearly.} Let $m_Z\triangleq\EE{}{|Z|}$. By \Cref{ec-mm-eq: linear}
and $|\query|\leq\truth+\sqrt{v}|Z|$,}
\begin{equation}\smrev{\label{ec-mm-eq: noise correction}
\left|\EE{}{Zg_v(\query)}\right|\leq\EE{}{|Z|\left(2|\query|+C_{\pi,v}\right)}
\leq2m_Z\truth+2\sqrt{v}+C_{\pi,v}m_Z,}
\end{equation}
\smrev{where we used $\EE{}{Z^2}=1$.}

\smrev{\emph{Conclusion.} By \Cref{ec-mm-eq: affine bound} and $\EE{}{|\query|}\geq|\EE{}{\query}|=\truth$, we have
$\EE{}{\Upsilon(\query)}\geq M\truth-C_M$. Substituting this and \Cref{ec-mm-eq: noise correction} into
\Cref{ec-mm-eq: whole risk} gives}
\[
\smrev{\Delta_{\pi,v}(\truth)\geq\left(M-4\sqrt{v}\,m_Z\right)\truth-C_M-4v-2C_{\pi,v}\sqrt{v}\,m_Z.}
\]
\smrev{Hence $\liminf_{\truth\rightarrow\infty}\Delta_{\pi,v}(\truth)/\truth\geq M-4\sqrt{v}\,m_Z$. Since $M$ is
arbitrary, this liminf is infinite, that is, $\Delta_{\pi,v}(\truth)/\truth\rightarrow\infty$ as
$\truth\rightarrow\infty$. By evenness, the same holds as $\truth\rightarrow-\infty$, with $|\truth|$ in place of
$\truth$. This proves \Cref{ec-mm-eq: superlinear gain}.}
\Halmos
\end{proof}

\begin{proposition}[No finite-time collapse under a first moment]\label{ec-mm-prop: noncollapse}
\smrev{Let $\truDist$ be symmetric with unbounded support (hence nondegenerate) and $m_1=\EE{}{|\truth|}<\infty$, and
let
$\pi_0=\truDist$. Let $(\costParam_t)_{t\geq0}$ be any sequence with $0<\costParam_t<\infty$ for every} $t$\smrev{,
with no
boundedness, monotonicity or convergence assumption. Then, for every sequence of tie-breaking rules,
$\pi_t\neq\delta_0$ for every finite $t$. If $\truDist$ is atomless, a tie-breaking rule exists at every round. If
$\truDist$ has atoms, the conclusion applies to every sequence of tie-breaking rules that exists.}
\end{proposition}

\begin{proof}{Proof of \Cref{ec-mm-prop: noncollapse}.}
\smrev{We prove by induction over rounds that every prior $\pi_t$ is symmetric, has unbounded support and has a
finite first moment. Since $\delta_0$ has bounded support, this gives $\pi_t\neq\delta_0$ for every $t$. At each
round, the tie-breaking rule is given; the last paragraph shows that one exists when $\truDist$ is atomless.}

\smrev{\emph{The fixed signal has a finite cost.} The finite first moment gives
$\EE{}{\ln(1+|\truth|)}\leq m_1<\infty$, and $\truDist$ is nondegenerate, so \Cref{ec-mm-lemma: cost}(i) gives
$\costEffort(1)=\mathcal{J}(1)<\infty$.}

\smrev{\emph{Base case and induction step.} At $t=0$, the prior $\pi_0=\truDist$ is symmetric, has unbounded support and
has the finite first moment $m_1$. Suppose that $\pi_t$ has these three properties, and let $v^\star_t$ be a
tie-breaking rule for $(\pi_t,\costParam_t)$. We show that $\pi_{t+1}$ has them too.}

\smrev{\emph{Sufficiently unique users send an informative signal.} By \Cref{ec-mm-lemma: gain}, applied to the
prior $\pi_t$ with $v=1$, the fidelity gain $\Delta_{\pi_t,1}(\truth)$ grows faster than $|\truth|$, so it
eventually exceeds the fixed number $\costParam_t\costEffort(1)$: there is a finite $R_t$ such that
$\Delta_{\pi_t,1}(\truth)>\costParam_t\costEffort(1)$ for every $|\truth|>R_t$. By the definition of the fidelity gain,
this is equivalent to}
\[
\smrev{\obj_{\pi_t,\costParam_t}(\truth,1)=\costQual_{\pi_t}(\truth,1)+\costParam_t\costEffort(1)<\truth^2
=\obj_{\pi_t,\costParam_t}(\truth,\infty)
\qquad(|\truth|>R_t).}
\]
\smrev{So the uninformative signal is not optimal for these types, and
$v_t^\star(\truth)<\infty$ for $\truDist$-almost every $\truth$ with $|\truth|>R_t$. Since $\truDist$ has unbounded
support, $\truDist(\{|\truth|>R_t\})>0$: a positive mass of users sends an informative signal.}

\smrev{\emph{Their outputs keep the support of the next prior unbounded.} The prior $\pi_t$ is symmetric and
nondegenerate, and $\truDist(\{\truth:v^\star_t(\truth)<\infty\})>0$ by the previous paragraph. So
\Cref{ec-mm-lemma: posterior}(iii), applied with the variance rule $\hat v=v^\star_t$, shows that the support of
$\pi_{t+1}$ is $H_{\pi_t}=\mathbb{R}$, which is unbounded. In particular, $\pi_{t+1}\neq\delta_0$.}

\smrev{\emph{Symmetry and the first moment.} By \Cref{ec-mm-lemma: wellposed}(ii), which does not assume that
$\truDist$ is atomless, $\pi_{t+1}$ is symmetric with $\int|y|\,\pi_{t+1}(dy)\leq2m_1<\infty$. This closes the
induction.}

\smrev{\emph{Existence of tie-breaking rules.} If $\truDist$ is atomless, then
\Cref{ec-mm-lemma: wellposed}(i) applies at every round, because the current prior is symmetric with a finite first
moment and $0<\costParam_t<\infty$; so a tie-breaking rule exists at every round. If $\truDist$ has atoms, the
induction above applies to every sequence of tie-breaking rules that exists.}
\Halmos
\end{proof}

\begin{proof}{Proof of \Cref{thm:unbounded-finite-variance-no-collapse}.}
\smrev{\Cref{thm:unbounded-finite-variance-no-collapse} is the case $\costParam_t\equiv\costParam$ of
\Cref{ec-mm-prop: noncollapse}. Its assumptions hold: $\truDist$ is symmetric with unbounded support, and its finite
variance gives $m_1=\EE{}{|\truth|}\leq\truSD<\infty$. Since $\truDist$ is atomless, \Cref{ec-mm-prop: noncollapse}
also provides a tie-breaking rule at every round. For every sequence of tie-breaking rules, it gives
$\pi_t\neq\delta_0$ for every finite $t\geq0$.}
\Halmos
\end{proof}

\paragraph{\smrev{Two-sided polynomial tails.}} \smrev{We prove \Cref{cor:two-sided-polynomial-tails} in the
following stronger form, which allows the interaction cost to vary across rounds and also states the mean of the
output.}

\begin{proposition}[Polynomial tails keep the output variance infinite]\label{ec-mm-prop: heavy tails}
\smrev{Suppose $\truDist$ satisfies the assumptions of \Cref{cor:two-sided-polynomial-tails}, with $1<\zeta<2$, and let
$\pi_0=\truDist$. Let $(\costParam_t)_{t\geq0}$ be any sequence with $0<\costParam_t<\infty$. Then a tie-breaking
rule exists at every round, and, for every sequence of tie-breaking rules and every finite $t\geq0$,}
\[
\smrev{\EE{}{|\outcome_t|}<\infty,\qquad\EE{}{\outcome_t}=0,\qquad\VV(\outcome_t)=\infty.}
\]
\end{proposition}

\smrev{Recall that the population has a density $p$ with $c_-|\truth|^{-\zeta-1}\leq p(\truth)\leq
c_+|\truth|^{-\zeta-1}$ for $|\truth|\geq A$. The proof of \Cref{ec-mm-prop: heavy tails} is an induction over
rounds on a class $\mathcal{T}_\zeta$ of priors with polynomial tails. The class is defined next, and a later
paragraph checks that the population belongs to it.}

\smrev{\emph{A tail class of priors.} Let $\mathcal{T}_\zeta$ be the class of symmetric distributions $\rho$ on
$\mathbb{R}$ for which there exist constants $\bar C,\underline{c},L>0$ and $\bar x\geq e$ such that, for every
$x\geq\bar x$,}
\begin{equation}\smrev{\label{ec-mm-eq: tail class}
\rho(\{y:|y|>x\})\leq\bar Cx^{-\zeta},
\qquad
\rho\left(\left[x-L\sqrt{\ln x},\,x+L\sqrt{\ln x}\right]\right)\geq\underline{c}x^{-\zeta-1}.}
\end{equation}
\smrev{The first inequality, the tail bound, gives every member a finite first moment (next paragraph). The second,
the window bound, says that every window of half-width $L\sqrt{\ln x}$ around a large point $x$ has probability at
least of order $x^{-\zeta-1}$; by symmetry, the same holds around $-x$. It is a lower bound on the probability of
intervals, which is weaker than a lower bound on a density. It keeps the denominator of
\Cref{ec-mi-eq: posterior mean} from being too small in the proof of \Cref{ec-mm-lemma: localization}. The constants
may depend on $\rho$. Apart from symmetry, the class constrains only
the tails. It allows atoms and priors without a density, which matters because the priors after round $0$ may have
an atom at $0$, created by users who send the uninformative signal.}

\smrev{Every member of $\mathcal{T}_\zeta$ has a finite first moment. Indeed, writing the first moment as the
integral of the tail and using the tail bound beyond $\bar x$,}
\[
\smrev{\int|y|\,\rho(dy)=\int_0^\infty\rho(\{y:|y|>r\})\,dr\leq\bar x+\bar C\int_{\bar x}^\infty r^{-\zeta}\,dr<\infty,}
\]
\smrev{because $\zeta>1$.}

\smrev{\emph{The population belongs to $\mathcal{T}_\zeta$.} The tail bound holds with $\bar C=2c_+/\zeta$ for
$x\geq A$, by \Cref{ec-mm-cor: heavy cost}. For the window bound, take $L=1$ and $x\geq\max\{A+1,e\}$. Then
$\sqrt{\ln x}\geq1$, so the window contains $[x-1,x+1]$, and $x-1\geq A$, so the lower density bound of
\Cref{cor:two-sided-polynomial-tails} holds on $[x-1,x+1]$. Since $x+1\leq2x$,}
\[
\smrev{\truDist([x-1,x+1])\geq2c_-(x+1)^{-\zeta-1}\geq2^{-\zeta}c_-\,x^{-\zeta-1}.}
\]
\smrev{Hence $\truDist\in\mathcal{T}_\zeta$, with $\bar x=\max\{A+1,e\}$ and $\underline{c}=2^{-\zeta}c_-$.}

\smrev{\emph{The three parts of the proof.} The proof of \Cref{ec-mm-prop: heavy tails} has three parts. First,
under a prior in $\mathcal{T}_\zeta$ and for noise variances in a compact subinterval of $(0,\infty)$, the AI's
output differs from the signal by at most a constant times $\sqrt{\ln(e+|\query|)}$
(\Cref{ec-mm-lemma: localization}, next). Second, comparing utility losses with the signal of noise variance $1$,
whose communication cost is finite, bounds the fidelity error of the optimal signal by a constant times
$\ln(e+|\truth|)$, which is small relative to $\truth^2$ (Step~2 of the proof). Third, the outputs of unique users
therefore stay close to their preferences, so the output has infinite variance (Step~3), and its distribution
satisfies the tail bound and the window bound of \Cref{ec-mm-eq: tail class} (Steps~4 and~5).}

\begin{lemma}[Under a prior in the tail class, the AI's output stays close to the signal]
\label{ec-mm-lemma: localization}
\smrev{If $\rho\in\mathcal{T}_\zeta$ and $0<v_1\leq v_2<\infty$, there is a finite constant $K_{\rho,v_1,v_2}$ such
that}
\begin{equation}\smrev{\label{ec-mm-eq: localization}
\left|\truthAIAlg{\rho}(\query,\sqrt{v})-\query\right|\leq K_{\rho,v_1,v_2}\sqrt{\ln(e+|\query|)}
\qquad(\query\in\mathbb{R},\ v_1\leq v\leq v_2).}
\end{equation}
\end{lemma}

\begin{proof}{Proof of \Cref{ec-mm-lemma: localization}.}
\smrev{For large positive signals, we write the AI's output minus the signal as a ratio, bound its denominator from
below with the window bound, and bound its numerator by splitting it at distance $r_\query$ (defined below) from the
signal. We then extend the bound to all signals by symmetry and continuity. Let $(\bar C,\underline{c},L,\bar x)$ be the
constants of $\rho$ in \Cref{ec-mm-eq: tail class}. Fix $v\in[v_1,v_2]$. The AI's output minus the signal is the
ratio of the numerator $\int(x-\query)\,e^{-(\query-x)^2/(2v)}\,\rho(dx)$ to the denominator}
\[
\smrev{\widetilde N_0(\query,v)\triangleq\int e^{-(\query-x)^2/(2v)}\,\rho(dx)}
\]
\smrev{of \Cref{ec-mi-eq: posterior mean} under the prior $\rho$.}

\smrev{\emph{Lower bound on the denominator.} Let $\query\geq\bar x$. On the window
$[\query-L\sqrt{\ln{\query}},\query+L\sqrt{\ln{\query}}]$, we have $(\query-x)^2\leq L^2\ln{\query}$, and $v\geq v_1$,
so
the Gaussian kernel is at least $e^{-L^2\ln{\query}/(2v_1)}=\query^{-L^2/(2v_1)}$. The window bound of
\Cref{ec-mm-eq: tail class} then gives}
\[
\smrev{\widetilde N_0(\query,v)\geq\underline{c}\,\query^{-\zeta-1}\query^{-L^2/(2v_1)}=\underline{c}\,\query^{-\beta},
\qquad\beta\triangleq\zeta+1+\frac{L^2}{2v_1}.}
\]

\smrev{\emph{Upper bound on the numerator.} We split the numerator at a distance $r_\query$ from the signal, chosen
so that the far part is small compared with the denominator bound. Choose $k>0$ with $k^2/(2v_2)>\beta$, and let
$r_\query\triangleq k\sqrt{\ln{\query}}$. From now on, let $\query\geq \query_1\triangleq\max\{\bar x,e^{v_2/k^2}\}$,
so that, in addition, $r_\query\geq\sqrt{v_2}$.}

\smrev{We split the numerator at $|x-\query|=r_\query$. On $\{|x-\query|\leq r_\query\}$, the absolute value of the
integrand is at most $r_\query$ times the kernel, so this part is at most $r_\query\widetilde N_0(\query,v)$ in
absolute value. On $\{|x-\query|>r_\query\}$, we use $e^{-r^2/(2v)}\leq e^{-r^2/(2v_2)}$ and the fact that
$r\mapsto re^{-r^2/(2v_2)}$ is decreasing for $r\geq\sqrt{v_2}$. So the absolute value of the integrand is at most
$r_\query e^{-r_\query^2/(2v_2)}=r_\query\,\query^{-k^2/(2v_2)}$, and so is this part of the numerator, because $\rho$
is a probability distribution. Together,}
\[
\smrev{\left|\int(x-\query)\,e^{-(\query-x)^2/(2v)}\,\rho(dx)\right|\leq r_\query\widetilde N_0(\query,v)
+r_\query\,\query^{-k^2/(2v_2)}\qquad(\query\geq\query_1).}
\]

\smrev{\emph{The bound for large positive signals.} We divide the two parts of the numerator bound by the
denominator. The near part, $r_\query\widetilde N_0(\query,v)$, gives exactly $r_\query$. The far part,
$r_\query\query^{-k^2/(2v_2)}$, divided by the lower bound $\underline{c}\,\query^{-\beta}$, gives
$\underline{c}^{-1}r_\query\,\query^{\beta-k^2/(2v_2)}$. Hence}
\[
\smrev{\left|\truthAIAlg{\rho}(\query,\sqrt{v})-\query\right|\leq
r_\query+\underline{c}^{-1}r_\query\,\query^{\beta-k^2/(2v_2)}
\leq(1+\underline{c}^{-1})\,k\sqrt{\ln{\query}},}
\]
\smrev{because $\query^{\beta-k^2/(2v_2)}\leq1$ for $\query\geq1$. Since $\ln{\query}\leq\ln(e+\query)$, this is
\Cref{ec-mm-eq: localization} for all sufficiently large $\query$, uniformly in $v\in[v_1,v_2]$.}

\smrev{\emph{All signals.} Because $\rho$ is symmetric, the AI's output is odd in the signal, so
$\truthAIAlg{\rho}(-\query,\sqrt{v})+\query=-(\truthAIAlg{\rho}(\query,\sqrt{v})-\query)$, and the same bound holds
for large negative $\query$. On the remaining compact range of $\query$, and for $v\in[v_1,v_2]$, the left-hand side
of \Cref{ec-mm-eq: localization} is bounded, because the AI's output is jointly continuous in $(\query,v)$
(\Cref{ec-mm-lemma: posterior}(i)). Since $\sqrt{\ln(e+|\query|)}\geq1$, enlarging the constant gives
\Cref{ec-mm-eq: localization} for every $\query$.}
\Halmos
\end{proof}

\begin{proof}{Proof of \Cref{ec-mm-prop: heavy tails}.}
\smrev{We show that one round of the loop maps a prior in $\mathcal{T}_\zeta$ to an output with the three
properties of the statement, whose distribution is again in $\mathcal{T}_\zeta$. Since $\pi_0=\truDist$ belongs to
$\mathcal{T}_\zeta$, induction over rounds then gives the proposition. The one-round argument has five steps: a
tie-breaking rule exists; the fidelity error of the optimal signal is at most logarithmic in $|\truth|$; the output
has a finite first moment, mean zero and infinite variance; its distribution has the tail bound of
\Cref{ec-mm-eq: tail class}; and it has the window bound.}

\smrev{\emph{One round: notation.} Let the current prior be $\rho\in\mathcal{T}_\zeta$ and the current cost
$\costParam\in(0,\infty)$. Let $v^\star$ be a tie-breaking rule for $(\rho,\costParam)$ (one exists, by Step~1
below), and let $Y^\truth_\rho$ be the
output of \Cref{ec-mi-eq: type output} when $\truth\sim\truDist$; its distribution $\pi_+$ is the next prior. Let}
\[
\smrev{\mathcal{E}^\star(\truth)\triangleq\EE{}{(Y^\truth_\rho-\truth)^2\,\big|\,\truth}}
\]
\smrev{be the fidelity error of user $\truth$ under the tie-breaking rule. It equals
$\costQual_\rho(\truth,\sqrt{v^\star(\truth)})$ when $v^\star(\truth)<\infty$ and $\truth^2$ otherwise.}

\smrev{\emph{Measurability and null sets.} As in the measurability paragraph of the proof of
\Cref{ec-mm-lemma: wellposed}(ii), $Y^\truth_\rho$ is a Borel function of $(\truth,Z)$, so $\mathcal{E}^\star$ is
Borel. This justifies the uses of Tonelli's theorem and of conditional Markov
inequalities below. The optimality inequalities hold for $\truDist$-almost every $\truth$, which suffices for every
bound that is integrated over $\truth$.}

\smrev{\emph{Step 1: a tie-breaking rule exists.} The prior $\rho$ is symmetric with a finite first moment, as every
member of $\mathcal{T}_\zeta$ is, and $\truDist$ has a density, hence is atomless, and $m_1<\infty$
(\Cref{ec-mm-cor: heavy cost}). So
\Cref{ec-mm-lemma: wellposed} applies to $\rho$ and every cost in $(0,\infty)$: a tie-breaking rule exists, and every
tie-breaking rule satisfies the conclusions of \Cref{ec-mm-lemma: wellposed}(ii). This is where the continuity of
$\mathcal{J}$ and its divergence as $u\rightarrow\infty$ (\Cref{ec-mm-cor: heavy cost}) are used.}

\smrev{\emph{Step 2: the fidelity error is at most logarithmic.} We compare the optimal signal with the signal of
noise variance $1$. Let $K_\rho\triangleq K_{\rho,1,1}$ be the constant of \Cref{ec-mm-lemma: localization} with
$v_1=v_2=1$. The fidelity error of the signal $\truth+Z$ is the expected square of
$\bigl(\truthAIAlg{\rho}(\truth+Z,1)-(\truth+Z)\bigr)+Z$. Using $(a+b)^2\leq2a^2+2b^2$ and $\EE{}{Z^2}=1$,}
\[
\smrev{\costQual_\rho(\truth,1)\leq2\,\EE{}{\left(\truthAIAlg{\rho}(\truth+Z,1)-(\truth+Z)\right)^2}+2.}
\]
\smrev{By \Cref{ec-mm-eq: localization} with the signal $\truth+Z$, and by
$\ln(e+|\truth+Z|)\leq\ln(e+|\truth|)+\ln(1+|Z|)$, which holds because $e+|\truth+Z|\leq(e+|\truth|)(1+|Z|)$,}
\[
\smrev{\EE{}{\left(\truthAIAlg{\rho}(\truth+Z,1)-(\truth+Z)\right)^2}\leq K_\rho^2\,\EE{}{\ln(e+|\truth+Z|)}
\leq K_\rho^2\left(\ln(e+|\truth|)+\EE{}{\ln(1+|Z|)}\right).}
\]
\smrev{Combining the two displays,}
\[
\smrev{\costQual_\rho(\truth,1)\leq2K_\rho^2\left(\ln(e+|\truth|)+\EE{}{\ln(1+|Z|)}\right)+2
\leq C_\rho\ln(e+|\truth|),}
\]
\smrev{where $C_\rho\triangleq2K_\rho^2\left(1+\EE{}{\ln(1+|Z|)}\right)+2$, and the last step uses
$\ln(e+|\truth|)\geq1$.}

\smrev{We now compare total losses. The communication cost of the signal with noise variance $1$ enters the
utility loss as $\costParam\mathcal{J}(1)$, a finite number by \Cref{ec-mm-cor: heavy cost}. The optimal signal has a
utility loss at most that of this signal, and at most $\truth^2$, the
loss of the uninformative signal. Its fidelity error is at most its utility loss, because its communication cost is
nonnegative. Let $\bar K_{\rho,\costParam}\triangleq C_\rho+\costParam\mathcal{J}(1)$. Using $\ln(e+|\truth|)\geq1$
once more, we get}
\begin{equation}\smrev{\label{ec-mm-eq: heavy risk}
\mathcal{E}^\star(\truth)\leq\min\left\{\truth^2,\ \bar K_{\rho,\costParam}\ln(e+|\truth|)\right\}\qquad\text{for
}\truDist\text{-almost every }\truth.}
\end{equation}

\smrev{\emph{Step 3: finite first moment, mean zero, infinite variance.} By \Cref{ec-mm-lemma: wellposed}(ii), the
output is symmetric and has first absolute moment at most $2m_1<\infty$. Its mean is therefore $0$.}

\smrev{For the variance, we bound the conditional second moment from below for unique users. The reverse triangle
inequality in conditional $L^2$ gives}
\[
\smrev{\EE{}{(Y^\truth_\rho)^2\,\big|\,\truth}^{1/2}\geq\left(|\truth|-\mathcal{E}^\star(\truth)^{1/2}\right)_+.}
\]
\smrev{Let $\bar\truth$ be such that $\bar K_{\rho,\costParam}\ln(e+|\truth|)\leq\truth^2/4$ for
$|\truth|\geq\bar\truth$. For
these types, \Cref{ec-mm-eq: heavy risk} gives $\mathcal{E}^\star(\truth)^{1/2}\leq|\truth|/2$, so
$\EE{}{(Y^\truth_\rho)^2\,|\,\truth}\geq\truth^2/4$ for $\truDist$-almost every such $\truth$. Integrating against the
lower density bound of \Cref{cor:two-sided-polynomial-tails},}
\[
\smrev{\EE{}{(Y^\truth_\rho)^2}\geq\int_{|\truth|\geq\max\{\bar\truth,A\}}\frac{\truth^2}{4}
\,c_-|\truth|^{-\zeta-1}\,d\truth
=\infty,}
\]
\smrev{because $\zeta<2$. Since the mean is $0$, the variance of the output is infinite.}

\smrev{\emph{Step 4: the tail bound for $\pi_+$.} It remains to show that $\pi_+\in\mathcal{T}_\zeta$, so that the
argument can be repeated. The distribution $\pi_+$ is symmetric by \Cref{ec-mm-lemma: wellposed}(ii). For the tail
bound, we split the users according to whether $|\truth|$ is large. For users with small $|\truth|$, a large output
requires a large error, whose probability we bound with the conditional Markov inequality; this produces a
truncated second moment of the population, which we bound first.}

\smrev{Writing $\truth^2=\int_0^{|\truth|}2y\,dy$ gives
$\int_{|\truth|\leq r}\truth^2\,\truDist(d\truth)=2\int_0^ry\,\truDist(\{y<|\truth|\leq r\})\,dy$, and
$\truDist(\{y<|\truth|\leq r\})\leq\truDist(\{|\truth|>y\})$. We then bound the tail $\truDist(\{|\truth|>y\})$
by $1$ on $[0,A]$ and by $(2c_+/\zeta)y^{-\zeta}$ beyond $A$ (the tail bound of \Cref{ec-mm-cor: heavy cost}), and
use $2-\zeta>0$. Finally, for $r$ large enough that $A^2\leq r^{2-\zeta}$, the constant $A^2$ is absorbed into the
power of $r$. This gives, for every $r$ with $r^{2-\zeta}\geq A^2$,}
\begin{equation}\smrev{\label{ec-mm-eq: truncated second}
\int_{|\truth|\leq r}\truth^2\,\truDist(d\truth)\leq2\int_0^ry\,\truDist(\{|\truth|>y\})\,dy
\leq A^2+\frac{4c_+}{\zeta(2-\zeta)}\,r^{2-\zeta}\leq C_pr^{2-\zeta},
\qquad C_p\triangleq1+\frac{4c_+}{\zeta(2-\zeta)}.}
\end{equation}
\smrev{We now bound the tail of $\pi_+$. Fix $x$ with $x/2\geq A$ and $(x/2)^{2-\zeta}\geq A^2$, and split the
users according to whether $|\truth|>x/2$. If $|\truth|\leq x/2$ and $|Y^\truth_\rho|>x$,
then $|Y^\truth_\rho-\truth|>x/2$. The conditional Markov inequality, with the bound
$\mathcal{E}^\star(\truth)\leq\truth^2$ of \Cref{ec-mm-eq: heavy risk}, bounds the conditional probability of this
event by $4\truth^2/x^2$. Hence, using \Cref{ec-mm-eq: truncated second} with $r=x/2$ and the tail bound of
\Cref{ec-mm-cor: heavy cost},}
\[
\smrev{\PP(|Y^\truth_\rho|>x)\leq\truDist(\{|\truth|>x/2\})+\frac{4}{x^2}\int_{|\truth|\leq
x/2}\truth^2\,\truDist(d\truth)
\leq\frac{2c_+}{\zeta}\,2^\zeta x^{-\zeta}+4C_p2^{\zeta-2}x^{-\zeta}=C'x^{-\zeta},
\qquad C'\triangleq2^\zeta\left(\frac{2c_+}{\zeta}+C_p\right).}
\]
\smrev{This is the tail bound of \Cref{ec-mm-eq: tail class} for $\pi_+$, for all $x$ as above.}

\smrev{\emph{Step 5: the window bound for $\pi_+$.} Take $x$ large enough that $x-1\geq A$, $\ln x\geq1$ and
$\ln(e+x+1)\leq2\ln x$. We show that users with $\truth\in[x-1,x+1]$ have outputs in a window around $x$ with
probability at least $1/2$. We first show that their outputs are close to their own type $\truth$, and then use
$|\truth-x|\leq1$. For $\truDist$-almost every such $\truth$, \Cref{ec-mm-eq: heavy risk} gives
$\mathcal{E}^\star(\truth)\leq\bar K_{\rho,\costParam}\ln(e+x+1)\leq2\bar K_{\rho,\costParam}\ln x$, and Markov's
inequality, applied to
$(Y^\truth_\rho-\truth)^2$, gives}
\[
\smrev{\PP\left(|Y^\truth_\rho-\truth|\leq2\sqrt{\bar K_{\rho,\costParam}}\sqrt{\ln
x}\,\middle|\,\truth\right)\geq\frac12.}
\]
\smrev{Let $L'\triangleq1+2\sqrt{\bar K_{\rho,\costParam}}$. Since $|\truth-x|\leq1\leq\sqrt{\ln x}$, this event places
$Y^\truth_\rho$ in the window $[x-L'\sqrt{\ln x},\,x+L'\sqrt{\ln x}]$. Integrating over these users and using the
lower density bound of \Cref{cor:two-sided-polynomial-tails} on $[x-1,x+1]$,}
\[
\smrev{\pi_+\left(\left[x-L'\sqrt{\ln x},\,x+L'\sqrt{\ln
x}\right]\right)\geq\frac12\,\truDist(}[\smrev{x-1,x+1}]\smrev{)
\geq c_-(x+1)^{-\zeta-1}\geq2^{-\zeta-1}c_-\,x^{-\zeta-1}.}
\]
\smrev{This is the window bound of \Cref{ec-mm-eq: tail class} for $\pi_+$, with constant $L'$. Together with the
symmetry of $\pi_+$ and the tail bound of Step~4, it shows that $\pi_+\in\mathcal{T}_\zeta$, with one threshold
$\bar x$ for both bounds: the larger of the thresholds of Steps~4 and~5, and at least $e$.}

\smrev{\emph{Induction over rounds.} The constants and thresholds introduced in Steps~1--5 depend on the current
prior $\rho$ and cost $\costParam$, so they may change from round to round; the induction needs no bound on them
that is uniform over rounds. The initial prior $\pi_0=\truDist$ belongs to $\mathcal{T}_\zeta$. Suppose that
$\pi_t\in\mathcal{T}_\zeta$. By Step~1, a tie-breaking rule for $(\pi_t,\costParam_t)$ exists. Fix any such rule.
The output at round $t$ is $\outcome_t=Y^\truth_{\pi_t}$, and its distribution is the next prior $\pi_{t+1}$. By
Step~3, $\EE{}{|\outcome_t|}<\infty$, $\EE{}{\outcome_t}=0$ and $\VV(\outcome_t)=\infty$. By Steps~4 and~5,
$\pi_{t+1}\in\mathcal{T}_\zeta$. Induction proves the claim at every finite round.}
\Halmos
\end{proof}

\begin{proof}{Proof of \Cref{cor:two-sided-polynomial-tails}.}
\smrev{\Cref{cor:two-sided-polynomial-tails} is the case $\costParam_t\equiv\costParam$ of
\Cref{ec-mm-prop: heavy tails}, whose conclusion $\VV(\outcome_t)=\infty$ for every finite $t\geq0$ holds for every
sequence of tie-breaking rules.}
\Halmos
\end{proof}
}
\section{Results in \Cref{sec: human stops death spiral}.}
\label{ec-sec: proofs mitigation}

\subsection{Proofs for \Cref{sec:mit-efficiency}}\label{ec-sec: proofs mit-efficiency}

\rev{
\paragraph{\smrev{Setting.}} \smrev{We use the setting and notation of \Cref{ec-sec: proofs unique key}. In particular,
$\truMean=0$, and the population distribution $\truDist$ does not change over time; only the AI's prior $\pi_t$ does.
A user of type $\truth$ chooses the noise variance $v\in(0,\infty]$ of her signal. A finite choice sends the signal
$\query=\truth+\sqrt{v}Z$, where $Z\sim\mathcal{N}(0,1)$ is independent of $\truth$, and $v=\infty$ is the uninformative
signal. Exact revelation ($v=0$) is excluded. We write $u=1/v$ for the precision of the signal, with $u=0$ for the
uninformative signal. Functions inherited from \Cref{ec-sec: proofs unique key} take the noise standard deviation
$\sqrt{v}$ as their argument, as in the AI output $\truthAIAlg{\rho}(\query,\sqrt{v})$ under a prior $\rho$, the
fidelity error $\costQual_\rho(\truth,\sqrt{v})$ and the communication cost $\costEffort(\sqrt{v})$; in terms of the
precision, this argument is $1/\sqrt{u}$.}

\smrev{From the part of \Cref{ec-sec: proofs unique key} on populations with unbounded support, we use two further
items. The first is the round loss $\obj_{\pi_t,\costParam_t}$ of \Cref{ec-mm-eq: round loss}, which a user of type
$\truth$ minimizes at round $t$; in it, the interaction cost $\costParam_t>0$ may vary across rounds. The second is the
communication cost as a function of the precision, $\mathcal{J}(u)=\costEffort(1/\sqrt{u})$, with $\mathcal{J}(0)=0$
for the uninformative signal.}

\smrev{Unless stated otherwise, $\truDist$ is symmetric (about $0$) and atomless, with
$\truSD^2=\EE{}{\truth^2}<\infty$ and $m_1=\EE{}{|\truth|}$; its support may be bounded (as in item~2 of
\Cref{prop: recover diversity}) or unbounded. Atomlessness implies that $\truth$ is not almost surely $0$, so
$\truSD^2>0$ and $m_1>0$.}

\emph{\smrev{Tie-breaking.}} \smrev{At round $t$, users break ties by a tie-breaking rule $v^\star_t$ for
$(\pi_t,\costParam_t)$, as defined in the Setting of \Cref{ec-sec: proofs unique key}, where $\pi_t$ is the prior that
the rules of the earlier rounds generate. Such a rule gives an optimal variance to every type in the support of
$\truDist$, and these types carry all the mass of $\truDist$; statements below about ``each type'' refer to them.}

\smrev{The quantifiers in \Cref{prop: recover diversity} are in this order. We first construct the cost sequences (and,
in item~2, the collapse time $T$) from $\truDist$ and $\costParam_0$ alone. The conclusions then hold for every sequence
$(v^\star_t)_{t\geq0}$ of tie-breaking rules, one per round. The rule used at round $t$ may depend on $t$, on the past
of the loop, and even on the whole sequence of costs, including future costs.}

\smrev{A tie-breaking rule exists at every round, by induction. The prior $\pi_0=\truDist$ is symmetric with a finite
first absolute moment. If $\pi_t$ is symmetric with a finite first absolute moment, then a tie-breaking rule for
$(\pi_t,\costParam_t)$ exists, by \Cref{ec-mm-lemma: wellposed}(i), and, whichever rule is used, $\pi_{t+1}$ is again
symmetric, with first absolute moment at most $2m_1$, by \Cref{ec-mm-lemma: wellposed}(ii).}

\paragraph{\smrev{Plan.}} \smrev{Item~1 of \Cref{prop: recover diversity} (fast improvement) is proved by choosing the
costs one round at a time, each from $\truDist$, $\costParam_0$ and the earlier costs only, as the order of quantifiers
above requires. The difficulty is that the prior that the loop reaches at round $t$ depends on the tie-breaking rules
of the earlier rounds, so no cost may depend on it.}

\smrev{The costs must do two things. First, the initial cost $\costParam_0$, which is given, must already improve on
the uninformative signal. We show that, under every tie-breaking rule, the population-averaged fidelity error
$\EE{}{(\outcome_0-\truth)^2}$ is below $\truSD^2$, the error of the uninformative signal, by a fixed amount $d_0>0$.
The reason is that, at round $0$, the AI's prior is the population itself, so its posterior mean does at least as well
as the best linear estimator. With a suitable variance common to all types, the population-averaged utility loss, and
hence the population-averaged fidelity error of optimal signals, is then below $\truSD^2$ when
$\costParam_0<2\truSD^2$. This improvement starts the construction: it is the hypothesis of
\Cref{ec-mm-lemma: compact support} below at round~$0$.}

\smrev{Second, each later cost $\costParam_t$ ($t\geq1$) must make the error small. It is chosen so small that, whatever
tie-breaking rules are used,}
\[
\smrev{\EE{}{(\outcome_t-\truth)^2}\leq\varepsilon_t=\truSD^2 4^{-t}.}
\]
\smrev{Since $\varepsilon_t\leq\truSD^2/4$, this bound also gives, at round $t$, the improvement
$d=\truSD^2-\varepsilon_t>0$ that \Cref{ec-mm-lemma: compact support} needs. We also take
$\costParam_t<\costParam_{t-1}/2$, so the costs strictly decrease to $0$. Finally, since $\varepsilon_t\rightarrow0$,
the outputs converge to the preferences in mean square, which forces $\VV(\outcome_t)\rightarrow\truSD^2$, as item~1
claims.}

\smrev{To make one choice of $\costParam_t$ that works for all tie-breaking rules, we enclose all the priors that the
loop can reach at round $t$, over all tie-breaking rules, in a family $\mathcal{K}_t$. The families and the costs are
built alternately. We take $\mathcal{K}_0=\{\truDist\}$. Given $\mathcal{K}_t$, we choose $\costParam_t$ ($t\geq1$;
$\costParam_0$ is given). Then $\mathcal{K}_{t+1}$ is the weak closure of all the output distributions that the priors
of $\mathcal{K}_t$ generate under the cost $\costParam_t$, over all tie-breaking rules; the closure is taken so that the
family is compact. So each family is fixed before the loop is run, and depends only on $\truDist$ and
$\costParam_0,\ldots,\costParam_{t-1}$. For $t\geq1$, each $\mathcal{K}_t$ has three properties:}
\begin{enumerate}
\item[\smrev{(a)}] \smrev{it is compact for weak convergence;}
\item[\smrev{(b)}] \smrev{its members have second moments at most $4\truSD^2$;}
\item[\smrev{(c)}] \smrev{its members have support exactly $H_{\truDist}$, where $H_\rho$ denotes the closed convex
hull of the support of a distribution $\rho$.}
\end{enumerate}
\smrev{Each property is used to choose the next cost, by comparing every tie-breaking rule with signals in which all
types use one common variance. For each prior, a small enough common variance gives a population-averaged fidelity
error below a given tolerance. Compactness lets finitely many common variances serve the whole family: each prior of
the family is served by one of them. The second-moment bound gives a growth bound on posterior means that is common to
the family, which supplies the domination needed to pass to limits. Property~(c) implies that the support of each
prior contains that of $\truDist$, so a precise enough signal gives a small fidelity error to $\truDist$-almost every
type.}

\smrev{Full support need not survive weak limits, so property~(c) needs more than full support of each output
distribution. For every fixed compact interval $I$ of positive length in the interior of $H_{\truDist}$, we prove a
lower bound on the mass that the output distributions put on $I$. This bound may depend on $I$ and on the round, but it
holds for every output distribution that the family generates, so it passes to weak limits.}

\smrev{Three lemmas serve this construction.}
\begin{itemize}
\item \smrev{\Cref{ec-mm-lemma: compact recovery} shows that, for every $\varepsilon>0$, one small cost makes the
population-averaged fidelity error at most $\varepsilon$ for every prior of a family with properties~(a)--(c) and every
tie-breaking rule. It chooses $\costParam_t$.}
\item \smrev{\Cref{ec-mm-lemma: compact support} shows that, if, for every prior of such a family, every tie-breaking
rule improves on the uninformative signal by a common amount $d>0$, then the output distributions generated from the
family lie in a new family with the same properties. It builds $\mathcal{K}_{t+1}$.}
\item \smrev{\Cref{ec-mm-lemma: uniform mean} is used only inside the proof of \Cref{ec-mm-lemma: compact support}. It
shows that the expected AI output of a user who sends a very imprecise signal is close to $0$.}
\end{itemize}

\smrev{\Cref{ec-mm-lemma: compact support} preserves full support as follows. The improvement $d$ must come from a set
of users, of probability at least $\alpha>0$, whose types satisfy $|\truth|\leq K$ and whose precisions lie in a fixed
interval $[\eta,U]$, where $0<\eta\leq U<\infty$; the constants $\alpha$, $K$, $\eta$ and $U$ do not depend on the
prior or the rule. The cutoff $K$ exists because $\truDist$ has a finite second moment, so the types with $|\truth|>K$
contribute little. The lower end $\eta$ exists because users with very imprecise signals cannot provide the
improvement. Expanding the square, the fidelity gain of a signal over the uninformative signal is at most $2\truth$
times the expected AI output (a signed number), and \Cref{ec-mm-lemma: uniform mean} shows that this expected output is
close to $0$ when the precision is small. The upper end $U$ exists because users with $|\truth|\leq K$ and precision
above $U$ do not exist: their communication cost would exceed $\truth^2\leq K^2$, the loss of the uninformative
signal.}

\smrev{Finally, the posterior mean is a strictly increasing map onto the interior of $H_{\truDist}$. So each user of
this set receives an output in $I$ with positive probability. This probability is continuous in the prior, the type
and the precision, which range over a compact set, so it has a positive lower bound $\beta_I$. Hence every output
distribution puts mass at least $\alpha\beta_I$ on $I$.}

\smrev{Item~2 (slow improvement) rests on \Cref{ec-mi-prop: collapse lower-bounded costs}. It extends
\Cref{thm:bounded-atomless-collapse} to costs that vary over time but stay above a fixed level $\ell>0$, with a collapse
time $T$ common to all such cost sequences and all tie-breaking rules. The slow sequence strictly decreases at every
round. In rounds $0,\ldots,T-1$, it stays above $\ell$, so the prior collapses to $\delta_0$ by round $T$. From round
$T$ on, it decreases to zero while staying positive. That comes too late, because $\delta_0$ is absorbing: once the
prior is $\delta_0$, the AI ignores every signal.}

\smrev{The subsection ends with two remarks. The first gives an example to which both items apply. The second explains
why item~1 assumes an atomless population, given that exact revelation is excluded (see the Setting). For the
two-point population $\frac12(\delta_{-1}+\delta_1)$ and any cost in $(0,1]$, which includes some initial costs in
$(0,2\truSD^2)$, no type has an optimal signal, so the loop is not defined and atomlessness cannot simply be dropped.}

\paragraph{\smrev{Fast improvement. Imprecise signals barely move the expected output.}} \smrev{Under a prior $\rho$,
the next lemma bounds the signed expected AI output $\EE{}{\truthAIAlg{\rho}(\truth+Z/\sqrt{u},1/\sqrt{u})}$ of a
type-$\truth$ user who sends a signal with a small precision $u$ (noise variance $1/u$), uniformly over types in a
compact interval and over symmetric priors with a common bound on the second moment. This is the quantity that matters
for the fidelity gain. By expanding the square, the gain of such a signal over the uninformative signal is at most
$2\truth$ times this expected output (the expansion is written out for Group~2 in the proof of
\Cref{ec-mm-lemma: compact support}). So very imprecise signals bring almost no fidelity gain.}

\begin{lemma}[Uniform small-precision bound for the expected output]\label{ec-mm-lemma: uniform mean}
\smrev{Let $\rho$ be a symmetric prior with $\int x^2\,\rho(dx)\leq\bar m_2<\infty$, and let $u>0$ be a precision. Then,
for every $\truth\in\mathbb{R}$,}
\begin{align}
\smrev{\EE{}{\truthAIAlg{\rho}\bigl(Z/\sqrt{u},1/\sqrt{u}\bigr)^2}}&\smrev{\leq e^{u\bar m_2/2}\,\bar m_2,
\label{ec-mm-eq: baseline second}}\\
\smrev{\Bigl|\EE{}{\truthAIAlg{\rho}\bigl(\truth+Z/\sqrt{u},1/\sqrt{u}\bigr)}\Bigr|} &\smrev{\leq\sqrt{e^{u\bar
m_2/2}\,\bar m_2\,\bigl(e^{u\truth^2}-1\bigr)}.\label{ec-mm-eq: small precision mean}}
\end{align}
\smrev{Consequently, $\EE{}{\truthAIAlg{\rho}(\truth+Z/\sqrt{u},1/\sqrt{u})}\rightarrow0$ as $u\downarrow0$, uniformly
over such priors and over $\truth$ in any compact interval.}
\end{lemma}

\begin{proof}{Proof of \Cref{ec-mm-lemma: uniform mean}.}
\smrev{The left-hand side of \Cref{ec-mm-eq: baseline second} is the second moment of the output of the type $\truth=0$;
the left-hand side of \Cref{ec-mm-eq: small precision mean} is the absolute value of the expected output of type
$\truth$. We bound the first,
and then use it for the second. A change of measure writes the expected output of type $\truth$ as an expectation over
the signal of type $0$: the expectation of the output multiplied by the likelihood ratio of the two signals. Subtracting
$1$ from the likelihood ratio does not change this expectation, because the expected output of type $0$ is zero by
symmetry. The Cauchy--Schwarz inequality then bounds the result by the square root of the product of two factors: the
type-$0$ second moment and the variance of the likelihood ratio.}

\smrev{\emph{The output as a function of the rescaled signal.} Fix $u>0$. We rescale a signal $\query$ with noise
variance $1/u$ as $y\triangleq\sqrt{u}\,\query$; for type $\truth$, $y=\sqrt{u}\,\truth+Z$. As in the proof of
\Cref{ec-mi-lemma: derivative}, the factor $e^{-u\query^2/2}$ cancels from the numerator and the denominator of
\Cref{ec-mi-eq: posterior mean}, so the posterior given $\query$ has density proportional to $e^{\sqrt{u}\,yx-ux^2/2}$
relative to $\rho$. Hence the AI output, as a function of the rescaled signal, is}
\[
\smrev{\psi(y)\triangleq\truthAIAlg{\rho}\bigl(y/\sqrt{u},1/\sqrt{u}\bigr)=\frac{N_1(y)}{N_0(y)}, \qquad
N_k(y)\triangleq\int x^k e^{\sqrt{u}\,yx-ux^2/2}\,\rho(dx)\quad(k=0,1,2).}
\]
\smrev{These are the weighted moments of the proof of \Cref{ec-mi-lemma: derivative}, under the prior $\rho$ and the
variance $1/u$, evaluated at the signal $\query=y/\sqrt{u}$. They are finite, because $e^{\sqrt{u}\,yx-ux^2/2}\leq
e^{y^2/2}$ and $\rho$ has a finite second moment.}

\smrev{\emph{The second moment at $\truth=0$.} We first bound the normalizer $N_0$ from below. Jensen's inequality for
the convex exponential, together with $\int x\,\rho(dx)=0$ (by symmetry), gives}
\[
\smrev{N_0(y)\geq\exp\Bigl\{\int\bigl(\sqrt{u}\,yx-ux^2/2\bigr)\,\rho(dx)\Bigr\}=e^{-u\int x^2\rho(dx)/2} \geq e^{-u\bar
m_2/2}\qquad(y\in\mathbb{R}).}
\]
\smrev{Next, Jensen's inequality under the posterior gives $\psi(y)^2\leq N_2(y)/N_0(y)$, the posterior second moment.
For type $0$, the rescaled signal is $y=Z$. For each $x$, $\EE{}{e^{\sqrt{u}\,Zx-ux^2/2}}=1$, the mean of a lognormal
variable, so Tonelli's theorem gives $\EE{}{N_2(Z)}=\int x^2\,\rho(dx)\leq\bar m_2$. Therefore}
\[
\smrev{\EE{}{\psi(Z)^2}\leq\EE{}{\frac{N_2(Z)}{N_0(Z)}}\leq e^{u\bar m_2/2}\,\EE{}{N_2(Z)}\leq e^{u\bar m_2/2}\,\bar
m_2.}
\]
\smrev{Since $\psi(Z)=\truthAIAlg{\rho}(Z/\sqrt{u},1/\sqrt{u})$, this is \Cref{ec-mm-eq: baseline second}.}

\smrev{\emph{The expected output of type $\truth$.} Because $\rho$ is symmetric, $\psi$ is odd, by
\Cref{ec-mm-lemma: posterior}(i) (if $\rho=\delta_0$, then $\psi\equiv0$). Since $Z$ is symmetric, $\EE{}{\psi(Z)}=0$.
The rescaled signal of type $\truth$, $\sqrt{u}\,\truth+Z$, has the distribution $\mathcal{N}(\sqrt{u}\,\truth,1)$,
whose likelihood ratio relative to $\mathcal{N}(0,1)$ is $r_\truth(y)\triangleq e^{\sqrt{u}\,\truth y-u\truth^2/2}$.
Hence}
\[
\smrev{\EE{}{\psi(\sqrt{u}\,\truth+Z)}=\EE{}{\psi(Z)\,r_\truth(Z)}=\EE{}{\psi(Z)\bigl(r_\truth(Z)-1\bigr)},}
\]
\smrev{where the second equality subtracts $\EE{}{\psi(Z)}=0$. A Gaussian computation gives
$\EE{}{r_\truth(Z)^2}=e^{u\truth^2}$, so $\EE{}{(r_\truth(Z)-1)^2}=e^{u\truth^2}-1$. The Cauchy--Schwarz inequality and
\Cref{ec-mm-eq: baseline second} then give}
\[
\smrev{\bigl|\EE{}{\psi(\sqrt{u}\,\truth+Z)}\bigr|\leq\sqrt{\EE{}{\psi(Z)^2}}\,\sqrt{e^{u\truth^2}-1} \leq\sqrt{e^{u\bar
m_2/2}\,\bar m_2\,\bigl(e^{u\truth^2}-1\bigr)}.}
\]
\smrev{This is \Cref{ec-mm-eq: small precision mean}, because
$\psi(\sqrt{u}\,\truth+Z)=\truthAIAlg{\rho}(\truth+Z/\sqrt{u},1/\sqrt{u})$. The right-hand side does not depend on
$\rho$ (only on the common bound $\bar m_2$), is nondecreasing in $\truth^2$, and tends to $0$ as $u\downarrow0$, which
gives the uniform convergence.}
\Halmos
\end{proof}

\paragraph{\smrev{One small cost for a compact family of priors.}} \smrev{The next lemma gives, for a compact family of
priors as in the Plan, one cost below which the population-averaged fidelity error is small for every prior of the
family and every tie-breaking rule.}

\begin{lemma}[Uniform recovery over compact families of priors]\label{ec-mm-lemma: compact recovery}
\smrev{Let $\mathcal{K}$ be a family of symmetric priors that is compact for weak convergence, with
$\sup_{\rho\in\mathcal{K}}\int x^2\,\rho(dx)\leq\bar m_2<\infty$, and suppose that the support of every
$\rho\in\mathcal{K}$ contains the support of $\truDist$. For every $\varepsilon>0$, there is $\bar\costParam>0$ with the
following property. For every cost $\costParam\in(0,\bar\costParam]$, every $\rho\in\mathcal{K}$ and every tie-breaking
rule for $(\rho,\costParam)$, the population-averaged fidelity error satisfies}
\begin{equation}\smrev{\label{ec-mm-eq: compact recovery claim}
\EE{}{\bigl(Y^\truth_\rho-\truth\bigr)^2}\leq\varepsilon,}
\end{equation}
\smrev{where $\truth\sim\truDist$ and $Y^\truth_\rho$ is the output of \Cref{ec-mi-eq: type output} under the prior
$\rho$, the cost $\costParam$ and that tie-breaking rule. The threshold $\bar\costParam$ depends on $\mathcal{K}$, $\bar
m_2$, $\varepsilon$ and $\truDist$, but not on $\rho$ or on the tie-breaking rule.}
\end{lemma}

\begin{proof}{Proof of \Cref{ec-mm-lemma: compact recovery}.}
\smrev{The idea is to compare every tie-breaking rule with a signal in which every user, whatever her type, sends the
same finite noise variance $v$; this $v$ may differ from one prior to another. For each prior in $\mathcal{K}$, a small
enough such $v$ gives a small fidelity error. The fidelity error at a fixed $v$ is continuous in the prior, so the same
$v$ works on an open set of priors, and compactness then lets finitely many variances $v_1,\ldots,v_m$ serve the whole
family. Finally, the cost is chosen so small that the penalties $\costParam\costEffort(\sqrt{v_i})$ of these finitely
many signals are negligible. A growth bound on posterior means that holds for the whole family supplies the domination
needed for the continuity.}

\smrev{\emph{A common growth bound.} Choose $K_*>0$ with $K_*^2\geq2\bar m_2$. By Chebyshev's inequality, every
$\rho\in\mathcal{K}$ satisfies $\rho(\{|x|>K_*\})\leq\bar m_2/K_*^2\leq1/2$, so $\rho([-K_*,K_*])\geq1/2$. By the
Cauchy--Schwarz inequality, the first absolute moment of $\rho$ is at most $\sqrt{\bar m_2}$. The growth bound
\Cref{ec-mm-eq: growth} of \Cref{ec-mm-lemma: posterior}(ii), applied with $R=K_*$ and $q=\rho([-K_*,K_*])\geq1/2$,
therefore gives, for every $\rho\in\mathcal{K}$,}
\begin{equation}\smrev{\label{ec-mm-eq: common growth}
\left|\truthAIAlg{\rho}(\query,\sqrt{v})\right|\leq2|\query|+C_*,\qquad C_*\triangleq K_*+2\sqrt{\bar m_2}
\qquad(\query\in\mathbb{R},\ v>0).}
\end{equation}

\smrev{\emph{The fidelity error at a fixed variance.} For a fixed finite $v>0$, let}
\[
\smrev{\bar{\costQual}_\rho(\sqrt{v})\triangleq\EE{}{\costQual_\rho(\truth,\sqrt{v})}
=\EE{}{\bigl(\truthAIAlg{\rho}(\truth+\sqrt{v}Z,\sqrt{v})-\truth\bigr)^2},\qquad\truth\sim\truDist,}
\]
\smrev{be the population-averaged fidelity error when every user sends a signal with noise variance $v$.}

\smrev{First, $\rho\mapsto\bar{\costQual}_\rho(\sqrt{v})$ is continuous on $\mathcal{K}$ for weak convergence. Weak
convergence of distributions on $\mathbb{R}$ is metrizable, so it suffices to consider a sequence
$\rho_n\rightarrow\rho$ in $\mathcal{K}$. For fixed $\query$ and $v$, the integrands $x\,e^{-(\query-x)^2/(2v)}$ and
$e^{-(\query-x)^2/(2v)}$ in \Cref{ec-mi-eq: posterior mean} are bounded continuous functions of $x$, and the denominator
under $\rho$ is positive. So $\truthAIAlg{\rho_n}(\query,\sqrt{v})\rightarrow\truthAIAlg{\rho}(\query,\sqrt{v})$ for
every $\query$. By \Cref{ec-mm-eq: common growth},
$|\truthAIAlg{\rho_n}(\truth+\sqrt{v}Z,\sqrt{v})-\truth|\leq3|\truth|+2\sqrt{v}|Z|+C_*$, whose square is integrable
because $\truSD^2<\infty$. Dominated convergence gives
$\bar{\costQual}_{\rho_n}(\sqrt{v})\rightarrow\bar{\costQual}_\rho(\sqrt{v})$.}

\smrev{Second, $\bar{\costQual}_\rho(\sqrt{v})\rightarrow0$ as $v\downarrow0$, for each $\rho\in\mathcal{K}$. By
\Cref{ec-mm-lemma: posterior}(ii), $\costQual_\rho(\truth,\sqrt{v})\rightarrow0$ for every $\truth$ in the support of
$\rho$. This holds for $\truDist$-almost every $\truth$, because the support of $\rho$ contains that of $\truDist$. For
$v\leq1$, the same domination gives $\costQual_\rho(\truth,\sqrt{v})\leq\EE{}{(3|\truth|+2|Z|+C_*)^2\,|\,\truth}$, which
is $\truDist$-integrable, so dominated convergence over $\truth$ gives the claim.}

\smrev{\emph{Choice of $\bar\costParam$.} By the two facts, the sets
$\mathcal{O}_v\triangleq\{\rho\in\mathcal{K}:\bar{\costQual}_\rho(\sqrt{v})<\varepsilon/2\}$, $0<v\leq1$, are open in
$\mathcal{K}$ and cover $\mathcal{K}$. By compactness, finitely many of them, with variances $v_1,\ldots,v_m$, cover
$\mathcal{K}$. Let $L\triangleq\max_i\costEffort(\sqrt{v_i})$, which is finite and positive by
\Cref{ec-mm-lemma: cost}(i), and let $\bar\costParam\triangleq\varepsilon/(2L)$.}

\smrev{Now fix $\costParam\in(0,\bar\costParam]$, $\rho\in\mathcal{K}$ and a tie-breaking rule for $(\rho,\costParam)$,
and pick $i$ with $\rho\in\mathcal{O}_{v_i}$. For each type $\truth$, the conditional fidelity error
$\EE{}{(Y^\truth_\rho-\truth)^2\,|\,\truth}$ is at most the type's utility loss at the chosen variance, because the
communication cost is nonnegative (for the uninformative signal, both equal $\truth^2$). That utility loss is the
minimal one, so it is at most the loss of the signal in which every type uses the variance $v_i$:}
\[
\smrev{\EE{}{(Y^\truth_\rho-\truth)^2\,|\,\truth}\leq\min_{v\in(0,\infty]}\obj_{\rho,\costParam}(\truth,\sqrt{v})
\leq\costQual_\rho(\truth,\sqrt{v_i})+\costParam\costEffort(\sqrt{v_i}).}
\]
\smrev{Averaging over $\truth\sim\truDist$ gives}
\[
\smrev{\EE{}{(Y^\truth_\rho-\truth)^2}\leq\bar{\costQual}_\rho(\sqrt{v_i})+\costParam L
<\frac{\varepsilon}{2}+\bar\costParam L =\varepsilon,}
\]
\smrev{which is \Cref{ec-mm-eq: compact recovery claim}.}
\Halmos
\end{proof}

\paragraph{\smrev{The outputs of a compact family lie in a compact family.}} \smrev{To apply
\Cref{ec-mm-lemma: compact recovery} again at the next round, we need a family with properties~(a)--(c) of the Plan
that contains every prior the loop can reach there. The next lemma provides it, provided that every tie-breaking rule
improves on the uninformative signal, in population-averaged fidelity error, by a common amount $d>0$. The uninformative
signal has fidelity error $\truth^2$, so its population-averaged fidelity error is $\truSD^2$, and the improvement
condition reads as in \Cref{ec-mm-eq: compact support hypothesis} below.}

\begin{lemma}[Compact families keep full support after an improvement]\label{ec-mm-lemma: compact support}
\smrev{By symmetry, $H_{\truDist}$ is either $[-B,B]$ for some $B>0$ or $\mathbb{R}$. Let $\mathcal{K}$ be a family of
symmetric priors that is compact for weak convergence, whose members have second moments at most $\bar m_2<\infty$ and
satisfy $H_\rho=H_{\truDist}$. Fix $\costParam>0$ and $d>0$, and suppose that, for every $\rho\in\mathcal{K}$ and every
tie-breaking rule for $(\rho,\costParam)$,}
\begin{equation}\smrev{\label{ec-mm-eq: compact support hypothesis}
\EE{}{\bigl(Y^\truth_\rho-\truth\bigr)^2}\leq\truSD^2-d,}
\end{equation}
\smrev{with $Y^\truth_\rho$ and $\truth\sim\truDist$ as in \Cref{ec-mm-lemma: compact recovery}. Let $\mathcal{K}_+$ be
the weak closure of the set of all the corresponding output distributions, as $\rho$ ranges over $\mathcal{K}$ and the
rule over all tie-breaking rules for $(\rho,\costParam)$. Then $\mathcal{K}_+$ is compact for weak convergence, and each
of its members is symmetric, has second moment at most $4\truSD^2$, and has support exactly $H_{\truDist}$.}
\end{lemma}

\begin{proof}{Proof of \Cref{ec-mm-lemma: compact support}.}
\smrev{The proof has three steps. Step~1 shows that, under every prior of $\mathcal{K}$ and every tie-breaking rule, a
set of users of probability at least $\alpha>0$ has $|\truth|\leq K$ and chooses a precision in a fixed interval
$[\eta,U]$, where $\alpha$, $K$, $\eta$ and $U$ do not depend on the prior or the rule. Step~2 shows that each of these
users receives an output in every fixed compact interval $I$ of positive length inside the interior of $H_{\truDist}$
with probability at least $\beta_I>0$. So every output distribution puts mass at least $\alpha\beta_I$ on $I$, which
gives full support, also for weak limits. Step~3 bounds second moments and applies Prokhorov's
theorem.}

\smrev{\emph{Step 1: a set of users of fixed size chooses a precision in a fixed interval.} The constants $K$, $\eta$,
$U$ and $\alpha$ chosen in this step depend only on $(\costParam,d,\bar m_2,\truDist)$. Fix $\rho\in\mathcal{K}$ and a
tie-breaking rule $v^\star_\rho$ for $(\rho,\costParam)$, and let $u^\star(\truth)\triangleq1/v^\star_\rho(\truth)$ be
the chosen precision, with $u^\star(\truth)=0$ for the uninformative signal. The quantities below depend on $\rho$ and
on the rule, which stay fixed until the end of Step~1; we omit them from the notation where no confusion arises. We
write
$\Delta_{\rho,v}(\truth)\triangleq\truth^2-\costQual_\rho(\truth,\sqrt{v})$ for the fidelity gain of a signal with
variance $v$ over the uninformative signal, with $\Delta_{\rho,\infty}\triangleq0$, and we write
$\Delta^\star(\truth)\triangleq\Delta_{\rho,v^\star_\rho(\truth)}(\truth)$ for the gain at the chosen variance. Since
$\EE{}{(Y^\truth_\rho-\truth)^2\,|\,\truth}=\truth^2-\Delta^\star(\truth)$ and $\EE{}{\truth^2}=\truSD^2$,
\Cref{ec-mm-eq: compact support hypothesis} says that the aggregate gain is at least $d$:}
\[
\smrev{\int\Delta^\star(\truth)\,\truDist(d\truth)\geq d.}
\]
\smrev{Moreover, $\Delta^\star(\truth)\leq\truth^2$, because a fidelity error is nonnegative, and
$\Delta^\star(\truth)\geq0$, because the chosen variance is optimal: its utility loss, which is at least its fidelity
error, is at most the loss $\truth^2$ of the uninformative signal.}

\smrev{We now show that users with $|\truth|\leq K$ whose chosen precision lies in a fixed interval $[\eta,U]$ have an
aggregate gain of at least $d/2$. To do so, we set aside the users who can contribute little, in three groups.}

\smrev{\emph{Group 1: users with $|\truth|>K$.} Choose $K>0$ so large that $\EE{}{\truth^2\1{|\truth|>K}}\leq d/4$,
which is possible because $\truSD^2<\infty$. Since $\Delta^\star(\truth)\leq\truth^2$, these users contribute at most
$d/4$ to the aggregate gain.}

\smrev{\emph{Group 2: users with $|\truth|\leq K$ and a small precision.} For a finite variance $v=1/u$, expanding the
square in
$\costQual_\rho$ gives}
\[
\smrev{\Delta_{\rho,1/u}(\truth)=2\truth\,\EE{}{\truthAIAlg{\rho}(\truth+Z/\sqrt{u},1/\sqrt{u})}
-\EE{}{\truthAIAlg{\rho}(\truth+Z/\sqrt{u},1/\sqrt{u})^2}
\leq2\truth\,\EE{}{\truthAIAlg{\rho}(\truth+Z/\sqrt{u},1/\sqrt{u})}.}
\]
\smrev{By \Cref{ec-mm-eq: small precision mean} of \Cref{ec-mm-lemma: uniform mean}, for $|\truth|\leq K$ and $u>0$,}
\[
\smrev{\Delta_{\rho,1/u}(\truth)\leq2K\sqrt{e^{u\bar m_2/2}\,\bar m_2\,\bigl(e^{uK^2}-1\bigr)}.}
\]
\smrev{The right-hand side does not depend on $\rho$ and tends to $0$ as $u\downarrow0$. Choose $\eta>0$ so small that
it is at most $d/4$ for $0<u<\eta$. The users with $|\truth|\leq K$ and $u^\star(\truth)<\eta$, including those who send
the uninformative signal (whose gain is $0$), therefore contribute at most $d/4$.}

\smrev{\emph{Group 3: users with $|\truth|\leq K$ and a large precision.} We show that there are none beyond a fixed
level $U$. Comparison with the uninformative signal gives $\costParam\mathcal{J}(u^\star(\truth))\leq\truth^2\leq K^2$,
because the fidelity error is nonnegative. By \Cref{ec-mm-lemma: cost}(i)--(ii), $\mathcal{J}$ is nondecreasing and
$\mathcal{J}(u)\rightarrow\infty$ as $u\rightarrow\infty$. So there is a finite $U\geq\eta$, depending only on
$(\costParam,K,\eta,\truDist)$, such that $\costParam\mathcal{J}(u)>K^2$ for every $u>U$. Hence}
\[
\smrev{u^\star(\truth)\leq U\qquad\text{whenever }|\truth|\leq K.}
\]

\smrev{\emph{The remaining users.} The first two groups contribute at most $d/2$, and the third is empty. So the
remaining users, those in}
\[
\smrev{E\triangleq\bigl\{\truth:|\truth|\leq K,\ \eta\leq u^\star(\truth)\leq U\bigr\},}
\]
\smrev{contribute at least $d/2$. Since $\Delta^\star(\truth)\leq\truth^2\leq K^2$ on $E$, we get $K^2\truDist(E)\geq
d/2$, that is,}
\begin{equation}\smrev{\label{ec-mm-eq: uniform coalition}
\truDist(E)\geq\alpha\triangleq\frac{d}{2K^2}>0.}
\end{equation}
\smrev{Only the set $E$ depends on the prior and the rule; the constants $K$, $\eta$, $U$ and $\alpha$ do not.}

\smrev{\emph{Step 2: every output distribution, and every weak limit of them, has full support.} Fix a compact interval
$I$ of positive length in the interior of $H_{\truDist}$. For $(\rho,\truth,u)\in\mathcal{K}\times[-K,K]\times[\eta,U]$,
let}
\[
\smrev{q_I(\rho,\truth,u)\triangleq\PP\bigl(\truthAIAlg{\rho}(\truth+Z/\sqrt{u},1/\sqrt{u})\in I\bigr)}
\]
\smrev{be the probability that a type-$\truth$ user who sends a signal with precision $u$ receives an output in $I$. We
show that $q_I$ is strictly positive and continuous, so that, by compactness, it has a positive minimum.}

\smrev{\emph{Positivity.} Each $\rho\in\mathcal{K}$ is symmetric and nondegenerate, because $H_\rho=H_{\truDist}$ has
positive length. By \Cref{ec-mm-lemma: posterior}(i), the posterior mean under $\rho$ is then a strictly increasing
continuous bijection from $\mathbb{R}$ onto $J_\rho$, the interior of $H_\rho=H_{\truDist}$. So the set of signals whose
output lies in $I$ is an interval of positive length, and the Gaussian signal falls in it with positive probability.}

\smrev{\emph{Continuity.} Let $(\rho_n,\truth_n,u_n)\rightarrow(\rho,\truth,u)$ in
$\mathcal{K}\times[-K,K]\times[\eta,U]$. The probability $q_I$ is the expectation, over $Z$, of the indicator that
the output lies in $I$. By dominated convergence, it suffices to show that, for almost every realization $z$ of $Z$,
this indicator converges. There are two obstacles: the integrands of \Cref{ec-mi-eq: posterior mean} must converge
uniformly in $x$ over the whole real line, which needs a bound on their tails, and the indicator is discontinuous at
the endpoints of $I$. We argue in four parts.}

\smrev{First, the signals converge. Fix a realization $z$ of $Z$. The signals $\query_n=\truth_n+z/\sqrt{u_n}$ converge
to $\query=\truth+z/\sqrt{u}$. Let $S\triangleq\sup_n|\query_n|$, which is finite.}

\smrev{Second, the integrands converge uniformly in $x$, by a tail bound for large $|x|$ and joint continuity on
bounded intervals. The integrands $x\,e^{-u_n(\query_n-x)^2/2}$ and $e^{-u_n(\query_n-x)^2/2}$ of
\Cref{ec-mi-eq: posterior mean} converge, uniformly in $x$, to their limits, which are bounded and continuous in $x$.
Indeed, since $u_n\geq\eta$, the absolute values of both integrands, and of their
limits, satisfy}
\[
\smrev{|f_n(x)|,\ |f(x)|\leq(1+|x|)\,e^{-\eta(|x|-S)^2/2}\qquad(|x|>S),}
\]
\smrev{where $f_n$ is either integrand and $f$ its limit. The right-hand side is uniformly small for $|x|$ large. On
each bounded interval of $x$, the convergence is uniform because the integrands are jointly continuous in
$(x,\query,u)$.}

\smrev{Third, the outputs converge. For each of the two integrands,}
\[
\smrev{\Bigl|\int f_n\,d\rho_n-\int f\,d\rho\Bigr|\leq\sup_x|f_n(x)-f(x)|+\Bigl|\int f\,d\rho_n-\int f\,d\rho\Bigr|.}
\]
\smrev{The first term tends to $0$ by the uniform convergence, and the second by the weak convergence
$\rho_n\rightarrow\rho$, because $f$ is bounded and continuous. So the numerator and the denominator of
\Cref{ec-mi-eq: posterior mean} converge. The limiting denominator is positive, so the outputs converge.}

\smrev{Fourth, the endpoints of $I$. Because the limiting posterior mean is strictly increasing, each endpoint of $I$
is the output of a single signal, which the Gaussian signal hits with probability $0$. So, for almost every $z$, the
limiting output is not an endpoint of $I$, and the indicator that the output lies in $I$ converges. Dominated
convergence gives $q_I(\rho_n,\truth_n,u_n)\rightarrow q_I(\rho,\truth,u)$.}

\smrev{\emph{The minimum.} Since $q_I$ is positive and continuous on the compact set
$\mathcal{K}\times[-K,K]\times[\eta,U]$, its minimum there is positive:}
\[
\smrev{\beta_I\triangleq\min_{\mathcal{K}\times[-K,K]\times[\eta,U]}q_I>0.}
\]

\smrev{\emph{Positive mass on every interval.} Every user in $E$ receives an output in $I$ with probability at least
$\beta_I$. So, by \Cref{ec-mm-eq: uniform coalition}, every output distribution $\nu$ considered in the lemma
satisfies}
\[
\smrev{\nu(I)\geq\truDist(E)\,\beta_I\geq\alpha\beta_I>0.}
\]
\smrev{Since $I$ is closed, the portmanteau theorem gives $\nu(I)\geq\limsup_n\nu_n(I)$ whenever $\nu_n\rightarrow\nu$
weakly. So the same bound holds for every member of $\mathcal{K}_+$:}
\[
\smrev{\nu(I)\geq\alpha\beta_I\qquad\text{for every }\nu\in\mathcal{K}_+.}
\]
\smrev{Every open interval that meets the interior of $H_{\truDist}$ contains such an interval $I$, so it has positive
mass under every member of $\mathcal{K}_+$.}

\smrev{\emph{Support contained in $H_{\truDist}$.} Every output lies in $H_{\truDist}$
(\Cref{ec-mm-lemma: posterior}(iii)), and $H_{\truDist}$ is closed. So, by the portmanteau theorem again, every member
of $\mathcal{K}_+$ is supported on $H_{\truDist}$.}

\smrev{Together, these two facts prove that every member of $\mathcal{K}_+$ has support exactly $H_{\truDist}$.}

\smrev{\emph{Step 3: second moments, symmetry and compactness.} We show that $\mathcal{K}_+$ is compact and that each of
its members is symmetric with second moment at most $4\truSD^2$.}

\smrev{First, the output distributions themselves. By \Cref{ec-mm-eq: pointwise bounds} of
\Cref{ec-mm-lemma: wellposed}(ii), $\EE{}{(Y^\truth_\rho)^2\,|\,\truth}\leq4\truth^2$, so every output distribution has
second moment at most $4\truSD^2$. By the same lemma, every output distribution is symmetric.}

\smrev{Second, compactness. A common second-moment bound makes the set of output distributions tight, by Chebyshev's
inequality. So, by Prokhorov's theorem, its weak closure $\mathcal{K}_+$ is compact.}

\smrev{Third, it remains to check that the second-moment bound and symmetry pass to weak limits. For the second moment,
let $\nu_k\rightarrow\nu$ weakly. Then $\int\min\{x^2,n\}\,\nu(dx)=\lim_k\int\min\{x^2,n\}\,\nu_k(dx)\leq4\truSD^2$ for
every $n$, and letting $n\rightarrow\infty$ gives $\int x^2\,\nu(dx)\leq4\truSD^2$. For symmetry, the reflection
$x\mapsto-x$ is continuous, so a weak limit of symmetric distributions is symmetric. This completes the proof.}
\Halmos
\end{proof}

\paragraph{\smrev{Fast improvement: proof of item~1.}}

\begin{proof}{Proof of \Cref{prop: recover diversity}, item 1.}
\smrev{Item~1 states that, when $\truDist$ is symmetric and atomless with finite variance and
$\costParam_0\in(0,2\truSD^2)$, some sequence of costs that strictly decreases to $0$ gives
$\lim_{t\rightarrow\infty}\VV(\outcome_t)=\truSD^2$, whatever tie-breaking rule is used. The proof has two inductions.
The first constructs the costs together with the families $\mathcal{K}_0,\mathcal{K}_1,\ldots$ of the Plan; the second
shows that, whatever tie-breaking rules are used, the prior of round $t$ lies in $\mathcal{K}_t$. The roles of the
first round, of \Cref{ec-mm-lemma: compact recovery} and of \Cref{ec-mm-lemma: compact support} are as described in the
Plan.}

\smrev{\emph{The first round.} We show that there is $d_0>0$ such that}
\begin{equation}\smrev{\label{ec-mm-eq: first round gain}
\EE{}{(\outcome_0-\truth)^2}\leq\truSD^2-d_0\qquad\text{for every tie-breaking rule at round }0.}
\end{equation}
\smrev{At round $0$, the AI's prior is the population itself, $\pi_0=\truDist$. So, for a finite variance $v$, the AI
output $\truthAIAlg{\truDist}(\query,\sqrt{v})$ is the conditional mean $\EE{}{\truth\,|\,\truth+\sqrt{v}Z=\query}$, and
the population-averaged fidelity error $\int\costQual_{\truDist}(\truth,\sqrt{v})\,\truDist(d\truth)$ is the mean
squared error of this conditional mean as an estimator of $\truth$. The conditional mean minimizes the mean squared
error, so this error is at most that of the best linear estimator $\truSD^2(\truth+\sqrt{v}Z)/(\truSD^2+v)$:}
\begin{equation}\smrev{\label{ec-mm-eq: first round linear}
\int\costQual_{\truDist}(\truth,\sqrt{v})\,\truDist(d\truth)\leq\frac{\truSD^2v}{\truSD^2+v}\qquad(0<v<\infty).}
\end{equation}

\smrev{For the comparison, consider the signal with the same variance $v=\truSD^2/r$ for every type, for a parameter
$r>0$. This signal need not be optimal for anyone. The bound \Cref{ec-mm-eq: first round linear} becomes
$\truSD^2v/(\truSD^2+v)=\truSD^2/(1+r)$. By the upper bound in \Cref{ec-mm-eq: cost bounds},
$\costEffort(\sqrt{v})\leq\frac12\ln(1+r)$. The optimal loss of each type is at most its loss at this variance, so}
\begin{equation}\smrev{\label{ec-mm-eq: first round bound}
\EE{}{\min_{v'\in(0,\infty]}\obj_{\truDist,\costParam_0}(\truth,\sqrt{v'})}
\leq\frac{\truSD^2}{1+r}+\frac{\costParam_0}{2}\ln(1+r).}
\end{equation}
\smrev{At $r=0$, the right-hand side equals $\truSD^2$, and its derivative in $r$ there is $-\truSD^2+\costParam_0/2$,
which is negative because $\costParam_0<2\truSD^2$. So there is $r>0$ at which the right-hand side equals $\truSD^2-d_0$
for some $d_0>0$.}

\smrev{For every tie-breaking rule at round $0$, the conditional fidelity error of each type is at most its optimal
loss, as in the proof of \Cref{ec-mm-lemma: compact recovery}. Averaging over $\truth$ and using
\Cref{ec-mm-eq: first round bound} gives \Cref{ec-mm-eq: first round gain}.}

\smrev{\emph{The first two families.} Let $\mathcal{K}_0\triangleq\{\truDist\}$, and let $\mathcal{K}_1$ be the weak
closure of the set of all output distributions of round $0$, that is, under the prior $\truDist$ and the cost
$\costParam_0$, over all tie-breaking rules. The family $\mathcal{K}_0$ satisfies the hypotheses of
\Cref{ec-mm-lemma: compact support} with $\bar m_2=\truSD^2$, and \Cref{ec-mm-eq: first round gain} is its hypothesis
\Cref{ec-mm-eq: compact support hypothesis} with $d=d_0$. So $\mathcal{K}_1$ is compact, and each of its members is
symmetric, has second moment at most $4\truSD^2$, and has support $H_{\truDist}$; in particular, its support contains
the support of $\truDist$.}

\smrev{\emph{The induction.} Let $t\geq1$, and suppose that $\costParam_0,\ldots,\costParam_{t-1}$ and a compact family
$\mathcal{K}_t$ have been constructed, whose members are symmetric, have second moment at most $4\truSD^2$, and have
support $H_{\truDist}$.}

\smrev{\emph{Choice of $\costParam_t$.} Apply \Cref{ec-mm-lemma: compact recovery} to $\mathcal{K}_t$ with $\bar
m_2=4\truSD^2$ and $\varepsilon_t\triangleq\truSD^2 4^{-t}$; this gives $\bar\costParam_t>0$. Choose any}
\[
\smrev{0<\costParam_t<\min\bigl\{\costParam_{t-1}/2,\ \bar\costParam_t\bigr\}.}
\]
\smrev{Tie-breaking rules for $(\rho,\costParam_t)$ exist for every $\rho\in\mathcal{K}_t$, by
\Cref{ec-mm-lemma: wellposed}(i), because $\rho$ has a finite second moment. By the choice of $\costParam_t$ and
\Cref{ec-mm-eq: compact recovery claim},}
\begin{equation}\smrev{\label{ec-mm-eq: induction error}
\EE{}{\bigl(Y^\truth_\rho-\truth\bigr)^2}\leq\varepsilon_t\qquad\text{for every }\rho\in\mathcal{K}_t\text{ and every
tie-breaking rule for }(\rho,\costParam_t).}
\end{equation}

\smrev{\emph{The next family.} Let $\mathcal{K}_{t+1}$ be the weak closure of the set of all the output distributions
in \Cref{ec-mm-eq: induction error}. Since $\varepsilon_t\leq\truSD^2/4$, the bound \Cref{ec-mm-eq: induction error} is
\Cref{ec-mm-eq: compact support hypothesis} with $d=\truSD^2-\varepsilon_t>0$. So \Cref{ec-mm-lemma: compact support}
applies with $\bar m_2=4\truSD^2$: the family $\mathcal{K}_{t+1}$ is compact, and its members are symmetric, have second
moment at most $4\truSD^2$, and have support $H_{\truDist}$. This completes the induction.}

\smrev{The construction uses only $\truDist$ and $\costParam_0$. That every
prior of the loop lies in the corresponding family is shown next.}

\smrev{\emph{Every run of the loop stays in the families.} Fix any sequence of tie-breaking rules, one per round, and
run the loop with the costs $(\costParam_t)$. We show by induction that $\pi_t\in\mathcal{K}_t$ for every $t\geq0$.
First, $\pi_0=\truDist\in\mathcal{K}_0$. If $\pi_t\in\mathcal{K}_t$, then $\pi_{t+1}$ is the output distribution under
the prior $\pi_t$, the cost $\costParam_t$ and a tie-breaking rule for $(\pi_t,\costParam_t)$, so
$\pi_{t+1}\in\mathcal{K}_{t+1}$ by the definition of $\mathcal{K}_{t+1}$. Hence, at every round $t\geq1$,
\Cref{ec-mm-eq: induction error} applies to $\rho=\pi_t$ and gives}
\[
\smrev{\EE{}{(\outcome_t-\truth)^2}\leq\varepsilon_t=\truSD^2 4^{-t}.}
\]

\smrev{\emph{The variance of outputs.} By \Cref{ec-mm-lemma: wellposed}(ii), $\outcome_t$ is symmetric with a finite
second moment, so $\VV(\outcome_t)=\EE{}{(\outcome_t)^2}$; likewise, $\truSD^2=\EE{}{\truth^2}$. The triangle inequality
in $L^2$, $\bigl|\|\outcome_t\|_2-\|\truth\|_2\bigr|\leq\|\outcome_t-\truth\|_2$, therefore gives}
\[
\smrev{\Bigl|\sqrt{\VV(\outcome_t)}-\truSD\Bigr|\leq\sqrt{\EE{}{(\outcome_t-\truth)^2}}\leq\truSD\,2^{-t}
\qquad(t\geq1).}
\]
\smrev{Therefore $\VV(\outcome_t)\rightarrow\truSD^2$ as $t\rightarrow\infty$. The costs are positive, strictly
decreasing, and satisfy $\costParam_t\leq2^{-t}\costParam_0\rightarrow0$. Since the tie-breaking rules were arbitrary,
this proves item~1.}
\Halmos
\end{proof}

\paragraph{\smrev{Slow improvement. With bounded support, the loop collapses while costs stay above a positive level.}}
\smrev{The next proposition is the extension of \Cref{thm:bounded-atomless-collapse} announced in the Plan. It assumes
bounded support but allows atoms: the condition \Cref{ec-mm-eq: coercive} below takes over the role that atomlessness
plays in \Cref{ec-mi-lemma: vmin}, namely ruling out very precise optimal signals.}

\begin{proposition}[Costs bounded away from zero cannot prevent collapse]
\label{ec-mi-prop: collapse lower-bounded costs} \smrev{Let $\truDist$ be symmetric and nondegenerate with
$H_{\truDist}=[-B,B]$ for some $B\in(0,\infty)$, and fix $\ell>0$} such that
\begin{equation}\smrev{\label{ec-mm-eq: coercive}
\ell\lim_{v\downarrow0}\costEffort(\sqrt{v})>B^2.}
\end{equation}
\smrev{This holds for every $\ell>0$ when $\truDist$ is atomless, by \Cref{ec-mi-lemma: cost}(iii). Then there is a
finite integer $T=T(\truDist,\ell)\geq1$ such that, for every sequence of costs $\costParam_t\geq\ell$ (not necessarily
constant, monotone or convergent) and every sequence of tie-breaking rules, $\pi_T=\delta_0$, and hence $\pi_t=\delta_0$
for every $t\geq T$. Under these assumptions, a tie-breaking rule exists at every round.}
\end{proposition}

\begin{proof}{Proof of \Cref{ec-mi-prop: collapse lower-bounded costs}.}
\smrev{The proof follows the proof of \Cref{thm:bounded-atomless-collapse}, with two changes.}

\smrev{First, every constant must be uniform over the costs $\costParam_t\geq\ell$. Two constants matter: the lower
bound $v_{\min}$ on finite optimal variances (\Cref{ec-mi-lemma: vmin}), which we replace by a level defined from
$\ell$ and \Cref{ec-mm-eq: coercive}, and the second-moment threshold $\varepsilon_0$ below which every user sends the
uninformative signal (\Cref{ec-mi-lemma: basin}).}

\smrev{Second, the collapse time must depend neither on the costs nor on the tie-breaking rules. So, instead of
following one run of the loop, we bound at once every sequence of priors generated by measurable variance rules that
treat $\truth$ and $-\truth$ alike and use variances at least $v_{\min}$ (Step~3 calls these controlled sequences).
These rules need not be optimal. Step~5 shows that there is one round $n$ at which the second moment of every such
sequence is at most $\varepsilon_0$. Optimality is used only in Steps~1 and~2: Step~1 shows that actual runs of the loop
are such sequences, and Step~2 shows that a second moment at most $\varepsilon_0$ forces collapse at the next round.}

\smrev{The proof has five steps.}
\begin{itemize}
\item \smrev{Step~1 shows that every finite optimal variance exceeds $v_{\min}$, for every cost at least $\ell$, and
that tie-breaking rules exist.}
\item \smrev{Step~2 shows that a prior with second moment at most $\varepsilon_0$ makes every finite variance strictly
worse than the uninformative signal, for every user and every cost at least $\ell$.}
\item \smrev{Step~3 shows that every controlled sequence puts mass at least $\delta_c$ on $[-c,c]$ from round $1$ on,
for every $c>0$, where $\delta_c>0$ depends only on $c$, $B$ and $v_{\min}$.}
\item \smrev{Step~4 shows that, given this central mass, the probability that a user's output exceeds a cutoff $a>c$ is
bounded by a function of the excess $\Lambda_a(\rho)=\int(x-a)_+\,\rho(dx)$ of the prior $\rho$ beyond $a$, and that
this function is small compared with the excess when the excess is small.}
\item \smrev{Step~5 shows that the excess beyond every cutoff tends to $0$ uniformly over all controlled sequences.
Since $m_2(\rho)\leq2B\Lambda_0(\rho)$ for a symmetric distribution $\rho$ on $[-B,B]$, this gives the common round
$n$; Step~2 then gives collapse at round $T=n+1$.}
\end{itemize}

\smrev{\emph{Preliminary: symmetry and support.} Every prior along the loop is symmetric and supported on $[-B,B]$.
Indeed, $\pi_0=\truDist$ has these properties. If $\pi_t$ has them, then every posterior mean under $\pi_t$ lies in
$[-B,B]$ and the uninformative signal gives the output $0$, so $\pi_{t+1}$ is supported on $[-B,B]$. A tie-breaking rule
treats $\truth$ and $-\truth$ alike, so $\pi_{t+1}$ is symmetric, by the argument in the proof of
\Cref{ec-mi-cor: support preserved}. The same argument applies to any measurable variance rule that treats $\truth$ and
$-\truth$ alike, optimal or not, which Steps~3--5 use. Since $\truDist([-B,B])=1$, we only need to consider types
$|\truth|\leq B$.}

\smrev{\emph{Step 1: finite optimal variances are bounded below, and tie-breaking rules exist.} The communication cost
$\costEffort(\sqrt{v})$ is nonincreasing in $v$, by \Cref{ec-mm-lemma: cost}(i), so the limit in
\Cref{ec-mm-eq: coercive} exists in $(0,\infty]$, and we can choose $v_{\min}>0$ such that}
\begin{equation}\smrev{\label{ec-mm-eq: vmin lower costs}
\ell\,\costEffort(\sqrt{v_{\min}})>B^2.}
\end{equation}
\smrev{This $v_{\min}$ depends only on $(\truDist,\ell)$. It plays the role of $v_{\min}$ in \Cref{ec-mi-lemma: vmin},
for all costs at least $\ell$ at once.}

\smrev{We show that, at every round $t$, every finite optimal variance $v$ of a type $|\truth|\leq B$ satisfies
$v>v_{\min}$. Comparison with the uninformative signal, and nonnegativity of the fidelity error, give
$\costParam_t\costEffort(\sqrt{v})\leq\obj_{\pi_t,\costParam_t}(\truth,\sqrt{v})\leq\truth^2\leq B^2$. If $v\leq
v_{\min}$, the monotonicity of the cost and $\costParam_t\geq\ell$ would instead give}
\[
\smrev{\costParam_t\costEffort(\sqrt{v})\geq\ell\,\costEffort(\sqrt{v_{\min}})>B^2,}
\]
\smrev{a contradiction. In terms of the precision $u=1/v$: every precision $u\geq1/v_{\min}$ has a loss larger than
$B^2\geq\truth^2$, the loss of the uninformative signal ($u=0$). So minimizing the loss over all precisions is the same
as minimizing it over $[0,1/v_{\min}]$.}

\smrev{\emph{Existence of tie-breaking rules.} We now show that tie-breaking rules exist, by the argument in the proof
of \Cref{ec-mi-lemma: selector} with $u_{\max}=1/v_{\min}$. That argument needs the loss
$F(\truth,u)\triangleq\obj_{\pi_t,\costParam_t}(\truth,1/\sqrt{u})$, with $F(\truth,0)\triangleq\truth^2$, to be jointly
continuous on $[-B,B]\times[0,1/v_{\min}]$, including at $u=0$. It remains to check this joint continuity.}

\smrev{\emph{Continuity at $u>0$.} Continuity follows as in the proof of \Cref{ec-mi-lemma: selector}, except that the
continuity of the cost $\mathcal{J}$ now comes from \Cref{ec-mm-lemma: cost}(i), which does not require
atomlessness.}

\smrev{\emph{Continuity at $u=0$.} Consider small $u>0$. The linear bound \Cref{ec-mi-eq: contraction pointwise},
applied with the variance $v=1/u$, gives}
\[
\smrev{\bigl|\truthAIAlg{\pi_t}(\truth+Z/\sqrt{u},1/\sqrt{u})\bigr|\leq B^2u\,\bigl|\truth+Z/\sqrt{u}\bigr| \leq
B^2\bigl(u|\truth|+\sqrt{u}\,|Z|\bigr).}
\]
\smrev{The right-hand side tends to $0$ in mean square as $u\downarrow0$, uniformly in $|\truth|\leq B$. Together with
$\mathcal{J}(u)\rightarrow0$, this gives, as in the proof of \Cref{ec-mi-lemma: selector},}
\[
\smrev{F(\truth,u)\rightarrow\truth^2\quad\text{as }u\downarrow0,\qquad\text{uniformly in }|\truth|\leq B.}
\]

\smrev{\emph{Conclusion of Step~1.} By the reduction to $[0,1/v_{\min}]$ shown above, a minimum over all precisions
exists for every type. As in the proof of \Cref{ec-mi-lemma: selector}, the smallest optimal precision gives a
tie-breaking rule for $(\pi_t,\costParam_t)$.}

\smrev{\emph{Step 2: a prior with a small second moment makes every user send the uninformative signal.} This step
follows the proof of \Cref{ec-mi-lemma: basin}: we bound the slope of the posterior mean by the prior's second moment
and compare the resulting fidelity gain with the communication cost. Two differences from that proof make the bounds
hold for every cost at least $\ell$ with the $v_{\min}$ of \Cref{ec-mm-eq: vmin lower costs}. First, Jensen's
inequality gives a lower bound on the posterior normalizer that does not depend on the signal; the slope bound still
grows with the signal, so we average it over the noise. Second, a single lower bound of the form $k/v$ on the
communication cost covers all $v\geq v_{\min}$, so the variances $v\geq v_{\min}$ need not be split into a bounded
range and a range of large variances, as in that proof. Let $\rho$ be a symmetric prior supported on $[-B,B]$, with
second moment $m_2(\rho)=\int x^2\,\rho(dx)$, and fix a finite variance $v$.}

\smrev{First, the slope of the posterior mean. As in the proof of \Cref{ec-mi-lemma: derivative}, the posterior given
$\query$ has density proportional to $e^{\query x/v-x^2/(2v)}$ relative to $\rho$. We use the weighted moments of that
proof, under the prior $\rho$, $N_k(\query)=\int x^ke^{\query x/v-x^2/(2v)}\,\rho(dx)$ for $k=0,2$. By
\Cref{ec-mi-eq: posterior derivative}, the slope is the posterior variance divided by $v$, so it is at most the
posterior second moment divided by $v$, which is $N_2(\query)/(vN_0(\query))$. Since $|x|\leq B$ on the support of
$\rho$, $N_2(\query)\leq m_2(\rho)\,e^{B|\query|/v}$. By Jensen's inequality and $\int x\,\rho(dx)=0$, $N_0(\query)\geq
e^{-m_2(\rho)/(2v)}\geq e^{-B^2/(2v)}$. Hence}
\begin{equation}\smrev{\label{ec-mm-eq: slope bound}
0\leq\frac{\partial}{\partial{\query}}\truthAIAlg{\rho}(\query,\sqrt{v})
\leq\frac{m_2(\rho)}{v}\exp\left\{\frac{B|\query|}{v}+\frac{B^2}{2v}\right\}.}
\end{equation}

\smrev{Second, the expected output. Let}
\[
\smrev{C\triangleq\exp\left(\frac{3B^2}{2v_{\min}}\right)\EE{}{\exp\left(\frac{B|Z|}{\sqrt{v_{\min}}}\right)}<\infty.}
\]
\smrev{It is chosen so that the slope bound \Cref{ec-mm-eq: slope bound}, averaged over the noise, is at most
$Cm_2(\rho)/v$, as we check next. The constant $C$ depends only on $B$ and $v_{\min}$, hence only on $\truDist$ and
$\ell$; it does not depend on the prior $\rho$, the round or the cost. For $|y|\leq B$ and $v\geq v_{\min}$, we have
$B|y+\sqrt{v}Z|/v\leq B^2/v_{\min}+B|Z|/\sqrt{v_{\min}}$, so \Cref{ec-mm-eq: slope bound} gives
$\EE{}{\partial_\query\truthAIAlg{\rho}(y+\sqrt{v}Z,\sqrt{v})}\leq Cm_2(\rho)/v$. The expected output
$\EE{}{\truthAIAlg{\rho}(\truth+\sqrt{v}Z,\sqrt{v})}$ is $0$ at $\truth=0$, because the posterior mean is odd. Its
derivative in $\truth$ is the expected slope; differentiation under the expectation is valid because the slope is at
most $B^2/v$, by \Cref{ec-mi-lemma: derivative}. Integrating the slope bound from $0$ to $\truth$ gives, for
$|\truth|\leq B$ and $v\geq v_{\min}$,}
\begin{equation}\smrev{\label{ec-mm-eq: expected output bound}
0\leq\truth\,\EE{}{\truthAIAlg{\rho}(\truth+\sqrt{v}Z,\sqrt{v})}\leq\frac{Cm_2(\rho)\,\truth^2}{v}.}
\end{equation}

\smrev{Third, the communication cost. Since $\truDist$ is symmetric and nondegenerate, $m_1=\EE{}{|\truth|}>0$ and
$\truSD^2=\EE{}{\truth^2}\leq B^2$. By the lower bound in \Cref{ec-mm-eq: cost bounds}, and since $v/(v+\truSD^2)$ is
increasing in $v$,}
\[
\smrev{\costEffort(\sqrt{v})\geq\frac{m_1^2}{2(v+\truSD^2)}=\frac{m_1^2}{2v}\cdot\frac{v}{v+\truSD^2}\geq\frac{k}{v}
\quad(v\geq v_{\min}),\qquad k\triangleq\frac{m_1^2v_{\min}}{2(v_{\min}+\truSD^2)}>0.}
\]

\smrev{Fourth, the comparison with the uninformative signal. Fix a cost $\costParam\geq\ell$, a type
$|\truth|\leq B$ and $v\geq v_{\min}$. We expand the square in the fidelity error, drop the nonnegative second moment of
the output, use \Cref{ec-mm-eq: expected output bound} for the fidelity term and $\costParam\geq\ell$ with the cost
bound for the communication term, and finally use $\truth^2\leq B^2$:}
\begin{align}
\smrev{\obj_{\rho,\costParam}(\truth,\sqrt{v})-\truth^2}
&\smrev{=\EE{}{\truthAIAlg{\rho}(\truth+\sqrt{v}Z,\sqrt{v})^2}
-2\truth\,\EE{}{\truthAIAlg{\rho}(\truth+\sqrt{v}Z,\sqrt{v})}+\costParam\costEffort(\sqrt{v})}\nonumber\\
&\smrev{\geq-2\truth\,\EE{}{\truthAIAlg{\rho}(\truth+\sqrt{v}Z,\sqrt{v})}+\costParam\costEffort(\sqrt{v})}\nonumber\\
&\smrev{\geq-\frac{2Cm_2(\rho)\,\truth^2}{v}+\frac{\ell k}{v}}\nonumber\\
&\smrev{\geq\frac{\ell k-2CB^2m_2(\rho)}{v}.\label{ec-mm-eq: basin comparison}}
\end{align}
\smrev{For $v<v_{\min}$, the comparison is direct: the communication cost alone exceeds $B^2\geq\truth^2$, by
\Cref{ec-mm-eq: vmin lower costs}, as in Step~1.}

\smrev{\emph{The threshold.} Define}
\begin{equation}\smrev{\label{ec-mm-eq: basin}
\varepsilon_0\triangleq\frac{\ell k}{4CB^2}.}
\end{equation}
\smrev{This $\varepsilon_0$ plays the role of $\varepsilon_0$ in \Cref{ec-mi-lemma: basin}, for all costs at least
$\ell$ at once. If $m_2(\rho)\leq\varepsilon_0$, the numerator in \Cref{ec-mm-eq: basin comparison} is at least
$\ell k-2CB^2\varepsilon_0=\ell k/2>0$. Together with the case $v<v_{\min}$, this proves the claim of Step~2: if
$m_2(\rho)\leq\varepsilon_0$, then}
\[
\smrev{\obj_{\rho,\costParam}(\truth,\sqrt{v})>\truth^2\qquad\text{for every finite }v>0\text{, every }
\costParam\geq\ell\text{ and every }|\truth|\leq B.}
\]
\smrev{So every finite variance is strictly worse than the uninformative signal. Every user then sends the
uninformative signal, and the next prior is $\delta_0$.}

\smrev{\emph{Step 3: central mass for all sequences of priors generated with variances at least $v_{\min}$.} In
Steps~3--5, optimality plays no role. We call a} \emph{\smrev{controlled sequence}} \smrev{any sequence
$(\rho_t)_{t\geq0}$ of symmetric distributions on $[-B,B]$ with $\rho_0=\truDist$, in which each $\rho_{t+1}$ is the
output distribution under the prior $\rho_t$ of a Borel variance rule $\hat v_t$ that treats $\truth$ and $-\truth$
alike and satisfies $\hat v_t(\truth)\in[v_{\min},\infty]$ for $\truDist$-almost every $\truth$. By Step~1, every run of
the loop with costs $\costParam_t\geq\ell$ and tie-breaking rules is a controlled sequence. As in the loop, the types
are always drawn from the fixed population $\truDist$; only the prior $\rho_t$ evolves.}

\smrev{The idea of this step is that a user with a finite variance $v\geq v_{\min}$ receives an output in $[-c,c]$
whenever the noise $Z$ falls in an interval of length $2h_c$ around $-\truth/\sqrt{v}$. The half-width $h_c$ is chosen
so that the linear bound \Cref{ec-mi-eq: contraction pointwise} then keeps the output within $c$, and this event has a
probability bounded below uniformly. We show
that, for every $c>0$, every controlled sequence satisfies}
\begin{equation}\smrev{\label{ec-mm-eq: central mass}
\rho_t(}[\smrev{-c,c}]\smrev{)\geq\delta_c\quad(t\geq1),\qquad \delta_c\triangleq2h_c\,\phi(\bar z+h_c),\qquad
h_c\triangleq\frac{c\sqrt{v_{\min}}}{B^2}, \qquad \bar z\triangleq\frac{B}{\sqrt{v_{\min}}}.}
\end{equation}
\smrev{Unlike \Cref{ec-mi-lemma: central mass}, this bound does not need a set of users who keep sending informative
signals, because users who send the uninformative signal output $0\in[-c,c]$.}

\smrev{Fix $t\geq1$ and a type $|\truth|\leq B$ with $\hat v_{t-1}(\truth)\in[v_{\min},\infty]$; these types carry all
the mass of $\truDist$. Let $v=\hat v_{t-1}(\truth)$. We show that the output of this type under the prior
$\rho_{t-1}$ lies in $[-c,c]$ with probability at least $\delta_c$. There are two cases, depending on whether the type
sends the uninformative signal.}

\smrev{\emph{Case $v=\infty$.} The output is $0\in[-c,c]$ with probability $1$, so it suffices that $\delta_c<1$.
Indeed, $\delta_c\leq2h_c\phi(h_c)\leq2\phi(1)<1$. The middle inequality holds because the function $h\phi(h)$ is
maximized at $h=1$.}

\smrev{\emph{Case $v<\infty$.} Then $v\geq v_{\min}$. The output lies in $[-c,c]$ whenever $|Z+\truth/\sqrt{v}|\leq
h_c$. Indeed, on this event, $|\truth+\sqrt{v}Z|\leq h_c\sqrt{v}$, and the linear bound
\Cref{ec-mi-eq: contraction pointwise} gives an output of absolute value at most
$(B^2/v)h_c\sqrt{v}=B^2h_c/\sqrt{v}\leq B^2h_c/\sqrt{v_{\min}}=c$. The event is an interval of length $2h_c$ for $Z$,
centred at $-\truth/\sqrt{v}\in[-\bar z,\bar z]$. So it lies inside $[-\bar z-h_c,\bar z+h_c]$, where the standard
normal density is at least $\phi(\bar z+h_c)$. Hence}
\[
\smrev{\PP\bigl(|Z+\truth/\sqrt{v}|\leq h_c\bigr)\geq2h_c\,\phi(\bar z+h_c)=\delta_c.}
\]

\smrev{In both cases, the output lies in $[-c,c]$ with probability at least $\delta_c$. Averaging over
$\truth\sim\truDist$ gives \Cref{ec-mm-eq: central mass}.}

\smrev{\emph{Step 4: given central mass, few outputs exceed a cutoff.} We show that, when the prior puts mass at least
$\delta$ on $[-c,c]$, the probability that an output exceeds a cutoff $a>c$ is bounded by a function of the excess of
the prior beyond $a$, defined next, and that this function is negligible compared with the excess when the excess is
small.}

\smrev{\emph{The excess.} For a distribution $\rho$ on $[-B,B]$ and $a\in[0,B]$, let
$\Lambda_a(\rho)\triangleq\int(x-a)_+\,\rho(dx)$, as in \Cref{ec-mi-lemma: edge}. We call it the} \emph{\smrev{excess}}
\smrev{of $\rho$ beyond the cutoff $a$. It is the mass of $\rho$ beyond $a$, weighted by the distance beyond $a$, not
the tail probability $\rho((a,B])$.}

\smrev{\emph{The edge bound.} Fix $0<c<a<B$ and $\delta>0$, and let $\gamma$ and $C_0$ be the constants of
\Cref{ec-mi-lemma: edge} computed with the $v_{\min}$ of \Cref{ec-mm-eq: vmin lower costs}. The other hypotheses of
that lemma, $\rho$ symmetric on $[-B,B]$ with $\rho([-c,c])\geq\delta$ and a type in $[-B,B]$, are imposed below. Its
proof uses $v_{\min}$ only through the bound $v\geq v_{\min}$ on the variance, which holds with the present $v_{\min}$.
So that lemma applies here. For $K>0$, let}
\[
\smrev{\Xi_K(x)\triangleq\exp\left\{-\gamma\left[\ln(K/x)-C_0\right]_+^2\right\}\quad(x>0),\qquad \Xi_K(0)\triangleq0.}
\]
\smrev{If $\rho$ is symmetric and supported on $[-B,B]$ with $\rho([-c,c])\geq\delta$, then \Cref{ec-mi-eq: edge bound}
gives, for every type $|\truth|\leq B$ that sends a signal with a finite variance $v\geq v_{\min}$,}
\begin{equation}\smrev{\label{ec-mm-eq: edge controlled}
\PP\bigl(\truthAIAlg{\rho}(\truth+\sqrt{v}Z,\sqrt{v})>a\bigr)\leq\Xi_{K_a}\bigl(\Lambda_a(\rho)\bigr),\qquad
K_a\triangleq(a-c)\delta.}
\end{equation}
\smrev{A type that sends the uninformative signal outputs $0<a$.}

\smrev{We use three properties of $\Xi_K$. It is continuous and nondecreasing in $x$ on $[0,\infty)$; continuity at $0$
holds because $\ln(K/x)\rightarrow\infty$ as $x\downarrow0$. It is nonincreasing in $K$, so $K_a\geq K$ implies
$\Xi_{K_a}\leq\Xi_K$. And it is negligible compared with $x$ near $0$:}
\begin{equation}\smrev{\label{ec-mm-eq: superlinear edge}
\frac{\Xi_K(x)}{x}\longrightarrow0\qquad\text{as }x\downarrow0.}
\end{equation}
\smrev{Indeed, with $y\triangleq\ln(K/x)$, so that $x=Ke^{-y}$, we have $\Xi_K(x)/x=K^{-1}\exp\{y-\gamma(y-C_0)^2\}$
once $y>C_0$, and the exponent tends to $-\infty$ as $y\rightarrow\infty$. The constants $\gamma$ and $C_0$ depend only
on $B$, $c$ and $v_{\min}$, not on the cutoff $a$, while $K_a=(a-c)\delta$ depends also on $a$ and $\delta$. None of
them depends on the round, on the controlled sequence, or on the costs $\costParam_t\geq\ell$.}

\smrev{\emph{Step 5: the excess beyond every cutoff vanishes uniformly.} We show that the supremum of
$\Lambda_a(\rho_t)$ over all controlled sequences and all rounds $t\geq n$ tends to $0$ as $n\rightarrow\infty$, for
every cutoff $a\in[0,B]$, including $a=0$. We first show it for every $a\in(0,B]$, by contradiction, and then obtain
$a=0$ by continuity in $a$. This step adapts the contradiction argument in the proof of
\Cref{thm:bounded-atomless-collapse} (its case $R>0$; here $R$ is defined below as the supremum of the cutoffs with a positive
limiting excess), with the supremum below in place of the maximum over the limit points of one run.
Let $\mathcal{S}$ be the set of all controlled sequences; it is nonempty, because the rule under which every type sends
the uninformative signal generates one. For $n\geq1$ and $a\in[0,B]$, let}
\[
\smrev{M_n(a)\triangleq\sup_{(\rho_t)\in\mathcal{S}}\ \sup_{t\geq n}\Lambda_a(\rho_t), \qquad
M(a)\triangleq\lim_{n\rightarrow\infty}M_n(a).}
\]
\smrev{Each $M_n$ is nonnegative, nonincreasing in $a$, and $1$-Lipschitz in $a$, because $a\mapsto(x-a)_+$ has these
properties. Also, $M_n$ is nonincreasing in $n$, so the limit $M$ exists and has the same three properties. Moreover,
$M(B)=0$, because every $\rho_t$ is supported on $[-B,B]$. Let}
\[
\smrev{R\triangleq\sup\{a\in[0,B]:M(a)>0\},}
\]
\smrev{with $R\triangleq0$ if this set is empty. If $R>0$, then $M(a)>0$ for every $a<R$, because $M$ is nonincreasing,
and $M(R)=0$, because $M$ vanishes on $(R,B]$ and is continuous (if $R=B$, this is $M(B)=0$). We show that $R=0$.}

\smrev{Suppose, toward a contradiction, that $R>0$. Let $c\triangleq R/2$, let $\delta\triangleq\delta_c$ be given by
\Cref{ec-mm-eq: central mass}, and let $K\triangleq R\delta/4$. For $r\in(0,R/4)$, let $a\triangleq R-r$. Then $c<a<B$
and $K_a=(a-c)\delta=(R/2-r)\delta\geq K$, so \Cref{ec-mm-eq: edge controlled} holds with $\Xi_K$ in place of
$\Xi_{K_a}$. The function $\Xi_K$ does not depend on $r$. As in the proof of \Cref{thm:bounded-atomless-collapse}, we
split the excess beyond $a=R-r$ into the part between $a$ and $R$, which the edge bound controls, and the excess beyond
$R$, which vanishes in the limit because $M(R)=0$.}

\smrev{We bound the excess at round $t$ by the excess at round $t-1$, in three stages. First, a pointwise inequality.
For every $x$, $(x-a)_+\leq r\,\1{x>a}+(x-R)_+$: both sides vanish for $x\leq a$, the left side is at most $r$ for
$a<x\leq R$, and it equals $r+(x-R)$ for $x>R$. Integrating against $\rho_t$ gives $\Lambda_a(\rho_t)\leq
r\,\rho_t((a,\infty))+\Lambda_R(\rho_t)$.}

\smrev{Second, the edge bound. Let $n\geq1$ and $t\geq n+1$. Then $t-1\geq1$, so $\rho_{t-1}([-c,c])\geq\delta$ by
Step~3. By Step~4, for $\truDist$-almost every type, the output of round $t-1$ exceeds $a$ with probability at most
$\Xi_K(\Lambda_a(\rho_{t-1}))$. Averaging over types bounds $\rho_t((a,\infty))$ by the same quantity. Two substitutions
then remove the dependence on the round. Since $t-1\geq n$, we have $\Lambda_a(\rho_{t-1})\leq M_n(a)$, and $\Xi_K$ is
nondecreasing. Since $t\geq n+1$, we have $\Lambda_R(\rho_t)\leq M_{n+1}(R)$. Hence}
\[
\smrev{\Lambda_a(\rho_t)\leq r\,\Xi_K\bigl(\Lambda_a(\rho_{t-1})\bigr)+\Lambda_R(\rho_t) \leq
r\,\Xi_K\bigl(M_n(a)\bigr)+M_{n+1}(R).}
\]
\smrev{Third, the supremum. Taking the supremum over controlled sequences and over $t\geq n+1$ gives $M_{n+1}(a)\leq
r\,\Xi_K(M_n(a))+M_{n+1}(R)$. Letting $n\rightarrow\infty$, and using the continuity of $\Xi_K$ and $M(R)=0$, we obtain}
\begin{equation}\smrev{\label{ec-mm-eq: edge recursion}
0<M(a)\leq r\,\Xi_K\bigl(M(a)\bigr).}
\end{equation}

\smrev{We now choose $r$ to reach a contradiction. By \Cref{ec-mm-eq: superlinear edge}, there is $r_0>0$ such that
$\Xi_K(x)\leq x$ for $0<x\leq r_0$. Take $r<\min\{R/4,1,r_0\}$; this is possible because $\Xi_K$ (that is, $\gamma$,
$C_0$ and $K$), and hence $r_0$, do not depend on $r$. The Lipschitz property and $M(R)=0$ give $M(a)\leq
M(R)+(R-a)=r\leq r_0$, so $\Xi_K(M(a))\leq M(a)$. Then \Cref{ec-mm-eq: edge recursion} gives $M(a)\leq r\,M(a)<M(a)$, a
contradiction. Therefore $R=0$, that is, $M(a)=0$ for every $a\in(0,B]$, and continuity gives $M(0)=0$.}

\smrev{\emph{Conclusion.} For a symmetric distribution $\rho$ on $[-B,B]$, $m_2(\rho)\leq
B\int|x|\,\rho(dx)=2B\Lambda_0(\rho)$. Since $M_n(0)\downarrow M(0)=0$, there is a finite $n\geq1$, depending only on
$(\truDist,\ell)$, with $2BM_n(0)\leq\varepsilon_0$. Every controlled sequence then satisfies
$m_2(\rho_n)\leq2B\Lambda_0(\rho_n)\leq2BM_n(0)\leq\varepsilon_0$ at round $n$. Now consider any run of the loop with
costs $\costParam_t\geq\ell$ and tie-breaking rules. By Step~1, it is a controlled sequence, so
$m_2(\pi_n)\leq\varepsilon_0$, and Step~2, applied to the prior $\pi_n$ and the cost $\costParam_n\geq\ell$, gives
$\pi_{n+1}=\delta_0$. Let $T\triangleq n+1$. The prior $\delta_0$ is absorbing (\Cref{ec-mm-lemma: wellposed}(ii)), so
$\pi_t=\delta_0$ for every $t\geq T$.}
\Halmos
\end{proof}

\paragraph{\smrev{Slow improvement: proof of item~2.}}

\begin{proof}{Proof of \Cref{prop: recover diversity}, item 2.}
\smrev{Item~2 states that, if $\truDist$ also has bounded support, there are a sequence of costs
$(\costParam'_t)_{t\geq0}$ (the prime distinguishes this slow sequence from the fast sequence of item~1) that starts at
$\costParam'_0=\costParam_0$ and strictly decreases to $0$, and a finite $T$, such that $\pi_T=\delta_0$, and hence
$\VV(\outcome_t)=0$ for every $t\geq T-1$, whatever tie-breaking rule is used. The idea is to keep the costs above
$\ell\triangleq\costParam_0/2$, while still strictly decreasing, until the common collapse time of
\Cref{ec-mi-prop: collapse lower-bounded costs}, and from that round on to let them fall to $0$.}

\smrev{\emph{The collapse time.} Since $\truDist$ is symmetric and atomless (hence nondegenerate) with bounded support,
$H_{\truDist}=[-B,B]$ for some $B\in(0,\infty)$, and \Cref{ec-mm-eq: coercive} holds for every $\ell>0$. With
$\ell=\costParam_0/2$, let $T=T(\truDist,\ell)$ be the collapse time of \Cref{ec-mi-prop: collapse lower-bounded costs}.
It depends only on $\truDist$ and $\costParam_0$.}

\smrev{\emph{The cost sequence.} We use an auxiliary sequence $\check\costParam_t\triangleq\ell(1+2^{-t})$, $t\geq0$,
whose costs stay above $\ell$ at every round, so that \Cref{ec-mi-prop: collapse lower-bounded costs} applies to the
loop run with it. In the collapse argument below, the tie-breaking rules of the slow loop in rounds $0,\ldots,T-1$
are reused as rules for the loop run with this auxiliary sequence. This works even though a rule may depend on future
costs, because a tie-breaking rule of round $t$ only needs to be optimal for the prior and the cost of round $t$. The
auxiliary sequence starts at $\check\costParam_0=2\ell=\costParam_0$ and is strictly decreasing. The slow sequence
equals it up to round $T-1$ and then halves at each round:}
\[
\smrev{\costParam'_t\triangleq\begin{cases}\check\costParam_t,&0\leq t<T,\\ \check\costParam_{T-1}\,2^{-(t-T+1)},&t\geq
T.\end{cases}}
\]
\smrev{This sequence is positive and tends to $0$. It is strictly decreasing: $\costParam'_t<\costParam'_{t-1}$ for
$1\leq t<T$ because $(\check\costParam_t)$ is; at $t=T$ because
$\costParam'_T=\check\costParam_{T-1}/2<\check\costParam_{T-1}=\costParam'_{T-1}$; and for $t>T$ because each step
halves the cost.}

\smrev{\emph{Collapse.} Fix any sequence of tie-breaking rules for the loop run with the costs $(\costParam'_t)$. In
rounds $0,\ldots,T-1$, these costs equal $\check\costParam_t\geq\ell$. So the same rules are tie-breaking rules for the
loop run with the costs $(\check\costParam_t)$ in these rounds, and the two loops generate the same priors
$\pi_0,\ldots,\pi_T$, by induction on $t$, since $\pi_{t+1}$ is determined by $\pi_t$ and by the cost and the rule of
round $t$. Extend these rules by any tie-breaking rules for the costs $(\check\costParam_t)$ in rounds $t\geq
T$; they exist by \Cref{ec-mi-prop: collapse lower-bounded costs}. The result is a run of the loop with costs bounded
below by $\ell$, so \Cref{ec-mi-prop: collapse lower-bounded costs} gives $\pi_T=\delta_0$. Because $\delta_0$ is
absorbing (\Cref{ec-mm-lemma: wellposed}(ii)), $\pi_t=\delta_0$ for every $t\geq T$ also in the loop run with the costs
$(\costParam'_t)$. Since $\pi_{t+1}$ is the distribution of $\outcome_t$, we get $\outcome_t=0$ almost surely, so
$\VV(\outcome_t)=0$, for every $t\geq T-1$. Since the tie-breaking rules were arbitrary and $T$ was fixed before them,
this proves item~2.}
\Halmos
\end{proof}

\paragraph{\smrev{Both sequences from the same population and the same initial cost.}} \smrev{If $\truDist$ is
symmetric, atomless and has bounded support, then $0<\truSD^2<\infty$, and both items apply for every
$\costParam_0\in(0,2\truSD^2)$: from the same initial cost, one sequence restores diversity and the other leads to
collapse.}

\smrev{For example, $\truDist=\mathrm{Unif}[-1,1]$ is symmetric and atomless, with $H_{\truDist}=[-1,1]$ and
$\truSD^2=1/3$. The initial cost $\costParam_0=1/3$ lies in $(0,2\truSD^2)=(0,2/3)$.}

\smrev{As a concrete check of the first round of the proof of item~1, take the common variance $v=1/3$, that is, $r=1$
in \Cref{ec-mm-eq: first round bound}. Then \Cref{ec-mm-eq: first round bound} bounds the population-averaged optimal
utility loss by}
\[
\smrev{\frac{\truSD^2}{2}+\frac{\costParam_0}{2}\ln2=\frac16+\frac16\ln2<\frac13=\truSD^2.}
\]
\smrev{The fidelity error of each type is at most its optimal utility loss, so \Cref{ec-mm-eq: first round gain} holds
with $d_0=\frac16(1-\ln2)>0$: every tie-breaking rule has a population-averaged fidelity error below $\truSD^2$ in the
first round.}

\smrev{The fast sequence is the one constructed in the proof of item~1. The slow sequence equals $\frac16(1+2^{-t})$ for
$t<T$, where $T$ is the collapse time of \Cref{ec-mi-prop: collapse lower-bounded costs}, and then halves at each
round.}

\paragraph{\smrev{Atomlessness cannot simply be dropped from item~1 if exact revelation is excluded.}} \smrev{Recall
from the Setting that exact revelation ($v=0$) is excluded. This remark shows that, with this restriction, the loop need
not even be defined for a population with atoms. Take the binary population $\truDist=\frac12(\delta_{-1}+\delta_1)$,
with $\truSD^2=1$, and the cost $\costParam_0=\frac12\in(0,2\truSD^2)$. We show that, for every cost
$\costParam\in(0,1]$ (which includes $\costParam_0$), neither type has an optimal signal at round $0$, so no
tie-breaking rule exists.}

\smrev{The plan is to compare, for each precision $u$, the remaining fidelity error $e(u)$, which is the fidelity error
of either type at precision $u$ (defined below), with the remaining information $\ln2-\mathcal{J}(u)$ that a more
precise signal could still convey, and to show that the latter is smaller. The bridge is the identity
$\ln2-\mathcal{J}(u)=\frac12\int_u^\infty e(t)\,dt$, which follows from $\mathcal{J}'=e/2$ (shown below). A decay bound
on $e$ then gives $\ln2-\mathcal{J}(u)<e(u)$. The loss of either type at precision $u$ is
$e(u)+\costParam\mathcal{J}(u)$, and its limit as $u\rightarrow\infty$ is $\costParam\ln2$. So the loss exceeds its
limit exactly when $\costParam\bigl(\ln2-\mathcal{J}(u)\bigr)<e(u)$, which the bound gives for $\costParam\leq1$. So the
loss at every finite precision is strictly above its limit. Hence no finite precision is optimal, and the infimum of the
loss is reached only in the limit $v\downarrow0$.}

\smrev{\emph{The fidelity error.} Recall that $u=1/v$ is the precision, with $u=0$ for the uninformative signal. We
rescale the signal as $\tilde{\query}\triangleq\sqrt{u}\,\query=\sqrt{u}\,\truth+Z$ for $u>0$. Under the prior
$\truDist$, the posterior weights of $\truth=1$ and $\truth=-1$ given $\tilde{\query}$ are proportional to
$\phi(\tilde{\query}-\sqrt{u})$ and $\phi(\tilde{\query}+\sqrt{u})$, that is, to $e^{\sqrt{u}\,\tilde{\query}}$ and
$e^{-\sqrt{u}\,\tilde{\query}}$, so the AI output is $\tanh(\sqrt{u}\,\tilde{\query})$. For the type $\truth=1$, let
$W_u\triangleq\sqrt{u}\,\tilde{\query}=u+\sqrt{u}Z$. The fidelity error of this type is}
\[
\smrev{e(u)\triangleq\EE{}{(1-\tanh W_u)^2}\quad(u>0),\qquad e(0)\triangleq1,}
\]
\smrev{where $e(0)=\truth^2=1$ is the fidelity error of the uninformative signal.}

\smrev{\emph{The error as an expected posterior variance.} Let $\operatorname{sech}t\triangleq1/\cosh t$, so that
$1-\tanh^2t=\operatorname{sech}^2t$. We show that}
\begin{equation}\smrev{\label{ec-mm-eq: binary sech}
e(u)=\EE{}{\operatorname{sech}^2W_u}\qquad(u>0).}
\end{equation}
\smrev{By symmetry, the type $\truth=-1$ has the same error, so $e(u)$ is also the population-averaged fidelity error,
that is, the mean squared error of the posterior mean. That error equals the expected posterior variance, which is
$\EE{}{1-\tanh^2(\sqrt{u}\,\tilde{\query})}$, the expectation being over the population and the noise. Given
$\truth=1$, the variable $\sqrt{u}\,\tilde{\query}$ is $W_u$; given $\truth=-1$, it has the distribution of $-W_u$.
Since $\operatorname{sech}^2$ is even, this gives $e(u)=\EE{}{1-\tanh^2W_u}=\EE{}{\operatorname{sech}^2W_u}$.}

\smrev{\emph{A decay bound for the error.} Since $W_u$ has the distribution $\mathcal{N}(u,u)$,
\Cref{ec-mm-eq: binary sech} gives}
\[
\smrev{e(u)=\frac{1}{\sqrt{2\pi u}}\int\operatorname{sech}^2t\,e^{-(t-u)^2/(2u)}\,dt =\frac{e^{-u/2}}{\sqrt{2\pi
u}}\int\operatorname{sech}^2t\,e^{t}e^{-t^2/(2u)}\,dt.}
\]
\smrev{Write $e^t=\cosh t+\sinh t$. The odd part $\sinh t$ integrates to zero against the even function
$\operatorname{sech}^2t\,e^{-t^2/(2u)}$, and $\operatorname{sech}^2t\cosh t=\operatorname{sech}t$, so}
\[
\smrev{e(u)=\frac{e^{-u/2}}{\sqrt{2\pi u}}\int\operatorname{sech}t\,e^{-t^2/(2u)}\,dt
=e^{-u/2}\,\EE{}{\operatorname{sech}(\sqrt{u}Z)}.}
\]
\smrev{We now derive the decay bound. Because $\operatorname{sech}t$ strictly decreases in $|t|$,
$\EE{}{\operatorname{sech}(\sqrt{u}Z)}$ strictly decreases in $u$, and it is at most $1=e(0)$. Hence, for every $u\geq0$
and $h>0$,}
\begin{equation}\smrev{\label{ec-mm-eq: binary decay}
e(u+h)<e^{-h/2}\,e(u).}
\end{equation}

\smrev{\emph{The communication cost.} We show that $\mathcal{J}(u)=u-\EE{}{\ln\cosh W_u}$ for $u>0$. The population
mutual information $\mathcal{J}(u)=I(\truth;\tilde{\query})$ is the expected logarithm of the ratio of the density of
$\tilde{\query}$ given the type to its mixture density. Given $\truth=1$, we
divide the numerator and the denominator of this ratio by $\phi(\tilde{\query}-\sqrt{u})$ and use
$\phi(\tilde{\query}+\sqrt{u})/\phi(\tilde{\query}-\sqrt{u})=e^{-2\sqrt{u}\,\tilde{\query}}=e^{-2W_u}$:}
\[
\smrev{\ln\frac{\phi(\tilde{\query}-\sqrt{u})}{\frac12\bigl[\phi(\tilde{\query}-\sqrt{u})
+\phi(\tilde{\query}+\sqrt{u})\bigr]}=\ln\frac{2}{1+e^{-2W_u}}=\ln\frac{e^{W_u}}{\cosh W_u}=W_u-\ln\cosh W_u.}
\]
\smrev{Since $\EE{}{W_u}=u$, the expected logarithm given $\truth=1$ is $u-\EE{}{\ln\cosh W_u}$. By symmetry, the same
holds given $\truth=-1$. Hence}
\[
\smrev{\mathcal{J}(u)=u-\EE{}{\ln\cosh W_u}\qquad(u>0).}
\]

\smrev{\emph{The derivative of the communication cost.} The derivative computation below also uses the identity
$\EE{}{\tanh W_u}=1-e(u)$. It follows by comparing the two expressions of $e(u)$: $\EE{}{(1-\tanh
W_u)^2}=\EE{}{1-\tanh^2W_u}$ gives $\EE{}{\tanh W_u}=\EE{}{\tanh^2W_u}=1-e(u)$.}

\smrev{We now show that $\mathcal{J}'(u)=\frac12e(u)$ for every $u>0$, in two operations; both are justified because
$\tanh$ and its derivative are bounded. First, $W_u=u+\sqrt{u}Z$ has derivative $1+Z/(2\sqrt{u})$ in $u$, and the
derivative of $\ln\cosh$ is $\tanh$. Differentiating $\mathcal{J}(u)=u-\EE{}{\ln\cosh W_u}$ under the expectation
gives}
\[
\smrev{\mathcal{J}'(u)=1-\EE{}{\tanh W_u}-\frac{1}{2\sqrt{u}}\EE{}{Z\tanh W_u}.}
\]
\smrev{Second, Gaussian integration by parts gives $\EE{}{Z\tanh W_u}=\sqrt{u}\,\EE{}{\operatorname{sech}^2W_u}$.
Hence}
\[
\smrev{\mathcal{J}'(u)=1-\EE{}{\tanh W_u}-\frac12\EE{}{\operatorname{sech}^2W_u}=\frac12\,e(u).}
\]
\smrev{The last step uses the identity $\EE{}{\tanh W_u}=1-e(u)$ and \Cref{ec-mm-eq: binary sech}.}

\smrev{\emph{The limit of the communication cost.} We show that $\mathcal{J}(u)\rightarrow\ln2$ as
$u\rightarrow\infty$. Indeed, $\mathcal{J}(u)=\ln2-H(\truth\mid\tilde{\query})$, where $\ln2$ is the (Shannon) entropy
of the binary type and $H(\truth\mid\tilde{\query})$ is its conditional (Shannon) entropy given the signal, not a
differential entropy. It is the expected binary entropy of the posterior, where
$H_b(p)\triangleq-p\ln p-(1-p)\ln(1-p)$. Given $\truth=1$, the posterior
probability of $\truth=1$ is $(1+\tanh W_u)/2$, so the conditional entropy is $\EE{}{H_b((1+\tanh W_u)/2)}$; the case
$\truth=-1$ is symmetric. As $u\rightarrow\infty$, $W_u=u+\sqrt{u}Z\rightarrow\infty$ almost surely, so
$H_b((1+\tanh W_u)/2)\rightarrow0$, and $H_b\leq\ln2$ lets dominated convergence conclude.}

\smrev{\emph{The remaining information.} It remains to integrate $\mathcal{J}'$ from $u$ to $\infty$. The cost
$\mathcal{J}$ is continuous on $[0,\infty)$ with $\mathcal{J}(0)=0$ (\Cref{ec-mm-lemma: cost}(i)), so integrating
$\mathcal{J}'$ over $(u,\infty)$ is valid also at $u=0$. Using \Cref{ec-mm-eq: binary decay}, for every finite
$u\geq0$,}
\begin{equation}\smrev{\label{ec-mm-eq: binary information gap}
\ln2-\mathcal{J}(u)=\frac12\int_u^\infty e(t)\,dt=\frac12\int_0^\infty e(u+h)\,dh <\frac12\,e(u)\int_0^\infty
e^{-h/2}\,dh=e(u).}
\end{equation}

\smrev{\emph{No optimal signal.} Let $\costParam\in(0,1]$, which includes $\costParam_0=\frac12$. For every finite
$u\geq0$, the loss of either type is $e(u)+\costParam\mathcal{J}(u)$, and \Cref{ec-mm-eq: binary information gap}
gives}
\[
\smrev{e(u)+\costParam\mathcal{J}(u)-\costParam\ln2=e(u)-\costParam\bigl(\ln2-\mathcal{J}(u)\bigr)>(1-\costParam)\,e(u)
\geq0.}
\]
\smrev{On the other hand, $e(u)+\costParam\mathcal{J}(u)\rightarrow\costParam\ln2$ as $u\rightarrow\infty$, because
$e(u)\leq e^{-u/2}$. So the infimum of the loss is $\costParam\ln2$, and, by the strict inequality above, no finite
precision $u$ attains it. It is approached only as $u\rightarrow\infty$, that is, as $v\downarrow0$, which is
excluded. Both types have positive mass, so no tie-breaking rule exists at round $0$, and the loop is not defined.}

\smrev{If instead exact revelation is allowed, that is, if users may choose any $\qErrorSD\in[0,\infty]$, the signal
domain of \Cref{sec:model}, then $v=0$ has zero fidelity error and the finite cost $\costParam\ln2$ for this
population, so it attains the infimum.}

\smrev{This example does not decide whether diversity can be recovered for atomic populations for which optimal
signals do exist.}
}

\subsection{Proofs for \Cref{sec:mitigation_hc}}

\begin{proof}{Proof of \Cref{theorem: lower-bounded variance}.}
\smrev{We fix a round $t$ of \Cref{alg:self-training-loop} and show that the variance of outputs at that round is at
least a constant $C_{\costHumanOnly}>0$ that is the same at every round: it depends only on $\costHumanOnly$ and on
the population, and not on $t$, on the AI's prior $\pi_t$ or on which optimal signal each user chooses when several
are optimal. Since the round is arbitrary, this gives $\inf_{t\ge0}\VV(\outcome_t)\ge C_{\costHumanOnly}>0$, which
is the theorem's claim.}

\smrev{The argument has four steps. Step~1 concerns each user's output, and Step~2 the distribution of outputs. Step~1
shows that the expected fidelity error of a user's output is at most $\costHumanOnly$. The reason is the option to
work alone: a user works with the AI only if her utility loss, which includes her fidelity error, is at most
$\costHumanOnly$, and otherwise she works alone, with zero fidelity error. This holds whichever optimal signal she
uses. Step~2 deduces that the output, averaged over both the users and their signal noise, has a finite second moment,
so that the variance of outputs at round $t$ is defined. Since the distribution of outputs is the AI's prior at round
$t+1$, Step~2 also shows, by induction on $t$, that the AI's prior has a finite mean at every round, so that the AI's
default output is defined at every round.}

\smrev{Steps~3 and~4 turn the bound of Step~1 into the variance floor. Step~3 gives a lower bound on each user's
expected squared distance from an arbitrary constant $c$, namely $(|\truth-c|-\sqrt{\costHumanOnly})_+^2$, where
$(x)_+\triangleq\max\{x,0\}$. This lower bound is zero for users within $\sqrt{\costHumanOnly}$ of $c$. Step~4
averages this bound over users and applies it with $c=\EE{}{\outcome_t}$, so that the averaged left side,
$\EE{}{(\outcome_t-c)^2}$, becomes the variance of outputs. The mean $\EE{}{\outcome_t}$ is not computed and could
differ from $\truMean$, so Step~4 then minimizes the averaged bound over all $c$. The averaged bound is convex in $c$
and symmetric about $\truMean$, so its minimum is at $\truMean$. This minimum is $C_{\costHumanOnly}$, and it is
positive because the population is normal and $\costHumanOnly$ is finite, so a positive mass of users lies farther
than $\sqrt{\costHumanOnly}$ from $\truMean$.}

\smrev{\emph{Notation.} Let $\pi_t$ be the AI's prior at round $t$, so that $\pi_0=\truDist$. Its mean
$m_t\triangleq\int x\,\pi_t(dx)$ is the AI's default output at round $t$, the output it returns for the
uninformative signal ($\aMean$ in \Cref{sec:model}, where the prior is $\priorAI$). We take $m_t$ to be finite; this
is the induction hypothesis that Step~2 verifies.}

\smrev{For a user $\truth$ and a noise level $\qErrorSD$, we write}
\begin{equation}
 \smrev{\label{eq: hc loss under prior t}
 \obj_{\pi_t}(\truth,\qErrorSD)\triangleq
 \underbrace{\costQual_{\pi_t}(\truth,\qErrorSD)}_{\text{fidelity error}}
 +\underbrace{\costParam\costEffort(\qErrorSD)}_{\text{communication cost}}}
\end{equation}
\smrev{for her expected utility loss in Problem~\eqref{eq: loss def} when the AI's prior is $\pi_t$, as in
\Cref{alg:self-training-loop}. Here $\costQual_{\pi_t}(\truth,\qErrorSD)$ is her expected fidelity error, averaged
over the signal noise, when the AI outputs $\truthAIAlg{t}(\query,\qErrorSD)$, its posterior mean under the prior
$\pi_t$. This posterior mean is the middle expression of \Cref{eq: truthAI def}, with $\priorAI$ replaced by $\pi_t$.
The weighted-average formula on the right of \Cref{eq: truthAI def} holds only for a normal prior, and $\pi_t$ need
not be normal.}

\smrev{For a finite noise level, this loss has the same form as the loss $\obj_\pi(\truth,\sqrt v)$ of
\Cref{ec-mi-eq: user loss}, with $\pi=\pi_t$ and $v=\qErrorSD^2$. That equation is stated for priors symmetric
about $0$, but its formula for finite $v$, the expected squared distance between the AI's posterior mean under $\pi$
and $\truth$ plus the communication cost, does not use the symmetry.}

\smrev{Two facts about this loss are used below. First, the communication cost $\costEffort(\qErrorSD)$ of
\Cref{sec:model} is a mutual information, so it is nonnegative. Second, for the uninformative signal,
$\costEffort(\infty)=0$ and the AI returns its default output $m_t$, so
$\costQual_{\pi_t}(\truth,\infty)=\obj_{\pi_t}(\truth,\infty)=(\truth-m_t)^2$.}

\smrev{Let $\opteffortAlg{t}$ be the optimal signal of user $\truth$ at round $t$, that is, a noise level that
minimizes $\obj_{\pi_t}(\truth,\cdot)$, chosen measurably in $\truth$ so that the output $\outcome_t$ of
\Cref{eq: hc output} below is a random variable. When several signals are optimal, $\opteffortAlg{t}$ may be any such
choice among them, because no step below uses which one is chosen. In particular, it need not give the same signal to
two users $\truMean+x$ and $\truMean-x$ who are symmetric about $\truMean$. This differs from
\Cref{ec-sec: proofs unique key}, where $\truMean=0$ and the choice is even in $\truth$. By
\Cref{eq:decision w/ human only} with the prior $\pi_t$, user $\truth$ works with the AI if and only if}
\begin{equation}
 \smrev{\label{eq: hc works with AI}
 \obj_{\pi_t}\bigl(\truth,\opteffortAlg{t}\bigr)\le\costHumanOnly.}
\end{equation}
\smrev{So, as in \Cref{eq:decision w/ human only}, a tie between working with the AI and working alone goes to the
AI.}

\smrev{Her output is therefore, as in \Cref{eq:decision w/ human only},}
\begin{equation}
 \smrev{\label{eq: hc output}
 \outcome_t=\begin{cases}
 \truthAIAlg{t}\bigl(\query,\opteffortAlg{t}\bigr)
 &\text{if \Cref{eq: hc works with AI} holds (she works with the AI)},\\
 \truth&\text{otherwise (she works alone)}.
 \end{cases}}
\end{equation}
\smrev{Here $\query$ is drawn from her signal. If $\opteffortAlg{t}=\infty$, the signal is uninformative and the AI
returns its default output $m_t$, whatever $\query$. Expectations of $\outcome_t$ are over both the user
$\truth\sim\truDist$ and her signal noise. The conditional expectation
$\EE{}{\,\cdot\mid\truth}$ averages over the signal noise only.}

\smrev{\emph{Step 1: each user's expected fidelity error is at most $\costHumanOnly$.} We show that, for every user
$\truth$,}
\begin{equation}
 \smrev{\label{eq: hc fidelity cap}
 \EE{}{(\outcome_t-\truth)^2\mid\truth}\le\costHumanOnly.}
\end{equation}
\smrev{We consider the user's two options in turn.}

\smrev{If user $\truth$ works alone, then $\outcome_t=\truth$ by \Cref{eq: hc output}, so the left side of
\Cref{eq: hc fidelity cap} is $0\le\costHumanOnly$.}

\smrev{If user $\truth$ works with the AI, then, by \Cref{eq: hc output}, her output is the AI's output for her
optimal signal, so the left side of \Cref{eq: hc fidelity cap} is her expected fidelity error
$\costQual_{\pi_t}(\truth,\opteffortAlg{t})$. The communication cost is nonnegative and
\Cref{eq: hc works with AI} holds, so}
\[
 \smrev{\EE{}{(\outcome_t-\truth)^2\mid\truth}
 =\costQual_{\pi_t}\bigl(\truth,\opteffortAlg{t}\bigr)
 \le\costQual_{\pi_t}\bigl(\truth,\opteffortAlg{t}\bigr)+\costParam\costEffort\bigl(\opteffortAlg{t}\bigr)
 =\obj_{\pi_t}\bigl(\truth,\opteffortAlg{t}\bigr)\le\costHumanOnly.}
\]
\smrev{This proves Step~1.}

\smrev{\emph{Step 2: the output has a finite second moment.} We show that}
\begin{equation}
 \smrev{\label{eq: hc second moment}
 \EE{}{(\outcome_t)^2}\le2\EE{\truDist}{\truth^2}+2\costHumanOnly<\infty,}
\end{equation}
\smrev{so that $\VV(\outcome_t)$ is defined. We also show that the AI's prior has a finite mean at every round.}

\smrev{Write $\outcome_t=\truth+(\outcome_t-\truth)$ and use $(x+y)^2\le2x^2+2y^2$, which gives
$(\outcome_t)^2\le2\truth^2+2(\outcome_t-\truth)^2$. Conditional on $\truth$, Step~1 bounds the
expectation of the second term by $2\costHumanOnly$. Averaging over $\truth\sim\truDist$ then gives
\Cref{eq: hc second moment}, where $\EE{\truDist}{\truth^2}=\truMean^2+\truSD^2$ is finite.}

\smrev{By \Cref{alg:self-training-loop}, the next prior $\pi_{t+1}$ is the distribution of $\outcome_t$. By
\Cref{eq: hc second moment}, it has a finite second moment, hence a finite mean $m_{t+1}$. The initial prior
$\pi_0=\truDist$ has the finite mean $\truMean$. So, by induction on $t$, the mean $m_t$ is finite at every round,
as assumed in the notation paragraph, and Steps~1 and~2 hold at every round. This proves Step~2.}

\smrev{\emph{Step 3: a lower bound on each user's expected squared distance from a constant.}
We show that, for every user $\truth$ and every constant $c\in\mathbb R$,}
\begin{equation}
 \smrev{\label{eq: hc pointwise floor}
 \EE{}{(\outcome_t-c)^2\mid\truth}\ge\bigl(|\truth-c|-\sqrt{\costHumanOnly}\bigr)_+^2.}
\end{equation}
\smrev{In words, the bound combines Step~1 with the triangle inequality. By Step~1, the root-mean-square distance
between a user's output and her preference is at most $\sqrt{\costHumanOnly}$. So, by the triangle inequality, the
root-mean-square distance between her output and $c$ is at least $|\truth-c|-\sqrt{\costHumanOnly}$ whenever this
number is positive. Squaring gives \Cref{eq: hc pointwise floor}, whose right side is zero when this number is not
positive.}

\smrev{Fix $\truth$ and $c$. Conditional on $\truth$, the difference $\truth-c$ is a constant, and it splits as
$\truth-c=(\truth-\outcome_t)+(\outcome_t-c)$. The first term has a finite conditional second moment, by Step~1.
So does the second, because it equals the constant $\truth-c$ minus the first term. The triangle inequality for the
conditional $L^2$ norm $X\mapsto\EE{}{X^2\mid\truth}^{1/2}$ gives}
\[
 \smrev{|\truth-c|
 \le\EE{}{(\truth-\outcome_t)^2\mid\truth}^{1/2}+\EE{}{(\outcome_t-c)^2\mid\truth}^{1/2}
 \le\sqrt{\costHumanOnly}+\EE{}{(\outcome_t-c)^2\mid\truth}^{1/2},}
\]
\smrev{where the second inequality is Step~1.}

\smrev{Hence $\EE{}{(\outcome_t-c)^2\mid\truth}^{1/2}\ge|\truth-c|-\sqrt{\costHumanOnly}$. The left side is also
nonnegative, so it is at least $(|\truth-c|-\sqrt{\costHumanOnly})_+$. Squaring both sides, which are nonnegative,
gives \Cref{eq: hc pointwise floor}. This proves Step~3.}

\smrev{\emph{Step 4: averaging over users.} We show that $\VV(\outcome_t)\ge C_{\costHumanOnly}>0$, where}
\begin{equation}
 \smrev{\label{eq: hc floor function}
 \mathcal V(c)\triangleq\EE{\truDist}{\bigl(|\truth-c|-\sqrt{\costHumanOnly}\bigr)_+^2}
 \quad\text{for }c\in\mathbb R,
 \qquad
 C_{\costHumanOnly}\triangleq\mathcal V(\truMean).}
\end{equation}
\smrev{The function $\mathcal V$ depends only on $\costHumanOnly$ and on the population, so $C_{\costHumanOnly}$
depends neither on $t$, nor on $\pi_t$, nor on the users' signals.}

\smrev{Averaging \Cref{eq: hc pointwise floor} over $\truth\sim\truDist$ gives
$\EE{}{(\outcome_t-c)^2}\ge\mathcal V(c)$ for every constant $c$. We take $c=\EE{}{\outcome_t}$, which is finite
by Step~2. Then the left side is the variance of outputs:}
\begin{equation}
 \smrev{\label{eq: hc variance floor}
 \VV(\outcome_t)=\EE{}{\bigl(\outcome_t-\EE{}{\outcome_t}\bigr)^2}\ge\mathcal V\bigl(\EE{}{\outcome_t}\bigr).}
\end{equation}
\smrev{It remains to show that $\mathcal V$ attains its minimum at $\truMean$, and that this minimum is positive. This
is needed because the mean $\EE{}{\outcome_t}$ in \Cref{eq: hc variance floor} could a priori differ from $\truMean$.
Indeed, the choice among optimal signals need not give the same signal to users who are symmetric about $\truMean$, so
we do not assume that the outputs are symmetric about $\truMean$. We use three properties of $\mathcal V$: it is
finite, convex, and symmetric about $\truMean$.}

\smrev{First, $\mathcal V(c)$ is finite for every $c$. Since $\sqrt{\costHumanOnly}\ge0$, the integrand in
\Cref{eq: hc floor function} is at most $(\truth-c)^2$, whose expectation under the normal population is finite.}

\smrev{Second, $\mathcal V$ is convex. For each $\truth$, the quantity $(|\truth-c|-\sqrt{\costHumanOnly})_+$ is the
distance from $c$ to the interval $[\truth-\sqrt{\costHumanOnly},\truth+\sqrt{\costHumanOnly}]$. The distance to an
interval is a nonnegative convex function of $c$. Its square is also convex, because $x\mapsto x^2$ is convex and
increasing on $[0,\infty)$. So the integrand is convex in $c$ for each $\truth$, and an average of convex functions
is convex.}

\smrev{Third, $\mathcal V$ is symmetric about $\truMean$: $\mathcal V(\truMean+x)=\mathcal V(\truMean-x)$ for every
$x\in\mathbb R$. Indeed, the population is normal with mean $\truMean$, so $\truth-\truMean$ and $\truMean-\truth$
have the same distribution $\mathcal N(0,\truSD^2)$. Hence}
\[
 \smrev{\mathcal V(\truMean+x)
 =\EE{\truDist}{\bigl(|(\truth-\truMean)-x|-\sqrt{\costHumanOnly}\bigr)_+^2}
 =\EE{\truDist}{\bigl(|(\truMean-\truth)-x|-\sqrt{\costHumanOnly}\bigr)_+^2}
 =\mathcal V(\truMean-x).}
\]
\smrev{The middle equality replaces $\truth-\truMean$ by $\truMean-\truth$, which has the same distribution, and the
last uses $|(\truMean-\truth)-x|=|\truth-(\truMean-x)|$.}

\smrev{Convexity and symmetry together imply that $\mathcal V$ is smallest at $\truMean$. For every $x$,}
\[
 \smrev{\mathcal V(\truMean)\le\tfrac12\mathcal V(\truMean+x)+\tfrac12\mathcal V(\truMean-x)=\mathcal V(\truMean+x).}
\]
\smrev{The inequality is convexity at $\truMean$, the midpoint of $\truMean-x$ and $\truMean+x$, and the equality is
symmetry. Taking $x=c-\truMean$ gives}
\begin{equation}
 \smrev{\label{eq: hc floor minimum}
 \mathcal V(c)\ge\mathcal V(\truMean)=C_{\costHumanOnly}\qquad\text{for every }c\in\mathbb R.}
\end{equation}

\smrev{We now show that $C_{\costHumanOnly}>0$. By \Cref{eq: hc floor function},
$C_{\costHumanOnly}=\EE{\truDist}{(|\truth-\truMean|-\sqrt{\costHumanOnly})_+^2}$. The integrand is nonnegative,
and it is positive on the event $\{|\truth-\truMean|>\sqrt{\costHumanOnly}\}$. Since $\costHumanOnly$ is finite and
the population is normal, this event has positive probability. Hence $C_{\costHumanOnly}>0$.}

\smrev{Combining \Cref{eq: hc variance floor} with \Cref{eq: hc floor minimum} at $c=\EE{}{\outcome_t}$ gives}
\[
 \smrev{\VV(\outcome_t)\ge\mathcal V\bigl(\EE{}{\outcome_t}\bigr)\ge\mathcal V(\truMean)=C_{\costHumanOnly}>0.}
\]
\smrev{This proves Step~4.}

\smrev{\emph{Conclusion.} The round $t$ was arbitrary, and Step~2 shows that $m_t$ is finite at every round, so
Steps~1--4 hold at every round $t$, with the same $C_{\costHumanOnly}$. Hence
$\inf_{t\ge0}\VV(\outcome_t)\ge C_{\costHumanOnly}>0$. This proves the theorem.}
\Halmos
\end{proof}

\begin{proof}{Proof of \Cref{corollary: lower-bounded variance}.}
\smrev{We follow the four steps of the proof of \Cref{theorem: lower-bounded variance}. The common manual cost
$\costHumanOnly$ becomes the manual cost $\costHumanOnly(\truth)$ of each user. Unlike the common cost
$\costHumanOnly>0$ of \Cref{sec:model}, the corollary does not require $\costHumanOnly(\truth)\ge0$. A negative
manual cost $\costHumanOnly(\truth)<0$ means that user $\truth$ gains from doing the task herself. Step~1 shows that
such a user works alone, so that her output has zero fidelity error.}

\smrev{Replacing $\costHumanOnly$ by $\costHumanOnly(\truth)$ is not enough in every step. Steps~1, 2 and~4
change, and Step~3 carries over with a threshold that now varies across users:}
\begin{itemize}
\item \smrev{\emph{Step 1.} Because $\costHumanOnly(\truth)$ may be negative, the bound on a user's expected
fidelity error becomes $\max\{\costHumanOnly(\truth),0\}$.}
\item \smrev{\emph{Step 2.} The bound of Step~1 need not be integrable over the population, because
$\costHumanOnly(\cdot)$ is only assumed to be continuous. So the finite second moment uses a second bound, which
does not involve the manual cost: a user's expected fidelity error is at most $(\truth-m_t)^2$, the utility loss of
the uninformative signal. Indeed, this error is zero if she works alone. If she works with the AI, it is at most her
utility loss, because the communication cost is nonnegative, and her optimal signal does at least as well as the
uninformative signal. This bound is integrable
because the population is normal and $m_t$ is finite.}
\item \smrev{\emph{Step 3.} The argument is unchanged, with $\sqrt{\costHumanOnly}$ replaced by the square root
$r(\truth)$ of the bound of Step~1, defined in \Cref{eq: hc type radius} below.}
\item \smrev{\emph{Step 4.} The function $\mathcal V$ of \Cref{eq: hc floor function} is redefined with $r(\truth)$
in place of $\sqrt{\costHumanOnly}$. It is still finite and convex, by the same arguments: its integrand is still the
squared distance from $c$ to an interval. Two places change: (a) the symmetry of $\mathcal V$ about $\truMean$ now
also uses the symmetry of $\costHumanOnly(\cdot)$; (b) the positivity of its minimum uses the corollary's assumption
on the mass of users who would rather work alone.}
\end{itemize}

\smrev{\emph{Notation.} Throughout, $\pi_t$, $m_t$, $\obj_{\pi_t}$ (\Cref{eq: hc loss under prior t}) and
$\opteffortAlg{t}$ are as defined before Step~1 of the proof of \Cref{theorem: lower-bounded variance}; in
particular, the optimal signal is chosen measurably in $\truth$.}
\smrev{Given the AI's prior $\pi_t$, the only change in the users' decisions is in who works with the AI.
By \Cref{eq:decision w/ human only} with $\costHumanOnly$ replaced by $\costHumanOnly(\truth)$, user $\truth$ works
with the AI if and only if}
\begin{equation}
 \smrev{\label{eq: hc type works with AI}
 \obj_{\pi_t}\bigl(\truth,\opteffortAlg{t}\bigr)\le\costHumanOnly(\truth).}
\end{equation}
\smrev{Her output $\outcome_t$ is as in \Cref{eq: hc output}, with \Cref{eq: hc works with AI} replaced by
\Cref{eq: hc type works with AI}: she works with the AI if \Cref{eq: hc type works with AI} holds, and alone
otherwise.}

\smrev{In Step~1 of the proof of \Cref{theorem: lower-bounded variance}, $\sqrt{\costHumanOnly}$ bounds the
root-mean-square distance between a user's output and her preference. Here this bound becomes}
\begin{equation}
 \smrev{\label{eq: hc type radius}
 r(\truth)\triangleq\sqrt{\max\{\costHumanOnly(\truth),0\}},}
\end{equation}
\smrev{where the maximum with zero covers users with a negative manual cost. The function $r$ is continuous, because
$\costHumanOnly(\cdot)$ is continuous; Step~4 uses this to ensure that the integrand of the function $\mathcal V$
defined there is measurable.}

\smrev{\emph{Step 1: each user's expected fidelity error is at most $r(\truth)^2$.} We show that, for every user
$\truth$,}
\begin{equation}
 \smrev{\label{eq: hc type fidelity cap}
 \EE{}{(\outcome_t-\truth)^2\mid\truth}\le r(\truth)^2.}
\end{equation}
\smrev{We treat two cases, depending on the sign of $\costHumanOnly(\truth)$.}

\smrev{\emph{Case $\costHumanOnly(\truth)\ge0$.} Then $r(\truth)^2=\costHumanOnly(\truth)$. The argument of Step~1 of
the proof of \Cref{theorem: lower-bounded variance} applies, with $\costHumanOnly$ replaced by
$\costHumanOnly(\truth)$ and \Cref{eq: hc works with AI} by \Cref{eq: hc type works with AI}. If the user works
alone, her fidelity error is zero. If she works with the AI, her fidelity error is at most her utility loss, which is
at most $\costHumanOnly(\truth)$ by \Cref{eq: hc type works with AI}. So
$\EE{}{(\outcome_t-\truth)^2\mid\truth}\le\costHumanOnly(\truth)=r(\truth)^2$.}

\smrev{\emph{Case $\costHumanOnly(\truth)<0$.} Then $r(\truth)=0$, and we show that the user works alone. Her
utility loss from working with the AI, $\obj_{\pi_t}(\truth,\opteffortAlg{t})$, is nonnegative, because both the
fidelity error and the communication cost are. So it exceeds $\costHumanOnly(\truth)<0$, and
\Cref{eq: hc type works with AI} fails. Hence $\outcome_t=\truth$, and
$\EE{}{(\outcome_t-\truth)^2\mid\truth}=0=r(\truth)^2$.}

\smrev{This proves Step~1 for the corollary.}

\smrev{\emph{Step 2: the output has a finite second moment.} As in Step~2 of the proof of
\Cref{theorem: lower-bounded variance}, we show that $\EE{}{(\outcome_t)^2}<\infty$, so that $\VV(\outcome_t)$ is
defined, and that the AI's prior has a finite mean at every round.}

\smrev{The bound
$r(\truth)^2$ of Step~1 need not be integrable, because $\costHumanOnly(\cdot)$ is only assumed to be continuous, so
it cannot be averaged as in Step~2 of the proof of \Cref{theorem: lower-bounded variance}. As there, we argue by
induction on $t$ and take the mean $m_t$ of $\pi_t$ to be finite. We then use a second bound, which does not
involve the manual cost: for every user $\truth$,}
\begin{equation}
 \smrev{\label{eq: hc default cap}
 \EE{}{(\outcome_t-\truth)^2\mid\truth}\le(\truth-m_t)^2.}
\end{equation}

\smrev{To prove \Cref{eq: hc default cap}, consider the user's two options. If she works alone, the left side is
zero. If she works with the AI, then}
\[
 \smrev{\EE{}{(\outcome_t-\truth)^2\mid\truth}
 =\costQual_{\pi_t}\bigl(\truth,\opteffortAlg{t}\bigr)
 \le\obj_{\pi_t}\bigl(\truth,\opteffortAlg{t}\bigr)
 \le\obj_{\pi_t}(\truth,\infty)
 =(\truth-m_t)^2.}
\]
\smrev{The first inequality holds because her fidelity error is at most her utility loss, the communication cost
being nonnegative. The second holds because $\opteffortAlg{t}$ is optimal, so its utility loss is at most that of the
uninformative signal. The last equality is the loss of the uninformative signal, given after
\Cref{eq: hc loss under prior t}.}

\smrev{As in Step~2 of the proof of \Cref{theorem: lower-bounded variance},
$(\outcome_t)^2\le2\truth^2+2(\outcome_t-\truth)^2$. Taking expectations and using \Cref{eq: hc default cap}
gives}
\begin{equation}
 \smrev{\label{eq: hc type second moment}
 \EE{}{(\outcome_t)^2}\le2\EE{\truDist}{\truth^2}+2\EE{\truDist}{(\truth-m_t)^2}
 =2\bigl(\truMean^2+\truSD^2\bigr)+2\bigl(\truSD^2+(\truMean-m_t)^2\bigr)<\infty,}
\end{equation}
\smrev{because $m_t$ is finite.}

\smrev{This bound replaces \Cref{eq: hc second moment} in the induction on $t$ of Step~2 of the proof of
\Cref{theorem: lower-bounded variance}. The base case is the same: $\pi_0=\truDist$ has the finite mean $\truMean$.
For the inductive step, if $m_t$ is finite, \Cref{eq: hc type second moment} shows that $\outcome_t$ has a finite
second moment, so its distribution $\pi_{t+1}$ has a finite mean $m_{t+1}$. Hence $m_t$ is finite and Steps~1 and~2
hold at every round. This proves Step~2 for the corollary.}

\smrev{\emph{Step 3: a lower bound on each user's expected squared distance from a constant.} We show that, for
every user $\truth$ and every constant $c\in\mathbb R$,}
\begin{equation}
 \smrev{\label{eq: hc type pointwise floor}
 \EE{}{(\outcome_t-c)^2\mid\truth}\ge\bigl(|\truth-c|-r(\truth)\bigr)_+^2.}
\end{equation}
\smrev{The argument of Step~3 of the proof of \Cref{theorem: lower-bounded variance} applies, with the threshold
$\sqrt{\costHumanOnly}$ replaced by $r(\truth)$, which varies across users. The bound of Step~1 that it uses is now
\Cref{eq: hc type fidelity cap} instead of \Cref{eq: hc fidelity cap}. The only change is that the triangle
inequality now gives $|\truth-c|\le r(\truth)+\EE{}{(\outcome_t-c)^2\mid\truth}^{1/2}$. This proves Step~3 for the
corollary.}

\smrev{\emph{Step 4: averaging over users.} We show that $\VV(\outcome_t)\ge C_{\costHumanOnly}>0$, where the
function $\mathcal V$ and the constant $C_{\costHumanOnly}$ of \Cref{eq: hc floor function} become, now depending on
the whole function $\costHumanOnly(\cdot)$ rather than on a single cost,}
\begin{equation}
 \smrev{\label{eq: hc type floor function}
 \mathcal V(c)\triangleq\EE{\truDist}{\bigl(|\truth-c|-r(\truth)\bigr)_+^2}
 \quad\text{for }c\in\mathbb R,
 \qquad
 C_{\costHumanOnly}\triangleq\mathcal V(\truMean).}
\end{equation}

\smrev{Averaging \Cref{eq: hc type pointwise floor} over $\truth\sim\truDist$ and taking $c=\EE{}{\outcome_t}$, which
is finite by Step~2, gives $\VV(\outcome_t)\ge\mathcal V(\EE{}{\outcome_t})$, as in \Cref{eq: hc variance floor}.
It remains to check the three properties of $\mathcal V$ used in Step~4 of the proof of
\Cref{theorem: lower-bounded variance}, namely that it is finite, convex, and symmetric about $\truMean$, and then
that $C_{\costHumanOnly}>0$.}

\smrev{First, $\mathcal V(c)$ is well defined and finite for every $c$. The integrand is a measurable function of
$\truth$, because $r$ is continuous. It is at most $(\truth-c)^2$, because $r(\truth)\ge0$, so its expectation is
finite. This uses no integrability of $\costHumanOnly(\cdot)$.}

\smrev{Second, $\mathcal V$ is convex. For each $\truth$, the integrand, as a function of $c$, is the squared
distance from $c$ to the interval $[\truth-r(\truth),\truth+r(\truth)]$. This function is convex in $c$ for the
reason given in Step~4 of the proof of \Cref{theorem: lower-bounded variance}. An average of convex functions is
convex.}

\smrev{Third, $\mathcal V$ is symmetric about $\truMean$, and this is where the symmetry of the manual cost is used.
By assumption, $\costHumanOnly(\truMean+y)=\costHumanOnly(\truMean-y)$ for every $y$, so
$r(\truMean+y)=r(\truMean-y)$. Write $\truth=\truMean+y$, where $y=\truth-\truMean\sim\mathcal N(0,\truSD^2)$ has
the same distribution as $-y$. Then, for every $x\in\mathbb R$,}
\[
 \smrev{\begin{aligned}
 \mathcal V(\truMean+x)
 &=\EE{}{\bigl(|y-x|-r(\truMean+y)\bigr)_+^2}
 =\EE{}{\bigl(|-y-x|-r(\truMean-y)\bigr)_+^2}\\
 &=\EE{}{\bigl(|y+x|-r(\truMean+y)\bigr)_+^2}
 =\mathcal V(\truMean-x).
 \end{aligned}}
\]
\smrev{The first equality is the definition of $\mathcal V$ with $\truth=\truMean+y$. The second replaces $y$ by
$-y$, which has the same distribution. The third uses $|-y-x|=|y+x|$ and $r(\truMean-y)=r(\truMean+y)$. The last is
the definition of $\mathcal V$ at $\truMean-x$, because $|y-(-x)|=|y+x|$.}

\smrev{With convexity and symmetry, the midpoint argument that precedes \Cref{eq: hc floor minimum} applies
unchanged, so \Cref{eq: hc floor minimum} holds for the new $\mathcal V$: $\mathcal V(c)\ge\mathcal V(\truMean)=
C_{\costHumanOnly}$ for every $c$. Hence
$\VV(\outcome_t)\ge\mathcal V(\EE{}{\outcome_t})\ge C_{\costHumanOnly}$.}

\smrev{We now show that $C_{\costHumanOnly}>0$, which is where the corollary's assumption
$\PP_{\truDist}((\truth-\truMean)^2>\costHumanOnly(\truth))>0$ is used. By \Cref{eq: hc type floor function},
$C_{\costHumanOnly}=\EE{\truDist}{(|\truth-\truMean|-r(\truth))_+^2}$, whose integrand is nonnegative. So it
suffices to show that the integrand is positive on the event}
\[
 \smrev{\bigl\{\truth\ne\truMean:\ (\truth-\truMean)^2>\costHumanOnly(\truth)\bigr\}}
\]
\smrev{and that this event has positive probability. It is the event of the assumption with the single point
$\truth=\truMean$ removed. The population is normal, so this point has probability zero, and the event still has
positive probability by assumption. The point $\truMean$ is removed because the integrand is zero there. It belongs to
the event of the assumption when $\costHumanOnly(\truMean)<0$.}

\smrev{Let $\truth\ne\truMean$ satisfy $(\truth-\truMean)^2>\costHumanOnly(\truth)$. Since $\truth\ne\truMean$, also
$(\truth-\truMean)^2>0$. Together, these two inequalities give}
\[
 \smrev{(\truth-\truMean)^2>\max\{\costHumanOnly(\truth),0\}=r(\truth)^2.}
\]
\smrev{Taking square roots of these nonnegative numbers gives $|\truth-\truMean|>r(\truth)$, so the integrand
$(|\truth-\truMean|-r(\truth))_+^2$ is positive. Hence $C_{\costHumanOnly}>0$, and
$\VV(\outcome_t)\ge C_{\costHumanOnly}>0$. This proves Step~4 for the corollary.}

\smrev{\emph{Conclusion.} At every round $t$, $\VV(\outcome_t)\ge C_{\costHumanOnly}>0$. The constant
$C_{\costHumanOnly}$ depends only on $\costHumanOnly(\cdot)$ and on the population, and, as in the proof of
\Cref{theorem: lower-bounded variance}, not on which optimal signal each user chooses. Hence
$\inf_{t\ge0}\VV(\outcome_t)\ge C_{\costHumanOnly}>0$, which is the conclusion of
\Cref{theorem: lower-bounded variance}. This proves the corollary.}
\Halmos
\end{proof}

\section{Robustness Tests}
\label{sec: robust test}

In what follows, we further test the robustness of our results in more complex scenarios. Specifically, we examine \rev{\smrev{three}} additional cases. First, we explore the situation where the decision to use the AI is made ex-post rather than ex-ante. Second, we investigate scenarios where the distribution of users' preferences is not a normal distribution. \rev{\smrev{Third, we let the cost of human-AI interactions depend on a user's type.}}

\subsection{Ex-post decision of accepting the AI output}
In the original model presented in \Cref{sec:model}, we focus on the situation where users make an ex-ante decision about whether to use the AI to assist their \smrev{work}\rev{\smrev{,}} \smrev{based} on the expected utility loss. Note that in our model, every AI output is functional; as such, once a user has decided to use the AI, the user will be able to use the output, but she may still experience an ex-post fidelity loss. We now introduce and simulate a setting in which users only use the AI's output if it implies a moderate fidelity error.

Suppose that after observing the AI output $\truthAI(\query,\jgsim{\qErrorSD})$ and the realized fidelity error, $(\truth - \truthAI(\query,\jgsim{\qErrorSD}))^2$, a user decides to accept it if the realized fidelity error is less than the fixed utility cost $\costHumanOnly$. Otherwise, the user will ignore the output and do the work manually.  The output $\outcomeRetro$ chosen by a user $\truth$ is: 

\begin{equation*}
    \outcomeRetro \triangleq \begin{cases} \truthAI(\query,\jgsim{\qErrorSD}) & \text{if } (\truth - \truthAI(\query,\jgsim{\qErrorSD}))^2 \le \costHumanOnly\\ 
        \truth & \text{otherwise}
           \end{cases}.
\end{equation*}
Compared to our base model, the AI output is \rev{\smrev{replaced by the user's own work whenever it lands too far from her preference}}. In addition, since the user decides $\qErrorSD$ prior to deciding whether to accept the AI output, she must evaluate the expected fidelity error by considering the possibility of using the AI output:
\begin{equation*}
    \costQualRetro(\truth, \qErrorSD) \triangleq \EE{\pi(\cdot|\truth)}{(\outcomeRetro-\truth)^2}\jgsim{,}
\end{equation*}
\jgsim{which vanishes whenever she rejects the output, since then $\outcomeRetro=\truth$. She must also account for the \emph{expected manual cost}, the cost $\costHumanOnly$ of doing the work herself whenever she rejects the AI output:
\begin{equation*}
    \costManual(\truth, \qErrorSD) \triangleq \costHumanOnly\,\PP\big((\truth-\truthAI(\query,\qErrorSD))^2>\costHumanOnly\big).
\end{equation*}}
The utility loss and the optimal signal $\qErrorSDRetro^\star(\truth)$  are given by
\begin{equation*}
    \objRetro(\truth, \qErrorSD) \triangleq \costQualRetro(\truth, \qErrorSD) \jgsim{+ \costManual(\truth, \qErrorSD)} + \costParam \costEffort(\qErrorSD)\quad \text{and}\quad 
    \qErrorSDRetro^{\star}(\truth) \triangleq \argmin_{\qErrorSD \ge 0} \objRetro(\truth, \qErrorSD).
\end{equation*} 
\jgsim{Equivalently, $\costQualRetro+\costManual=\EE{\pi(\cdot|\truth)}{\min\{(\truth-\truthAI(\query,\qErrorSD))^2,\costHumanOnly\}}$: the user bears the fidelity error of an accepted AI output and the manual cost of a rejected one.}
\jgsim{The AI is then retrained on the outputs $\outcomeRetro$ that are actually produced (the AI output when it is accepted, the user's own work when it is rejected), so the next prior is the distribution of $\outcomeRetro$ and we report the variance of outputs $\VV(\outcomeRetro_t)$ in \Cref{fig: iter converge postGamma}.}
\begin{figure}[htbp!]
        \begin{center}
        \begin{subfigure}{0.3\textwidth}
            \includegraphics[width=\linewidth]{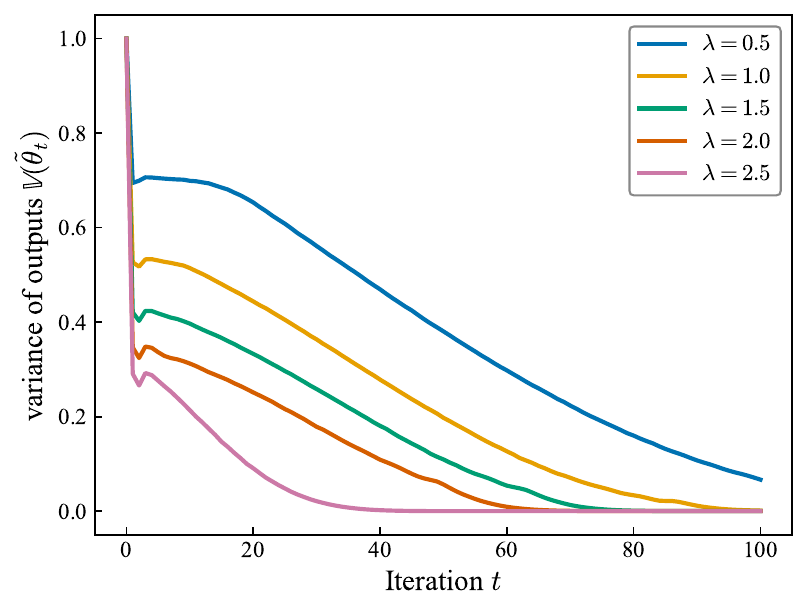}
            \caption{}
            \label{fig: iter converge postGamma-a}
        \end{subfigure}
        \begin{subfigure}{0.3\textwidth}
            \includegraphics[width=\linewidth]{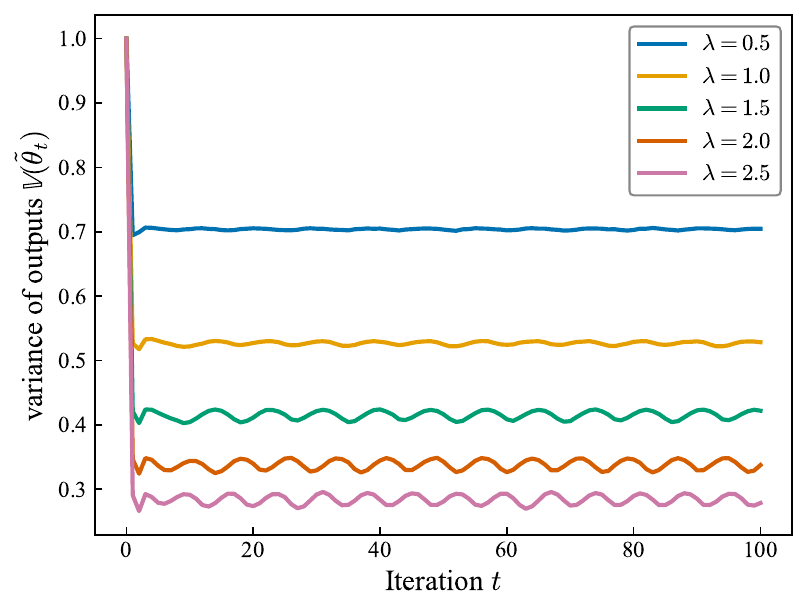}
            \caption{}
            \label{fig: iter converge postGamma-b}
        \end{subfigure}
        \begin{subfigure}{0.3\textwidth}
            \includegraphics[width=\linewidth]{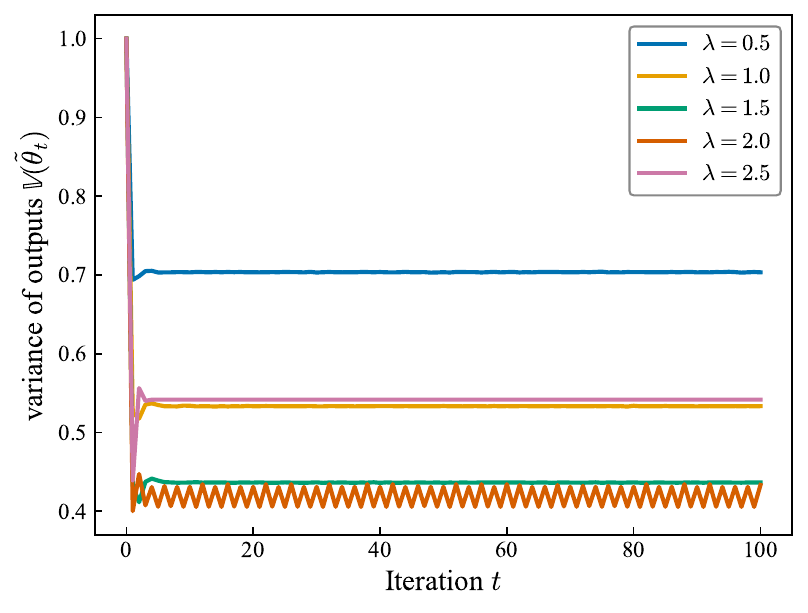}
            \caption{}
            \label{fig: iter converge postGamma-c}
        \end{subfigure}
    \end{center}
    \caption[The iterative change of the variance of outputs with an ex-post decision of accepting the AI output]{The iterative change of the variance of \jgsim{outputs $\VV(\outcomeRetro_t)$} with an ex-post decision of accepting the AI output. We use $\truMean=0$, $\truSD=1$. (a) $\costHumanOnly=\infty$; (b) $\costHumanOnly=10$; (c) $\costHumanOnly=2$.}
    \label{fig: iter converge postGamma}
\end{figure}
Despite being less \smrev{tractable} \rev{\smrev{(}}\smrev{we} would need to analyze a truncated AI \smrev{output}\rev{\smrev{)}}\smrev{,} \Cref{fig: iter converge postGamma} confirms and extends our finding in the base model to this setting. It showcases the same death spiral behaviors as our main model.
\smrev{In} \Cref{fig: iter converge postGamma-b}, the oscillations are much less pronounced than in \Cref{fig: iter converge base-b}. This is because, with ex-post decisions, the users tend to abandon the AI output earlier, rather than continuously accepting it until the expected fidelity error has significantly accumulated. As a result, the changes in the variance of outputs are less dramatic over time.

\subsection{Other population distribution of users' preferences}
To further test the robustness of our results, we numerically implement different population distributions of users' preferences. Specifically, we consider three additional types of distributions: uniform, a distribution with two symmetric peaks, and a distribution with two asymmetric peaks. The uniform distribution \jgsim{(uniform on $[-4,4]$)} represents an extreme case where every preference has the same density in the population, meaning that there is no majority preference. A distribution with two symmetric peaks features two large groups of users whose preferences are on opposite sides \rev{\smrev{of zero}} and have the same density. In contrast, a distribution with two asymmetric peaks also has two large groups of users with preferences on opposite sides, but the preferences in one of the groups are more concentrated (more homogeneous) while the other group's preferences are more diverse. The instances of the last two distribution types are illustrated in \Cref{fig: extra dist robust}.
\begin{figure}[htbp!]
        \begin{center}
        \begin{subfigure}{0.35\textwidth}
            \includegraphics[width=\linewidth]{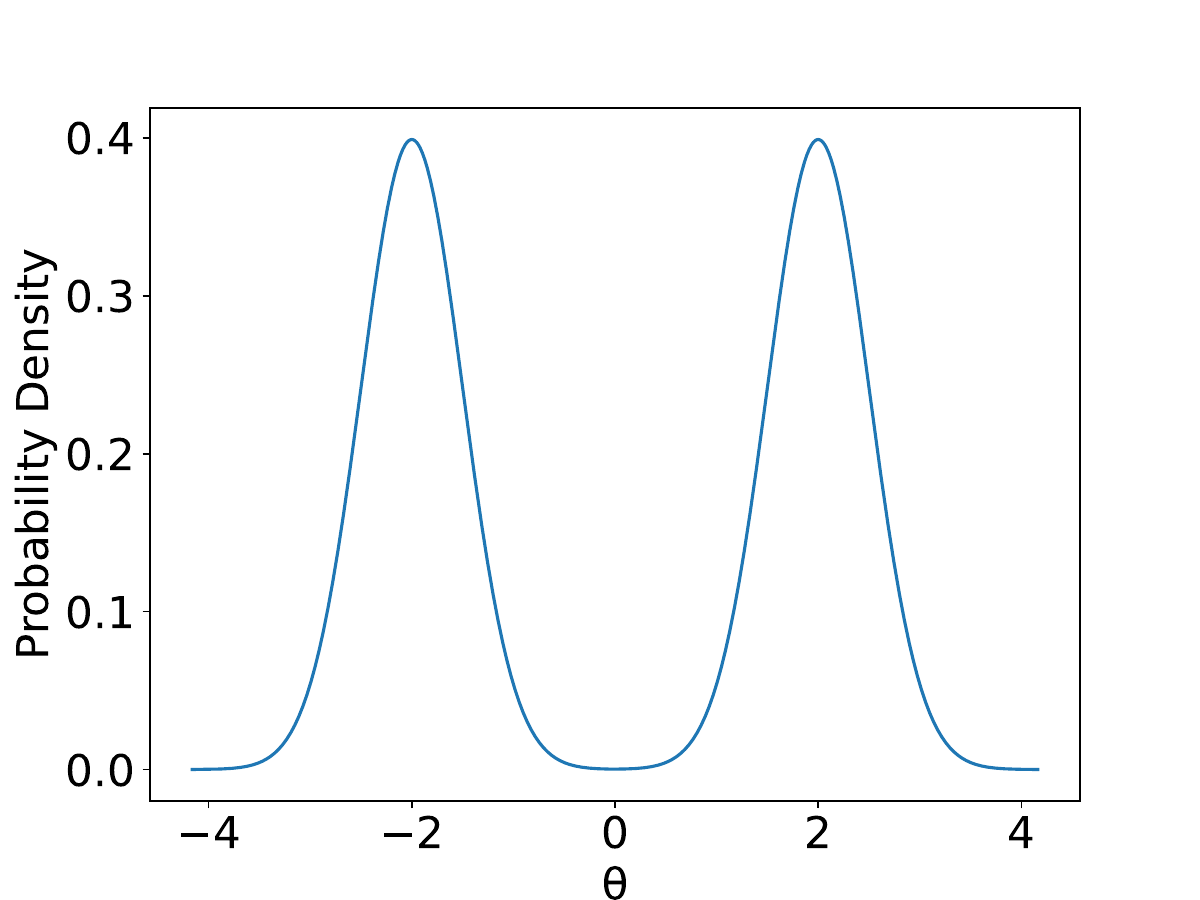}
            \caption{}
        \end{subfigure}
        \begin{subfigure}{0.35\textwidth}
            \includegraphics[width=\linewidth]{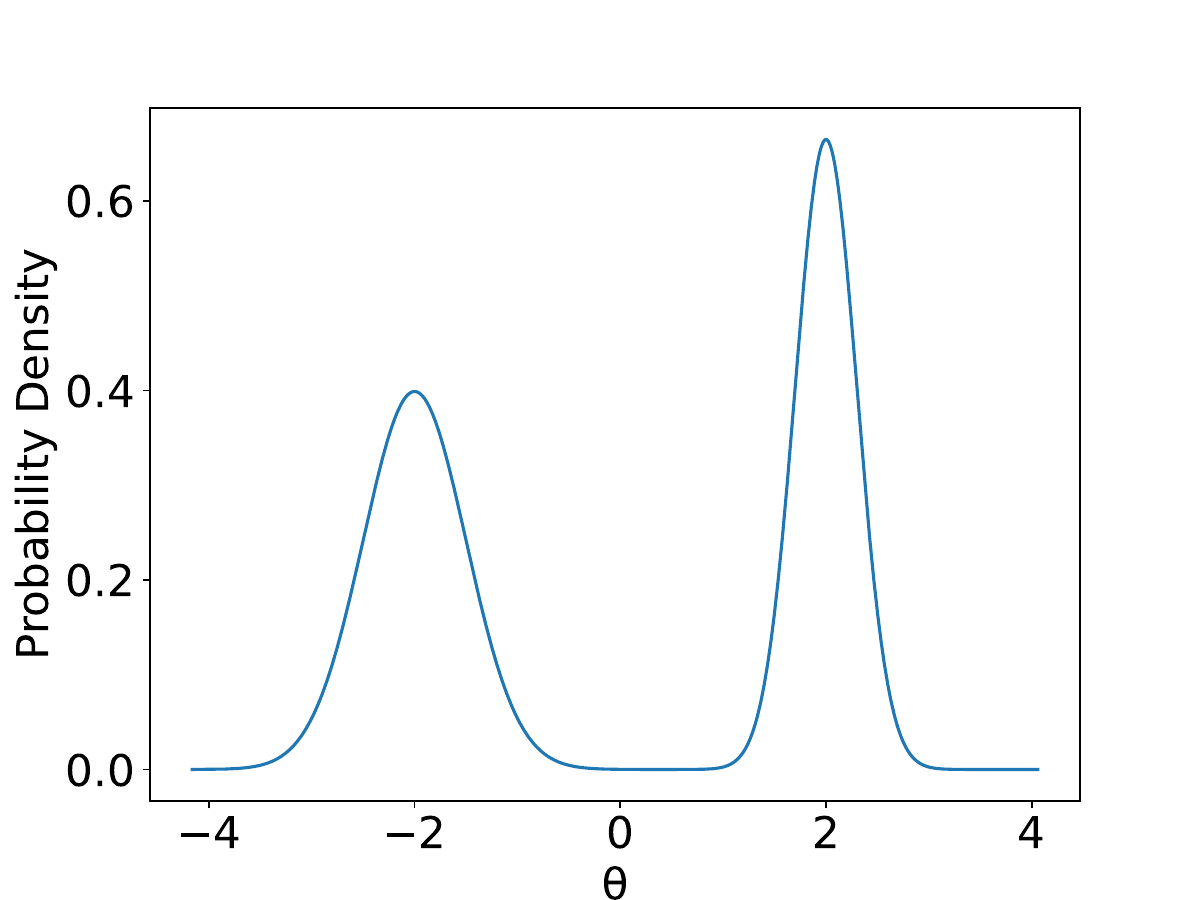}
            \caption{}
        \end{subfigure}
    \end{center}
    \caption[The last two extra distributions in the robustness test]{The last two extra distributions in the robustness test. (a) a mixed distribution between \jgsim{$\mathcal{N}(-2, 0.5^2)$ and $\mathcal{N}(2, 0.5^2)$}; (b) a mixed distribution between \jgsim{$\mathcal{N}(-2, 0.5^2)$ and $\mathcal{N}(2, 0.3^2)$}. The weight of each distribution is $0.5$.}
    \label{fig: extra dist robust}
\end{figure}

We present the numerical results in \Cref{fig: iter converge different dist everyone uses AI} and \Cref{fig: iter converge different dist may not use AI}. Regardless of the assumed distribution of $\truth$, our insights remain consistent. The diversity of outputs continues to diminish over time when everyone uses the AI. However, a low $\costParam$ or a low $\costHumanOnly$ can effectively mitigate the homogenization death spiral.

\begin{figure}[htbp!]
        \begin{center}
        \begin{subfigure}{0.3\textwidth}
            \includegraphics[width=\linewidth]{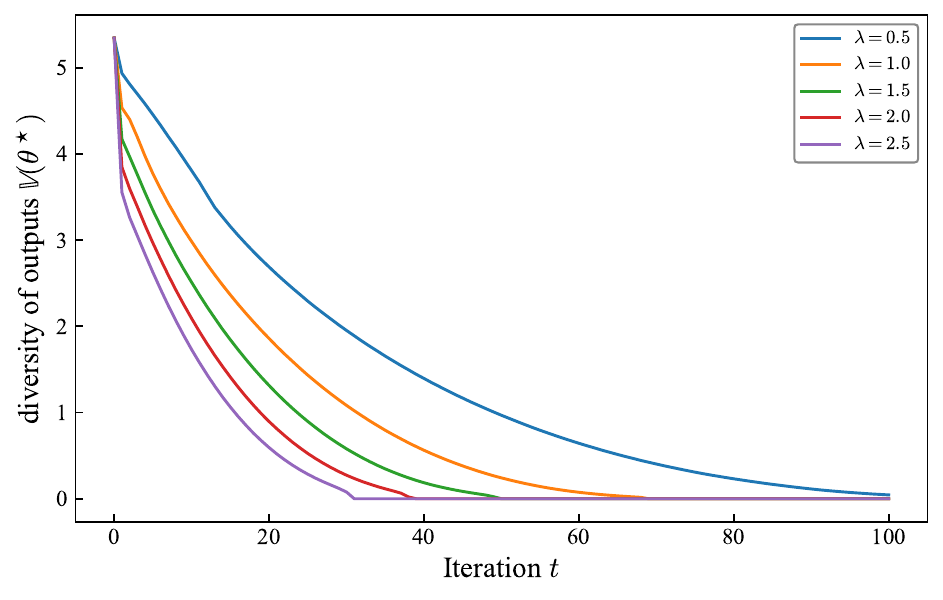}
            \caption{}
        \end{subfigure}
        \begin{subfigure}{0.3\textwidth}
            \includegraphics[width=\linewidth]{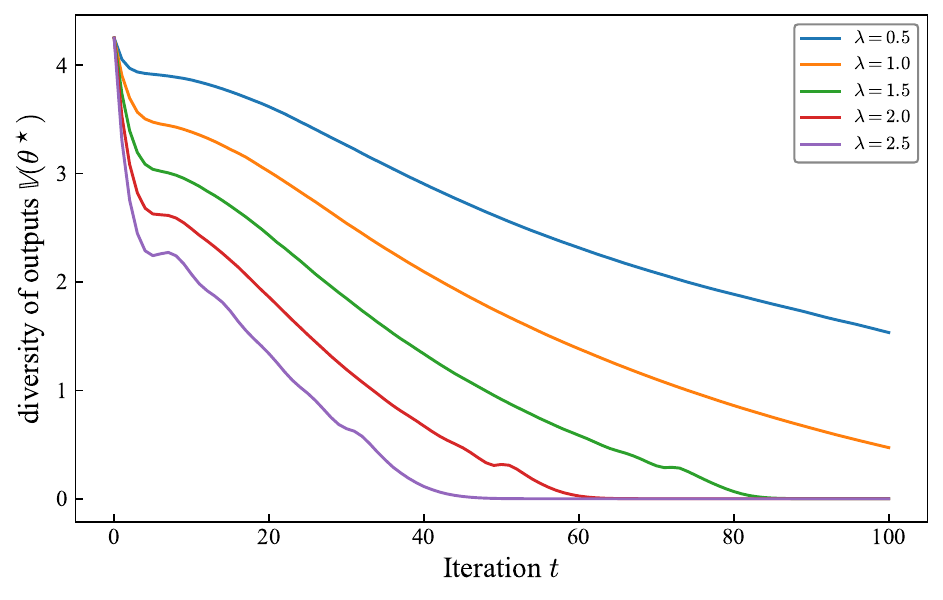}
            \caption{}
        \end{subfigure}
        \begin{subfigure}{0.3\textwidth}
            \includegraphics[width=\linewidth]{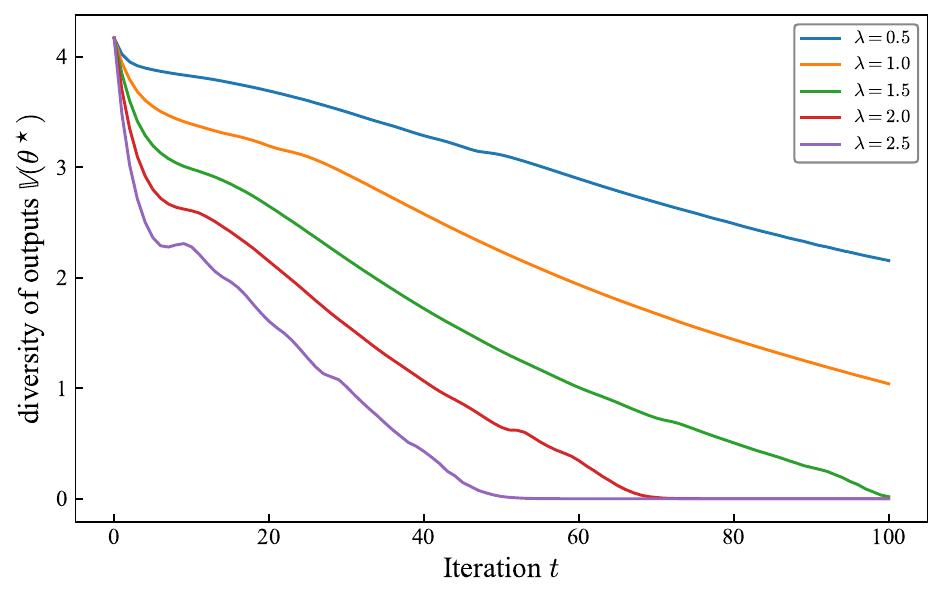}
            \caption{}
        \end{subfigure}
    \end{center}
    \caption[The iterative change of the variance of outputs in the three cases with a more complex distribution of $\truth$ when $\costHumanOnly = \infty$.]{The iterative \rev{\smrev{change}} of the variance of \jgsim{outputs $\VV(\outcome_t)$} in the three cases with a more complex distribution of $\truth$ when $\costHumanOnly = \infty$. (a) uniform \jgsim{on $[-4,4]$}; (b) a mixed distribution between \jgsim{$\mathcal{N}(-2, 0.5^2)$ and $\mathcal{N}(2, 0.5^2)$}; (c) a mixed distribution between \jgsim{$\mathcal{N}(-2, 0.5^2)$ and $\mathcal{N}(2, 0.3^2)$}.}
    \label{fig: iter converge different dist everyone uses AI}
\end{figure}

\begin{figure}[htbp!]
        \begin{center}
        \begin{subfigure}{0.3\textwidth}
            \includegraphics[width=\linewidth]{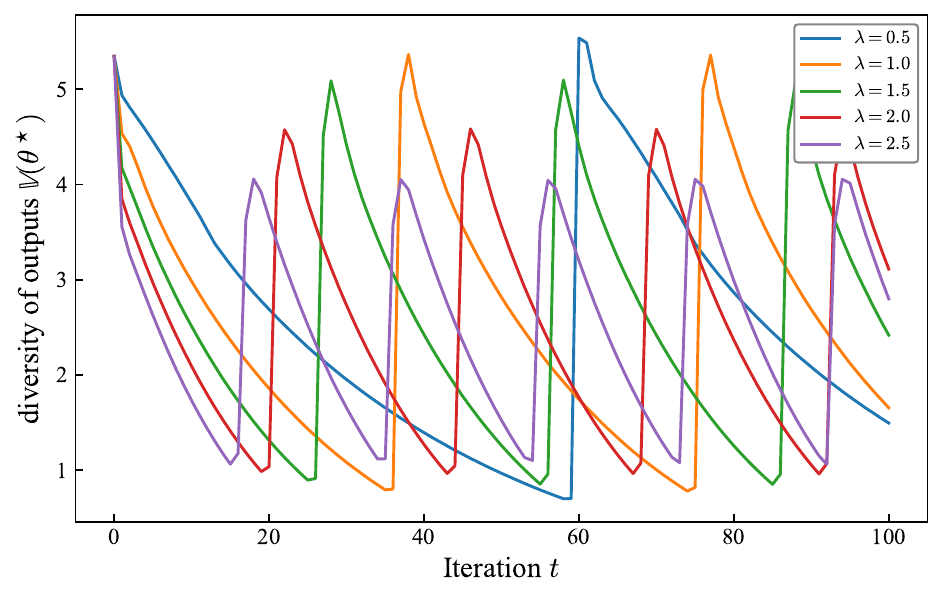}
            \caption{}
        \end{subfigure}
        \begin{subfigure}{0.3\textwidth}
            \includegraphics[width=\linewidth]{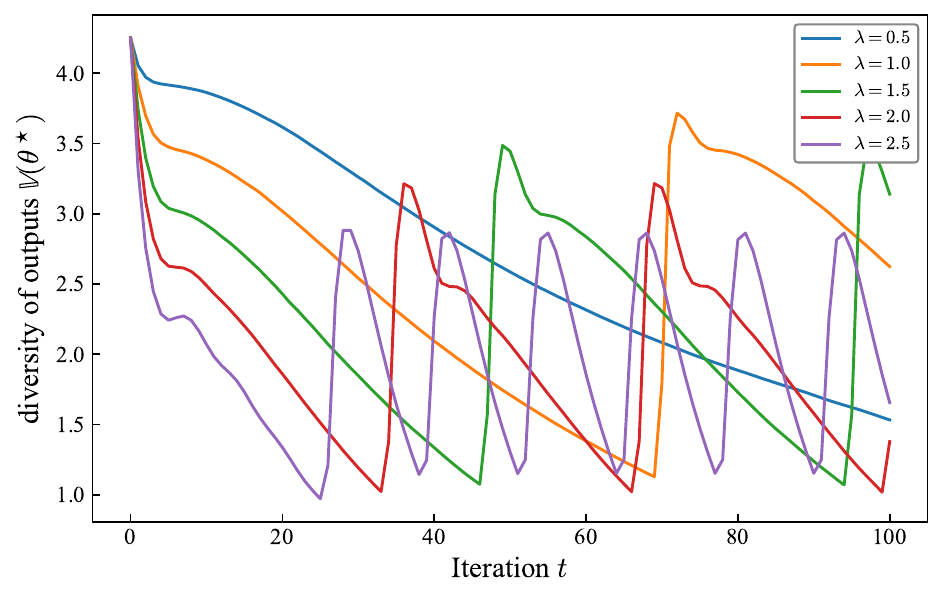}
            \caption{}
        \end{subfigure}
        \begin{subfigure}{0.3\textwidth}
            \includegraphics[width=\linewidth]{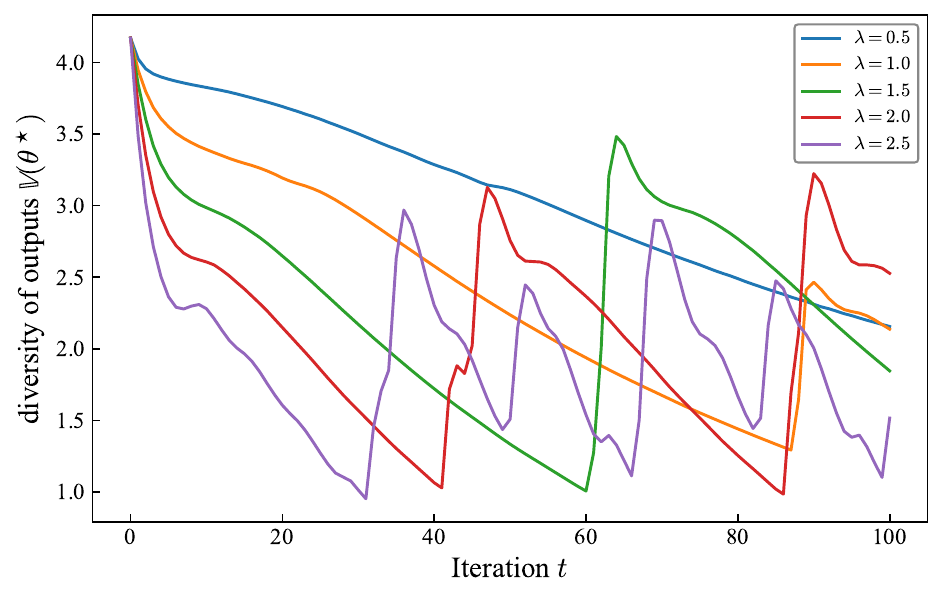}
            \caption{}
        \end{subfigure}
    \end{center}
    \caption[The iterative change of the variance of outputs in the three cases with a more complex distribution of $\truth$ when $\costHumanOnly = 10$.]{The iterative change of the variance of \jgsim{outputs $\VV(\outcome_t)$} in the three cases with a more complex distribution of $\truth$ when $\costHumanOnly = 10$. (a) uniform \jgsim{on $[-4,4]$}; (b) a mixed distribution between \jgsim{$\mathcal{N}(-2, 0.5^2)$ and $\mathcal{N}(2, 0.5^2)$}; (c) a mixed distribution between \jgsim{$\mathcal{N}(-2, 0.5^2)$ and $\mathcal{N}(2, 0.3^2)$}.}
    \label{fig: iter converge different dist may not use AI}
\end{figure}

\rev{
\subsection{Type-dependent interaction cost}
\label{sec: lambda theta}

\smrev{In \Cref{sec:mit-efficiency}, we discuss the case where the cost of human-AI interactions depends on a user's type, $\costParam(\truth)$, for example because the most unique users are also the most skilled at communicating with the AI. We simulate \Cref{alg:self-training-loop} with $\costHumanOnly=+\infty$, $\truMean=0$, and $\truSD=1$, replacing the constant $\costParam$ with one of three profiles that decrease in $|\truth|$: $\costParam(\truth)=a\,e^{-0.05\truth^2}$, $\costParam(\truth)=a/(1+0.1\truth^2)$, and $\costParam(\truth)=a/\ln(\truth^2+e)$. In each profile, the constant $a$ is chosen so that the population's average interaction cost is $\EE{\truDist}{\costParam(\truth)}=1$, and we compare the profiles with the constant cost $\costParam=1$ (\Cref{fig: lambda theta}(a)). Each user solves Problem \eqref{eq: loss def} with her own cost; everything else is as in \Cref{sec:training-loop}\jgsim{, except that these runs use a 2001-point grid with interpolated quantization and 150 iterations}.}

\Cref{fig: lambda theta}\smrev{(b) shows the variance of outputs across iterations. The first iteration produces a similar drop under every profile, since most users face a cost close to the population average. Afterwards, the variance declines more slowly when the interaction cost falls with uniqueness, and more slowly still the more it falls from the center to the tails: the most unique users keep expressing their preferences for longer, while the most common users are pulled toward the default. On the simulation grid, which truncates the support (see \Cref{sec:unique_key}), the population is bounded and every profile is bounded away from zero. We therefore expect the variance to vanish eventually, as conjectured in \Cref{sec:mit-efficiency} (\Cref{thm:bounded-atomless-collapse} covers only a cost common to all users). In these simulations, a type-dependent interaction cost delays the spiral but does not appear to stop it.}

\begin{figure}[htbp!]
    \begin{center}
    \begin{subfigure}{0.45\textwidth}
        \includegraphics[width=\linewidth]{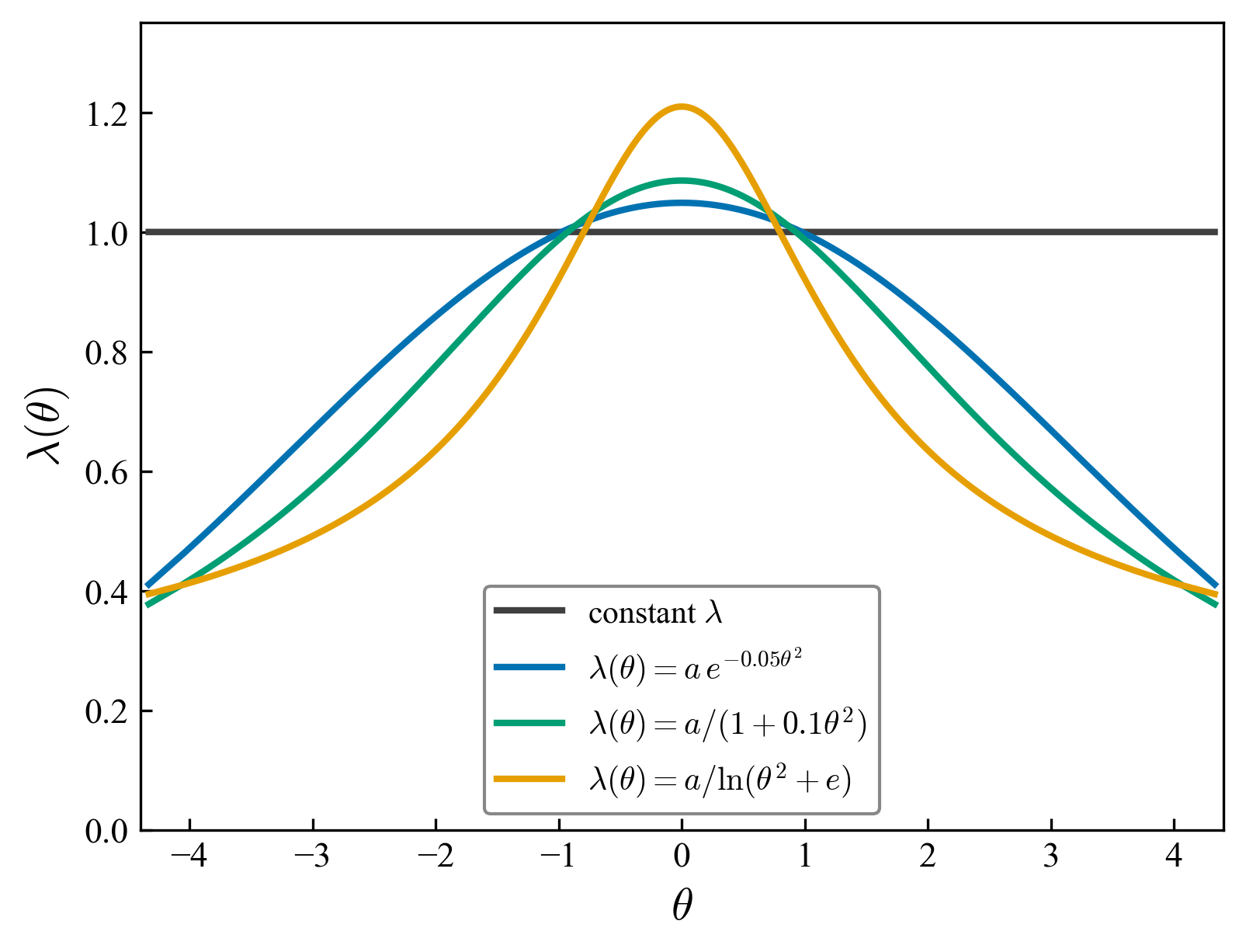}
        \caption{}
    \end{subfigure}
    \begin{subfigure}{0.45\textwidth}
        \includegraphics[width=\linewidth]{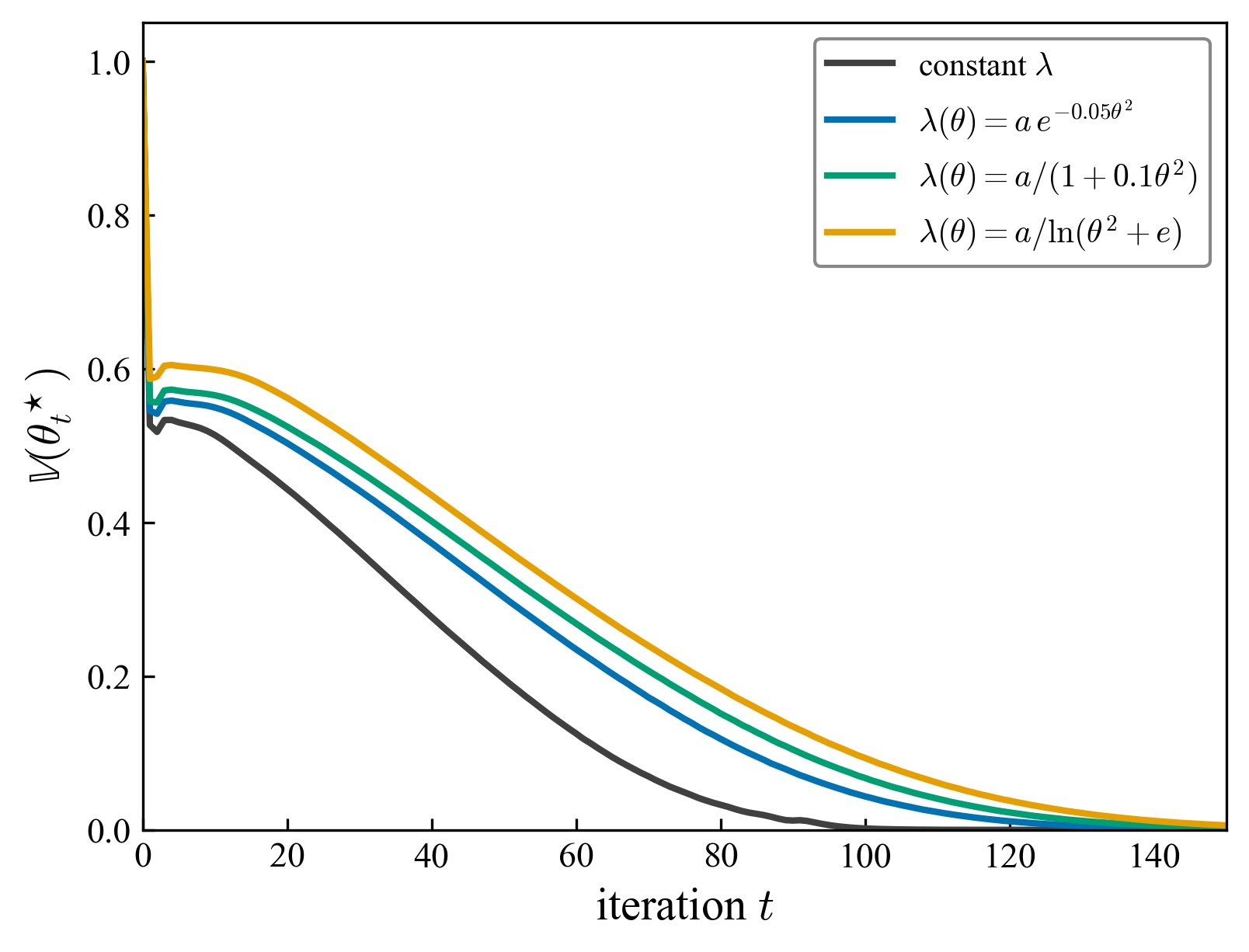}
        \caption{}
    \end{subfigure}
    \end{center}
    \caption{\smrev{Type-dependent interaction cost. (a) Interaction cost profiles $\costParam(\truth)$, each scaled so that $\EE{\truDist}{\costParam(\truth)}=1$; lower values mean users communicate with the AI more easily. (b) Variance of outputs $\VV(\outcome_t)$ across iterations. We use $\truMean=0$, $\truSD=1$, and $\costHumanOnly=+\infty$.}}
    \label{fig: lambda theta}
\end{figure}
}

\clearpage

\end{APPENDICES}

\end{document}